\documentclass[twoside,11pt]{article}
\usepackage[utf8]{inputenc}
\usepackage{times}
\usepackage{graphicx} % more modern
\usepackage[update,prepend]{epstopdf} 

\usepackage{natbib}
\usepackage[titletoc]{appendix}

\usepackage{algorithm}
\usepackage{algorithmic}

\usepackage{jmlr2e}
\usepackage{amsfonts}
\usepackage{amsmath}
\usepackage{bm}
\usepackage{dsfont} 
\usepackage{bbm}
\usepackage{mathtools}
\usepackage{empheq}
\usepackage{amssymb}
\usepackage{epigraph}
\usepackage{multicol}
\usepackage{multirow}
\usepackage{booktabs}
\usepackage{lscape}
\usepackage{hyperref}
\usepackage[section]{placeins}
\usepackage[dvipsnames]{xcolor}
\newcommand*\widefbox[1]{\fbox{\hspace{2em}#1\hspace{2em}}}

\usepackage{tikz}
\usetikzlibrary{matrix,chains,positioning,decorations.pathreplacing,arrows,decorations.markings,patterns}

\hypersetup{colorlinks=true, linkcolor=red,filecolor=magenta,urlcolor=green,citecolor=blue}
\jmlrheading{17}{2016}{1-87}{7/15; Revised 6/16}{8/16}{Yves-Laurent Kom Samo and Stephen J. Roberts}
\ShortHeadings{String and Membrane Gaussian Processes}{Kom Samo and Roberts}
\firstpageno{1}

\begin{document}

\title{String and Membrane Gaussian Processes}

\author{\name Yves-Laurent Kom Samo \email ylks@robots.ox.ac.uk  \AND
       \name Stephen J.\ Roberts \email sjrob@robots.ox.ac.uk \\
       \addr Department of Engineering Science and Oxford-Man Institute\\
       University of Oxford\\
       Eagle House, Walton Well Road,\\
       OX2 6ED, Oxford, United Kingdom}

\editor{Neil Lawrence}

\maketitle

\begin{abstract}%
In this paper we introduce a novel framework for making exact nonparametric Bayesian inference on latent functions that is particularly suitable for \emph{Big Data} tasks. Firstly, we introduce a class of stochastic processes we refer to as \emph{string Gaussian processes} (\emph{string GPs} which are not to be mistaken for Gaussian processes operating on text). We construct \emph{string GPs} so that their finite-dimensional marginals exhibit suitable \emph{local} conditional independence structures, which allow for \emph{scalable}, \emph{distributed}, and \emph{flexible} nonparametric Bayesian inference, without resorting to approximations, and while ensuring some mild global regularity constraints. Furthermore, \emph{string GP} priors naturally cope with heterogeneous input data, and the gradient of the learned latent function is readily available for explanatory analysis. Secondly, we provide some theoretical results relating our approach to the \emph{standard GP paradigm}. In particular, we prove that some \emph{string GPs} are Gaussian processes, which provides a complementary \emph{global} perspective on our framework. Finally, we derive a scalable and distributed MCMC scheme for supervised learning tasks under \emph{string GP} priors. The proposed MCMC scheme has computational time complexity $\mathcal{O}(N)$ and memory requirement $\mathcal{O}(dN)$, where $N$ is the data size and $d$ the dimension of the input space. We illustrate the efficacy of the proposed approach on several synthetic and real-world data sets, including a data set with $6$ millions input points and $8$ attributes.
\end{abstract}


%
%
%
%
%	SECTION: INTRODUCTION
%
%
%
%
\section{Introduction}
Many problems in statistics and machine learning involve inferring a latent function from training data (for instance regression, classification, inverse reinforcement learning, inference on point processes to name but a few). Real-valued stochastic processes, among which Gaussian processes (GPs), are often used as functional priors for such problems, thereby allowing for a full Bayesian nonparametric treatment. In the machine learning community,  interest in  GPs grew out of the observation that some Bayesian neural networks converge to GPs as the number of hidden units approaches infinity (\cite{neal}). Since then, other similarities have been established between GPs and popular models such as Bayesian linear regression, Bayesian basis function regression, spline models and support vector machines (\cite{rasswill}). However, they often perform poorly on \emph{Big Data} tasks primarily for two reasons. Firstly, large data sets are likely to exhibit multiple types of local patterns that should appropriately be accounted for by flexible and possibly nonstationary covariance functions, the development of which is still an active subject of research. Secondly, inference under GP priors often consists of looking at the values of the GP at all input points as a jointly Gaussian vector with fully dependent coordinates, which induces a memory requirement and time complexity respectively squared and cubic in the training data size, and thus is intractable for large data sets. We refer to this approach as the \emph{standard GP paradigm}. The framework we introduce in this paper addresses both of the above limitations. 

Our work  is rooted in the observation that, from a Bayesian nonparametric perspective, it is inefficient to define a stochastic process through \emph{fully-dependent} marginals, as it is the case for Gaussian processes. Indeed, if a stochastic process $(f(x))_{x \in \mathbb{R}^d}$ has fully dependent marginals and exhibits no additional conditional independence structure then, when $f$ is used as functional prior and some observations related to $\left(f(x_1), \dots, f(x_n)\right)$ are gathered, namely $\left(y_1, \dots, y_n\right)$, the \emph{additional} memory required to take into account an additional piece of information $(y_{n+1}, x_{n+1})$ grows in $\mathcal{O}(n)$, as one has to keep track of the extent to which $y_{n+1}$ informs us about $f(x_i)$ for every $i\leq n$, typically through a covariance matrix whose size will increase by $2n +1$ terms. Clearly, this is inefficient, as $y_{n+1}$ is unlikely to be informative about $f(x_i)$, unless $x_i$ is sufficiently close to $x_{n+1}$. More generally, the larger $n$, the less information a single additional pair $(y_{n+1}, x_{n+1})$ will add to existing data, and yet the increase in memory requirement will be much higher than that required while processing earlier and more informative data. This inefficiency in resource requirements extends to computational time, as the \emph{increase} in computational time resulting from adding $(y_{n+1}, x_{n+1})$ typically grows in $\mathcal{O}(n^2)$, which is the difference between the numbers of operations required to invert a $n \times n$ matrix and to invert a $(n+1) \times (n+1)$ matrix. A solution for addressing this inefficiency is to appropriately limit the extent to which values $f(x_1), \dots, f(x_n)$ are related to each other. Existing approaches such as sparse Gaussian processes (see \cite{FTCI} for a review), resort to an \emph{ex-post} approximation of fully-dependent Gaussian marginals with multivariate Gaussians exhibiting conditional independence structures. Unfortunately, these approximations trade-off accuracy for scalability through a control variable, namely the number of inducing points, whose choice is often left to the user. The approach we adopt in this paper consists of going back to stochastic analysis basics, and constructing stochastic processes whose finite-dimensional marginals exhibit suitable conditional independence structures so that we need not resorting to \emph{ex-post} approximations. Incidentally, unlike sparse GP techniques, the conditional independence structures we introduce also allow for flexible and principled learning of local patterns, and this increased flexibility does not come at the expense of scalability.

The contributions of this paper are as follows. We introduce  a novel class of stochastic processes, string Gaussian processes (\emph{string GPs}), that may be used as priors over latent functions within a Bayesian nonparametric framework, especially for large scale problems and in the presence of possibly multiple types of local patterns. We propose a framework for analysing the flexibility of random functions and surfaces, and prove that our approach yields more flexible stochastic processes than isotropic Gaussian processes. We demonstrate that exact inference under a \emph{string GP} prior scales considerably better than in the \emph{standard GP paradigm}, and is amenable to distributed computing. We illustrate that popular stationary kernels can be well approximated within our framework, making \emph{string GPs} a scalable alternative to commonly used GP models. We derive the joint law of a \emph{string GP} and its gradient, thereby allowing for explanatory analysis on the learned latent function. We propose a reversible-jump Markov Chain Monte Carlo sampler for automatic learning of model complexity and local patterns from data.

The rest of the paper is structured as follows. In Section \ref{sct:related_work} we review recent advances on Gaussian processes in relation to inference on large data sets. In Section \ref{sct:model} we formally construct \emph{string GPs} and derive some important results. In Section \ref{sct:comp} we provide detailed illustrative and theoretical comparisons between \emph{string GPs} and the \emph{standard GP paradigm}. In Section \ref{sct:infer} we propose methods for inferring latent functions under \emph{string GP} priors with time complexity and memory requirement that are linear in the size of the data set. The efficacy of our approach compared to competing alternatives is illustrated in Section \ref{sct:exp}. Finally, we finish with a discussion in Section \ref{sct:discussion}.

%
%
%
%
%	SECTION: RELATED WORK
%
%
%
%
\section{Related Work}
\label{sct:related_work}
The two primary drawbacks of the \emph{standard GP paradigm} on large scale problems are the lack of scalability resulting from postulating a full multivariate Gaussian prior on function values at \emph{all} training inputs, and the difficulty postulating \emph{a priori} a class of covariance functions capable of capturing intricate and often local patterns likely to occur in large data sets. A tremendous amount of work has been published that attempt to address either of the aforementioned limitations. However, scalability is often achieved either through approximations or for specific applications, and nonstationarity is usually introduced at the expense of scalability, again for specific applications.

\subsection{Scalability Through Structured Approximations}

As far as scalability is concerned, sparse GP methods have been developed  that approximate the multivariate Gaussian probability density function (pdf) over training data with the marginal over a smaller set of inducing points multiplied by an approximate conditional pdf (\cite{Smola01sparsegreedy, lawr03, seeger2003pac, Seeger03bayesiangaussian, Snelson06sparsegaussian}). This approximation yields a time complexity linear---rather than cubic---in the data size and squared in the number of inducing points. We refer to \cite{FTCI} for a review of sparse GP approximations. More recently, \cite{gpbigdatareg,gpbigdatacls} combined sparse GP methods with Stochastic Variational Inference (\cite{JMLR:v14:hoffman13a}) for GP regression and GP classification. However, none of these sparse GP methods addresses the selection of the number of inducing points (and the size of the minibatch in the case of \cite{gpbigdatareg,gpbigdatacls}), although this may greatly affect scalability. More importantly, although these methods do not impose strong restrictions on the covariance function of the GP model to approximate, they do not address the need for flexible covariance functions inherent to large scale problems, which are more likely to exhibit intricate and local patterns, and applications considered by the authors typically use the vanilla squared exponential kernel.

\cite{sparsespectrum} proposed approximating stationary kernels with truncated Fourier series in Gaussian process regression. An interpretation of the resulting sparse spectrum Gaussian process model as Bayesian basis function regression with a finite number $K$ of trigonometric basis functions allows making inference in time complexity and memory requirement that are both linear in the size of the training sample. However, this model has two major drawbacks. Firstly, it is prone to over-fitting. In effect, the learning machine will aim at inferring the $K$ major spectral frequencies evidenced in the training data. This will only lead to appropriate prediction out-of-sample when the underlying latent phenomenon can be appropriately characterised by a finite discrete spectral decomposition that is expected to be the same everywhere on the domain. Secondly, this model implicitly postulates that the covariance between the values of the GP at two points does not vanish as the distance between the points becomes arbitrarily large. This imposes \emph{a priori} the view that the underlying function is highly structured, which might be unrealistic in many real-life non-periodic applications. This approach is generalised by the so-called \emph{random Fourier features} methods (\cite{rahimi07, le13, le15}). Unfortunately all existing random Fourier features methods give rise to stationary covariance functions, which might not be appropriate for data sets exhibiting local patterns.

The bottleneck of inference in the \emph{standard GP paradigm} remains inverting and computing the determinant of a covariance matrix, normally achieved through the Cholesky decomposition or Singular Value Decomposition. Methods have been developed that speed-up these decompositions through low rank approximations (\cite{Williams01usingthe}) or by exploiting specific structures in the covariance function and in the input data (\cite{saatchi11, GPatt}), which typically give rise to Kronecker or Toeplitz covariance matrices. While the Kronecker method used by \cite{saatchi11} and \cite{GPatt} is restricted to inputs that form a Cartesian grid and to separable kernels,\footnote{That is multivariate kernel that can be written as product of univariate kernels.} low rank approximations such as the Nystr$\ddot{\text{o}}$m method used by \cite{Williams01usingthe} modify the covariance function and hence the functional prior in a non-trivial way. Methods have also been proposed to interpolate the covariance matrix on a uniform or Cartesian grid in order to benefit from some of the computational gains of Toeplitz and Kronecker techniques even when the input space is not structured (\cite{wilson2015kernel}). However, none of these solutions is general as they require that either the covariance function be separable (Kronecker techniques), or the covariance function be stationary and the input space be one-dimensional (Toeplitz techniques).

\subsection{Scalability Through Data Distribution}
A family of methods have been proposed to scale-up inference in GP models that are based on the observation that it is more computationally efficient to compute the pdf of $K$ independent small Gaussian vectors with size $n$ than to compute the pdf of a single bigger Gaussian vector of size $nK$. For instance, \cite{kim} and \cite{gramacy} partitioned the input space, and put independent stationary GP priors on the restrictions of the latent function to the subdomains forming the partition, which can be regarded as independent \emph{local GP experts}. \cite{kim} partitioned the domain using Voronoi tessellations, while \cite{gramacy} used tree based partitioning. These two approaches are provably equivalent to postulating a (nonstationary) GP prior on the whole domain that is discontinuous along the boundaries of the partition, which might not be desirable if the latent function we would like to infer is continuous, and might affect predictive accuracy. The more local experts there are, the more scalable the model will be, but the more discontinuities the latent function will have, and subsequently the less accurate the approach will be. 

Mixtures of Gaussian process experts models (MoE) (\cite{tresp, Rasmussen01infinitemixtures, Meeds_analternative, icml2013_ross13a}) provide another implementation of this idea. MoE models assume that there are multiple latent functions to be inferred from the data, on which it is placed independent GP priors, and each training input is associated to one latent function. The number of latent functions and the repartition of data between latent functions can then be performed in a full Bayesian nonparametric fashion (\cite{Rasmussen01infinitemixtures, icml2013_ross13a}). When there is a single continuous latent function to be inferred, as it is the case for most regression models, the foregoing Bayesian nonparametric approach will learn a single latent function, thereby leading to a time complexity and a memory requirement that are the same as in the \emph{standard GP paradigm}, which defies the scalability argument. 

The last implementation of the idea in this section consists of distributing the training data over multiple independent but identical GP models. In regression problems, examples include the \emph{Bayesian Committee Machines} (BCM) of \cite{tresp2000bayesian}, the \emph{generalized product of experts} (gPoE) model of \cite{cao2014generalized}, and the \emph{robust Bayesian Committee Machines} (rBCM) of \cite{deisenroth2015distributed}. These models propose splitting the training data in small subsets, each subset being assigned to a different GP regression model---referred to as an expert---that has the same hyper-parameters as the other experts, although experts are assumed to be mutually independent. Training is performed by maximum marginal likelihood, with time complexity (resp. memory requirement) linear in the number of experts and cubic (resp. squared) in the size of the largest data set processed by an expert. Predictions are then obtained by aggregating the predictions of all GP experts in a manner that is specific to the method used (that is the BCM, the gPoE or the rBCM). However, these methods present major drawbacks in the training and testing procedures. In effect, the assumption that experts have identical hyper-parameters is inappropriate for data sets exhibiting local patterns. Even if one would allow GP experts to be driven by different hyper-parameters as in \cite{nguyen2014fast} for instance, learned hyper-parameters would lead to overly simplistic GP experts and poor aggregated predictions when the number of training inputs assigned to each expert is small---this is a direct consequence of the (desirable) fact that maximum marginal likelihood GP regression abides by Occam's razor. Another critical pitfall of BCM, gPoE and rBCM is that their methods for aggregating expert predictions are Kolmogorov \emph{inconsistent}. For instance, denoting $\hat{p}$ the predictive distribution in the BCM, it can be easily seen from Equations (2.4) and (2.5) in \cite{tresp2000bayesian} that the predictive distribution $\hat{p}(f(x_1^*) \vert \mathcal{D})$ (resp. $\hat{p}(f(x_2^*) \vert \mathcal{D})$)\footnote{Here $f$ is the latent function to be inferred, $x_1^*, x_2^*$ are test points and $\mathcal{D}$ denotes training data.} provided by the aggregation procedure of the BCM is \emph{not} the marginal over $f(x_2^*)$ (resp. over $f(x_1^*)$) of the multivariate predictive distribution $\hat{p}(f(x_1^*), f(x_2^*) \vert \mathcal{D})$ obtained from experts multivariate predictions $p_k(f(x_1^*), f(x_2^*) \vert \mathcal{D})$ using the same aggregation procedure: $\hat{p} (f(x_1^*) \vert \mathcal{D}) \neq \int \hat{p} (f(x_1^*), f(x_2^*)  \vert \mathcal{D}) df(x_2^*)$. Without Kolmogorov consistency, it is impossible to make principled Bayesian inference of latent function values. A principled Bayesian nonparametric model should not provide predictions about $f(x_1^*)$ that differ depending on whether or not one is also interested in predicting other values $f(x_i^*)$ simultaneously. This pitfall might be the reason why \cite{cao2014generalized} and \cite{deisenroth2015distributed} restricted their expositions to predictive distributions about a single function value at a time $\hat{p}(f(x^*) \vert \mathcal{D})$, although their procedures (Equation 4 in \cite{cao2014generalized} and Equation 20 in \cite{deisenroth2015distributed}) are easily extended to posterior distributions over multiple function values. These extensions would also be Kolmogorov \emph{inconsistent}, and restricting the predictions to be of exactly one function value is unsatisfactory as it does not allow determining the posterior covariance between function values at two test inputs.

\subsection{Expressive Stationary Kernels}
In regards to flexibly handling complex patterns likely to occur in large data sets, \cite{wilson2013gaussian} introduced a class of expressive stationary kernels obtained by summing up convolutions of Gaussian basis functions with Dirac delta functions in the spectral domain. The sparse spectrum kernel can be thought of as the special case where the convolving Gaussian is degenerate. Although such kernels perform particularly well in the presence of globally repeated patterns in the data, their stationarity limits their utility on data sets with local patterns. Moreover the proposed covariance functions generate infinitely differentiable random functions, which might be too restrictive in some applications. 

\subsection{Application-Specific Nonstationary Kernels}
As for nonstationary kernels, \cite{paciorek2004nonstationary} proposed a method for constructing nonstationary covariance functions from any stationary one that involves introducing $n$ input dependent $d \times d$ covariance matrices that will be inferred from the data. \cite{plagemann08ecml} proposed a faster approximation to the model of \cite{paciorek2004nonstationary}. However, both approaches scale poorly with the input dimension and the data size as they have time complexity $\mathcal{O}\left(\max(nd^3, n^3) \right)$. \cite{gp_intro}, \cite{Schmidt03}, and \cite{calandra2014manifold} proposed kernels that can be regarded as stationary after a non-linear transformation $d$ on the input space: $k(x, x^\prime) = h \left(\| d(x) - d(x^\prime) \| \right),$ where $h$ is positive semi-definite. Although for a given deterministic function $d$ the kernel $k$ is nonstationary, \cite{Schmidt03} put a GP prior on $d$ with mean function $m(x)=x$ and covariance function invariant under translation, which unfortunately leads to a kernel that is (unconditionally) stationary, albeit more flexible than $h \left(\| x - x^\prime \| \right).$ To model nonstationarity, \cite{ggpm} introduced a functional prior of the form $y(x)= f(x) \exp{g(x)}$ where $f$ is a stationary GP and $g$ is some scaling function on the domain. For a given non-constant function $g$ such a prior indeed yields a nonstationary Gaussian process. However, when a stationary GP prior is placed on the function $g$ as \cite{ggpm} did, the resulting functional prior $y(x)= f(x) \exp{g(x)}$ becomes stationary. The piecewise GP (\cite{kim}) and treed GP (\cite{gramacy}) models previously discussed also introduce nonstationarity. The authors' premise is that heterogeneous patterns might be locally homogeneous. However, as previously discussed such models are inappropriate for modelling continuous latent functions.

\subsection{Our Approach}
The approach we propose in this paper for inferring latent functions in large scale problems, possibly exhibiting locally homogeneous patterns, consists of constructing a novel class of \emph{smooth}, \emph{nonstationary} and \emph{flexible} stochastic processes we refer to as \emph{string Gaussian processes} (\emph{string GPs}), whose finite dimensional marginals are structured enough so that full Bayesian nonparametric inference scales linearly with the sample size, without resorting to approximations. Our approach is analogous to MoE models in that, when the input space is one-dimensional, a \emph{string GP} can be regarded as a \emph{collaboration of local GP experts} on non-overlapping supports, that implicitly exchange messages with one another, and that are independent conditional on the aforementioned messages. Each local GP expert only shares just enough information with adjacent local GP experts for the whole stochastic process to be sufficiently smooth (for instance continuously differentiable), which is an important improvement over MoE models as the latter generate discontinuous latent functions. These messages will take the form of boundary conditions, conditional on which each local GP expert will be independent from any other local GP expert. Crucially, unlike the BCM, the gPoE and the rBCM, we do not assume that local GP experts share the same prior structure (that is mean function, covariance function, or hyper-parameters). This allows each local GP expert to flexibly learn local patterns from the data if there are any, while preserving global smoothness, which will result in improved accuracy. Similarly to MoEs, the computational gain in our approach stems from the fact that the conditional independence of the local GP experts conditional on shared boundary conditions will enable us to write the joint distribution over function and derivative values at a large number of inputs as the product of pdfs of much smaller Gaussian vectors. The resulting effect on time complexity is a decrease from $\mathcal{O}(N^3)$ to $\mathcal{O}(\underset{k}{\max} ~ n_k^3)$, where $N=\sum_{k} n_k, ~ n_k \ll N$. In fact, in Section \ref{sct:infer} we will propose Reversible-Jump Monte Carlo Markov Chain (RJ-MCMC) inference methods that achieve memory requirement and time complexity $\mathcal{O}(N)$, without any loss of flexibility. All these results are preserved by our extension of \emph{string GPs} to multivariate input spaces, which we will occasionally refer to as \emph{membrane Gaussian processes} (or membrane GPs). Unlike the BCM, the gPoE and the rBCM, the approach we propose in this paper, which we will refer to as the \emph{string GP paradigm}, is Kolmogorov consistent, and enables principled inference of the posterior distribution over the values of the latent function at multiple test inputs.

%
%
%
%
%	SECTION: THE MODEL
%
%
%
%
\section{Construction of String and Membrane Gaussian Processes}
\label{sct:model}
In this section we formally construct \emph{string} Gaussian processes, and we provide some important theoretical results including smoothness, and the joint law of \emph{string GPs} and their gradients. We construct \emph{string GPs} indexed on $\mathbb{R}$, before generalising to \emph{string GPs} indexed on $\mathbb{R}^d$, which we will occasionally refer to as \emph{membrane GPs} to stress that the input space is multivariate. We start by considering the joint law of a differentiable GP on an interval and its derivative, and introducing some related notions that we will use in the construction of \emph{string GPs}.

\begin{proposition}(\textbf{Derivative Gaussian processes})\\
\label{prop:derivative_processes}
Let $I$ be an interval, $k: I \times I \rightarrow \mathbb{R}$ a $\mathcal{C}^2$ symmetric positive semi-definite function,\footnote{$\mathcal{C}^1$ (resp. $\mathcal{C}^2$) functions denote functions that are once (resp. twice) continuously differentiable on their domains.} $m: I \rightarrow \mathbb{R}$ a $\mathcal{C}^1$ function.\\

\noindent (A) There exists a $\mathbb{R}^2$-valued stochastic process $\left(D_t\right)_{t \in I}, ~D_t=(z_t, z_t^\prime)$, such that for all $t_1, \dots, t_n \in I$, $(z_{t_1}, \dots, z_{t_n}, z_{t_1}^\prime, \dots, z_{t_n}^\prime)$ is a Gaussian vector with mean 
$\left(m(t_1), \dots, m(t_n), \frac{\text{d}m}{\text{dt}}(t_1), \dots, \frac{\text{d}m}{\text{dt}}(t_n)\right)$ and covariance matrix such that 
$$\text{cov}(z_{t_i}, z_{t_j})=k(t_i, t_j), ~~~ \text{cov}(z_{t_i}, z_{t_j}^\prime)=\frac{\partial k}{\partial y }(t_i, t_j), ~~~\text{and}~~~ \text{cov}(z_{t_i}^\prime, z_{t_j}^\prime)=\frac{\partial^2 k}{\partial x \partial y }(t_i, t_j),$$ 
where $\frac{\partial}{\partial x}$ (resp. $\frac{\partial}{\partial y}$) refers to the partial derivative with respect to the first (resp. second) variable of $k$. We herein refer to $(D_t)_{t \in I}$ as a \textbf{derivative Gaussian process}.\\

\noindent (B) $(z_t)_{t \in I}$ is a Gaussian process with mean function $m$, covariance function $k$ and that is $\mathcal{C}^1$ in the $L^2$ (mean square) sense. \\

\noindent (C) $(z^\prime_t)_{t \in I}$ is a Gaussian process with mean function $\frac{\text{d}m}{\text{dt}}$ and covariance function $\frac{\partial^2 k}{\partial x \partial y }$. Moreover, $(z^\prime_t)_{t \in I}$ is the $L^2$ derivative of the process $(z_t)_{t \in I}$.
\end{proposition}
\begin{proof}
Although this result is known in the Gaussian process community, we provide a proof for the curious reader in \ref{app:derivative_processes}.
\end{proof}
We will say of a kernel $k$ that it is \textbf{degenerate at} $a$ when a \emph{derivative Gaussian process} $(z_t, z_t^\prime)_{t \in I}$ with kernel $k$ is such that $z_a$ and $z_a^\prime$ are perfectly correlated,\footnote{Or equivalently when the Gaussian vector $(z_a, z_a^\prime)$ is degenerate.} that is  $$\vert \text{corr}(z_a, z_a^\prime) \vert= 1.$$ As an example, the linear kernel $k(u,v) = \sigma^2(u-c)(v-c)$ is degenerate at $0$. Moreover, we will say of a kernel $k$ that it is \textbf{degenerate at} $b$ \textbf{given} $a$ when it is not degenerate at $a$ and when the \emph{derivative Gaussian process} $(z_t, z_t^\prime)_{t \in I}$ with kernel $k$ is such that the variances of $z_b$ and $z_b^\prime$ conditional on $(z_a, z_a^\prime)$ are both zero.\footnote{Or equivalently when the Gaussian vector $(z_a, z_a^\prime)$ is not degenerate but $(z_a, z_a^\prime, z_b, z_b^\prime)$ is.} For instance, the periodic kernel proposed by \cite{gp_intro} with period $T$ is degenerate at $u+T$ given $u$.

An important subclass of \emph{derivative Gaussian processes} in our construction are the processes resulting from conditioning paths of a \emph{derivative Gaussian process} to take specific values at certain times $(t_1, \dots, t_c)$. We herein refer to those processes as \textbf{\emph{conditional derivative Gaussian process}}.  As an illustration, when $k$ is $\mathcal{C}^3$ on $I \times I$ with $I=[a,b]$, and neither degenerate at $a$ nor degenerate at $b$ given $a$, the \emph{conditional derivative Gaussian process} on $I=[a,b]$ with unconditional mean function $m$ and unconditional covariance function $k$ that is conditioned to start at $(\tilde{z}_a, \tilde{z}_a^\prime)$ is the \emph{derivative Gaussian process} with mean function
\begin{align}
\label{eq:single_cond_mean}
& \forall ~ t \in I, ~~~ m_c^a(t; \tilde{z}_a, \tilde{z}^\prime_a) = m(t) +  \tilde{\textbf{K}}_{t; a}  \textbf{K}_{a; a}^{-1} \begin{bmatrix} \tilde{z}_a - m(a) \\ \tilde{z}_a^\prime - \frac{d m}{dt}(a) \end{bmatrix},
\end{align}
 and covariance function $k_c^a$ that reads
\begin{align}
\label{eq:single_cond_cov}
 \forall ~ t, s \in I, ~~~ k_c^a(t, s) =  k(t, s) -  \tilde{\textbf{K}}_{t; a}  \textbf{K}_{a; a}^{-1} \tilde{\textbf{K}}_{s; a}^T
\end{align}
where $~~~
\textbf{K}_{u; v} = 
\begin{bmatrix} 
k(u, v) & \frac{\partial k}{\partial y}(u, v) \\
\frac{\partial k}{\partial x}(u, v)  & \frac{\partial^2 k}{\partial x \partial y}(u, v) 
\end{bmatrix},~~~$ and $~~~
\tilde{\textbf{K}}_{t; a} =
\begin{bmatrix} 
k(t, a) & \frac{\partial k}{\partial y}(t, a) 
\end{bmatrix}.$
Similarly, when the process is conditioned to start at $(\tilde{z}_a, \tilde{z}_a^\prime)$ and to end at $(\tilde{z}_b, \tilde{z}_b^\prime)$, the mean function reads
\begin{align}
\label{eq:double_cond_mean}
& \forall ~ t \in I,~~~ m_c^{a, b}(t; \tilde{z}_a, \tilde{z}^\prime_a,  \tilde{z}_b, \tilde{z}^\prime_b) = m(t) +  \tilde{\textbf{K}}_{t; (a, b)} \textbf{K}_{(a,b); (a, b)}^{-1}  \begin{bmatrix} \tilde{z}_a - m(a) \\ \tilde{z}_a^\prime - \frac{d m}{dt}(a) \\ \tilde{z}_b - m(b) \\ \tilde{z}_b^\prime - \frac{d m}{dt}(b)\end{bmatrix},
\end{align}
 and the covariance function $k_c^{a,b}$ reads
\begin{align}
\label{eq:double_cond_cov}
 \forall ~ t, s \in I, ~~~k_c^{a,b}(t, s) =  k(t, s) - \tilde{\textbf{K}}_{t; (a, b)} \textbf{K}_{(a,b); (a, b)}^{-1}  \tilde{\textbf{K}}_{s; (a, b)}^T,
\end{align}
where
$~~~\textbf{K}_{(a, b); (a, b)} = 
\begin{bmatrix} 
\textbf{K}_{a; a} & \textbf{K}_{a; b} \\
\textbf{K}_{b; a} & \textbf{K}_{b; b}
\end{bmatrix},~~~$and $~~~
\tilde{\textbf{K}}_{t; (a, b)} = 
\begin{bmatrix} 
\tilde{\textbf{K}}_{t; a} & \tilde{\textbf{K}}_{t; b}
\end{bmatrix}.$
It is important to note that both $\textbf{K}_{a; a}$ and $\textbf{K}_{(a,b); (a, b)}$ are indeed invertible because the kernel is assumed to be neither degenerate at $a$ nor degenerate at $b$ given $a$. Hence, the support of $(z_a, z_a^\prime, z_b, z_b^\prime)$ is $\mathbb{R}^4$, and any function and derivative values can be used for conditioning. Figure \ref{fig:example_plot} illustrates example independent draws from a \emph{conditional derivative Gaussian process}.
\begin{figure}[p]
\begin{center}
\centerline{\includegraphics[width=0.7\textwidth]{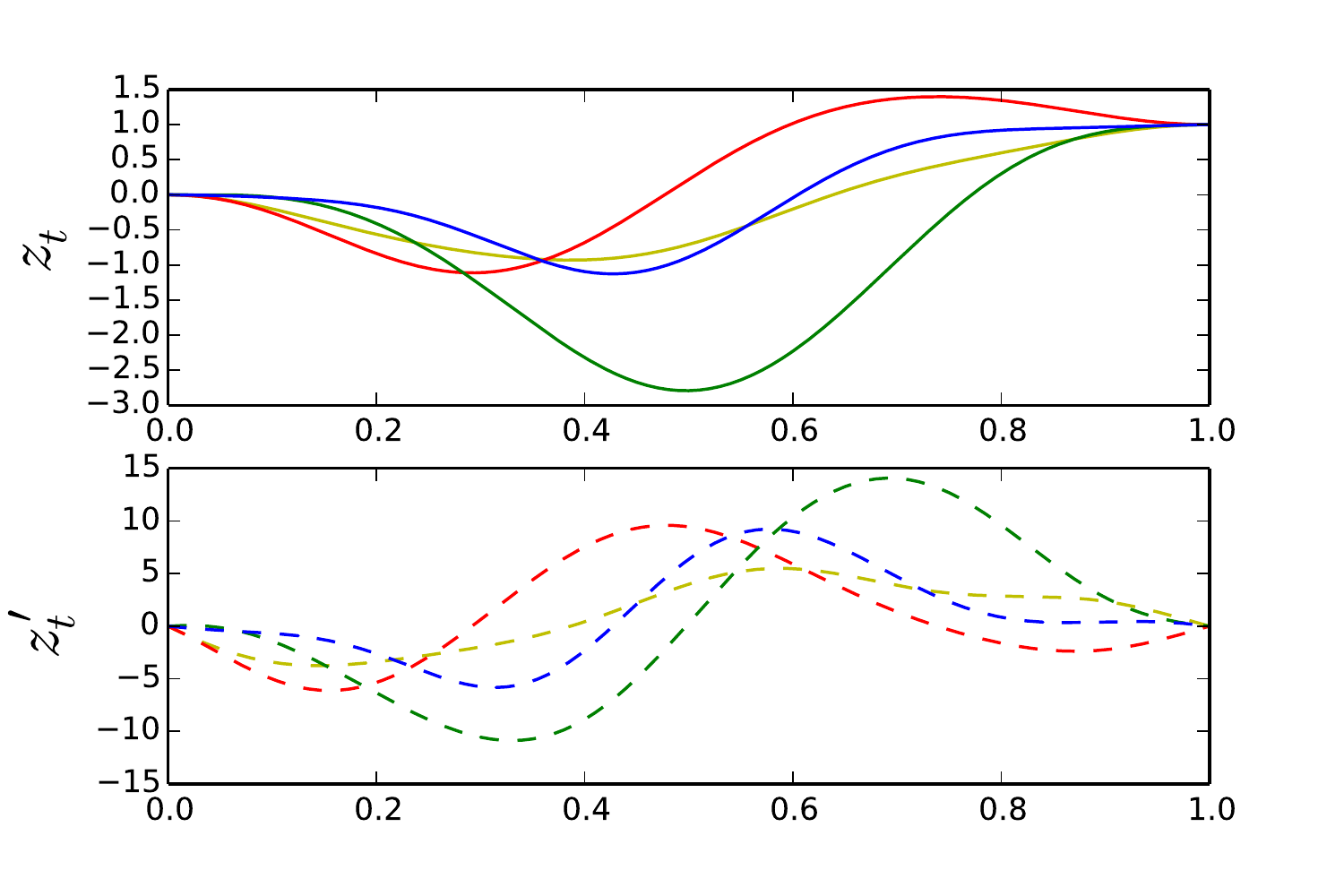}}
\caption{Draws from a conditional derivative GP conditioned to start at $0$ with derivative $0$ and to finish at $1.0$ with derivative $0.0$. The unconditional kernel is the squared exponential kernel with variance $1.0$ and input scale $0.2$.}
\label{fig:example_plot}
\end{center}
\end{figure}
%
%
%
%
%	SUBSECTION: STRING GP ON R
%
%
%
%
\subsection{String Gaussian Processes on $\mathbb{R}$}
The intuition behind string Gaussian processes on an interval comes from the analogy of collaborative local GP experts we refer to as \emph{strings} that are connected but independent of each other conditional on some regularity boundary conditions. While each string is tasked with representing local patterns in the data, a string only shares the \emph{states} of its extremities (value and derivative) with adjacent strings. Our aim is to preserve global smoothness and limit the amount of information shared between strings, thus reducing computational complexity. Furthermore, the conditional independence between strings will allow for distributed inference, greater flexibility and principled nonstationarity construction. 

The following theorem at the core of our framework establishes that it is possible to connect together GPs on a partition of an interval $I$, in a manner consistent enough that the newly constructed stochastic object will be a stochastic process on $I$ and in a manner restrictive enough that any two connected GPs will share just enough information to ensure that the constructed stochastic process is continuously differentiable ($\mathcal{C}^1$) on $I$ in the $L^2$ sense.
\begin{theorem} \textbf{(String Gaussian process)}
\label{theo:sgp}
\\Let $a_0<\dots<a_k< \dots<a_K$, $I=[a_0, a_K]$ and let $p_\mathcal{N}(x; \mu, \Sigma)$ be the multivariate Gaussian density with mean vector $\mu$ and covariance matrix $\Sigma$. Furthermore, let $(m_k:  [a_{k-1}, a_k] \to \mathbb{R})_{k \in [1..K]}$ be $\mathcal{C}^1$ functions, and   $(k_k: [a_{k-1}, a_k] \times  [a_{k-1}, a_k]\to \mathbb{R})_{k \in [1..K]}$ be $\mathcal{C}^3$ symmetric positive semi-definite functions, neither degenerate at $a_{k-1}$, nor degenerate at $a_k$ given $a_{k-1}$.\\

\noindent (A) There exists an $\mathbb{R}^2$-valued stochastic process $(SD_{t})_{t \in I}, ~ SD_t=(z_t, z_t^\prime)$  satisfying the following conditions:

\noindent 1) The probability density of $(SD_{a_0}, \dots, SD_{a_K})$ reads: 
\begin{equation}
\label{eq:bound_pdf}
p_{b}(x_0, \dots, x_K) := \prod_{k=0}^K p_\mathcal{N}\left(x_k;  \mu^b_k, \Sigma_k^b\right)
\end{equation}
\begin{equation}
\label{eq:sb0k}
\text{where:}~~~~
\Sigma_0^b = {}_1\textbf{K}_{a_0; a_0}, ~~
\forall ~ k>0 ~~~\Sigma_k^b = {}_k\textbf{K}_{a_k; a_k} - {}_k\textbf{K}_{a_k; a_{k-1}}~{}_k\textbf{K}_{a_{k-1}; a_{k-1}}^{-1}~{}_k\textbf{K}_{a_k; a_{k-1}}^T,
\end{equation}
\begin{equation}
\label{eq:mb0k}
\mu_0^b={}_1\textbf{M}_{a_0}, ~~
\forall ~ k>0 ~~~\mu^b_k={}_k\textbf{M}_{a_k} + {}_k\textbf{K}_{a_k; a_{k-1}}~{}_k\textbf{K}_{a_{k-1}; a_{k-1}}^{-1}(x_{k-1}-{}_k\textbf{M}_{a_{k-1}}), 
\end{equation}
\[
\text{with}~~~~{}_k\textbf{K}_{u;v} = 
\begin{bmatrix} 
k_k(u, v) & \frac{\partial k_k}{\partial y}(u, v) \\
\frac{\partial k_k}{\partial x}(u, v)  & \frac{\partial^2 k_k}{\partial x \partial y}(u, v) 
\end{bmatrix}, ~~~~\text{and}~~~~
{}_k\textbf{M}_u =\begin{bmatrix} m_k(u) \\ \frac{d m_k}{dt}(u) \end{bmatrix}.
\]

\noindent 2) Conditional on $(SD_{a_k} = x_k)_{k \in [0..K]}$, the restrictions $(SD_{t})_{t \in  ]a_{k-1}, a_k[},~ k \in [1..K]$ are \textbf{independent conditional derivative Gaussian processes}, respectively with unconditional mean function $m_k$ and unconditional covariance function $k_k$ and that are conditioned to take values $x_{k-1}$ and $x_k$ at $a_{k-1}$ and $a_k$ respectively. We refer to $(SD_{t})_{t \in I}$  as a \textbf{string derivative Gaussian process}, and to its first coordinate $(z_{t})_{t \in I}$ as a \textbf{string Gaussian process} namely, \[(z_{t})_{t \in I} \sim \mathcal{SGP}(\{a_k\}, \{m_k\}, \{k_k\}).\]

\noindent (B) The \textbf{string Gaussian process} $(z_t)_{t \in I}$ defined in (A) is $\mathcal{C}^1$ in the $L^2$ sense and its $L^2$ derivative is the process $(z_t^\prime)_{t \in I}$ defined in (A).
\end{theorem}
\begin{proof}
See \ref{app:sgp}.
\end{proof}
In our \emph{collaborative local GP experts} analogy, Theorem \ref{theo:sgp} stipulates that each local expert takes as message from the previous expert its left hand side boundary conditions, conditional on which it generates its right hand side boundary conditions, which it then passes on to the next expert. Conditional on their boundary conditions local experts are independent of each other, and resemble vibrating pieces of string on fixed extremities, hence the name \emph{string Gaussian process}.
%
%
%
%
%	SUBSUBSECTION: PATHWISE REGULARITY
%
%
%
%
\subsection{Pathwise Regularity}
\label{sct:reg_upgrade}
Thus far we have dealt with regularity only in the $L^2$ sense. However, we note that a sufficient condition for the process $(z_t^\prime)_{t \in I}$ in Theorem \ref{theo:sgp} to be almost surely continuous (i.e. sample paths are continuous with probability $1$) and to be the almost sure  derivative of the string Gaussian process $(z_t)_{t \in I}$, is that the Gaussian processes on $I_k=[a_{k-1}, a_k]$ with mean and covariance functions $m_{ck}^{a_{k-1}, a_k}$ and $k^{a_{k-1}, a_k}_{ck}$ (as per Equations \ref{eq:double_cond_mean} and \ref{eq:double_cond_cov} with $m:= m_k$ and $k:= k_k$) are themselves almost surely $\mathcal{C}^1$ for every boundary condition.\footnote{The proof is provided in \ref{app:path_reg}.} We refer to  \citep[Theorem 2.5.2]{adlertaylor} for a sufficient condition under which a $\mathcal{C}^1$ in $L^2$ Gaussian process is also almost surely $\mathcal{C}^1$. As the above question is provably equivalent to that of the almost sure continuity of a Gaussian process   \citep[see][p. 30]{adlertaylor}, \emph{Kolmogorov's continuity theorem} \citep[see][Theorem 2.2.3]{oks} provides a more intuitive, albeit stronger, sufficient condition than that of \citep[Theorem 2.5.2]{adlertaylor}.
%
%
%
%
%	SUBSUBSECTION: SIMULATION
%
%
%
%
\subsection{Illustration}
\label{sct:simu}
Algorithm \ref{alg:simulation} illustrates sampling jointly from a string Gaussian process and its derivative on an interval $I=[a_0, a_K]$. We start off by sampling the string boundary conditions $(z_{a_k}, z^\prime_{a_k})$ sequentially, conditional on which we sample the values of the stochastic process on each string. This we may do in parallel as the strings are  independent of each other conditional on boundary conditions. The resulting time complexity is the sum of $\mathcal{O}(\max ~ n_k^3)$ for sampling values within strings, and  $\mathcal{O}(n)$ for sampling boundary conditions, where the sample size is $n=\sum_{k} n_k$. The memory requirement grows as the sum of  $\mathcal{O}(\sum_{k} n_k^2)$, required to store conditional covariance matrices of the values within strings, and $\mathcal{O}(K)$ corresponding to the storage of covariance matrices of boundary conditions. In the special case where strings are all empty, that is inputs and boundary times are the same, the resulting time complexity and memory requirement are $\mathcal{O}(n)$. Figure \ref{fig:example_plot_string} illustrates a sample from a string Gaussian process, drawn using this approach. 

\begin{figure}[p]
\begin{center}
\centerline{\includegraphics[width=0.7\textwidth]{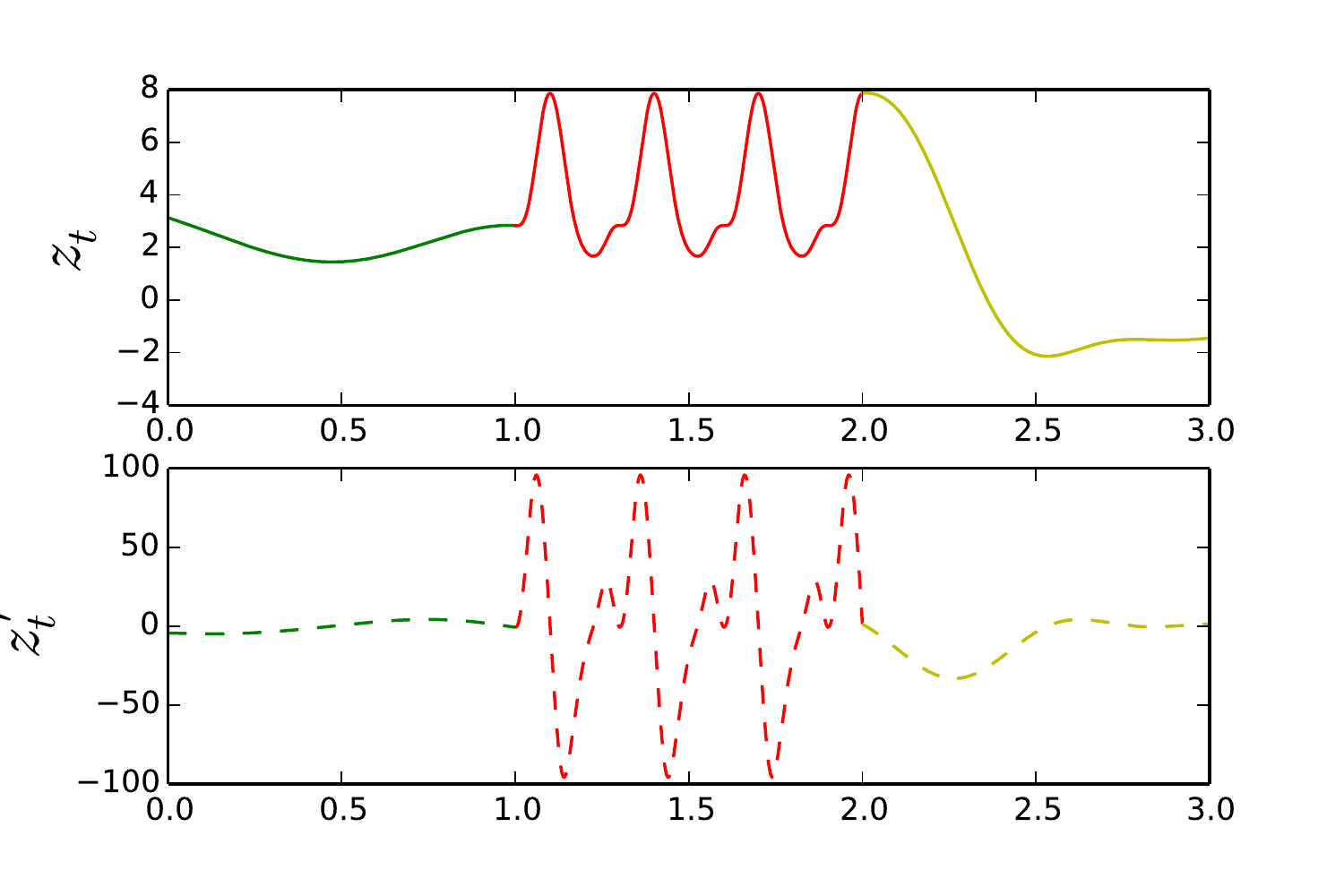}}
\caption{Draw from a \emph{string GP} $(z_t)$ with 3 strings and its derivative $(z_t^\prime)$, under squared exponential kernels (green and yellow strings), and the periodic kernel of \cite{gp_intro} (red string).}
\label{fig:example_plot_string}
\end{center}
\end{figure} 

\begin{algorithm}[p]
   \caption{Simulation of a string derivative Gaussian process}
   \label{alg:simulation}
\begin{algorithmic}
   \STATE {\bfseries Inputs:}  boundary times $a_0<\dots<a_K$, string times $\{t_j^k \in ]a_{k-1}, a_k[\}_{j \in [1..n_k], k \in [1..K]}$, unconditional mean (resp. covariance) functions $m_k$ (resp. $k_k$)
   \STATE {\bfseries Output:} $\{\dots, z_{a_k}, z_{a_k}^\prime, \dots, z_{t_j^k}, z_{t_j^k}^\prime, \dots\}$.
   \\\hrulefill
   \STATE {\bfseries \underline{Step 1}: sample the boundary conditions sequentially.}
    \FOR{$k=0$ {\bfseries to} $K$}
    \STATE Sample $(z_{a_k}, z_{a_k}^\prime) \sim \mathcal{N}\Big(\mu_k^b, \Sigma_k^b\Big)$, with $\mu_k^b$ and $\Sigma_k^b$ as per Equations (\ref{eq:mb0k}) and (\ref{eq:sb0k}).
   \ENDFOR
   \STATE {\bfseries \underline{Step 2}: sample the values on each string conditional on the boundary conditions in \emph{parallel}.}
\PARFOR{$k=1$ {\bfseries to} $K$}
\STATE Let ${}_k \textbf{M}_{u}$ and ${}_k \textbf{K}_{u;v}$ be as in Theorem  \ref{theo:sgp},
 \[
{}_k \Lambda = \begin{bmatrix}
{}_k \textbf{K}_{t_1^k; a_{k-1}} & {}_k \textbf{K}_{t_1^k; a_{k}} \\ 
\dots& \dots\\ 
{}_k \textbf{K}_{t_{n_k}^k; a_{k-1}} & {}_k \textbf{K}_{t_{n_k}^k; a_{k}}
\end{bmatrix}
\begin{bmatrix}
{}_k \textbf{K}_{a_{k-1}; a_{k-1}} & {}_k \textbf{K}_{a_{k-1}; a_{k}} \\ 
{}_k \textbf{K}_{a_{k}; a_{k-1}} & {}_k \textbf{K}_{a_{k}; a_{k}}
\end{bmatrix}^{-1},
\]
\STATE \begin{equation}
\label{eq:msk}
\mu_k^s=
\begin{bmatrix}
{}_k \textbf{M}_{t_1^k} \\ 
\dots\\ 
{}_k \textbf{M}_{t_{n_k}^k}
\end{bmatrix} +
{}_k \Lambda
\begin{bmatrix}
z_{a_{k-1}}-m_k(a_{k-1}) \\ 
z_{a_{k-1}}^\prime-\frac{dm_k}{dt}(a_{k-1}) \\ 
z_{a_{k}}-m_k(a_{k}) \\ 
z_{a_{k}}^\prime-\frac{dm_k}{dt}(a_{k}) \\ 
\end{bmatrix},
\end{equation}
\STATE \begin{equation}
\label{eq:ssk}
\Sigma_k^s = 
 \begin{bmatrix}
{}_k \textbf{K}_{t_1^k; t_1^k} & \dots & {}_k \textbf{K}_{t_1^k;t_{n_k}^k} \\ 
\dots & \dots& \dots\\ 
{}_k \textbf{K}_{t_{n_k}^k; t_1^k} & \dots & {}_k \textbf{K}_{t_{n_k}^k;t_{n_k}^k}
\end{bmatrix}- {}_k \Lambda
\begin{bmatrix}
{}_k \textbf{K}_{t_1^k; a_{k-1}} & {}_k \textbf{K}_{t_1^k; a_{k}} \\ 
\dots& \dots\\ 
{}_k \textbf{K}_{t_{n_k}^k; a_{k-1}} & {}_k \textbf{K}_{t_{n_k}^k; a_{k}}
\end{bmatrix}^T.
\end{equation}
\STATE Sample $\left(z_{t_1^k}, z_{t_1^k}^\prime \dots, z_{t_{n_k}^k}, z_{t_{n_k}^k}^\prime\right) \sim \mathcal{N}\left(\mu_k^s, \Sigma_k^s\right)$.
   \ENDPARFOR
\end{algorithmic}
\end{algorithm}

%
%
%
%
%
%	SUBSECTION: STRING GPS ON R^d
%
%
%
%
\subsection{String Gaussian Processes on $\mathbb{R}^d$}
So far the input space has been assumed to be an interval. We generalise \emph{string GPs} to hyper-rectangles in $\mathbb{R}^d$ as stochastic processes of the form:
\begin{equation}
\label{eq:multi_inputs_prior}
f(t_1, \dots, t_d)= \phi\left(z^1_{t_1}, \dots, z^d_{t_d}\right),
\end{equation}
where the \emph{link function} $\phi: \mathbb{R}^d \to \mathbb{R}$ is a $\mathcal{C}^1$ function and $(z^j_t)$ are $d$ independent ($\perp$) latent string Gaussian processes on intervals. We will occasionally refer to \emph{string GPs} indexed on $\mathbb{R}^d$ with $d>1$ as \emph{membrane GPs} to avoid any ambiguity. We note that when $d=1$ and when the link function is $\phi(x)=x$, we recover \emph{string GPs} indexed on an interval as previously defined. When the \emph{string GPs} $(z_t^j)$ are a.s. $\mathcal{C}^1$, the \emph{membrane GP} $f$ in Equation (\ref{eq:multi_inputs_prior}) is also a.s. $\mathcal{C}^1$, and the partial derivative with respect to the $j$-th coordinate reads:
\begin{equation}
\label{eq:sgp_gradient}
\frac{\partial f}{\partial t_j}(t_1, \dots, t_d) = z^{j \prime}_{t_j} \frac{\partial \phi}{\partial t_j} \left(z^1_{t_1}, \dots, z^d_{t_d}\right).
\end{equation}
Thus in high dimensions, \emph{string GPs} easily allow an explanation of the sensitivity of the learned latent function to inputs. 
%
%
%
%
%	SUBSECTION: CHOICE OF LINK FUNCTION
%
%
%
%
\subsection{Choice of Link Function}
\label{sct:need_not}
Our extension of \emph{string GPs} to $\mathbb{R}^d$ departs from the \emph{standard GP paradigm} in that we did not postulate a covariance function on $\mathbb{R}^d \times \mathbb{R}^d$ directly. Doing so usually requires using a metric on $\mathbb{R}^d$, which is often problematic for heterogeneous input dimensions, as it introduces an arbitrary comparison between distances in each input dimension. This problem has been partially addressed by approaches such as Automatic Relevance Determination (ARD) kernels, that allow for a linear rescaling of input dimensions to be learned jointly with kernel hyper-parameters. However,   inference under a \emph{string GP} prior can be thought of as learning a coordinate system in which the latent function $f$ resembles the link function $\phi$ through non-linear rescaling of input dimensions. In particular, when $\phi$ is symmetric, the learned univariate \emph{string GPs} (being interchangeable in $\phi$) implicitly aim at normalizing input data across dimensions, making \emph{string GPs} naturally cope with heterogeneous data sets. 

An important question arising from our extension is  whether or not the link function $\phi$ needs to be learned to achieve a flexible functional prior. The flexibility of a \emph{string GP} as a functional prior depends on both the \emph{link function} and the covariance structures of the underlying \emph{string GP} building blocks $(z_t^j)$. To address the impact of the choice of $\phi$ on flexibility, we constrain the \emph{string GP} building blocks by restricting them to be independent identically distributed \emph{string GPs} with one string each (i.e. $(z_t^j)$ are i.i.d Gaussian processes). Furthermore, we restrict ourselves to isotropic kernels as they provide a consistent basis for putting the same covariance structure in $\mathbb{R}$ and $\mathbb{R}^d$. One question we might then ask, for a given \emph{link function} $\phi_0$, is whether or not an isotropic GP indexed on $\mathbb{R}^d$ with covariance function $k$ yields more flexible random surfaces than the stationary \emph{string GP} $f(t_1, \dots, t_d) = \phi_0(z_{t_1}^1, \dots, z_{t_d}^d)$, where $(z_{t_j}^j)$ are stationary GPs indexed on $\mathbb{R}$ with the same covariance function $k$. If we find a \emph{link function} $\phi_0$ generating more flexible random surfaces than isotropic GP counterparts it would suggest  $\phi$ need not to be inferred in dimension $d>1$ to be more flexible than any GP using one of the large number of commonly used isotropic kernels, among which squared exponential kernels, rational quadratic kernels, and Mat\'ern kernels to name but a few.

Before discussing whether such a $\phi_0$ exists, we need to introduce a rigorous meaning to `flexibility'. An intuitive qualitative definition of the flexibility of a stochastic process indexed on $\mathbb{R}^d$ is the ease with which it can generate surfaces with varying shapes from one random sample to another independent one. We recall that the tangent hyperplane to a $\mathcal{C}^1$ surface $y-f(x) =0, x \in \mathbb{R}^d$ at some point $x_0=(t^0_1, \dots, t^0_d)$ has equation $\nabla f(x_0)^T(x-x_0)-(y-f(x_0))=0$ and admits as normal vector $(\frac{\partial f}{\partial t_1}(t^0_1), \dots, \frac{\partial f}{\partial t_d}(t^0_d), -1)$. As tangent hyperplanes approximate a surface locally, a first criterion of flexibility for a random surface $y-f(x)=0, ~x \in \mathbb{R}^d$ is the proclivity of the (random) direction of its tangent hyperplane at any point $x$---and hence the proclivity of $\nabla f(x)$---to vary. This criterion alone, however, does not capture the difference between the local shapes of the random surface at two distinct points. A complementary second criterion of flexibility is the proclivity of the (random) directions of the tangent hyperplanes at any two distinct points $x_0, x_1 \in \mathbb{R}^d$---and hence the proclivity of $\nabla f(x_0)$ and $\nabla f(x_1)$---to be independent. The first criterion can be measured using the entropy of the gradient at a point, while the second criterion can be measured through the mutual information between the two gradients. The more flexible a stochastic process, the higher the entropy of its gradient at any point, and the lower the mutual information between its gradients at any two distinct points. This is formalised in the definition below.

\begin{definition} \textbf{(Flexibility of stochastic processes)}\\
Let $f$ and $g$ be two real valued, almost surely $\mathcal{C}^1$ stochastic processes indexed on $\mathbb{R}^d$, and whose gradients have a finite entropy everywhere (i.e. $\forall ~ x, ~ H(\nabla f(x)), H(\nabla g(x)) < \infty$). We say that $f$ is more flexible than $g$ if the following conditions are met:

\noindent 1)$~\forall ~ x, ~ H(\nabla f(x)) \geq H(\nabla g(x)),$

\noindent 2)$~\forall ~ x \neq y, ~ I(\nabla f(x); \nabla f(y)) \leq I(\nabla g(x); \nabla g(y))$,\\
where $H$ is the entropy operator, and $I(X; Y)= H(X) + H(Y) - H(X, Y)$ stands for the mutual information between $X$ and $Y$.
\end{definition}

\noindent The following proposition establishes that the \emph{link function} $\phi_s\left(x_1, \dots, x_d\right) = \sum_{i=j}^d x_j$ yields more flexible stationary \emph{string GPs} than their isotropic GP counterparts, thereby providing a theoretical underpinning for not inferring $\phi$.  

\begin{proposition}\textbf{(Additively separable \emph{string GPs} are flexible)}\\
\label{prop:diversity}Let $k(x, y) := \rho\left(\vert\vert x-y \vert\vert^2_{{L^2}}\right)$ be a stationary covariance function generating a.s. $\mathcal{C}^1$ GP paths indexed on $\mathbb{R}^d, ~ d>0$, and $\rho$ a function that is $\mathcal{C}^2$ on $]0, +\infty[$ and continuous at $0$. Let $\phi_s(x_1, \dots, x_d)=\sum_{j=1}^d x_j$, let $(z_t^j)_{t \in I^j, ~ j \in [1..d]}$ be independent stationary Gaussian processes with mean $0$ and covariance function $k$ (where the $L^2$ norm is on $\mathbb{R}$), and let $f(t_1, \dots, t_d)=\phi_s(z_{t_1}^1, \dots, z_{t_d}^d)$ be the corresponding stationary string GP. Finally, let $g$ be an isotropic Gaussian process indexed on $I^1 \times \dots \times I^d$ with mean 0 and covariance function $k$ (where the $L^2$ norm is on $\mathbb{R}^d$). Then: \\\\
1)$~\forall ~ x \in I^1 \times \dots \times I^d, ~ H(\nabla f(x)) = H(\nabla g(x))$,\\
2)$~\forall ~ x \neq y \in I^1 \times \dots \times I^d, ~ I(\nabla f(x); \nabla f(y)) \leq I(\nabla g(x); \nabla g(y))$.
\end{proposition}
\begin{proof}
 See \ref{app:diversity}.
\end{proof}
Although the link function need not be inferred in a full nonparametric fashion to yield comparable if not better results than most isotropic kernels used in the \emph{standard GP paradigm}, for some problems certain link functions might outperform others. In Section \ref{sct:means_kerns} we analyse a broad family of link functions, and argue that they extend successful anisotropic approaches such as the Automatic Relevance Determination (\cite{gp_intro}) and the \text{additive kernels} of \cite{DuvenaudNR2012}. Moreover, in Section \ref{sct:infer} we propose a scalable inference scheme applicable to any link function.
%
%
%
%
%	SECTION: COMPARISON WITH THE STANDARD GP PARADIGM
%
%
%
%
\section{Comparison with the Standard GP Paradigm}
\label{sct:comp}
We have already established that sampling \emph{string GPs} scales better than sampling GPs under the \emph{standard GP paradigm} and is amenable to distributed computing. We have also established that stationary additively separable \emph{string GPs} are more flexible than their isotropic counterparts in the \emph{standard GP paradigm}. In this section, we provide further theoretical results relating the \emph{string GP paradigm} to the \emph{standard GP paradigm}. Firstly we establish that \emph{string GPs} with link function $\phi_s\left(x_1, \dots, x_d\right) = \sum_{i=j}^d x_j$ are GPs. Secondly, we derive the global mean and covariance functions induced by the \emph{string GP} construction for a variety of link functions. Thirdly, we provide a sense in which the \emph{string GP paradigm} can be thought of as extending the \emph{standard GP paradigm}. And finally, we show that the \emph{string GP paradigm} may serve as a scalable approximation of commonly used stationary kernels.
%
%
%
%
%
%	SUBSECTION: STRING GPs AS GPs
%
%
%
%
\subsection{Some String GPs are GPs}
On one hand we note from Theorem \ref{theo:sgp} that the restriction of a \emph{string GP} defined on an interval to the support of the first string---in other words the first local GP expert---is a Gaussian process. On the other hand, the messages passed on from one local GP expert to the next are not necessarily consistent with the unconditional law of the receiving local expert, so that overall a \emph{string GP} defined on an interval, that is when looked at globally and unconditionally, might not be a Gaussian process. However, the following proposition establishes that some \emph{string GPs} are indeed Gaussian processes.
\begin{proposition}\textbf{(Additively separable \emph{string GPs} are GPs)}\\
\label{prop:is_gp}String Gaussian processes on $\mathbb{R}$ are Gaussian processes. Moreover, string Gaussian processes on $\mathbb{R}^d$ with link function $\phi_s(x_1, \dots, x_d) = \sum_{j=1}^d x_j$ are also Gaussian processes.
\end{proposition}
\begin{proof}
The intuition behind this proof lies in the fact that if $X$ is a multivariate Gaussian, and if conditional on $X$, $Y$ is a multivariate Gaussian, providing that the conditional mean of $Y$ depends linearly on $X$ and the conditional covariance matrix of $Y$ does not depend on $X$, the vector $(X, Y)$ is jointly Gaussian. This will indeed be the case for our collaboration of local GP experts as the boundary conditions picked up by an expert from the previous will not influence the conditional covariance structure of the expert (the conditional covariance strucuture depends only on the partition of the domain, not the values of the boundary conditions) and will affect the mean linearly. See \ref{app:is_gp} for the full proof.
\end{proof}

The above result guarantees that commonly used closed form predictive equations under GP priors are still applicable under some \emph{string GP} priors, providing the global mean and covariance functions, which we derive in the following section, are available. Proposition \ref{prop:is_gp} also guarantees stability of the corresponding \emph{string GPs} in the GP family under addition of independent Gaussian noise terms as in regression settings. Moreover, it follows from Proposition \ref{prop:is_gp} that inference techniques developed for Gaussian processes can be readily used under \emph{string GP} priors. In Section \ref{sct:infer} we provide an additional MCMC scheme that exploits the conditional independence between strings to yield greater scalability and distributed inference.
%
%
%
%
%
%	SUBSECTION: STRING GP KERNELS AND STRING GP MEAN FUNCTIONS
%
%
%
%
\subsection{String GP Kernels and String GP Mean Functions}
\label{sct:means_kerns}
The approach we have adopted in the construction of \emph{string GPs} and \emph{membrane GPs} did not require explicitly postulating a global mean function or covariance function.  In \ref{app:global_mean_cov} we derive the global mean and covariance functions that result from our construction. The global covariance function could be used for instance as a stand-alone kernel in any kernel method, for instance GP models under the \emph{standard GP paradigm}, which would provide a flexible and nonstationary alternative to commonly used kernels that may be used to learn local patterns in data sets---some successful example applications are provided in Section \ref{sct:infer}. That being said, adopting such a global approach should be limited to small scale problems as the conditional independence structure of \emph{string GPs} does not easily translate into structures in covariance matrices over \emph{string GP} values (without derivative information) that can be exploited to speed-up SVD or Cholesky decomposition. Crucially, marginalising out all derivative information  in the distribution of \emph{derivative string GP} values at some inputs  would destroy any conditional independence structure, thereby limiting opportunities for scalable inference. In Section \ref{sct:infer} we will provide a RJ-MCMC inference scheme that fully exploits the conditional independence structure in \emph{string GPs} and scales to \emph{very large} data sets. 
\subsection{Connection Between Multivariate String GP Kernels and Existing Approaches}
\label{sct:ex_kern}
We recall that for $n\leq d$, the $n$-th order \emph{elementary symmetric polynomial} (\cite{macdonald}) is given by
\begin{align}
e_0(x_1, \dots, x_d) := 1, ~~~ \forall 1 \leq n \leq d ~~ e_n(x_1, \dots, x_d)= \sum_{1 \leq j_1 < j_2 < \dots < j_n \leq d} ~~ \prod_{k=1}^n x_{j_k}.
\end{align}
As an illustration, \[e_1(x_1, \dots, x_d) = \sum_{j=1}^{d} x_j = \phi_s(x_1, \dots, x_d),\] \[e_2(x_1, \dots, x_d) = x_1x_2 + x_1x_3 + \dots +  x_1x_d + \dots +  x_{d-1}x_d,\]\[\dots\]
\[e_d(x_1, \dots, x_d) = \prod_{j=1}^{d} x_j = \phi_p(x_1, \dots, x_d).\]
Covariance kernels of \emph{string GPs}, using as link functions \emph{elementary symmetric polynomials} $e_n$, extend most popular approaches that combine unidimensional kernels over features for greater flexibility or cheaper design experiments.

The first-order polynomial $e_1$ gives rise to additively separable Gaussian processes, that can be regarded as Bayesian nonparametric \emph{generalised additive models} (GAM), particularly popular for their interpretability. Moreover, as noted by \cite{durrande}, additively separable Gaussian processes are considerably cheaper than alternate transformations in design experiments with high-dimensional input spaces. In addition to the above, additively separable \emph{string GPs} also allow postulating the existence of local properties in the experimental design process at no extra cost.

The $d$-th order polynomial $e_d$ corresponds to a product of unidimensional kernels, also known as separable kernels. For instance, the popular squared exponential kernel is separable. Separable kernels have been successfully used on large scale inference problems where the inputs form a grid \citep{saatchi11, GPatt}, as they yield covariance matrices that are Kronecker products, leading to maximum likelihood inference in linear time complexity and with linear memory requirement. Separable kernels are often used in conjunction with the \emph{automatic relevance determination} (ARD) model, to learn the relevance of features through global linear rescaling. However, ARD kernels might be limited in that we might want the relevance of a feature to depend on its value. As an illustration, the market value of a watch can be expected to be a stronger indicator of its owner's wealth when it is in the top 1 percentile, than when it is in the bottom 1 percentile; the rationale being that possessing a luxurious watch is an indication that one can afford it, whereas possessing a cheap watch might be either an indication of lifestyle or an indication that one cannot afford a more expensive one. Separable \emph{string GP} kernels extend ARD kernels, in that strings between input dimensions and within an input dimension may have unconditional kernels with different hyper-parameters, and possibly different functional forms, thereby allowing for \emph{automatic local relevance determination} (ALRD).

More generally, using as link function the $n$-th order elementary symmetric polynomial $e_n$ corresponds to the $n$-th order interaction of the \emph{additive kernels} of \cite{DuvenaudNR2012}. We also note that the class of link functions $\phi(x_1, \dots, x_d)= \sum_{i=1}^d \sigma_i e_i(x_1, \dots, x_d)$ yield full \emph{additive kernels}. \cite{DuvenaudNR2012} noted that such kernels are `exceptionally well-suited' to learn non-local structures in data. \emph{String GPs} complement \emph{additive kernels} by allowing them to learn local structures as well.
\subsection{String GPs as Extension of the Standard GP Paradigm}
The following proposition provides a perspective from which \emph{string GPs} may be considered as extending Gaussian processes on an interval.

\begin{proposition}\textbf{(Extension of the \emph{standard GP paradigm})}\\
\label{prop:extension}Let $K \in \mathbb{N}^{*}$, let $I=[a_0, a_K]$ and $I_k= [a_{k-1}, a_k]$ be intervals with $a_0 < \dots < a_K$. Furthermore, let $m: I \to \mathbb{R}$ be a $\mathcal{C}^1$ function, $m_k$ the restriction of $m$ to $I_k$, $h: I \times I \to \mathbb{R}$ a $\mathcal{C}^3$ symmetric positive semi-definite function, and $h_k$ the restriction of $h$ to $I_k \times I_k$. If \[(z_t)_{t \in I} \sim \mathcal{SGP}(\{a_k\}, \{m_k\}, \{h_k\}),\] then \[ \forall ~ k \in [1..K], ~ (z_t)_{t \in I_k} \sim \mathcal{GP}(m, h).\]
\end{proposition}
\begin{proof}
See \ref{app:extension}.
\end{proof}

We refer to the case where unconditional string mean and kernel functions are restrictions of the same functions as in Proposition \ref{prop:extension} as \emph{uniform string GPs}. Although uniform \emph{string GPs} are not guaranteed to be as much regular at boundary times as their counterparts in the \emph{standard GP paradigm}, we would like to stress that they may well generate paths that are. In other words, the functional space induced by a uniform \emph{string GP} on an interval extends the functional space of the GP with the same mean and covariance functions $m$ and $h$ taken globally and unconditionally on the whole interval as in the \emph{standard GP paradigm}. This allows for (but does not enforce) less regularity at the boundary times. When \emph{string GPs} are used as functional prior, the posterior mean can in fact have more regularity at the boundary times than the continuous differentiability enforced in the \emph{string GP paradigm}, providing such regularity is evidenced in the data.

We note from Proposition \ref{prop:extension} that when $m$ is constant and $h$ is stationary, the restriction of the uniform \emph{string GP} $(z_t)_{t \in I}$ to any interval whose interior does not contain a boundary time, the largest of which being the intervals $[a_{k-1}, a_k]$, is a stationary GP. We refer to such cases as \emph{partition stationary string GPs}. 
\subsection{Commonly Used Covariance Functions and their String GP Counterparts}
\label{sct:approx}
Considering the superior scalability of the \emph{string GP paradigm}, which we may anticipate from the scalability of sampling \emph{string GPs}, and which we will confirm empirically in Section \ref{sct:infer}, a natural question that comes to mind is whether or not kernels commonly used in the \emph{standard GP paradigm} can be well approximated by \emph{string GP} kernels, so as to take advantage of the improved scalability of the \emph{string GP paradigm}. We examine the distortions to commonly used covariance structures resulting from restricting strings to share only $\mathcal{C}^1$ boundary conditions, and from increasing the number of strings.

Figure \ref{fig:string_effects} compares some popular stationary kernels on $[0,1] \times [0, 1]$ (first column) to their uniform \emph{string GP} kernel counterparts with 2, 4, 8 and 16 strings of equal length. The popular kernels considered are the squared exponential kernel (SE), the rational quadratic kernel $k_{RQ}(u,v) = \left(1+\frac{2(u-v)^2}{\alpha}\right)^{-\alpha}$ with $\alpha = 1$ (RQ 1) and  $\alpha = 5$ (RQ 5), the Mat\'ern 3/2 kernel (MA 3/2), and the Mat\'ern 5/2 kernel (MA 5/2), each with output scale (variance) $1$ and input scale $0.5$. Firstly, we observe that each of the popular kernels considered coincides with its uniform \emph{string GP} counterparts regardless of the number of strings, so long as the arguments of the covariance function are less than an input scale apart. Except for the Mat\'ern 3/2, the loss of information induced by restricting strings to share only $\mathcal{C}^1$ boundary conditions becomes noticeable when the arguments of the covariance function are more than $1.5$ input scales apart, and the effect is amplified as the number of strings increases. As for the Mat\'ern 3/2, no loss of information can been noticed, as further attests Table \ref{table:string_effects}. In fact, this comes as no surprise given that stationary Mat\'ern 3/2 GP are $1$-Markov, that is the corresponding derivative Gaussian process is a Markov process so that the vector $(z_t, z_t^\prime)$ contains as much information as all \emph{string GP} or derivative values prior to $t$ (see \cite{doob44}). Table \ref{table:string_effects} provides some statistics on the absolute errors between each of the popular kernels considered and uniform \emph{string GP} counterparts.
\begin{figure}[p]
\begin{center}
\centerline{\includegraphics[width=0.7\textwidth]{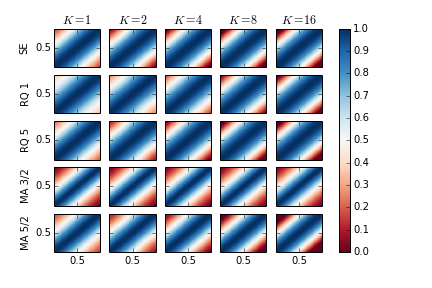}}
\caption{Commonly used covariance functions on $[0, 1] \times [0, 1]$ with the same input and output scales (first column) and their uniform \emph{string GP} counterparts with $K>1$ strings of equal length.}
\label{fig:string_effects}
\end{center}
\end{figure} 
\begin{landscape}
\begin{table*}
\centering
\begin{tabular}{@{}rrrrcrrrcrrrcrrr@{}}\toprule
& \multicolumn{3}{c}{$K = 2$} & \phantom{abc}& \multicolumn{3}{c}{$K = 4$} &
\phantom{abc} & \multicolumn{3}{c}{$K = 8$}& \phantom{abc} & \multicolumn{3}{c}{$K = 16$}\\
\cmidrule{2-4} \cmidrule{6-8} \cmidrule{10-12} \cmidrule{14-16}
& min & avg & max 		&& 	min & avg & max  	&& min & avg & max	&& min & avg & max \\ \midrule
SE & 0 & 0.01 & 0.13 	&& 	0 & 0.02 & 0.25 	&& 0 & 0.03 & 0.37 	&& 0 & 0.04 & 0.44\\
RQ 1 & 0 & 0.01 & 0.09 	&& 	0 & 0.03 & 0.20 	&& 0 & 0.05 & 0.37 	&& 0 & 0.07 & 0.52\\
RQ 5 & 0 & 0.01 & 0.12 	&& 	0 & 0.02 & 0.24 	&& 0 & 0.04 & 0.37 	&& 0 & 0.05 & 0.47\\
MA 3/2 & 0 & 0 & 0		 && 	0 & 0	 & 0	 	&& 0 & 0 & 0		 	&& 0 & 0	 & 0\\
MA 5/2 & 0 & 0.01 & 0.07 	&& 	0 & 0.03 & 0.15 	&& 0 & 0.05 & 0.29 	&& 0 & 0.08 & 0.48\\
\bottomrule
\end{tabular}
\caption{Minimum, average, and maximum absolute errors between some commonly used stationary covariance functions on $[0, 1] \times [0, 1]$ (with unit variance and input scale 0.5) and their uniform \emph{string GP} counterparts with $K>1$ strings of equal length.}
\label{table:string_effects}
\end{table*}
\end{landscape}

\section{Inference under String and Membrane GP Priors}
\label{sct:infer}
In this section we move on to developing inference techniques for Bayesian nonparametric inference of latent functions under \emph{string GP} priors. We begin with marginal likelihood inference in regression problems. We then propose a novel reversible-jump MCMC sampler that enables automatic learning of model complexity (that is the number of different unconditional kernel configurations) from the data, with a time complexity and memory requirement both linear in the number of training inputs.

\subsection{Maximum Marginal Likelihood for Small Scale Regression Problems}
\label{sct:ml}
Firstly, we leverage the fact that additively separable \emph{string GPs} are Gaussian processes to perform Bayesian nonparametric regressions in the presence of local patterns in the data, using standard Gaussian process techniques \citep[see][p.112 \S 5.4.1]{rasswill}. We use as generative model 
\[~~y_i = f(x_i) + \epsilon_i,~~~ \epsilon_i \overset{\text{}}{\sim} \mathcal{N}\left(0, \sigma^2_{k_i}\right), ~~~\sigma^2_{k_i} >0,~~~x_i \in I^1 \times \dots \times I^d, ~~~y_i, \epsilon_i \in \mathbb{R}\] we are given the training data set $\mathcal{D}=\{\tilde{x}_i, \tilde{y}_i\}_{i \in [1..N]},$ and we place a mean-zero additively separable \emph{string GP} prior on $f$, namely
\begin{equation}
f(x)= \sum_{j=1}^d z^j_{x[j]},~~~(z_t^j) \sim \mathcal{SGP}\left(\{a_k^j\}, \{0\}, \{k_k^j\}\right), ~~~ \forall j<l, ~ (z_t^j) \perp (z_t^l) \nonumber,
\end{equation}
which we assume to be independent of the measurement noise process. Moreover, the noise terms are assumed to be independent, and the noise variance $\sigma^2_{k_i}$ affecting $f(x_i)$ is assumed to be the same for any two inputs whose coordinates lie on the same string intervals. Such a heteroskedastic noise model fits nicely within the \emph{string GP paradigm}, can be very useful when the dimension of the input space is small, and may be replaced by the typical constant noise variance assumption in high-dimensional input spaces.

Let us define $\textbf{y}=(\tilde{y}_1, \dots, \tilde{y}_N)$, $\textbf{X}=(\tilde{x}_1, \dots, \tilde{x}_N)$, $\textbf{f}=(f(\tilde{x}_1), \dots, f(\tilde{x}_N))$ and let $\bar{\textbf{K}}_{\textbf{X};\mathbf{X}}$ denote the auto-covariance matrix of $\textbf{f}$ (which we have derived in Section \ref{sct:means_kerns}), and let $\textbf{D} = \text{diag}(\{\sigma^2_{k_i}\})$ denote the diagonal matrix of noise variances. It follows that $\textbf{y}$ is a Gaussian vector with mean $\bm{0}$ and auto-covariance matrix  $\textbf{K}_{\textbf{y}}:=\bar{\textbf{K}}_{\textbf{X};\textbf{X}}+\textbf{D}$ and that the log marginal likelihood reads:
\begin{align}
\label{eq:gp_margin}
\log p\left( \textbf{y} \Big| \textbf{X}, \{\sigma_{k_i}\}, \{\theta_k^j\}, \{a_k^j\}\right) = -\frac{1}{2} \textbf{y}^{T}\textbf{K}_{\textbf{y}}^{-1}\textbf{y} -\frac{1}{2} \log{\text{det}\left(\textbf{K}_{\textbf{y}}\right)} - \frac{n}{2} \log{2\pi}.
\end{align}
We obtain estimates of the string measurement noise standard deviations $\{\hat{\sigma}_{k_i}\}$ and estimates of the string hyper-parameters $\{\hat{\theta}_k^j\}$ by maximising the marginal likelihood for a given domain partition $\{a_k^j\}$, using gradient-based methods. We deduce the predictive mean and covariance matrix of the latent function values $\textbf{f}^*$ at test points $\mathbf{X}^*$, from the estimates $\{\hat{\theta}_k^j\}, \{\hat{\sigma}_{k_i}\}$ as
\begin{equation}
\text{E}(\textbf{f}^*|\textbf{y}) = \bar{\textbf{K}}_{\textbf{X}^*;\textbf{X}}\textbf{K}_{\textbf{y}}^{-1}\textbf{y}~~~~\text{and}~~~~\text{cov}(\textbf{f}^*|\textbf{y}) = \bar{\textbf{K}}_{\textbf{X}^*;\textbf{X}^*} - \bar{\textbf{K}}_{\textbf{X}^*;\textbf{X}}\textbf{K}_{\textbf{y}}^{-1}\bar{\textbf{K}}_{\textbf{X};\textbf{X}^*},
\end{equation}
using the fact that $(\textbf{f}^*, \textbf{y})$ is jointly Gaussian, and that the cross-covariance matrix between $\textbf{f}^*$ and $\textbf{y}$ is $\bar{\textbf{K}}_{\textbf{X}^*;\textbf{X}}$ as the additive measurement noise is assumed to be independent from the latent process $f$.\\

\subsubsection{Remarks} 
The above analysis and equations still hold when a GP prior is placed on $f$ with one of the multivariate \emph{string GP} kernels derived in Section \ref{sct:means_kerns} as covariance function. 

It is also worth noting from the derivation of \emph{string GP} kernels in \ref{app:global_mean_cov} that the marginal likelihood Equation (\ref{eq:gp_margin}) is continuously differentiable in the locations of boundary times. Thus, for a given number of boundary times, the positions of the boundary times can be determined as part of the marginal likelihood maximisation. The derivatives of the marginal log-likelihood (Equation \ref{eq:gp_margin}) with respect to the aforementioned locations $\{a_k^j\}$ can be determined from the recursions of \ref{app:global_mean_cov}, or approximated numerically by finite differences. The number of boundary times in each input dimension can then be learned by trading off model fit (the maximum marginal log likelihood) and model simplicity (the number of boundary times or model parameters), for instance using information criteria such as AIC and BIC. When the input dimension is large, it might be advantageous to further constrain the hypothesis space of boundary times before using information criteria, for instance by assuming that the number of boundary times is the same in each dimension. An alternative Bayesian nonparametric approach to learning the number of boundary times will be discussed in section  \ref{sct:flashback}.

This method of inference cannot exploit the structure of \emph{string GPs} to speed-up inference, and as a result it scales like the \emph{standard GP paradigm}. In fact, any attempt to marginalize out univariate derivative processes, including in the prior, will inevitably destroy the conditional independence structure. Another perspective to this observation is found by noting from the derivation of global \emph{string GP} covariance functions in \ref{app:global_mean_cov} that the conditional independence structure does not easily translate in a matrix structure that may be exploited to speed-up matrix inversion, and that marginalizing out terms relating to derivatives processes as in Equation (\ref{eq:gp_margin}) can only make things worse.
\subsection{Generic Reversible-Jump MCMC Sampler for Large Scale Inference}
\label{sct:gen_rjmcmc}
More generally, we consider learning a smooth real-valued latent function $f$, defined on a $d$-dimensional hyper-rectangle, under a generative model with likelihood $p\left(\mathcal{D} \vert \mathbf{f}, \mathbf{u} \right)$, where $\mathbf{f}$ denotes values of $f$ at training inputs points and $\mathbf{u}$ denotes other likelihood parameters that are not related to $f$. A large class of machine learning problems aiming at inferring a latent function have a likelihood model of this form. Examples include celebrated applications such as nonparametric regression and nonparametric binary classification problems, but also more recent applications such as learning a profitable portfolio generating-function in \emph{stochastic portfolio theory} (\cite{karatzas2009stochastic}) from the data. In particular, we do not assume that $p\left(\mathcal{D} \vert \mathbf{f}, \mathbf{u} \right)$ factorizes over training inputs. Extensions to likelihood models that depend on the values of multiple latent functions are straight-forward and will be discussed in Section \ref{sct:multi}.

\subsubsection{Prior Specification}
We place a prior $p(\mathbf{u})$ on other likelihood parameters. For instance, in regression problems under a Gaussian noise model, $\mathbf{u}$ can be the noise variance and we may choose $p(\mathbf{u})$ to be the inverse-Gamma distribution for conjugacy. We place a mean-zero \emph{string GP} prior on $f$
\begin{equation}
\label{eq:recov}
f(x)= \phi\left(z^1_{x[1]}, \dots, z^d_{x[d]}\right),~~~~(z_t^j) \sim \mathcal{SGP}\left(\{a_k^j\}, \{0\}, \{k_k^j\}\right), ~~~~ \forall j<l, ~ (z_t^j) \perp (z_t^l).
\end{equation}
As discussed in Section \ref{sct:need_not}, the link function $\phi$ need not be inferred as the symmetric sum was found to yield a sufficiently flexible functional prior. Nonetheless, in  this section we do not impose any restriction on the link function $\phi$ other than continuous differentiability. Denoting $\mathbf{z}$ the vector of univariate \emph{string GP} processes and their derivatives, evaluated at all distinct input coordinate values, we may re-parametrize the likelihood as $p\left(\mathcal{D} \vert \mathbf{z}, \mathbf{u} \right)$, with the understanding that $\mathbf{f}$ can be recovered from $\mathbf{z}$ through the link function $\phi$. To complete our prior specification, we need to discuss the choice of boundary times $\{a_k^j\}$ and the choice of the corresponding unconditional kernel structures $\{k_k^j\}$. Before doing so, we would like to stress that key requirements of our sampler are that i) it should decouple the need for scalability from the need for flexibility, ii) it should scale linearly with the number of training and test inputs, and iii) the user should be able to express prior views on model complexity/flexibility in an intuitive way, but the sampler should be able to validate or invalidate the prior model complexity from the data. While the motivations for the last two requirements are obvious, the first requirement is motivated by the fact that a massive data set may well be more homogeneous than a much smaller data set.

\subsubsection{Scalable Choice of Boundary Times}
To motivate our choice of boundary times that achieves great scalability, we first note that the evaluation of the likelihood, which will naturally be needed by the MCMC sampler, will typically have at least linear time complexity and linear memory requirement, as it will require performing computations that use each training sample at least once. Thus, the best we can hope to achieve overall is linear time complexity and linear memory requirement. Second, in MCMC schemes with functional priors, the time complexity and memory requirements for sampling from the posterior $$p\left(\mathbf{f} \vert \mathcal{D} \right) \propto p\left( \mathcal{D} \vert \mathbf{f} \right) p(\mathbf{f})$$ are often the same as the resource requirements for sampling from the prior $p\left(\mathbf{f} \right)$, as evaluating the model likelihood is rarely the bottleneck. Finally, we note from Algorithm \ref{alg:simulation} that, when each input coordinate in each dimension is a boundary time, the sampling scheme has time complexity and memory requirement that are linear in the maximum number of unique input coordinates across dimensions, which is at most the number of training samples. In effect, each univariate derivative \emph{string GP} is sampled in \emph{parallel} at as many times as there are unique input coordinates in that dimension, before being combined through the link function. In a given input dimension, univariate \emph{derivative string GP} values are sampled sequentially, one boundary time conditional on the previous. The foregoing sampling operation is very scalable not only asymptotically but also in absolute terms; it merely requires storing and inverting at most as many $2 \times 2$ matrices as the number of input points. We will evaluate the actual overall time complexity and memory requirement when we discuss our MCMC sampler in greater details. For now, we would like to stress that i) choosing each distinct input coordinate value as a boundary time in the corresponding input dimension before training is a perfectly valid choice, ii) we expect this choice to result in resource requirements that grow linearly with the sample size and iii) in the \emph{string GP} theory we have developed thus far there is no requirement that two adjacent strings be driven by different kernel hyper-parameters.

\subsubsection{Model Complexity Learning as a Change-Point Problem}
\label{sct:km}
The remark iii) above pertains to model complexity. In the simplest case, all strings are driven by the same kernel and hyper-parameters as it was the case in Section \ref{sct:approx}, where we discussed how this setup departs from postulating the unconditional string covariance function $k_k^j$ globally similarly to the \emph{standard GP paradigm}. The more distinct unconditional covariance structures there are, the more complex the model is, as it may account for more types of local patterns. Thus, we may identify model complexity to the number of different kernel configurations across input dimensions. In order to learn model complexity, we require that some (but not necessarily all) strings share their kernel configuration.\footnote{That is, the functional form of the unconditional kernel $k_k^j$ and its hyper-parameters.} Moreover, we require kernel membership to be dimension-specific in that two strings in different input dimensions may not explicitly share a kernel configuration in the prior specification, although the posterior distribution over their hyper-parameters might be similar if the data support it. 

In each input dimension $j$, kernel membership is defined by a partition of the corresponding domain operated by a (possibly empty) set of change-points,\footnote{We would like to stress that  change-points do not introduce new input points or boundary times, but solely define a partition of the domain of each input dimension.} as illustrated in Figure \ref{fig:kernel_membership}. When there is no change-point as in Figure \ref{fig:kernel_membership}-(a), all strings are driven by the same kernel and hyper-parameters. Each change-point $c_p^j$ induces a new kernel configuration $\theta_p^j$ that is shared by all strings whose boundary times $a_k^j$ and $a_{k+1}^j$ both lie in $[c_p^j, c_{p+1}^j[$. When one or multiple change-points $c_p^j$ occur between two adjacent boundary times as illustrated in Figures \ref{fig:kernel_membership}-(b-d), for instance $a_k^j \leq c_p^j \leq a_{k+1}^j$, the kernel configuration of the string defined on $[a_k^j,  a_{k+1}^j]$ is that of the largest change-point that lies in $[a_k^j,  a_{k+1}^j]$ (see for instance Figure \ref{fig:kernel_membership}-(d)). For consistency, we denote $\theta_0^j$ the kernel configuration driving the first string in the $j$-th dimension; it also drives strings that come before the first change-point, and all strings when there is no change-point.

\begin{figure}[p]
\begin{center}
\centerline{\includegraphics[width=0.75\textwidth]{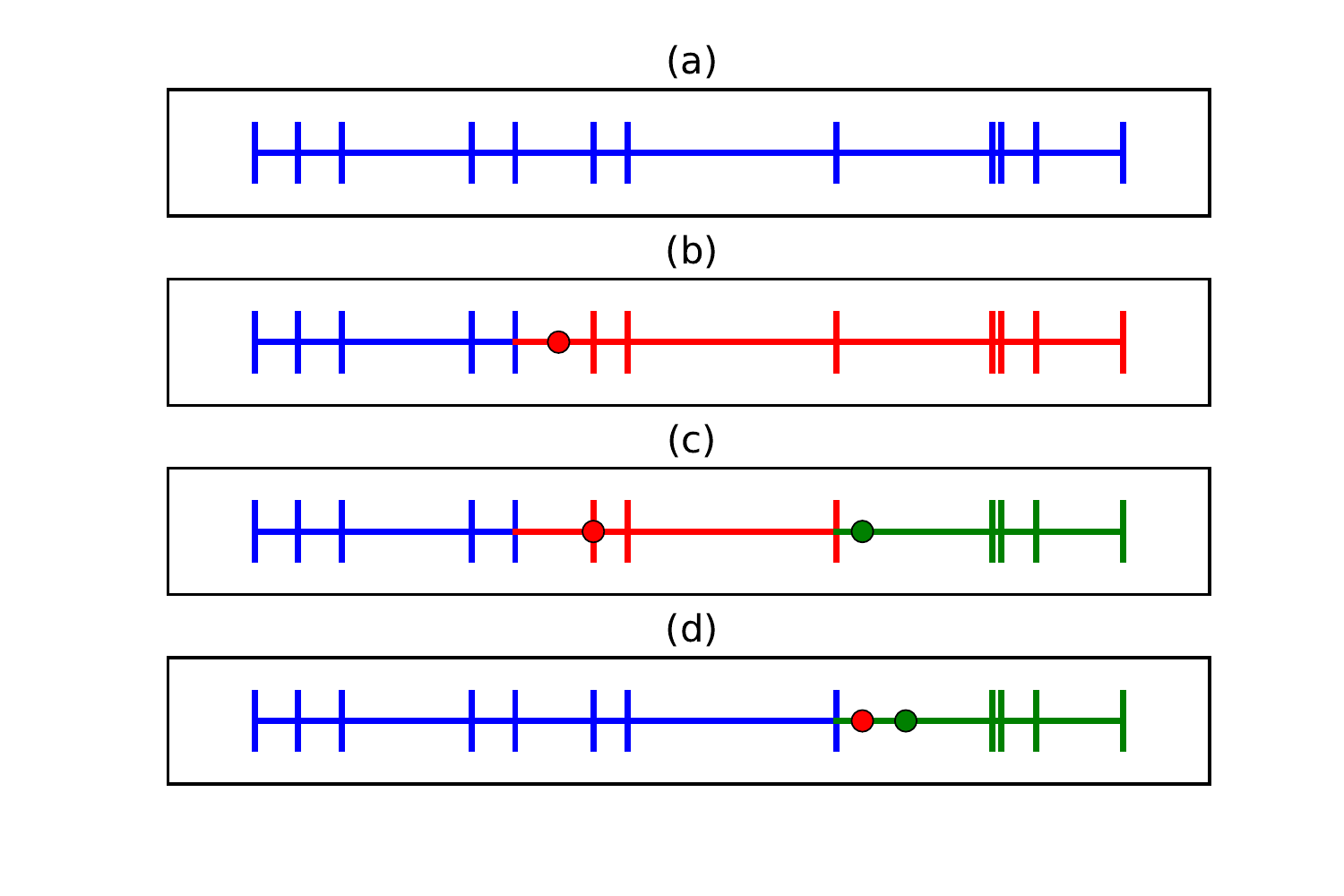}}
\caption{Effects of domain partition through change-points (coloured circles), on kernel membership. Each vertical bar corresponds to a distinct boundary time $a_k^j$. For the same collection of boundary times, we consider four scenarios: (a) no partition, (b) partition of the domain in two by a single  change-point that does not coincide with any existing boundary time, (c) partition of the domain in three by two change-points, one of which coincides with an existing boundary time, and (d) partition of the domain in two by two distinct change-points. In each scenario, kernel membership is illustrated by colour-coding. The colour of the interval between two consecutive boundary times $a_k^j$ and $a_{k+1}^j$ reflects what kernel configuration drives the corresponding string; in particular, the colour of the vertical bar corresponding to boundary time $a_{k+1}^j$ determines what kernel configuration should be used to compute the conditional distribution of the value of the \emph{derivative string GP} $(z_t^j, z_t^{j\prime})$ at $a_{k+1}^j$, given its value at $a_k^j$.}
\label{fig:kernel_membership}
\end{center}
\end{figure}

To place a prior on model complexity, it suffices to define a joint probability measure on the set of change-points and the corresponding kernel configurations. As kernel configurations are not shared across input dimensions, we choose these priors to be independent across input dimensions. Moreover, $\{c_p^j\}$ being a random collection of points on an interval whose number and positions are both random, it is \emph{de facto} a point process (\cite{Daley08}). To keep the prior specification of change-points uninformative, it is desirable that conditional on the number of change-points, the positions of change-points be i.i.d. uniform on the domain. As for the number of change-points, it is important that the support of its distribution not be bounded, so as to allow for an arbitrarily large model complexity if warranted. The two requirements above are satisfied by a homogeneous Poisson process or HPP (\cite{Daley08}) with constant intensity $\lambda^j$. More precisely, the prior probability measure on $\left(\{c_p^j, \theta_p^j\}, \lambda^j  \right)$ is constructed as follows:
\begin{equation}
\label{eq:comp_prior}
\begin{cases}
\lambda^j \sim \Gamma(\alpha^j, \beta^j),\\
\{c_p^j\} \big\vert \lambda^j  \sim \text{HPP}(\lambda^j)\\
 \theta_p^j[i] \big\vert  \{c_p^j\}, \lambda^j  ~~\overset{\text{i.i.d}}{\sim} ~~\log \mathcal{N}(0, \rho^j)\\
\forall (j, p) \neq (l, q) ~ \theta_p^j \perp  \theta_q^l,
\end{cases},
\end{equation}
where we choose the Gamma distribution $\Gamma$ as prior on the intensity $\lambda^j $ for conjugacy, we assume all kernel hyper-parameters are positive as is often the case in practice,\footnote{This may easily be relaxed if needed, for instance by putting normal priors on parameters that may be negative and log-normal priors on positive parameters.} the coordinates of the hyper-parameters of a kernel configuration are assumed i.i.d., and kernel hyper-parameters are assumed independent between kernel configurations. Denoting the domain of the $j$-th input $[a^j, b^j]$, it follows from applying the laws of total expectation and total variance on Equation (\ref{eq:comp_prior}) that the expected number of change-points in the $j$-th dimension under our prior is
\begin{equation}
\text{E}\left( \# \{c_p^j\}  \right) = \left(b^j-a^j\right) \frac{\alpha^j}{\beta^j},
\end{equation}
and the variance of the number of change-points in the $j$-dimension under our prior is
\begin{equation}
\text{Var}\left( \# \{c_p^j\}  \right) = \left(b^j-a^j\right) \frac{\alpha^j}{\beta^{j}} \left(1+\frac{\left(b^j-a^j\right) }{\beta^j} \right).
\end{equation}
The two equations above may guide the user when setting the parameters $\alpha^j$ and $\beta^j$. For instance, these values may be set so that the expected number of change-points in a given input dimension be a fixed fraction of the number of boundary times in that input dimension, and so that the prior variance over the number of change-points be large enough that overall the prior isn't too informative.

We could have taken a different approach to construct our prior on change-points. In effect, assuming for the sake of the argument that the boundaries of the domain of the $j$-th input, namely $a^j$ and $b^j$, are the first and last change-point in that input dimension, we note that the mapping $$\left(\dots, c_p^j, \dots \right) \to  \left(\dots, p_p^j, \dots\right) := \left(\dots, \frac{c_{p+1}^j-c_p^j}{b^j-a^j},\dots \right)$$ defines a bijection between the set of possible change-points in the $j$-th dimension and the set of all discrete probability distributions. Thus, we could have placed as prior on  $\left(\dots, p_p^j, \dots\right)$ a Dirichlet process (\cite{ferguson1973bayesian}), a Pitman-Yor process (\cite{pitman1997two}), more generally \emph{normalized completely random measures} (\cite{kingman1967completely}) or any other probability distribution over partitions. We prefer the point process approach primarily because it provides an easier way of expressing prior belief about model complexity through the expected number of change-points $\# \{c_p^j\}$, while remaining uninformative about positions thereof. 

One might also be tempted to regard change-points in an input dimension $j$ as inducing a partition, not of the domain $[a^j, b^j]$, but of the set of boundary times $a_k^j$ in the same dimension, so that one may define a prior over kernel memberships through a prior over partitions of the set of boundary times. However, this approach would be inconsistent with the aim to learn local patterns in the data if the corresponding random measure is \emph{exchangeable}. In effect, as boundary times are all input coordinates, local patterns may only arise in the data as a result of adjacent strings sharing kernel configurations. An exchangeable random measure would postulate a priori that two kernel membership assignments that have the same kernel configurations (i.e. the same number of configurations and the same set of hyper-parameters) and the same \emph{number} of boundary times in each kernel cluster (although not exactly the same boundary times), are equally likely to occur, thereby possibly putting more probability mass on kernel membership assignments that do not respect boundary time adjacency. Unfortunately,  \emph{exchangeable} random measures (among which the Dirichlet process and the Pitman-Yor process) are by far more widely adopted by the machine learning community than non-exchangeable random measures. Thus, this approach might be perceived as overly complex. That being said, as noted by \cite{foti2015survey}, non-exchangeable normalized random measures may be regarded as Poisson point processes (with varying intensity functions) on some augmented spaces, which makes this choice of prior specification somewhat similar, but stronger (that is more informative) than the one we adopt in this paper.

Before deriving the sampling algorithm, it is worth noting that the prior defined in Equation (\ref{eq:comp_prior}) does not admit a density with respect to the same base measure,\footnote{That is the joint prior probability measure is neither discrete, nor continuous.} as the number of change-points $\#\{c_p^j\}$, and subsequently the number of kernel configurations, may vary from one sample to another. Nevertheless, the joint distribution over the data $\mathcal{D}$ and all other model parameters is well defined and, as we will see later, we may leverage reversible-jump MCMC techniques (\cite{green1995reversible,green09}) to construct a Markov chain that converges to the posterior distribution.

\subsubsection{Overall Structure of the MCMC Sampler} 
To ease notations, we denote $\bm{c}$ the set of all change-points in all input dimensions, we denote $\bm{n}= \left(\dots, \#\{c_p^j \}, \dots \right) \in \mathbb{N}^d$ the vector of the numbers of change-points in each input dimension, we denote $\bm{\theta}$ the set of kernel hyper-parameters,\footnote{To simplify the exposition, we assume without loss of generality that each kernel configuration has the same kernel functional form, so that configurations are defined by kernel hyper-parameters.} and $\bm{\rho} := (\dots, \rho^j, \dots)$ the vector of variances of the independent log-normal priors  on $\bm{\theta}$. We denote $\bm{\lambda} := (\dots, \lambda^j, \dots)$ the vector of change-points intensities, we denote $\bm{\alpha} := (\dots, \alpha^j, \dots)$ and $\bm{\beta} := (\dots, \beta^j, \dots)$ the vectors of parameters of the Gamma priors we placed on the change-points intensities across the $d$ input dimensions, and we recall that $\mathbf{u}$ denotes the vector of likelihood parameters other than the values of the latent function $f$. 

We would like to sample from the posterior distribution $p(\mathbf{f}, \mathbf{f}^*, \nabla \mathbf{f}, \nabla \mathbf{f}^* | \mathcal{D}, \bm{\alpha}, \bm{\beta}, \bm{\rho})$, where $\mathbf{f}$ and $\mathbf{f}^*$ are the vectors of values of the latent function $f$ at training and test inputs respectively, and $\nabla \mathbf{f}, \nabla \mathbf{f}^*$ the corresponding gradients. Denoting $\mathbf{z}$ the vector of univariate \emph{string GP} processes and their derivatives, evaluated at all distinct training and test input coordinate values, we note that to sample from $p(\mathbf{f}, \mathbf{f}^*,  \nabla \mathbf{f}, \nabla \mathbf{f}^* | \mathcal{D}, \bm{\alpha}, \bm{\beta}, \bm{\rho})$, it suffices to sample from $p(\mathbf{z} | \mathcal{D}, \bm{\alpha}, \bm{\beta}, \bm{\rho})$, compute $\mathbf{f}$ and $\mathbf{f}^*$ using the link function, and compute the gradients using Equation (\ref{eq:sgp_gradient}). To sample from $p(\mathbf{z} | \mathcal{D}, \bm{\alpha}, \bm{\beta}, \bm{\rho})$, we may sample from the target distribution
\begin{equation}
\pi(\bm{n}, \bm{c}, \bm{\theta}, \bm{\lambda}, \mathbf{z}, \mathbf{u}) := p(\bm{n}, \bm{c}, \bm{\theta}, \bm{\lambda}, \mathbf{z}, \mathbf{u} \vert \mathcal{D}, \bm{\alpha}, \bm{\beta}, \bm{\rho}),
\end{equation} 
and discard variables that are not of interest. As previously discussed, $\pi$ is not absolutely continuous with respect to the same base measure, though we may still decompose it as 
\begin{equation}
\label{eq:decomp}
\pi(\bm{n}, \bm{c}, \bm{\theta}, \bm{\lambda}, \mathbf{z}, \mathbf{u}) = \color{green} \frac{1}{p(\mathcal{D} \vert \bm{\alpha}, \bm{\beta}, \bm{\rho})} \color{blue} p(\bm{n} \vert \bm{\lambda}) p(\bm{\lambda} \vert \bm{\alpha}, \bm{\beta}) \color{red} p(\bm{c}\vert \bm{n}) p(\bm{\theta} \vert \bm{n}, \bm{\rho}) p(\mathbf{u}) p(\mathbf{z} \vert \mathbf{c}, \bm{\theta})\color{black} p(\mathcal{D} \vert  \mathbf{z}, \mathbf{u}),
\end{equation}
where we use the notation $p(.)$ and $p(. \vert .)$ to denote probability measures rather than probability density functions or probability mass functions, and where product and scaling operations are usual measure operations. Before proceeding any further, we will introduce a slight re-parametrization of Equation (\ref{eq:decomp}) that will improve the inference scheme.

Let $\mathbf{n}_a = (\dots, \#\{a_k^j\}_k, \dots)$ be the vector of the numbers of unique boundary times in all $d$ input dimensions. We recall from our prior on $f$ that
\begin{equation}
\label{eq:comp_z}
p(\mathbf{z} \vert \mathbf{c}, \bm{\theta}) = \prod_{j=1}^d p\left(z_{a_0^j}^j, z_{a_0^j}^{j\prime}\right) \prod_{k=1}^{\bm{n}_a[j]-1} p\left(z_{a_k^j}^j, z_{a_k^j}^{j\prime} \Big\vert z_{a_{k-1}^j}^j, z_{a_{k-1}^j}^{j\prime} \right),
\end{equation}
where each factor in the decomposition above is a bivariate Gaussian density whose mean vector and covariance matrix is obtained from the partitions $\mathbf{c}$, the kernel hyper-parameters $\bm{\theta}$, and the kernel membership scheme described in Section \ref{sct:km} and illustrated in Figure \ref{fig:kernel_membership}, and using Equations (\ref{eq:sb0k}-\ref{eq:mb0k}). Let ${}_k^j\textbf{K}_{u; v}$ be the unconditional covariance matrix between $\left(z_{u}^j, z_{u}^{j\prime}\right)$ and $\left(z_{v}^j, z_{v}^{j\prime}\right)$ as per the unconditional kernel structure driving the string defined on the interval $[a_k^j, a_{k+1}^j[$. Let $\Sigma_0^j := {}_0^j\textbf{K}_{a_0^j; a_0^j}$ be the auto-covariance matrix of $\left(z_{a_0^j}^j, z_{a_0^j}^{j\prime}\right)$. Let  $$\Sigma_k^j := {}_k^j\textbf{K}_{a_k^j; a_k^j} - {}_k^j\textbf{K}_{a_k^j; a_{k-1}^j} {}_k^j\textbf{K}_{a_{k-1}^j; a_{k-1}^j}^{-1} {}_k^j\textbf{K}_{a_k^j; a_{k-1}^j}^{T}$$ be the covariance matrix of $\left(z_{a_k^j}^j, z_{a_k^j}^{j\prime}\right)$ given  $\left(z_{a_{k-1}^j}^j, z_{a_{k-1}^j}^{j\prime}\right)$, and $$M_k^j = {}_k^j\textbf{K}_{a_k^j; a_{k-1}^j} {}_k^j\textbf{K}_{a_{k-1}^j; a_{k-1}^j}^{-1}.$$ Finally, let  $L_k^j := U_k^j (D_k^j)^{\frac{1}{2}}$ with $\Sigma_k^j =  U_k^j D_k^j (U_k^j)^T$ the singular value decomposition (SVD) of $\Sigma_k^j$. We may choose to represent $\left(z_{a_0^j}^j, z_{a_0^j}^{j\prime}\right)$ as 
\begin{equation}
\label{eq:xrep0}
\begin{bmatrix}
z_{a_0^j}^j \\
 z_{a_0^j}^{j\prime}
\end{bmatrix} =  L_0^j x_0^j,
\end{equation}
and for $k>0$ we may also choose to represent $\left(z_{a_k^j}^j, z_{a_k^j}^{j\prime}\right)$ as 
\begin{equation}
\label{eq:xrep1}
\begin{bmatrix}
z_{a_k^j}^j \\
 z_{a_k^j}^{j\prime}
\end{bmatrix} = M_k^j
\begin{bmatrix}
z_{a_{k-1}^j}^j \\
 z_{a_{k-1}^j}^{j\prime}
\end{bmatrix} + L_k^j x_k^j,
\end{equation}
where $\{x_k^j\}$ are independent bivariate standard normal vectors. Equations (\ref{eq:xrep0}-\ref{eq:xrep1}) provide an equivalent representation. In effect, we recall that if $Z=M + LX,$ where $X \sim \mathcal{N}(0, I)$ is a standard multivariate Gaussian, $M$ is a real vector, and $L$ is a real matrix, then $Z \sim \mathcal{N}(M, LL^T)$. Equations (\ref{eq:xrep0}-\ref{eq:xrep1}) result from applying this result to $\left(z_{a_0^j}^j, z_{a_0^j}^{j\prime}\right)$ and $\left(z_{a_k^j}^j, z_{a_k^j}^{j\prime}\right) \Big\vert \left(z_{a_{k-1}^j}^j, z_{a_{k-1}^j}^{j\prime}\right)$. We note that at training time, $M_k^j$ and $L_k^j$ only depend on kernel hyper-parameters. Denoting $\mathbf{x}$ the vector of all $x_k^j$,  $\mathbf{x}$ is a so-called `whitened' representation of $\mathbf{z}$, which we prefer for reasons we will discuss shortly. In the whitened representation, the target distribution $\pi$ is re-parameterized as 
\begin{equation}
\label{eq:decomp2}
\pi(\bm{n}, \bm{c}, \bm{\theta}, \bm{\lambda}, \mathbf{x}, \mathbf{u}) = \color{green} \frac{1}{p(\mathcal{D} \vert \bm{\alpha}, \bm{\beta}, \bm{\rho})} \color{blue} p(\bm{n} \vert \bm{\lambda}) p(\bm{\lambda} \vert \bm{\alpha}, \bm{\beta}) \color{red} p(\bm{c}\vert \bm{n}) p(\bm{\theta} \vert \bm{n}, \bm{\rho}) p(\mathbf{u}) p(\mathbf{x})\color{black} p(\mathcal{D} \vert  \mathbf{x}, \mathbf{c}, \bm{\theta}, \mathbf{u}),
\end{equation}
where the dependency of the likelihood term on the partitions and the hyper-parameters stems from the need to recover $\mathbf{z}$ and subsequently $\mathbf{f}$ from $\mathbf{x}$ through Equations (\ref{eq:xrep0}) and (\ref{eq:xrep1}). The whitened representation Equation (\ref{eq:decomp2}) has two primary advantages. Firstly, it is robust to ill-conditioning of $\Sigma_k^j$, which would typically occur when two adjacent boundary times are too close to each other. In the representation of Equation (\ref{eq:decomp}), as one needs to evaluate the density $p(\mathbf{z} \vert \mathbf{c}, \bm{\theta})$, ill-conditioning of $\Sigma_k^j$ would result in numerical instabilities. In contrast, in the whitened representation, one needs to evaluate the density $p(\mathbf{x})$, which is that of i.i.d. standard Gaussians and as such can be evaluated robustly. Moreover, the SVD required to evaluate $L_k^j$ is also robust to ill-conditioning of $\Sigma_k^j$, so that Equations (\ref{eq:xrep0}) and  (\ref{eq:xrep1}) hold and can be robustly evaluated for degenerate Gaussians too. The second advantage of the whitened representation is that it improves mixing by establishing a link between kernel hyper-parameters and the likelihood.

Equation (\ref{eq:decomp2}) allows us to cast our inference problem as a Bayesian model selection problem under a countable family of models indexed by $\bm{n} \in \mathbb{N}^d$, each defined on a different parameter subspace $\mathcal{C}_{\bm{n}}$, with cross-model normalizing constant $\color{green} p(\mathcal{D}\vert \bm{\alpha}, \bm{\beta}, \bm{\rho})$, model probability driven by $ \color{blue} p(\bm{n} \vert \bm{\lambda}) p(\bm{\lambda} \vert \bm{\alpha}, \bm{\beta})$, model-specific prior $\color{red} p(\bm{c}\vert \bm{n}) p(\bm{\theta} \vert \bm{n}, \bm{\rho}) p(\mathbf{u})p(\mathbf{x})$, and likelihood $p(\mathcal{D} \vert  \mathbf{x}, \mathbf{c}, \bm{\theta}, \mathbf{u})$. Critically, it can be seen from Equation (\ref{eq:decomp2}) that the conditional probability distribution $$\pi( \bm{c}, \bm{\theta}, \bm{\lambda}, \mathbf{x}, \mathbf{u} \vert \bm{n})$$ admits a density with respect to Lebesgue's measure on $\mathcal{C}_{\bm{n}}$.

Our setup is therefore analogous to that which motivated the seminal paper \cite{green1995reversible}, so that to sample from the posterior $\pi( \bm{c}, \bm{\theta}, \bm{\lambda}, \mathbf{x}, \mathbf{u}, \bm{n})$ we may use any Reversible-Jump Metropolis-Hastings (RJ-MH) scheme satisfying detailed balance and dimension-matching as described in section 3.3 of \cite{green1995reversible}. To improve mixing of the Markov chain, we will alternate between a \emph{between-models} RJ-MH update with target distribution $\pi(\bm{n}, \bm{c}, \bm{\theta}, \bm{\lambda}, \mathbf{x}, \mathbf{u})$, and a \emph{within-model} MCMC-within-Gibbs sampler with target distribution $\pi( \bm{c}, \bm{\theta}, \bm{\lambda}, \mathbf{x}, \mathbf{u} \vert \bm{n})$. Constructing reversible-jump samplers by alternating between within-model sampling and between-models sampling is standard practice, and it is well-known that doing so yields a Markov chain that converges to the target distribution of interest \citep[see][p. 50]{brooks2011handbook}. 

In a slight abuse of notation, in the following we might use the notations $p(. \vert .)$ and $p(.)$, which we previously used to denote probability measures, to refer to the corresponding probability density functions or probability mass functions.

\subsubsection{Within-Model Updates}
We recall from Equation (\ref{eq:decomp2}) that $\bm{c}, \bm{\theta}, \bm{\lambda}, \mathbf{x}, \mathbf{u} \vert \mathcal{D}, \bm{\alpha}, \bm{\beta}, \bm{\rho}, \bm{n}$ has probability density function
\begin{equation}
\label{eq:wtt}
\color{blue} p(\bm{n} \vert \bm{\lambda}) p(\bm{\lambda} \vert \bm{\alpha}, \bm{\beta}) \color{red} p(\bm{c}\vert \bm{n}) p(\bm{\theta} \vert \bm{n}, \bm{\rho}) p(\mathbf{u})p(\mathbf{x})\color{black} p(\mathcal{D} \vert  \mathbf{x},  \mathbf{c}, \bm{\theta}, \mathbf{u}),
\end{equation}
up to a normalizing constant.

\textbf{Updating $\bm{\lambda}$}: By independence of the priors over $(\bm{\lambda}[j], \bm{n}[j])$, the distributions $\bm{\lambda}[j] ~ \big\vert ~ \bm{n}[j]$ are also independent, so that the updates may be performed in \emph{parallel}. Moreover, recalling that the prior number of change-points in the $j$-th input dimension is Poisson distributed with intensity $\bm{\lambda}[j] \left(b^j - a^j\right)$, and by conjugacy of the Gamma distribution to the Poisson likelihood, it follows that 
\begin{equation}
\label{eq:lam_up}
\bm{\lambda}[j] ~ \big\vert ~ \bm{n}[j] \sim \Gamma\left( \frac{\bm{n}[j]}{b^j - a^j} + \bm{\alpha}[j], 1+\bm{\beta}[j]\right).
\end{equation}
This update step has memory requirement and time complexity both constant in the number of training and test samples.

\textbf{Updating $\mathbf{u}$}: When the likelihood has additional parameters $\mathbf{u}$, they may be updated with a Metropolis-Hastings step. Denoting $q(\mathbf{u} \to \mathbf{u}^\prime)$ the proposal probability density function, the acceptance ratio reads
\begin{equation}
\label{eq:u}
r_\mathbf{u} = \min ~ \left(1, \frac{p(\mathbf{u}^\prime) p\left(\mathcal{D} \vert  \mathbf{x}, \mathbf{c}, \bm{\theta}, \mathbf{u}^\prime\right)q(\mathbf{u}^\prime \to \mathbf{u})}{p(\mathbf{u}) p\left(\mathcal{D} \vert  \mathbf{x}, \mathbf{c}, \bm{\theta}, \mathbf{u}\right)q(\mathbf{u} \to \mathbf{u}^\prime)}\right).
\end{equation}
In some cases however, it might be possible and more convenient to choose $p(\mathbf{u})$ to be conjugate to the likelihood $p\left(\mathcal{D} \vert  \mathbf{x}, \mathbf{c}, \bm{\theta}, \mathbf{u}\right)$. For instance, in regression problems under a Gaussian noise model, we may take $\mathbf{u}$ to be the noise variance on which we may place an inverse-gamma prior. Either way, the computational bottleneck of this step is the evaluation of the likelihood $p\left(\mathcal{D} \vert  \mathbf{x},  \mathbf{c}, \bm{\theta}, \mathbf{u}^\prime\right)$, which in most cases can be done with a time complexity and memory requirement that are both linear in the number of training samples.

\textbf{Updating $\mathbf{c}$}: We update the positions of change-points sequentially using the Metropolis-Hastings algorithm, one input dimension $j$ at a time, and for each input dimension we proceed in increasing order of change-points. The proposal new position for the change-point $c^j_p$ is sampled uniformly at random on the interval $]c^j_{p-1}, c^j_{p+1}[$, where $c^j_{p-1}$ (resp. $c^j_{p+1}$) is replaced by $a^j$ (resp. $b^j$) for the first (resp. last) change-point. The acceptance probability of this proposal is easily found to be
\begin{equation}
\label{eq:update_cjp}
r_{c_p^j} = \min~ \left(1, \frac{ p\left(\mathcal{D} \vert  \mathbf{x}, \mathbf{c}^\prime, \bm{\theta}, \mathbf{u}\right)}{ p\left(\mathcal{D} \vert  \mathbf{x}, \mathbf{c}, \bm{\theta}, \mathbf{u}\right)} \right),
\end{equation}
where $\mathbf{c}^\prime$ is identical to $\mathbf{c}$ except for the change-point to update. This step requires computing the factors $\{L_k^j, M_k^j\}$ corresponding to inputs in $j$-th dimension whose kernel configuration would change if the proposal were to be accepted, the corresponding vector of \emph{derivative string GP} values $\textbf{z}$, and the observation likelihood under the proposal $p\left(\mathcal{D} \vert  \mathbf{x}, \mathbf{c}^\prime, \bm{\theta}, \mathbf{u}\right)$. The computational bottleneck of this step is therefore once again the evaluation of the new likelihood $p\left(\mathcal{D} \vert  \mathbf{x}, \mathbf{c}^\prime, \bm{\theta}, \mathbf{u}\right)$.

\textbf{Updating $\mathbf{x}$}: The target conditional density of $\mathbf{x}$ is proportional to
\begin{equation}
\label{eq:w_rep}
p(\textbf{x})p(\mathcal{D} \vert \mathbf{x}, \bm{c}, \bm{\theta}, \mathbf{u}).
\end{equation}
Recalling that $p(\mathbf{x})$ is a multivariate standard normal, it follows that the form of Equation (\ref{eq:w_rep}) makes it convenient to use elliptical slice sampling (\cite{Murray09b}) to sample from the unormalized conditional $p(\textbf{x})p(\mathcal{D} \vert \mathbf{x}, \bm{c}, \bm{\theta}, \mathbf{u})$. The two bottlenecks of this update step are sampling a new proposal from $p(\mathbf{x})$ and evaluating the likelihood $p(\mathcal{D} \vert \mathbf{x}, \bm{c}, \bm{\theta}, \mathbf{u})$. Sampling from the multivariate standard normal $p(\mathbf{x})$ may be \emph{massively parallelized}, for instance by using GPU Gaussian random number generators. When no parallelism is available, the overall time complexity reads $\mathcal{O}\left(\sum_{j=1}^d \mathbf{n}_a[j]\right)$, where we recall that $\mathbf{n}_a[j]$ denotes the number of distinct training and testing input coordinates in the $j$-th dimension. In particular, if we denote $N$ the total number of training and testing $d$-dimensional input samples, then $\sum_{j=1}^d  \mathbf{n}_a[j] \leq dN$, although for many classes of data sets with sparse input values such as images, where each input (single-colour pixel value) may have at most $256$ distinct values, we might have $\sum_{j=1}^d \mathbf{n}_a[j] \ll dN$. As for the memory required to sample from $p(\mathbf{x})$, it grows proportionally to the size of $\mathbf{x}$, that is in $\mathcal{O}\left(\sum_{j=1}^d  \mathbf{n}_a[j]\right)$. In regards to the evaluation of the likelihood $p(\mathcal{D} \vert \mathbf{x}, \bm{c}, \bm{\theta}, \mathbf{u})$, as previously discussed its resource requirements are application-specific, but it will typically have time complexity that grows in $\mathcal{O}\left(N\right)$ and memory requirement that grows in $\mathcal{O}\left(dN\right)$. For instance, the foregoing resource requirements always hold for i.i.d. observation models such as in nonparametric regression and nonparametric classification problems.

\textbf{Updating $\bm{\theta}$}: We note from Equation (\ref{eq:wtt}) that the conditional distribution of $\bm{\theta}$ given everything else has unormalized density

\begin{equation}
p(\bm{\theta} \vert \bm{n}, \bm{\rho}) p(\mathcal{D} \vert \mathbf{x}, \bm{c}, \bm{\theta}, \mathbf{u}),
\end{equation}
which we may choose to represent as
\begin{equation}
\label{eq:up_log_theta}
p(\log \bm{\theta} \vert \bm{n}, \bm{\rho}) p(\mathcal{D} \vert \mathbf{x}, \bm{c}, \log \bm{\theta}, \mathbf{u}).
\end{equation}
As we have put independent log-normal priors on the coordinates of $\bm{\theta}$ (see Equation \ref{eq:comp_prior}), we may once again use elliptical slice sampling to sample from $\log \bm{\theta}$ before taking the exponential. The time complexity of  generating a new sample from $p(\log \bm{\theta} \vert \bm{n}, \bm{\rho})$ will typically be at most linear in the total number of distinct kernel hyper-parameters. Overall, the bottleneck of this update is the evaluation of the likelihood $p(\mathcal{D} \vert \mathbf{x}, \bm{c}, \log \bm{\theta}, \mathbf{u})$.  In this update, the latter operation requires recomputing the factors $M_k^j$ and $L_k^j$ of Equations (\ref{eq:xrep0}) and (\ref{eq:xrep1}), which requires computing and taking the SVD of unrelated $2\times 2$ matrices, computations we may perform in \emph{parallel}. Once the foregoing factors have been computed, we evaluate $\mathbf{z}$, the \emph{derivative string GP} values at boundary times, parallelizing over input dimensions, and running a sequential update within an input dimension using Equations (\ref{eq:xrep0}) and (\ref{eq:xrep1}). Updating $\mathbf{z}$ therefore has time complexity that is, in the worst case where no distributed computing is available, $\mathcal{O}\left(dN\right)$, and $\mathcal{O}\left(N\right)$ when there are up to $d$ computing cores. The foregoing time complexity will also be that of this update step, unless the observation likelihood is more expensive to evaluate. The memory requirement, as in previous updates, is $\mathcal{O}\left(dN\right)$.

\textbf{Overall resource requirement}: To summarize previous remarks, the overall computational bottleneck of a \emph{within-model} iteration is the evaluation of the likelihood $p(\mathcal{D} \vert \mathbf{x}, \bm{c}, \bm{\theta}, \mathbf{u})$. For i.i.d. observation models such as classification and regression problems for instance, the corresponding time complexity grows in $\mathcal{O}(N)$ when $d$ computing cores are available, or $\mathcal{O}(dN)$ otherwise, and the memory requirement grows in $\mathcal{O}(dN)$.

\subsubsection{Between-Models Updates}
\label{sct:btw_md}
Our reversible-jump Metropolis-Hastings update proceeds as follows. We choose an input dimension, say $j$, uniformly at random. If $j$ has no change-points, that is $\mathbf{n}[j] = 0$, we randomly choose between not doing anything, and adding a change-point, each outcome having the same probability. If $\mathbf{n}[j] > 0$, we either do nothing, add a change-point, or delete a change-point, each outcome having the same probability of occurrence.

Whenever we choose not to do anything, the acceptance ratio is easily found to be one:
\begin{equation}
 r_{j0} = 1.
\end{equation}

Whenever we choose to add a change-point, we sample the position $c^j_*$ of the proposal new change-point uniformly at random on the domain $[a^j, b^j]$ of the $j$-th input dimension. This proposal will almost surely break an existing kernel membership cluster, say the $p$-th, into two; that is $c^j_p < c^j_*< c_{p+1}^j$ where we may have $a^j=c^j_p$ and/or $b^j=c^j_{p+1}$. In the event $c^j_*$ coincides with an existing change-point, which should happen with probability $0$, we do nothing. When adding a change-point, we sample a new vector of hyper-parameters $\theta_*^j$ from the log-normal prior of Equation (\ref{eq:comp_prior}), and we propose as hyper-parameters for the tentative new clusters $[c^j_p, c^j_*[$ and $[c^j_*, c^j_{p+1}[$ the vectors $\theta_{\text{add-left}}^j$ and $\theta_{\text{add-right}}^j$ defined as
\begin{equation}
\label{eq:theta_add_left}
 \log \theta_{\text{add-left}}^j := \cos(\alpha) \log \theta_p^j - \sin(\alpha) \log \theta_*^j
\end{equation}
and
\begin{equation}
\label{eq:theta_add_right}
 \log \theta_{\text{add-right}}^j :=  \sin(\alpha) \log \theta_p^j + \cos(\alpha) \log \theta_*^j
\end{equation}
 respectively, where $\alpha \in [0, \frac{\pi}{2}]$ and $\theta_p^j$ is the vector of hyper-parameters currently driving the kernel membership defined by the cluster $[c^j_p, c_{p+1}^j[$. We note that if $\theta_p^j$ is distributed as per the prior  in Equation (\ref{eq:comp_prior}) then $\theta_{\text{add-left}}^j$ and $\theta_{\text{add-right}}^j$ are i.i.d. distributed as per the foregoing prior. More generally, this elliptical transformation determines the extent to which the new proposal kernel configurations should deviate from the current configuration $\theta_p^j$. $\alpha$ is restricted to $ [0, \frac{\pi}{2}]$ so as to give a positive weight the the current vector of hyper-parameters $\theta_p^j$. When $\alpha=0$, the left hand-side cluster $[c^j_p, c^j_*[$ will fully exploit the current kernel configuration, while the right hand-side cluster $[c^j_*, c^j_{p+1}[$ will use the prior to explore a new set of hyper-parameters. When $\alpha= \frac{\pi}{2}$ the reverse occurs. To preserve symmetry between the left and right hand-side kernel configurations, we choose 
 \begin{equation}
 \alpha = \frac{\pi}{4}.
 \end{equation}

Whenever we choose to delete a change-point, we choose an existing change-point uniformly at random, say $c^j_p$. Deleting $c^j_p$, would merge the clusters $[c^j_{p-1}, c^j_p[$ and $[c^j_p, c^j_{p+1}[$, where we may have $a^j=c^j_{p-1}$ and/or $b^j=c^j_{p+1}$. We propose as vector of hyper-parameters for the tentative merged cluster $[c^j_{p-1}, c^j_{p+1}[$ the vector $\theta_{\text{del-merged}}^j$ satisfying:
\begin{equation}
 \log \theta_{\text{del-merged}}^j = \cos(\alpha) \log \theta^j_{p-1} + \sin(\alpha) \log \theta^j_{p},
\end{equation}
which together with 
\begin{equation}
 \log \theta_{\text{del-*}}^j = -\sin(\alpha) \log \theta^j_{p-1} + \cos(\alpha) \log \theta^j_{p},
\end{equation}
constitute the inverse of the transformation defined by Equations (\ref{eq:theta_add_left}) and (\ref{eq:theta_add_right}).

Whenever a proposal to add or delete a change-point occurs, the factors $L_k^j$ and $M_k^j$ that would be affected by the change in kernel membership structure are recomputed, and so are the affected coordinates of $\mathbf{z}$.

This scheme satisfies the reversibility and dimension-matching requirements of \cite{green1995reversible}. Moreover, the absolute value of the Jacobian of the mapping $$\left( \log \theta_p^j, ~\log \theta_*^j \right) \to \left(\log \theta_{\text{add-left}}^j , ~\log \theta_{\text{add-right}}^j \right)$$ of the move to add a change-point in $[c^j_p, c^j_{p+1}[$ reads
\begin{equation}
\label{eq:jac_add}
\left\vert \frac{\partial \left(\log \theta_{\text{add-left}}^j , \log \theta_{\text{add-right}}^j \right)}{\partial \left(\log \theta_p^j, \log \theta_*^j \right)} \right\vert  = 1.
\end{equation}
Similarly, the absolute value of the Jacobian of the mapping corresponding to a move to delete change-point $c^j_p$, namely  $$\left(\log \theta^j_{p-1}, ~\log \theta^j_{p} \right) \to \left(\log \theta_{\text{del-merged}}^j , ~\log \theta_{\text{del-*}}^j \right),$$  reads:
\begin{equation}
\label{eq:jac_del}
\Bigg\vert \frac{\partial \left( \log \theta_{\text{del-merged}}^j , \log \theta_{\text{del-*}}^j \right)}{\partial \left(\log \theta^j_{p-1}, \log \theta^j_{p} \right)} \Bigg\vert  = 1.
\end{equation}
Applying the standard result Equation (8) of \cite{green1995reversible}, the acceptance ratio of the move to add a change-point is found to be
\begin{align}
\label{eq:r_add}
r_{j+} = \min \left(1,  {\color{ForestGreen}{\frac{p(\mathcal{D} \vert \mathbf{x}, \bm{c}_+, \bm{\theta}_+, \mathbf{u})}{p(\mathcal{D} \vert \mathbf{x}, \bm{c}, \bm{\theta}, \mathbf{u})}}} {\color{red}{\frac{\bm{\lambda}[j] \left(b^j - a^j\right)}{1 + \mathbf{n}[j]}}} {\color{blue}{\frac{p_{\log \theta^j}\left(\log \theta^j_{\text{add-left}}\right)p_{\log \theta^j}\left(\log \theta^j_{\text{add-right}}\right)}{p_{\log \theta^j}\left(\log \theta^j_p\right)p_{\log \theta^j}\left(\log \theta^j_*\right)}}}  \right)
\end{align}
where  $p_{\log \theta^j}$ is the prior over log hyper-parameters in the $j$-th input dimension (as per the prior specification Equation \ref{eq:comp_prior}), which we recall is i.i.d. centred Gaussian with variance $\bm{\rho}[j]$, and $\bm{c}_+$ and $\bm{\theta}_+$ denote the proposal new vector of change-points and the corresponding vector of hyper-parameters. The three coloured terms in the acceptance probability are very intuitive. The green term ${\color{ForestGreen}{\frac{p(\mathcal{D} \vert \mathbf{x}, \bm{c}_+, \bm{\theta}_+, \mathbf{u})}{p(\mathcal{D} \vert \mathbf{x}, \bm{c}, \bm{\theta}, \mathbf{u})}}}$ represents the fit improvement that would occur if the new proposal is accepted. In the red term ${\color{red}{\frac{\bm{\lambda}[j] \left(b^j - a^j\right)}{1 + \mathbf{n}[j]}}}$, ${\color{red}{\bm{\lambda}[j] \left(b^j - a^j\right)}}$ represents the average number of change-points in the $j$-th input dimension as per the HPP prior, while ${\color{red}{1 + \mathbf{n}[j]}}$ corresponds to the proposal new number of change-points in the $j$-th dimension, so that the whole red term acts as a complexity regulariser. Finally, the blue term ${\color{blue}{\frac{p_{\log \theta^j}\left(\log \theta^j_{\text{add-left}}\right)p_{\log \theta^j}\left(\log \theta^j_{\text{add-right}}\right)}{p_{\log \theta^j}\left(\log \theta^j_p\right)p_{\log \theta^j}\left(\log \theta^j_*\right)}}}$ plays the role of hyper-parameters regulariser.

Similarly, the acceptance ratio of the move to delete change-point $c_p^j$, thereby changing the number of change-points in the $j$-th input dimension from $\mathbf{n}[j]$ to $\mathbf{n}[j]-1$, is found to be 
\begin{equation}
\label{eq:r_del}
r_{j-}= \min \left(1, {\color{ForestGreen}{\frac{p(\mathcal{D} \vert \mathbf{x}, \bm{c}_-, \bm{\theta}_-, \mathbf{u})}{p(\mathcal{D} \vert \mathbf{x}, \bm{c}, \bm{\theta}, \mathbf{u})}}} {\color{red}{\frac{\mathbf{n}[j]}{\bm{\lambda}[j] \left(b^j - a^j \right)}}}{\color{blue}{\frac{p_{\log \theta^j}\left(\log \theta^j_{\text{del-merged}}\right)p_{\log \theta^j}\left(\log \theta^j_{\text{del-*}}\right)}{p_{\log \theta^j}\left(\log \theta^j_{p-1}\right)p_{\log \theta^j}\left(\log \theta^j_p\right)}}} \right),
\end{equation}
where $\bm{c}_-$ and $\bm{\theta}_-$ denote the proposal new vector of change-points and the corresponding vector of hyper-parameters. Once more, each coloured term plays the same intuitive role as its counterpart in Equation (\ref{eq:r_add}).
\begin{algorithm}[p]
   \caption{MCMC sampler for nonparametric Bayesian inference of a real-valued latent function under a \emph{string GP} prior}
   \label{alg:mcmc}
\begin{algorithmic}
\vspace{0.15cm}
   \STATE {\bfseries Inputs:}  Likelihood model $p(\mathcal{D} \vert \mathbf{f}, \mathbf{u})$, link function $\phi$, training data $\mathcal{D}$, test inputs, type of unconditional kernel, prior parameters $\bm{\alpha}, \bm{\beta}, \bm{\rho}$.
   \STATE {\bfseries Outputs:} Posterior samples of the values of the latent function at training and test inputs $\mathbf{f}$ and $\mathbf{f}^*$, and the corresponding gradients $\nabla \mathbf{f}$ and $\nabla \mathbf{f}^*$.
   
\hrulefill
\vspace{0.15cm}

\STATE {\bfseries \underline{Step 0}:} Set $\bm{n}=0$ and $\bm{c}= \emptyset$, and sample $\bm{\theta}, \bm{\lambda}, \mathbf{x}, \mathbf{u}$ from their priors.
\REPEAT
\STATE {\bfseries \underline{Step 1}:} Perform a within-model update.
\STATE\hspace{0.3cm}{\bfseries 1.1:} Update each $\bm{\lambda}[j]$ by sampling from the Gamma distribution in Equation (\ref{eq:lam_up}).
\STATE\hspace{0.3cm}{\bfseries 1.2:} Update $\mathbf{u}$, the vector of other likelihood parameters, if any, using Metropolis-Hastings (MH) with proposal $q$ and acceptance ratio Equation (\ref{eq:u}) or by sampling directly from the posterior when $p(\mathbf{u})$ is conjugate to the likelihood model.
\STATE\hspace{0.3cm}{\bfseries 1.3:} Update $\bm{\theta}$, using elliptical slice sampling (ESS) with target distribution Equation (\ref{eq:up_log_theta}), and record the newly computed factors $\{L_k^j, M_k^j\}$ that relate $\mathbf{z}$ to its whitened representation $\mathbf{x}$.
\STATE\hspace{0.3cm}{\bfseries 1.4:} Update $\mathbf{x}$ using ESS with target distribution Equation (\ref{eq:w_rep}).
\STATE\hspace{0.3cm}{\bfseries 1.5:} Update change-point positions $\mathbf{c}$ sequentially using MH, drawing a proposal update for $c^j_p$ uniformly at random on $]c^J_{p-1}, c^j_{p+1}[$, and accepting the update with probability $r_{c^j_p}$ (defined Equation \ref{eq:update_cjp}). On accept, update the factors $\{L_k^j, M_k^j\}$.
\STATE {\bfseries \underline{Step 2}}: Perform a between-models update.
\STATE\hspace{0.3cm}{\bfseries 2.1:} Sample a dimension to update, say $j$, uniformly at random.
\STATE\hspace{0.3cm}{\bfseries 2.2:} Consider adding or deleting a change-point
\STATE\hspace{1.0cm} \textbf{if} $\mathbf{n}[j]=0$ \textbf{then}
\STATE\hspace{1.4cm}Randomly choose to add a change-point with probability $1/2$.
\STATE\hspace{1.4cm}\textbf{if} we should consider adding a change-point  \textbf{then}
\STATE\hspace{1.7cm} Construct proposals to update following Section \ref{sct:btw_md}.
\STATE\hspace{1.7cm} Accept proposals with probability $r_+^j$ (see Equation \ref{eq:r_add}).
\STATE\hspace{1.4cm}\textbf{end if}
\STATE\hspace{1.0cm}\textbf{else}
\STATE\hspace{1.4cm}Randomly choose to add/delete a change-point with probability $1/3$.
\STATE\hspace{1.4cm}\textbf{if} we should consider adding a change-point  \textbf{then}
\STATE\hspace{1.7cm} Construct proposals to update following Section \ref{sct:btw_md}.
\STATE\hspace{1.7cm} Accept proposals with probability $r_+^j$ (see Equation \ref{eq:r_add}).
\STATE\hspace{1.4cm}\textbf{else if} we should consider deleting a change-point  \textbf{then}
\STATE\hspace{1.7cm} Construct proposals to update following Section \ref{sct:btw_md}.
\STATE\hspace{1.7cm} Accept proposals with probability $r_-^j$ (see Equation \ref{eq:r_del}).
\STATE\hspace{1.4cm}\textbf{else}
\STATE\hspace{1.7cm} Continue.
\STATE\hspace{1.4cm}\textbf{end if}
\STATE\hspace{1.0cm} \textbf{end if}
\STATE {\bfseries \underline{Step 3}}: Compute $\mathbf{f}, \mathbf{f}^*, \nabla \mathbf{f}$ and $\nabla \mathbf{f}^*$, first recovering $\mathbf{z}$ from $\mathbf{x}$, and then recalling that $f(x) = \phi \left(z_{x[1]}^1, \dots, z_{x[d]}^d\right)$ and $\nabla f(x) = \left( z_{x[1]}^{1\prime} \frac{\partial \phi}{\partial x[1]}(x), \dots, z_{x[d]}^{d\prime} \frac{\partial \phi}{\partial x[d]}(x) \right)$.
\UNTIL{enough samples are generated after mixing.}
\end{algorithmic}
\end{algorithm}

\textbf{Overall resource requirement}: The bottleneck of between-models updates is the evaluation of the new likelihoods $p(\mathcal{D} \vert \mathbf{x}, \bm{c}_+, \bm{\theta}_+, \mathbf{u})$ or $p(\mathcal{D} \vert \mathbf{x}, \bm{c}_{-}, \bm{\theta}_{-}, \mathbf{u})$, whose resource requirements, which are the same as those of within-models updates, we already discussed.

Algorithm \ref{alg:mcmc} summarises the proposed MCMC sampler.

\subsection{Multi-Output Problems}
\label{sct:multi}
Although we have restricted ourselves to cases where the likelihood model depends on a single real-valued function for brevity and to ease notations, cases where the likelihood depends on vector-valued functions, or equivalently multiple real-valued functions, present no additional theoretical or practical challenge. We may simply put independent \emph{string GP} priors on each of the latent functions. An MCMC sampler almost identical to the one introduced herein  may be used to sample from the posterior. All that is required to adapt the proposed MCMC sampler to multi-outputs problems is to redefine $\mathbf{z}$ to include all univariate \emph{derivative string GP} values across input dimensions and across latent functions, perform step 1.1 of Algorithm \ref{alg:mcmc}  for each of the latent function, and update step 2.1 so as to sample uniformly at random not only what dimension to update but also what latent function. Previous analyses and derived acceptance ratios remain unchanged. The resource requirements of the resulting multi-outputs MCMC sampler on a problem with $K$ latent functions, $N$ training and test $d$-dimensional inputs, are the same as those of the MCMC sampler for a single output (Algorithm \ref{alg:mcmc}) with $N$ training and test $dK$-dimensional inputs. The time complexity is $\mathcal{O}(N)$ when $dK$ computing cores are available, $\mathcal{O}(dKN)$ when no distributed computing is available, and the memory requirement becomes $\mathcal{O}(dKN)$.

\subsection{Flashback to Small Scale GP Regressions with String GP Kernels}
\label{sct:flashback}
In Section \ref{sct:ml} we discussed maximum marginal likelihood inference in Bayesian nonparametric regressions under additively separable \emph{string GP} priors, or GP priors with \emph{string GP} covariance functions. We proposed learning the positions of boundary times, conditional on their number, jointly with kernel hyper-parameters and noise variances by maximizing the marginal likelihood using gradient-based techniques. We then suggested learning the number of strings in each input dimension by trading off goodness-of-fit with model simplicity using information criteria such as AIC and BIC. In this section, we propose a fully Bayesian nonparametric alternative.

Let us consider the Gaussian process regression model
\begin{equation}
y_i = f(x_i) + \epsilon_i,~~~ f \sim \mathcal{GP}\left(0, k_{\text{SGP}}(.,.)\right), ~~~ \epsilon_i \overset{\text{}}{\sim} \mathcal{N}\left(0, \sigma^2\right),
\end{equation}
\begin{equation}
x_i \in [a^1, b^1] \times \dots \times [a^d, b^d], ~~~y_i, \epsilon_i \in \mathbb{R},
\end{equation}
where $ k_{\text{SGP}}$ is the covariance function of some \emph{string GP} with boundary times $\{a_k^j\}$ and corresponding unconditional kernels $\{k_k^j \}$ in the $j$-th input dimension. It is worth stressing that we place a GP (not \emph{string GP}) prior on the latent function $f$, but the covariance function of the GP is a \emph{string GP} covariance function (as discussed in Section \ref{sct:means_kerns} and as derived in \ref{app:global_mean_cov}). Of course when the \emph{string GP} covariance function $k_{\text{SGP}}$ is separately additive, the two functional priors are the same. However, we impose no restriction on the link function of the \emph{string GP} that $k_{\text{SGP}}$ is the covariance function of, other than continuous differentiability. To make full Bayesian nonparametric inference, we may place on the boundary times $\{a_k^j\}$ independent homogeneous Poisson process priors, each with intensity $\lambda^j$. Similarly to the previous section (Equation \ref{eq:comp_prior}) our full prior specification of the \emph{string GP} kernel reads
\begin{equation}
\begin{cases}
\lambda^j \sim \Gamma(\alpha^j, \beta^j),\\
\{a_k^j\} \big\vert \lambda^j  \sim \text{HPP}(\lambda^j)\\
 \theta_k^j[i] \big\vert  \{a_k^j\}, \lambda^j  ~\overset{\text{i.i.d}}{\sim}~ \log \mathcal{N}(0, \rho^j)\\
\forall (j, k) \neq (l, p) ~ \theta_k^j \perp  \theta_p^l,
\end{cases},
\end{equation}
where $ \theta_k^j$ is the vector of hyper-parameters driving the unconditional kernel $k_k^j$. The method developed in the previous section and the resulting MCMC sampling scheme (Algorithm \ref{alg:mcmc}) may be reused to sample from the posterior over function values, pending the following two changes. First,  gradients $\nabla \mathbf{f}$ and $\nabla \mathbf{f}^*$ are no longer necessary. Second, we may work with function values $(\mathbf{f},\mathbf{f}^*)$ directly (that is in the original as opposed to whitened space). The resulting (Gaussian) distribution of function values $(\mathbf{f},\mathbf{f}^*)$ conditional on all other variables is then analytically derived using standard Gaussian identities, like it is done in vanilla Gaussian process regression, so that the within-model update of $(\mathbf{f},\mathbf{f}^*)$ is performed using a single draw from a multivariate Gaussian.

This approach to model complexity learning is advantageous over the information criteria alternative of Section \ref{sct:ml} in that it scales better with large input-dimensions. Indeed, rather than performing complete maximum marginal likelihood inference a number of times that grows exponentially with the input dimension, the approach of this section alternates between exploring a new combination of numbers of kernel configurations in each input dimension, and exploring function values and kernel hyper-parameters (given their number). That being said, this approach should only be considered as an alternative to commonly used kernels for \emph{small scale} regression problems to enable the learning of local patterns. Crucially, it scales as poorly as the \emph{standard GP paradigm}, and Algorithm \ref{alg:mcmc} should be preferred for large scale problems.

\section{Experiments}
\label{sct:exp}
We now move on to presenting empirical evidence for the efficacy of \emph{string GPs} in coping with local patterns in data sets, and in doing so in a scalable manner. Firstly we consider maximum marginal likelihood inference on two small scale problems exhibiting local patterns. We begin with a toy experiment that illustrates the limitations of the \emph{standard GP paradigm} in extrapolating and interpolating simple local periodic patterns. Then, we move on to comparing the accuracy of Bayesian nonparametric regression under a \emph{string GP} prior to that of the standard Gaussian process regression model and existing mixture-of-experts alternatives on the motorcycle data set of \cite{silverman}, commonly used for the local patterns and heteroskedasticity it exhibits. Finally, we illustrate the performance of the previously derived MCMC sampler on two large scale Bayesian inference problems, namely the prediction of U.S. commercial airline arrival delays of \cite{gpbigdatareg} and a new large scale dynamic asset allocation problem.

\subsection{Extrapolation and Interpolation of Synthetic Local Patterns}
\label{sct:exp1}
In our first experiment, we illustrate a limitation of the standard approach consisting of postulating a global covariance structure on the domain, namely that this approach might result in unwanted global extrapolation of local patterns, and we show that this limitation is addressed by the \emph{string GP paradigm}. To this aim, we use $2$ toy regression problems. We consider the following functions:
\begin{equation}
\label{eq:synth1d}
  \begin{array}{c c c}
f_0(t) = \left\{ 
  \begin{array}{l l}
    \sin(60\pi t) & ~ t \in [0, 0.5]\\
  \frac{15}{4}\sin(16\pi t) & ~ t \in ]0.5, 1]
  \end{array} \right., &
  f_1(t) = \left\{ 
  \begin{array}{l l}
    \sin(16\pi t) & ~ t \in [0, 0.5]\\
  \frac{1}{2}\sin(32\pi t) & ~ t \in ]0.5, 1]
  \end{array} \right..
    \end{array}
\end{equation}
$f_0$ (resp. $f_1$) undergoes a sharp (resp. mild) change in frequency and amplitude at $t=0.5$. We consider using their restrictions to $[0.25, 0.75]$ for training. We sample those restrictions with frequency $300$, and we would like to extrapolate the functions to the rest of their domains using Bayesian nonparametric regression. 

We compare marginal likelihood \emph{string GP} regression models, as described in Section \ref{sct:ml}, to vanilla GP regression models using popular and expressive kernels. All \emph{string GP} models have two strings and the partition is learned in the marginal likelihood maximisation. Figure \ref{fig:synth_bench} illustrates plots of the posterior means for each kernel used,  and Table \ref{table:synth_bench} compares predictive errors. Overall, it can be noted that the \emph{string GP kernel} with the periodic kernel (\cite{gp_intro}) as building block outperforms competing kernels, including the expressive spectral mixture kernel $$k_{\text{SM}}(r) = \sum_{k=1}^{K}  \sigma_k^2 \exp(-2\pi^2r^2\gamma^2_k)\cos(2\pi r \mu_k)$$ of \cite{wilson2013gaussian} with $K=5$ mixture components.\footnote{The sparse spectrum kernel of \cite{sparsespectrum} can be thought of as the special case $\gamma_k \ll 1$.}

The comparison between the spectral mixture kernel and the string spectral mixture kernel is of particular interest, since spectral mixture kernels are pointwise dense in the family of stationary kernels, and thus can be regarded as flexible enough for learning \emph{stationary} kernels from the data. In our experiment, the \emph{string spectral mixture kernel} with  a single mixture component per string significantly outperforms the spectral mixture kernel with $5$ mixture components. This intuitively can be attributed to the fact that, regardless of the number of mixture components in the spectral mixture kernel, the learned kernel must account for both types of patterns present in each training data set. Hence, each local extrapolation on each side of $0.5$ will attempt to make use of both amplitudes and both frequencies evidenced in the corresponding training data set, and will struggle to recover the true local sine function. We would expect that the performance of the spectral mixture kernel in this experiment will not improve drastically as the number of mixture components increases. However, under a \emph{string GP} prior, the left and right hand side strings are independent conditional on the (unknown) boundary conditions. Therefore, when the \emph{string GP} domain partition occurs at time $0.5$, the training data set on $[0.25, 0.5]$ influences the hyper-parameters of the string to the right of $0.5$ only to the extent that both strings should agree on the value of the latent function and its derivative at $0.5$. To see why this is a weaker condition, we consider the family of pair of functions: $$(\alpha \omega_1 \sin(\omega_2 t), ~  \alpha \omega_2 \sin(\omega_1 t)), ~~~ \omega_i = 2\pi k_i, ~~~ k_i \in \mathbb{N}, ~~~ \alpha \in \mathbb{R}.$$
Such functions always have the same value and derivative at $0.5$, regardless of their frequencies, and they are plausible GP paths under a spectral mixture kernel with one single mixture component ($\mu_k = k_i$ and $\gamma_k \ll 1$),  and under a periodic kernel. As such it is not surprising that extrapolation under a string spectral mixture kernel or a string periodic kernel should perform well. 

To further illustrate that \emph{string GPs} are able to learn local patterns that GPs with commonly used and expressive kernels can't, we consider interpolating two bivariate functions $f_2$ and $f_3$ that exhibit local patterns. The functions are defined as:
\begin{equation}
\forall u, v \in [0.0, 1.0] ~~~ f_2(u, v)=f_0(u)f_1(v), ~~~  f_3(u, v)=\sqrt{f_0(u)^2 + f_1(v)^2}.
\end{equation}
We consider recovering the original functions as the posterior mean of a GP regression model trained on $[0.0, 0.4] \cup [0.6, 1.0] \times [0.0, 0.4] \cup [0.6, 1.0]$. Each bivariate kernel used is a product of two univariate kernels in the same family, and we used standard Kronecker techniques to speed-up inference \citep[see][p.134]{saatchi11}.  The univariate kernels we consider are the same as previously. Each univariate \emph{string GP} kernel has one change-point (two strings) whose position is learned by maximum marginal likelihood. Results are illustrated in Figures \ref{fig:pred_f2} and \ref{fig:pred_f3}. Once again it can be seen that unlike any competing kernel, the product of string periodic kernels recover both functions almost perfectly. In particular, it is impressive to see that, despite $f_3$ not being a separable function, a product of string periodic kernels recovered it almost perfectly. The interpolations performed by the spectral mixture kernel (see Figures \ref{fig:pred_f2} and \ref{fig:pred_f3}) provide further evidence for our previously developed narrative: the spectral mixture kernel tries to blend all local patterns found in the training data during the interpolation. The periodic kernel learns a single global frequency characteristic of the whole data set, ignoring local patterns, while the squared exponential, Mat\'{e}rn and rational quadratic kernels merely attempt to perform interpolation by smoothing.

Although we used synthetic data to ease illustrating our argument, it is reasonable to expect that in real-life problems the bigger the data set, the more likely there might be local patterns that should not be interpreted as noise and yet are not indicative of the data set as whole.

\begin{figure}[p]
\begin{center}
\centerline{\includegraphics[width=0.8\textwidth]{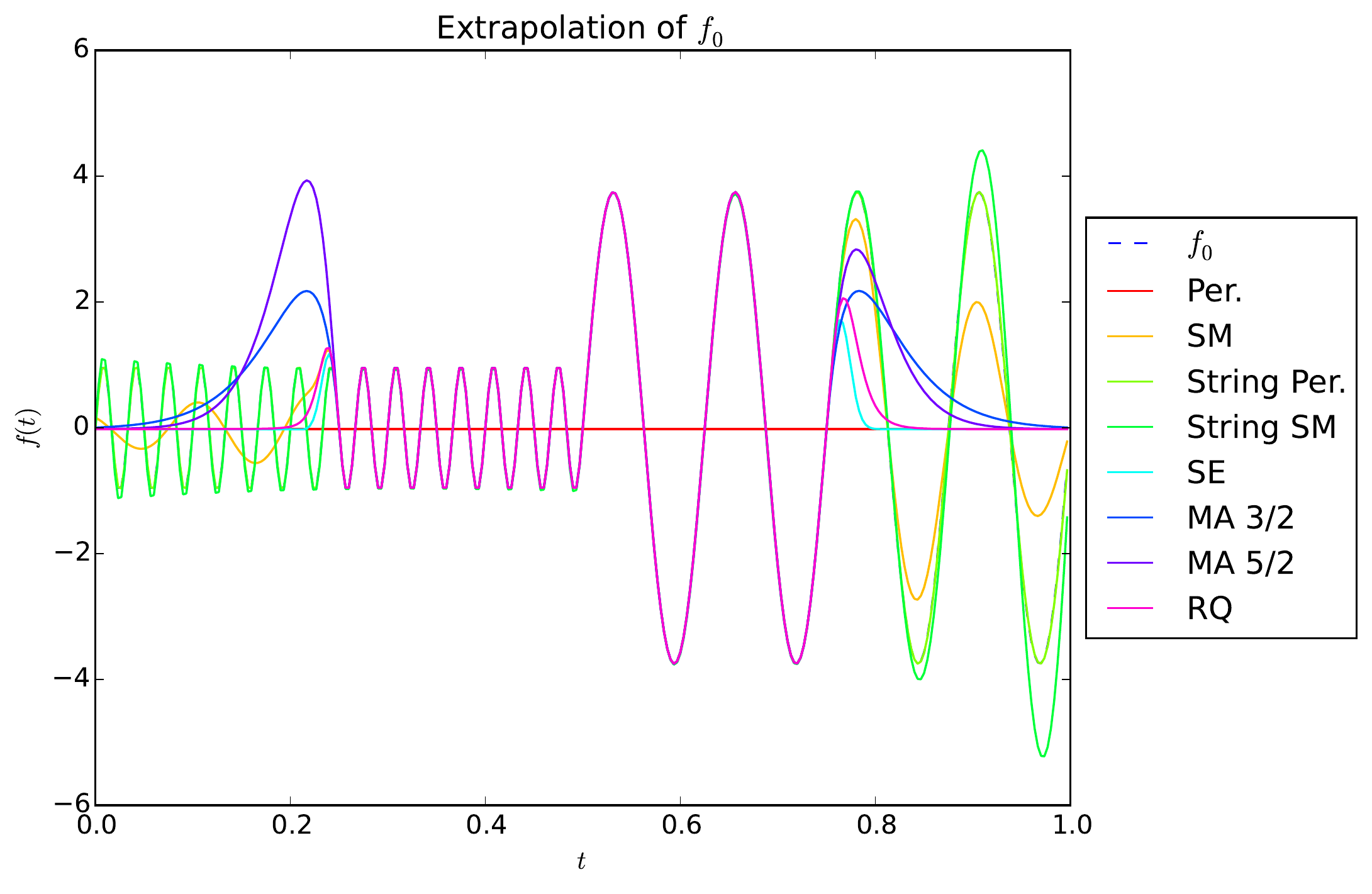}}
\centerline{\includegraphics[width=0.8\textwidth]{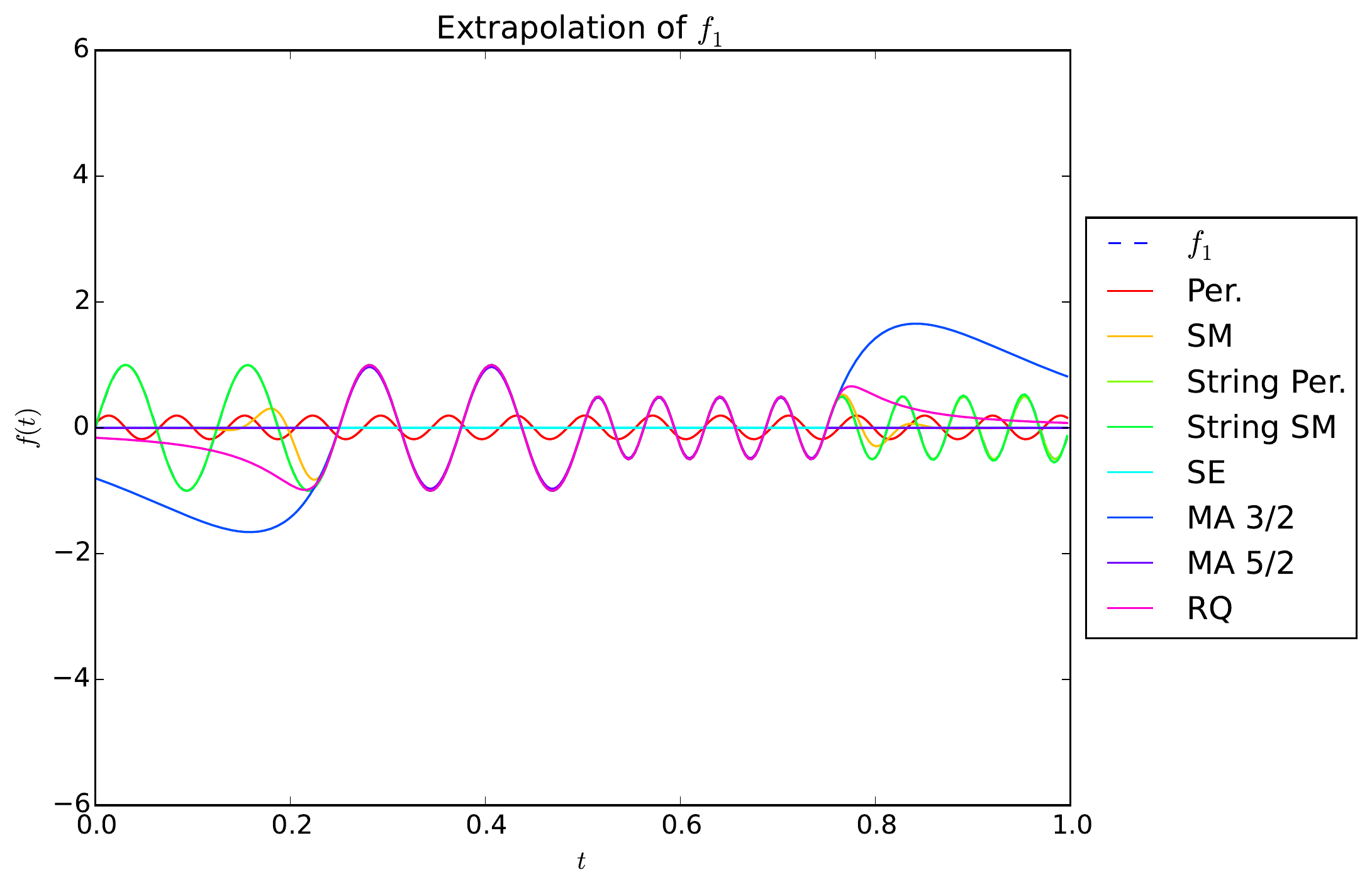}}
\caption{Extrapolation of two functions $f_0$ and $f_1$ through Bayesian nonparametric regression under \emph{string GP} priors and vanilla GP priors with popular and expressive kernels. Each model is trained on $[0.25, 0.5]$ and extrapolates to $[0, 1.0]$.}
\label{fig:synth_bench}
\end{center}
\end{figure} 

\begin{landscape}
\begin{table*}
\centering
\begin{tabular}{l l l l l  }\toprule
&\multicolumn{2}{c}{Absolute Error} & \multicolumn{2}{c}{Squared Error} \\
\cmidrule{2-5} 
Kernel & $f_0$ & $f_1$  & $f_0$ & $f_1$  \\
\cmidrule{1-5}
Squared exponential & $1.44 \pm 2.40$ & $0.48 \pm 0.58$  & $3.50 \pm 9.20$  & $0.31 \pm 0.64$ \\
Rational quadratic  & $1.39 \pm 2.31$ &$0.51 \pm 0.83$ &  $3.28 \pm 8.79$ & $0.43 \pm 1.15$	\\
Mat\'{e}rn 3/2 & $1.63 \pm 2.53$ & $1.26 \pm 1.37$  & $4.26 \pm 11.07$ & $2.06 \pm 3.55$ 	\\
Mat\'{e}rn 5/2 & $1.75 \pm 2.77$ & $0.48 \pm 0.58$  & $5.00 \pm 12.18$ & $0.31 \pm 0.64$ 	\\
Periodic & $1.51 \pm 2.45$ & $0.53 \pm 0.60$  & $3.79\pm 9.62$ & $0.37 \pm 0.72$ 	\\
Spec. Mix. (5 comp.) & $0.75 \pm 1.15$ & $0.39\pm 0.57$ & $0.94 \pm 2.46$ & $0.24 \pm 0.58$   \\ 
String Spec. Mix. (2 strings, 1 comp.) & $0.23 \pm 0.84$ & $0.01\pm 0.03$ & $0.21 \pm 1.07$ & $\bf{0.00} \pm \bf{0.00}$ 	\\
String Periodic &$\bf{0.02} \pm \bf{0.02}$ & $\bf{0.00} \pm \bf{0.01}$ & $\bf{0.00} \pm \bf{0.00}$ & $\bf{0.00} \pm \bf{0.00}$  \\
\bottomrule
\end{tabular}
\caption{Predictive accuracies in the extrapolation of the two functions $f_0$ and $f_1$ of Section \ref{sct:exp1} through Bayesian nonparametric regression under \emph{string GP} priors and vanilla GP priors with popular and expressive kernels. Each model is trained on $[0.25, 0.5]$ and extrapolates to $[0, 1.0]$. The predictive errors are reported as average $\pm$ 2 standard deviations.}
\label{table:synth_bench}
\end{table*}
\end{landscape}

\begin{figure}[p]
\begin{center}
\centerline{\includegraphics[width=0.44\textwidth]{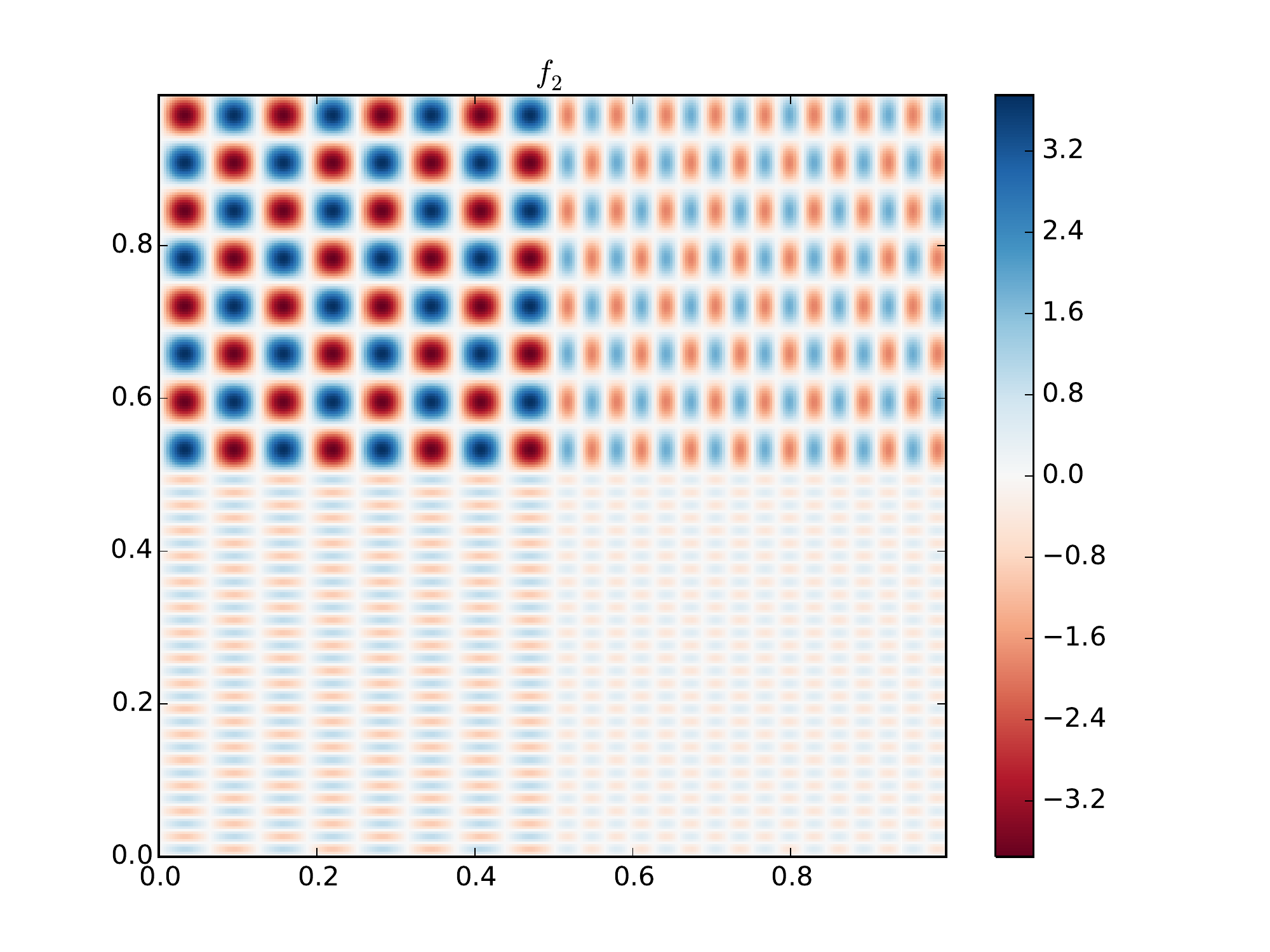} \includegraphics[width=0.44\textwidth]{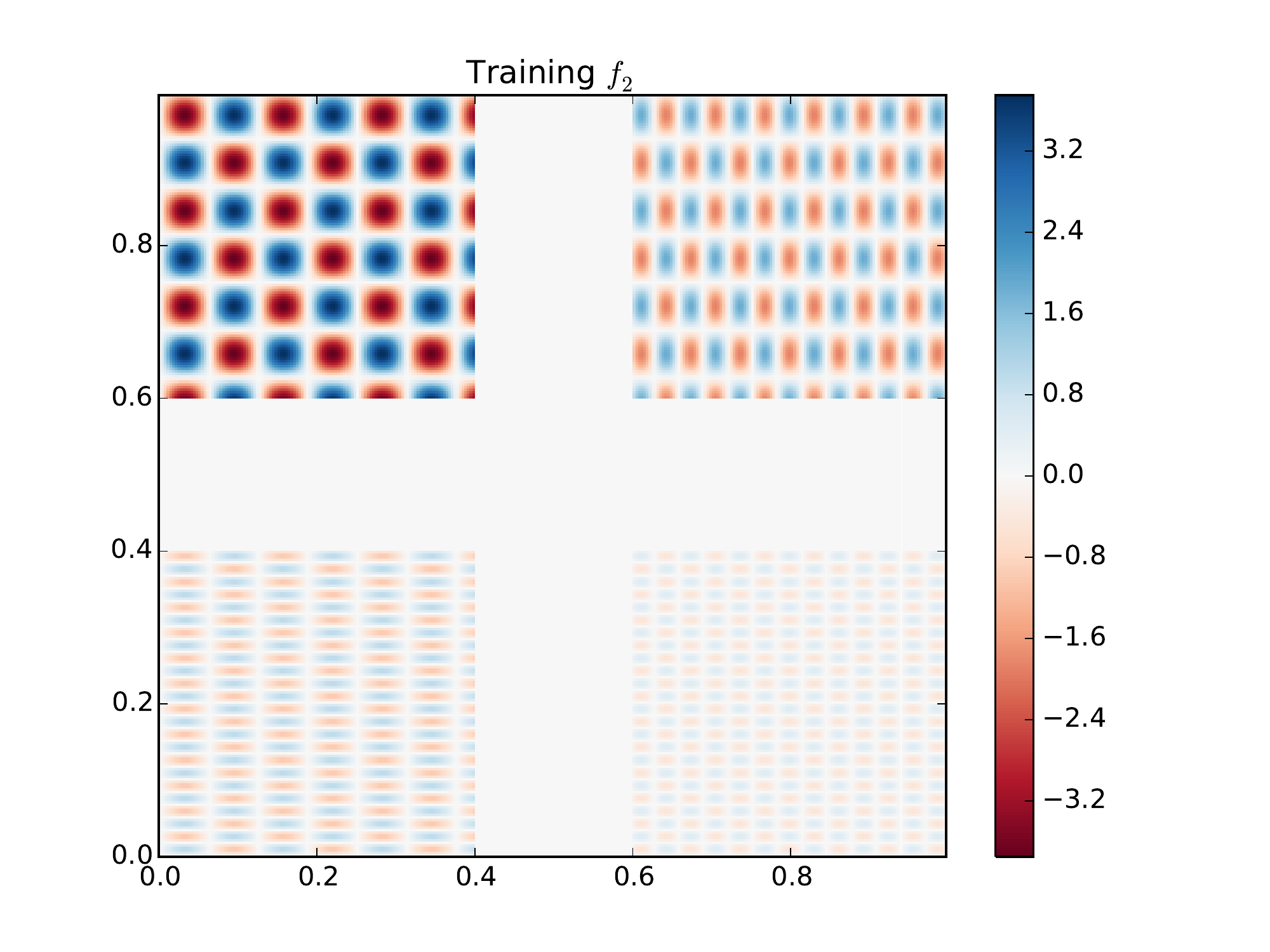}} 
\centerline{\includegraphics[width=0.44\textwidth]{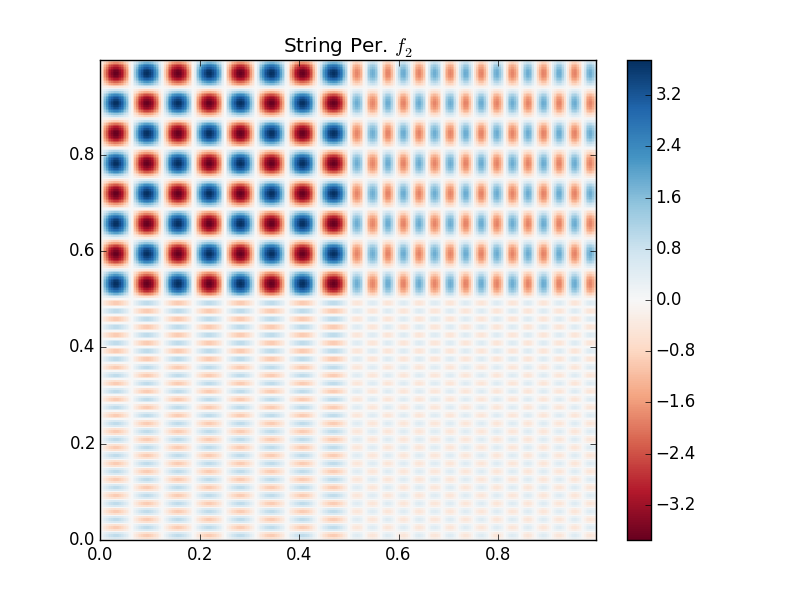} \includegraphics[width=0.44\textwidth]{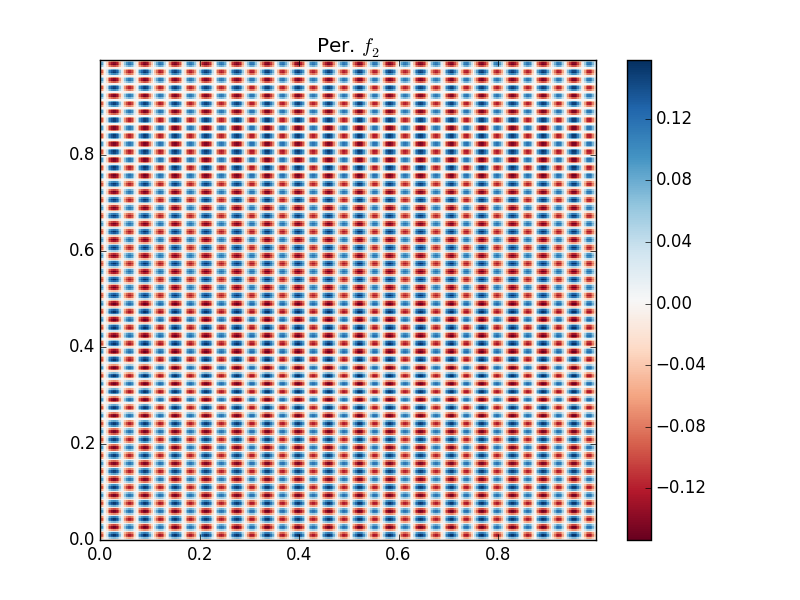}}
\centerline{\includegraphics[width=0.44\textwidth]{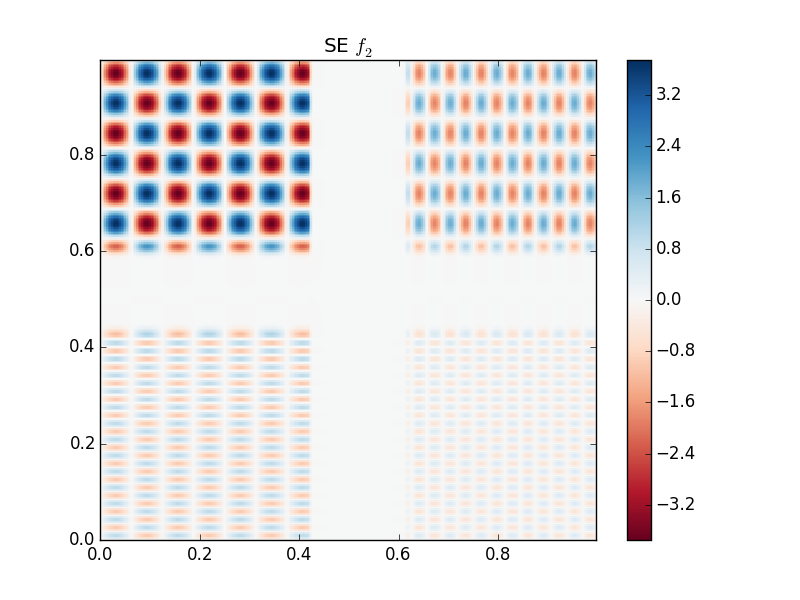} \includegraphics[width=0.44\textwidth]{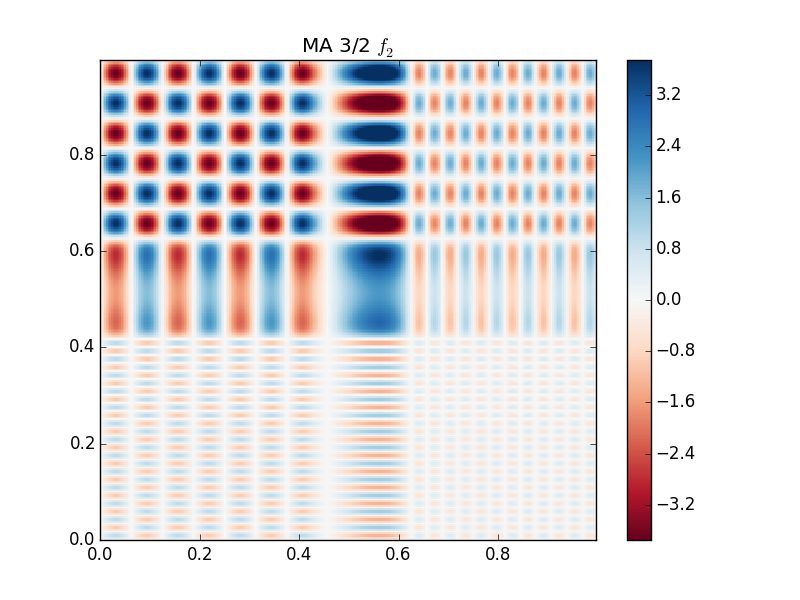}}
\centerline{\includegraphics[width=0.44\textwidth]{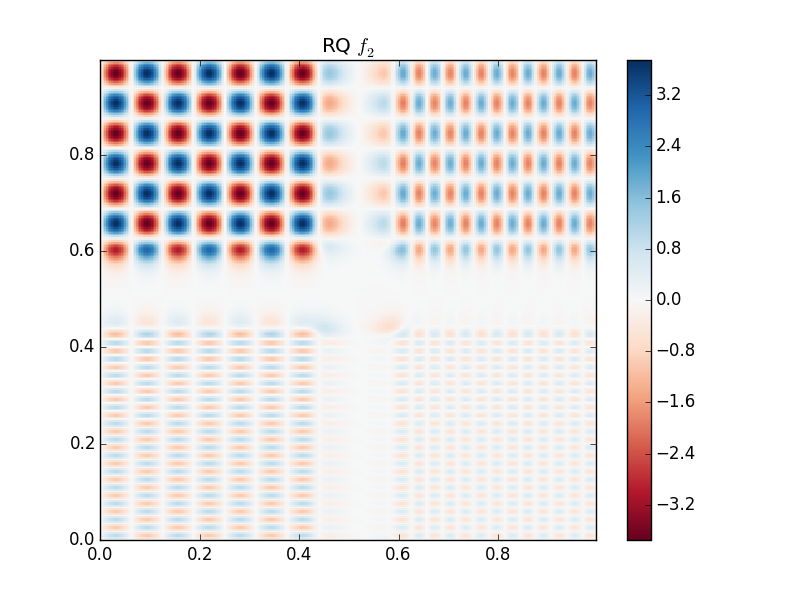} \includegraphics[width=0.44\textwidth]{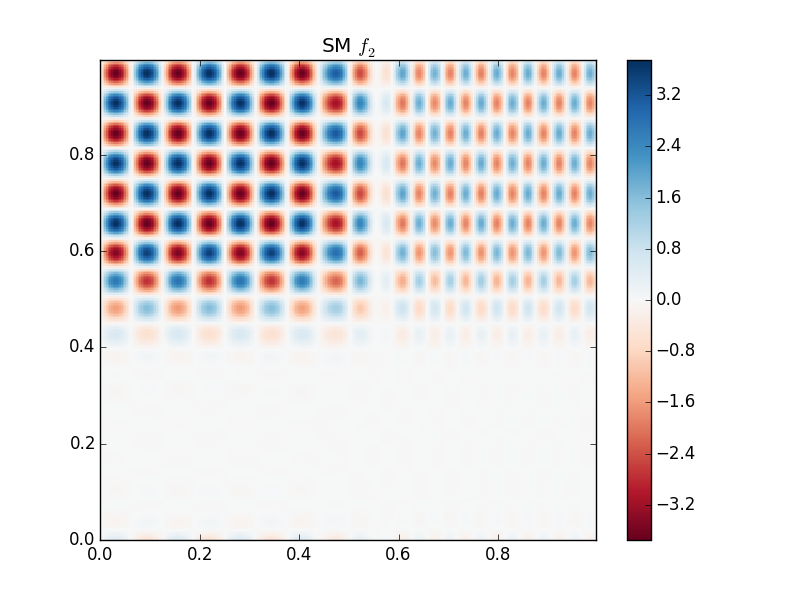}}
\caption{Extrapolation of a synthetic function $f_2$ (top left corner), cropped in the middle for training (top right corner), using \emph{string GP} regression and vanilla GP regression with various popular and expressive kernels.}
\label{fig:pred_f2}
\end{center}
\end{figure}

\begin{figure}[p]
\begin{center}
\centerline{\includegraphics[width=0.44\textwidth]{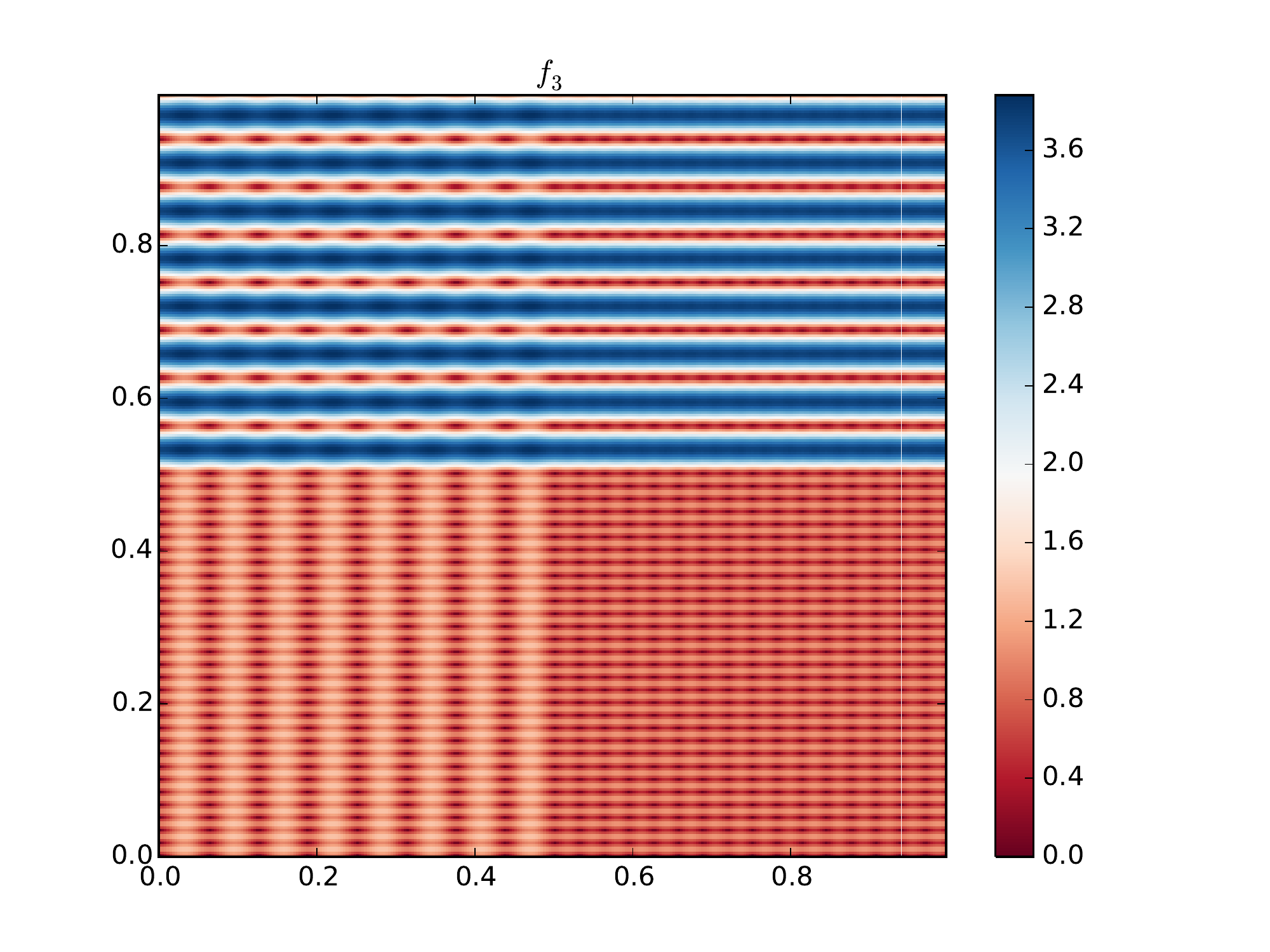} \includegraphics[width=0.44\textwidth]{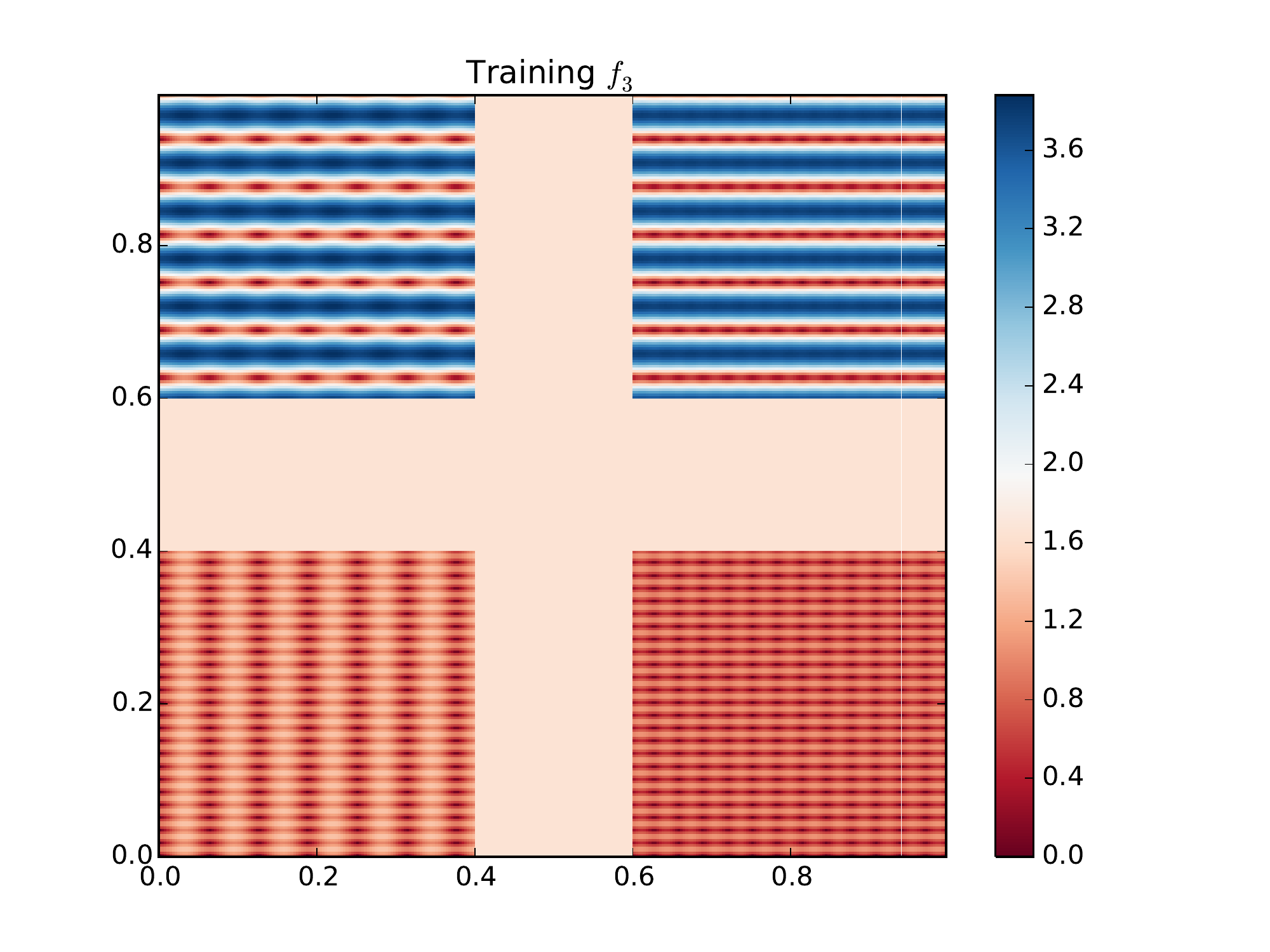}} 
\centerline{\includegraphics[width=0.44\textwidth]{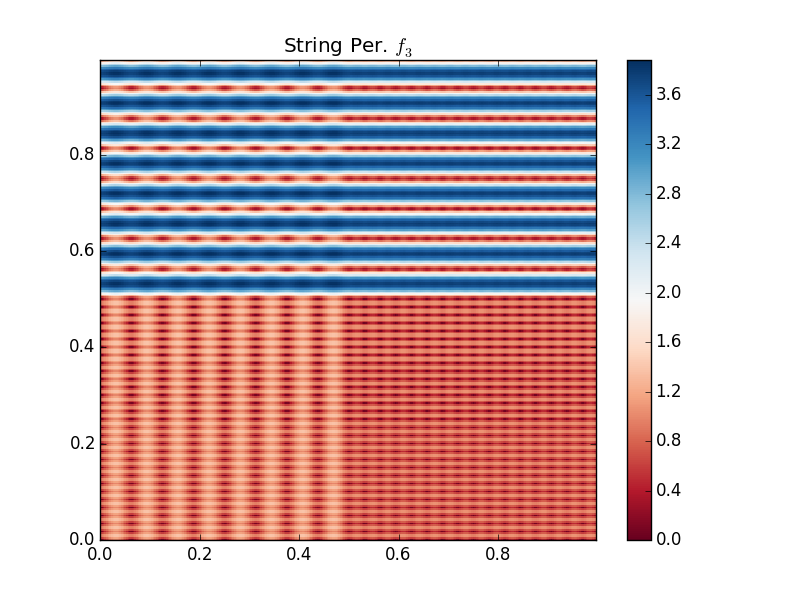} \includegraphics[width=0.44\textwidth]{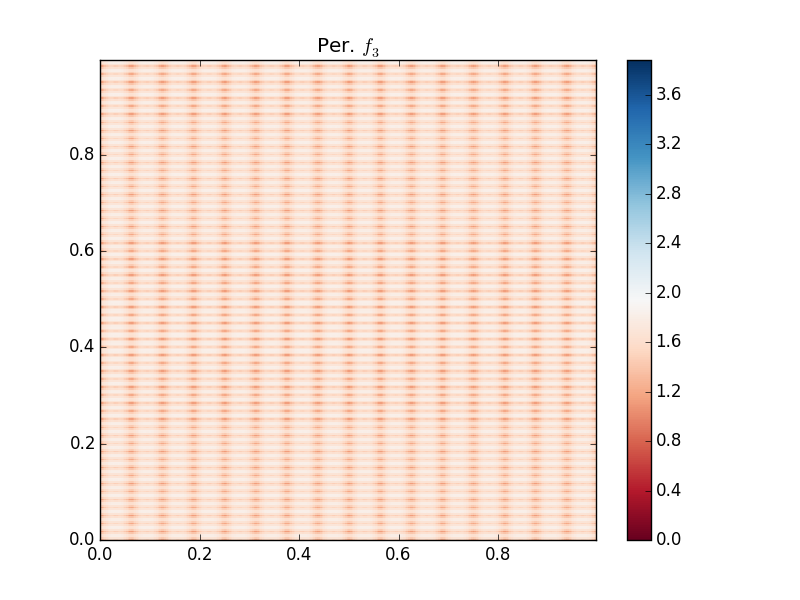}}
\centerline{\includegraphics[width=0.44\textwidth]{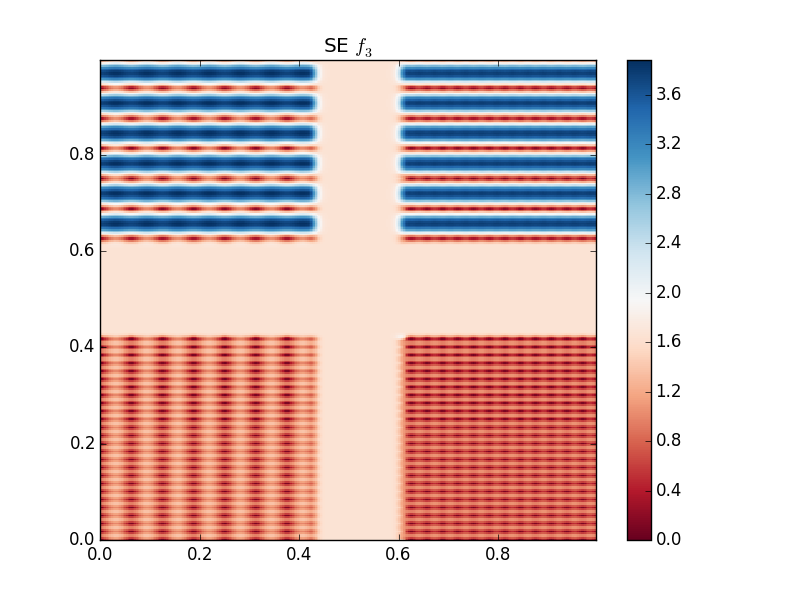} \includegraphics[width=0.44\textwidth]{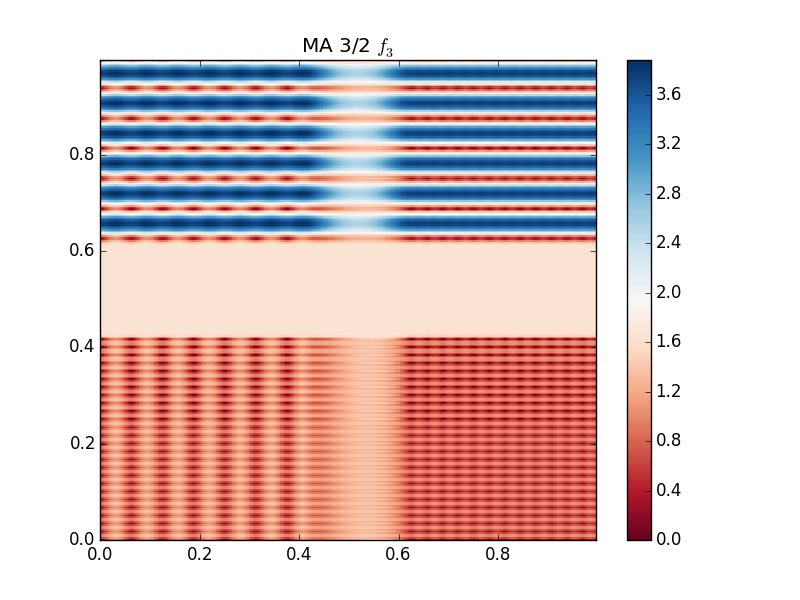}}
\centerline{\includegraphics[width=0.44\textwidth]{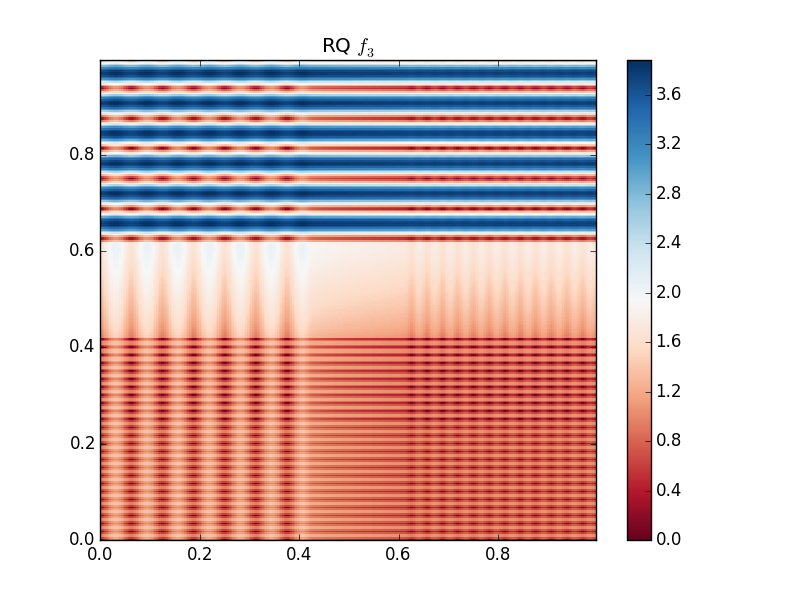} \includegraphics[width=0.44\textwidth]{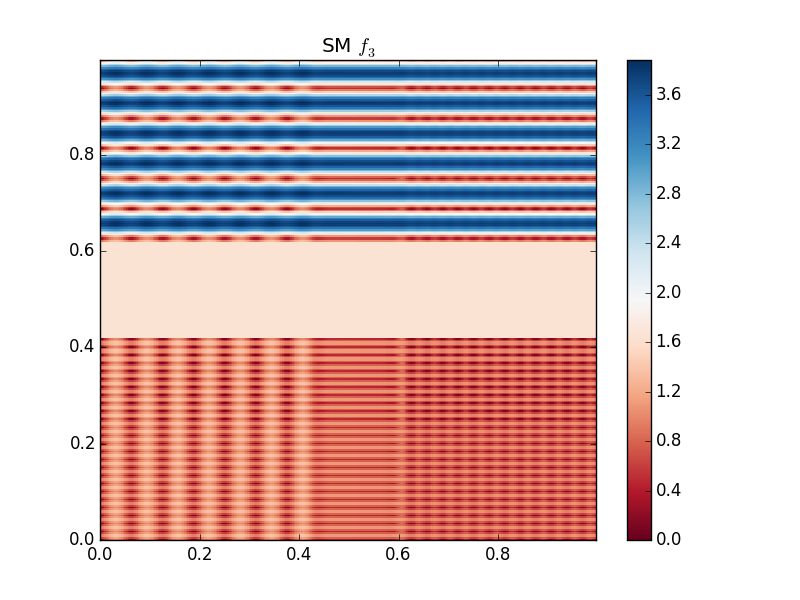}}
\caption{Extrapolation of a synthetic function $f_3$ (top left corner), cropped in the middle for training (top right corner), using \emph{string GP} regression and vanilla GP regression with various popular and expressive kernels.}
\label{fig:pred_f3}
\end{center}
\end{figure}

\subsection{Small Scale Heteroskedastic Regression}
In our second experiment, we consider illustrating the advantage of the \emph{string GP paradigm} over the \emph{standard GP paradigm}, but also over the alternatives of  \cite{kim}, \cite{gramacy}, \cite{tresp2000bayesian} and \cite{deisenroth2015distributed} that consist of considering independent GP experts on disjoint parts of the domain or handling disjoint subsets of the data. Using the motorcycle data set of  \cite{silverman}, commonly used for the local patterns and heteroskedasticity it exhibits, we show that our approach outperforms the aforementioned competing alternatives, thereby providing empirical evidence that the collaboration between consecutive GP experts introduced in the \emph{string GP paradigm} vastly improves predictive accuracy and certainty in regression problems with local patterns. We also illustrate learning of the derivative of the latent function, solely from noisy measurements of the latent function.

The observations consist of accelerometer readings taken through time in an experiment on the efficacy of crash helmets. It can be seen at a glance  in Figure \ref{fig:motorcycle_jerk} that the data set exhibits roughly $4$ regimes. Firstly, between $0$ms and $15$ms the acceleration was negligible. Secondly, the impact slowed down the helmet, resulting in a sharp deceleration between $15$ms and $28$ms. Thirdly, the helmet seems to have bounced back between $28$ms and $32$ms, before it finally gradually slowed down and came to a stop between $32$ms and $60$ms. It can also be noted that the measurement noise seems to have been higher in the second half of the experiment. 

We ran $50$ independent random experiments, leaving out $5$ points selected uniformly at random from the data set for prediction, the rest being used for training. The models we considered include the vanilla GP regression model, the \emph{string GP} regression model with marginal maximum likelihood inference as described in Section \ref{sct:ml}, mixtures of independent GP experts acting on disjoint subsets of the data both for training and testing, the Bayesian committee machine (\cite{tresp2000bayesian}), and the robust Bayesian committee machine (\cite{deisenroth2015distributed}). We considered \emph{string GPs} with $4$ and $6$ strings whose boundary times are learned as part of the maximum likelihood inference. For consistency, we used the resulting partitions of the domain to define the independent experts in the competing alternatives we considered. The Mat\'{e}rn 3/2 kernel was used throughout. The results are reported in Table \ref{table:motorcycle_bench}. To gauge the ability of each model to capture the physics of the helmets crash experiment, we have also trained all models with all data points. The results are illustrated in Figures \ref{fig:motorcycle_jerk} and \ref{fig:motorcycle}.

It can be seen at a glance from Figure \ref{fig:motorcycle} that mixtures of independent GP experts are inappropriate for this experiment as i) the resulting posterior means exhibit discontinuities (for instance at $t=30\text{ms}$ and $t=40\text{ms}$) that are inconsistent with the physics of the underlying phenomenon, and ii) they overfit the data towards the end. The foregoing discontinuities  do not come as a surprise as each GP regression expert acts on a specific subset of the domain that is disjoint from the ones used by the other experts, both for training and prediction. Thus, there is no guarantee of consistency between expert predictions at the boundaries of the domain partition. Another perspective to this observation is found in noting that postulating independent GP experts, each acting on an element of a partition of the domain, is equivalent to putting as prior on the whole function a stochastic process that is discontinuous at the boundaries of the partition. Thus, the posterior stochastic process should not be expected to be continuous at the boundaries of the domain either.

This discontinuity issue is addressed by the Bayesian committee machine (BCM) and the robust Bayesian committee machine (rBCM) because, despite each independent expert being trained on a disjoint subset of the data, each expert is tasked with making predictions about all test inputs, not just the ones that fall into its input subspace. Each GP expert prediction is therefore continuous on the whole input domain,\footnote{So long as the functional prior is continuous, which is the case here.} and the linear weighting schemes operated by the BCM and the rBCM on expert predictions to construct the overall predictive mean preserve continuity. However, we found that the BCM and the rBCM suffer from three pitfalls. First, we found them to be less accurate than any other alternative out-of-sample on this data set (see Table \ref{table:motorcycle_bench}). Second, their predictions of latent function values are overly uncertain. This might be due to the fact that, each GP expert being trained only with training samples that lie on its input subspace, its predictions about test inputs that lie farther away from its input subspace will typically be much more uncertain, so that, despite the weighting scheme of the Bayesian committee machine putting more mass on `confident' experts, overall the posterior variance over latent function values might still be much higher than in the \emph{standard GP paradigm} for instance. This is well illustrated by both the last column of Table \ref{table:motorcycle_bench} and the BCM and rBCM plots in Figure \ref{fig:motorcycle}. On the contrary, no \emph{string GP} model suffers from this excess uncertainty problem. Third, the posterior means of the BCM, the rBCM and the vanilla GP regression exhibit oscillations towards the end ($t >40 \text{ms}$) that are inconsistent with the experimental setup; the increases in acceleration as the helmet slows down suggested by these posterior means would require an additional source of energy after the bounce.

In addition to being more accurate and more certain about predictions than vanilla GP regression, the BCM and the rBCM (see Table \ref{table:motorcycle_bench}), \emph{string GP} regressions yield posterior mean acceleration profiles that are more consistent with the physics of the experiment: steady speed prior to the shock, followed by a deceleration resulting from the shock, a brief acceleration resulting from the change in direction after the bounce, and finally a smooth slow down due to the dissipation of kinetic energy. Moreover, unlike the vanilla GP regression, the rBCM and the BCM, \emph{string GP} regressions yield smaller posterior variances towards the beginning and the end of the experiment than in the middle, which is consistent with the fact that the operator would be less uncertain about the acceleration at the beginning and at the end of the experiment---one would indeed expect the acceleration to be null at the beginning and at the end of the experiment. This desirable property can be attributed to the heteroskedasticity of the noise structure in the \emph{string GP} regression model.

We also learned the derivative of the latent acceleration with respect to time, purely from noisy acceleration measurements using the joint law of a \emph{string GP} and its derivative (Theorem \ref{theo:sgp}). This is illustrated in Figure \ref{fig:motorcycle_jerk}.
\begin{landscape}
\begin{center}
\begin{table*}
\centering
\begin{tabular}{lllllll}
\toprule
& \multicolumn{1}{c}{Training} && \multicolumn{4}{c}{Prediction} \\
\cmidrule{2-2} \cmidrule{4-7}
& 							Log. lik.				&& Log. lik. & 				Absolute Error &  				Squared Error &		Pred. Std  \\ 
\midrule
String GP (4 strings)&			$-388.36 \pm 0.36$	&& $-22.16 \pm 0.41$  & 	$\bf{15.70 \pm 1.05}$ & 	$\bf{466.47 \pm 50.74}$ &	$0.70 / 2.25 / 3.39$\\ 
String GP (6 strings)&			$\bf{-367.21 \pm 0.43}$	&& $-21.99 \pm 0.37$  &		$15.89 \pm 1.06$ & 		$475.59 \pm 51.95$ &	$\bf{0.64 / 2.21 / 3.46}$	\\
Vanilla GP&					$-420.69 \pm 0.24$ 	&& $-22.77 \pm 0.24$ & 		$16.84 \pm 1.09$ & 		$524.18 \pm 58.33$ &	$2.66 / 3.09 / 4.94$ \\

Mix. of 4 GPs & 				$-388.37 \pm 0.38$ 	&& $-20.90 \pm 0.38$  & 	$16.61 \pm 1.10$ & 		$512.30 \pm 56.08$ &	$1.67 / 2.85 / 4.59$\\
Mix. of 6 GPs & 				$-369.05 \pm 0.45$ 	&& $\bf{-20.11 \pm 0.45}$  & $16.05 \pm 1.11$ & 		$500.43 \pm 58.26$ &	$0.62 / 2.83 / 4.63$ \\ 

BCM with 4 GPs & 				 $-419.08 \pm 0.30$	 && $-22.94 \pm 0.26$  & 	$17.17 \pm 1.13$ & 		$538.94 \pm 61.91$ &	$7.20 / 9.92 / 22.92$\\
BCM with 6 GPs & 				 $-422.15 \pm 0.30$ 	 && $-22.91 \pm 0.26$  & 	$16.93 \pm 1.12$ & 		$533.21 \pm 61.78$ &	$7.09 / 9.93 / 25.10$ \\

rBCM with 4 GPs & 				 $-419.08 \pm 0.30$	 && $-22.99 \pm 0.27$  & 	$17.29 \pm 1.11$ & 		$546.95 \pm 61.21$ &	$5.86 / 9.08 / 27.52$\\
rBCM with 6 GPs & 				 $-422.15\pm 0.30$ 	 && $-22.96 \pm 0.28$  & 	$16.79 \pm 1.12$ & 		$542.95 \pm 61.95$ &	$5.15 / 8.61 / 29.15$ \\
\bottomrule
\end{tabular}
\caption{Performance comparison between \emph{string GPs}, vanilla GPs, mixture of independent GPs, the  Bayesian committee machine (\cite{tresp2000bayesian}) and the robust Bayesian committee machine (\cite{deisenroth2015distributed}) on the motorcycle data set of \cite{silverman}. The Mat\'{e}rn 3/2 kernel was used throughout. The domain partitions were learned in the \emph{string GP} experiments by maximum likelihood. The learned partitions were then reused to allocate data between GP experts in other models. $50$ random runs were performed, each run leaving $5$ data points out for testing and using the rest for training. All results (except for predictive standard deviations) are reported as average over the $50$ runs $\pm$ standard error. The last column contains the minimum, average and maximum of the predictive standard deviation of the values of the latent (noise-free) function at all test points across random runs.}
\label{table:motorcycle_bench}
\end{table*}
\end{center}
\end{landscape}

\begin{figure}[p]
\begin{center}
\centerline{\includegraphics[width=0.7\textwidth]{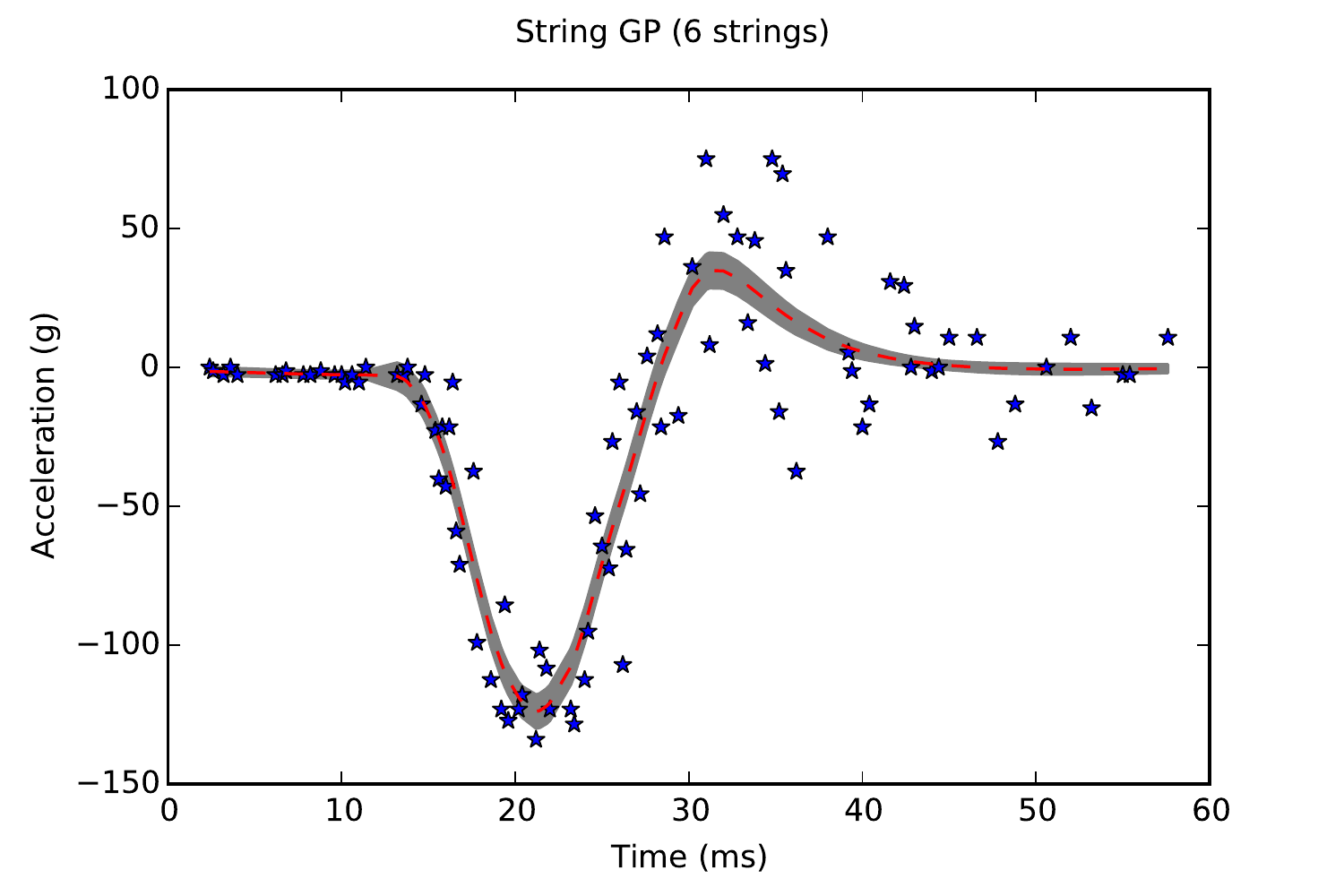}}
\centerline{\includegraphics[width=0.7\textwidth]{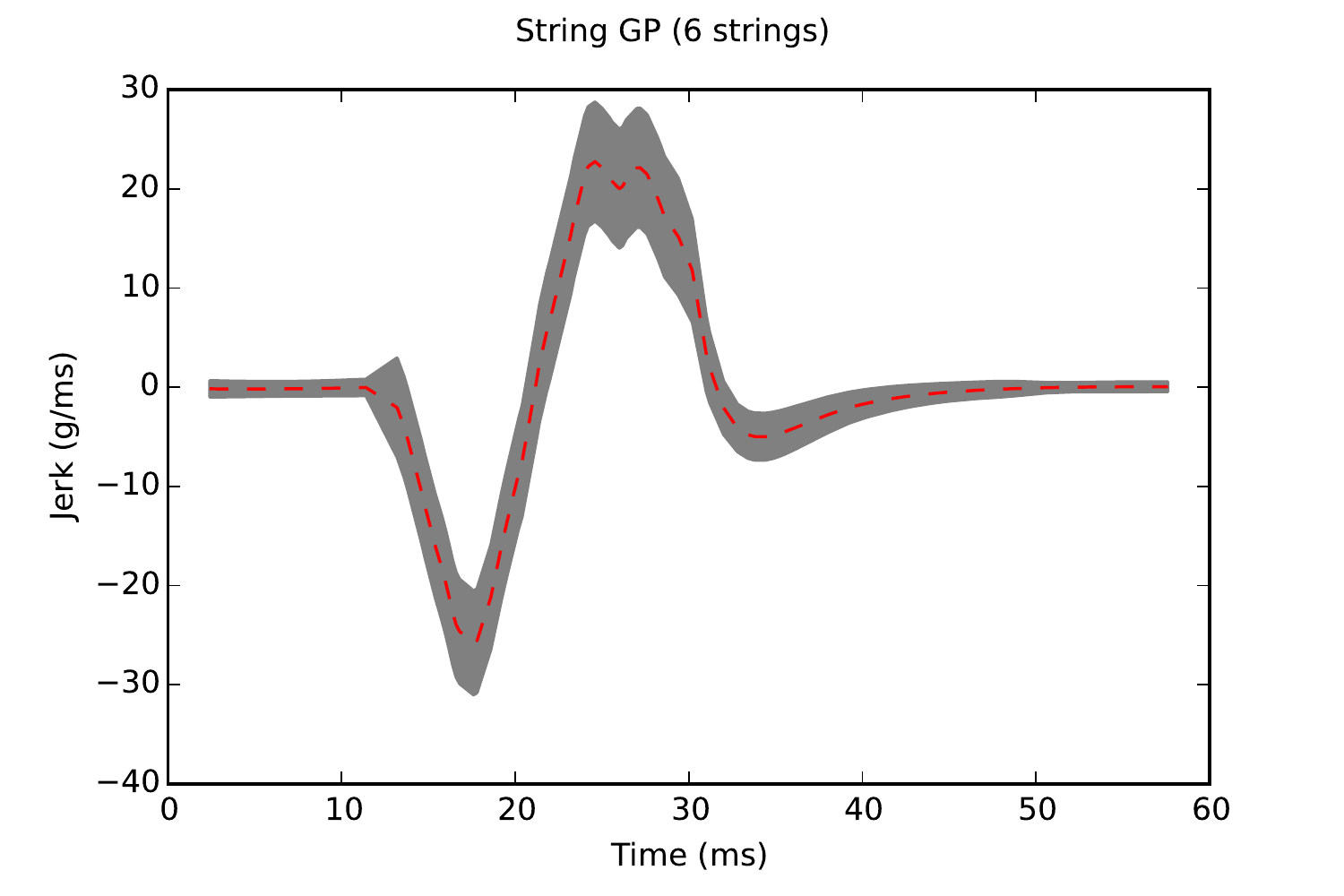}}
\caption{Posterior mean $\pm$ 2 predictive standard deviations on the motorcycle data set \citep[see][]{silverman}, under a Mat\'{e}rn 3/2 derivative \emph{string GP} prior with 6 learned strings. The top figure shows the noisy accelerations measurements and the learned latent function. The bottom function illustrates the derivative of the acceleration with respect to time learned from noisy acceleration samples. Posterior credible bands are over the latent functions rather than  noisy measurements, and as such they do not include the measurement noise.}
\label{fig:motorcycle_jerk}
\end{center}
\end{figure} 

\begin{landscape}
\begin{figure}[p]
\begin{center}
\centerline{\includegraphics[width=0.43\textwidth]{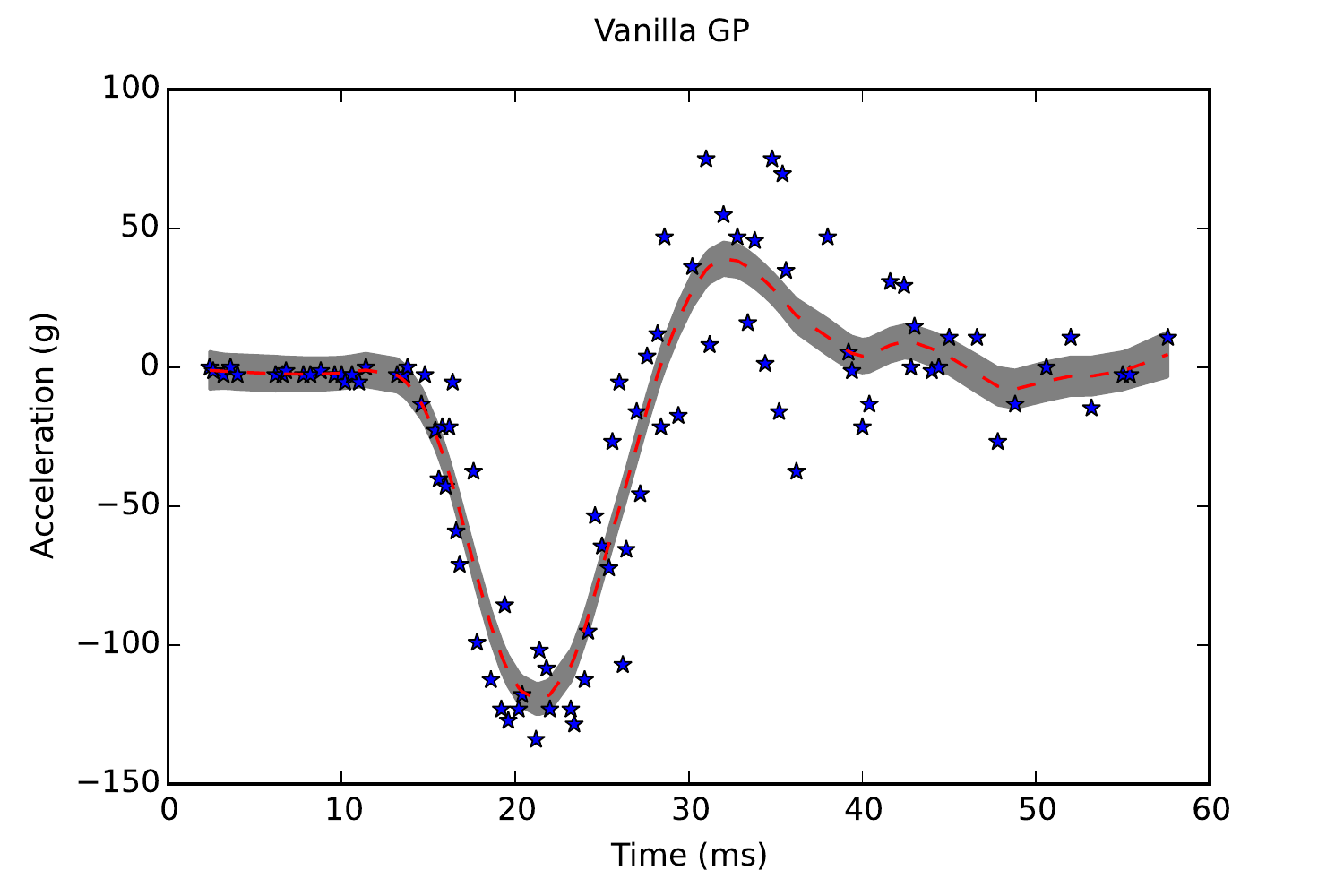}\includegraphics[width=0.43\textwidth]{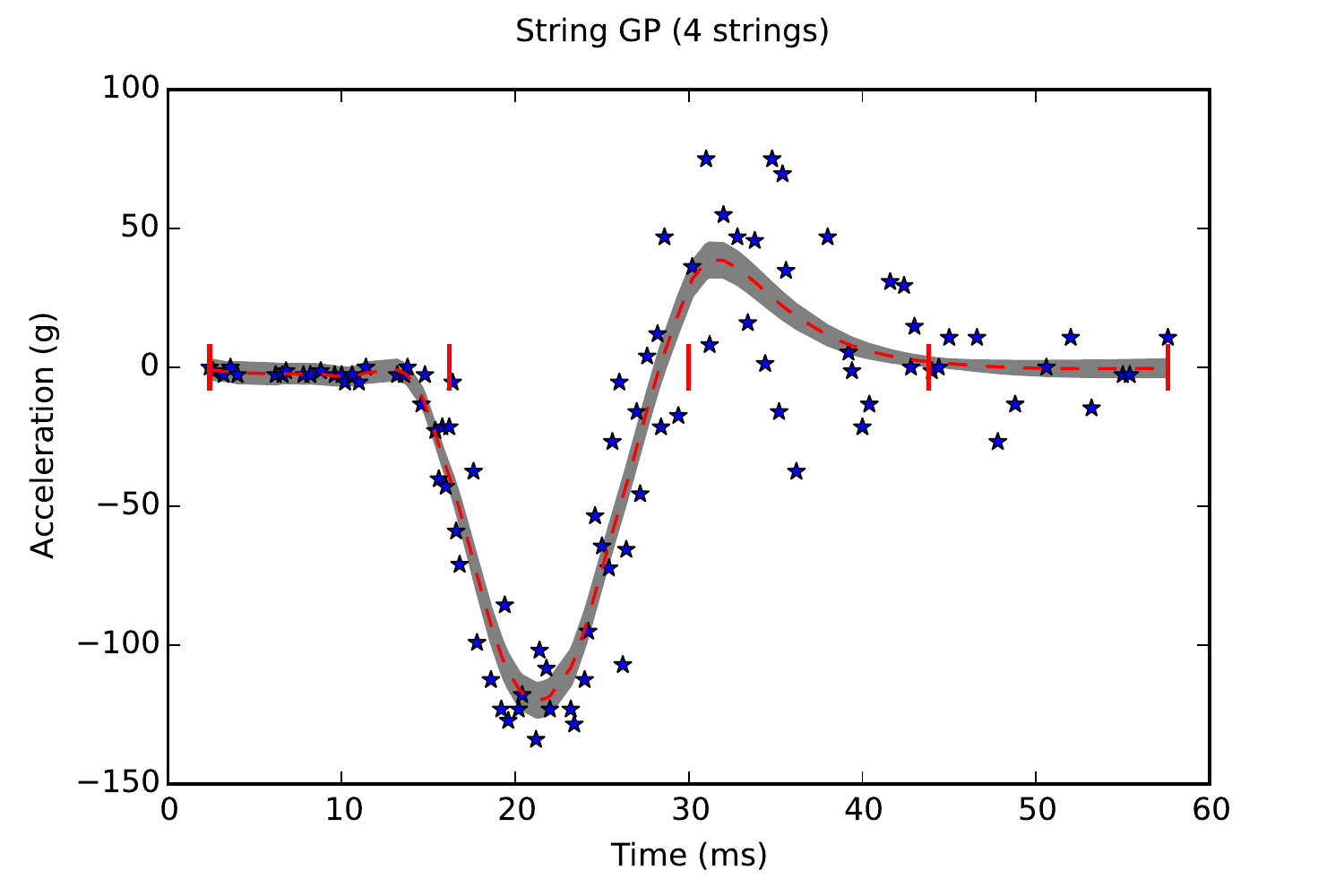}\includegraphics[width=0.43\textwidth]{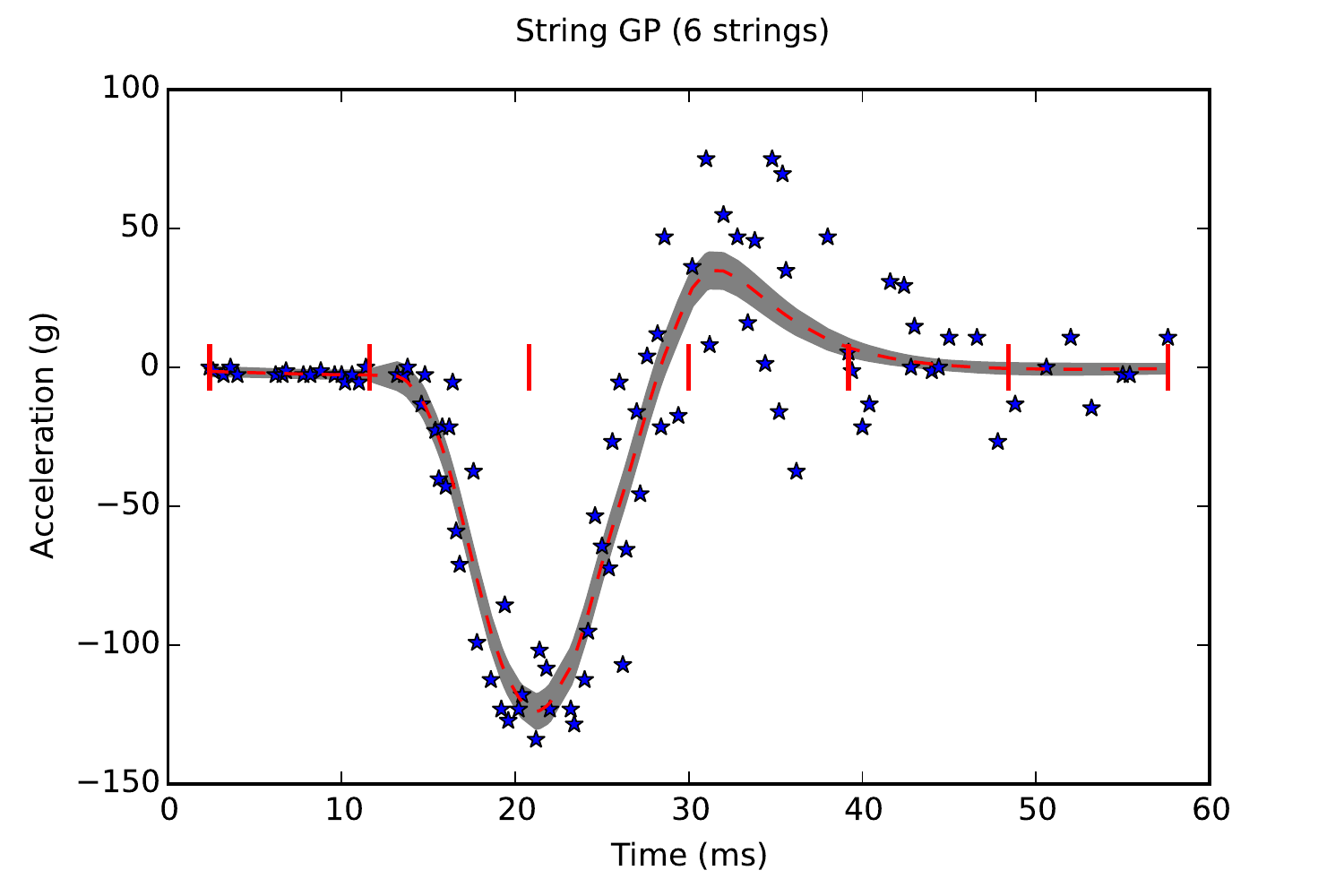}}
\centerline{\includegraphics[width=0.43\textwidth]{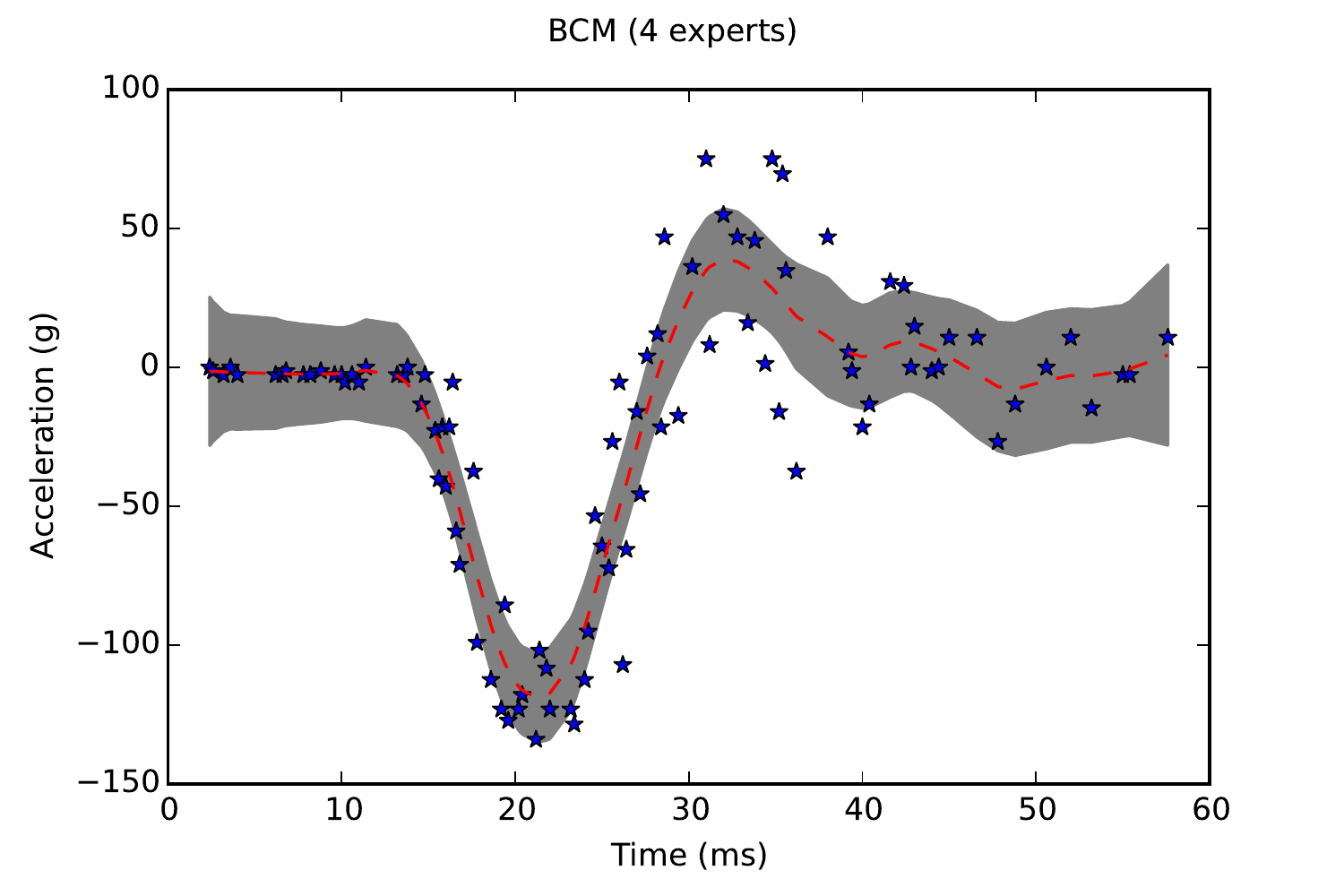}\includegraphics[width=0.43\textwidth]{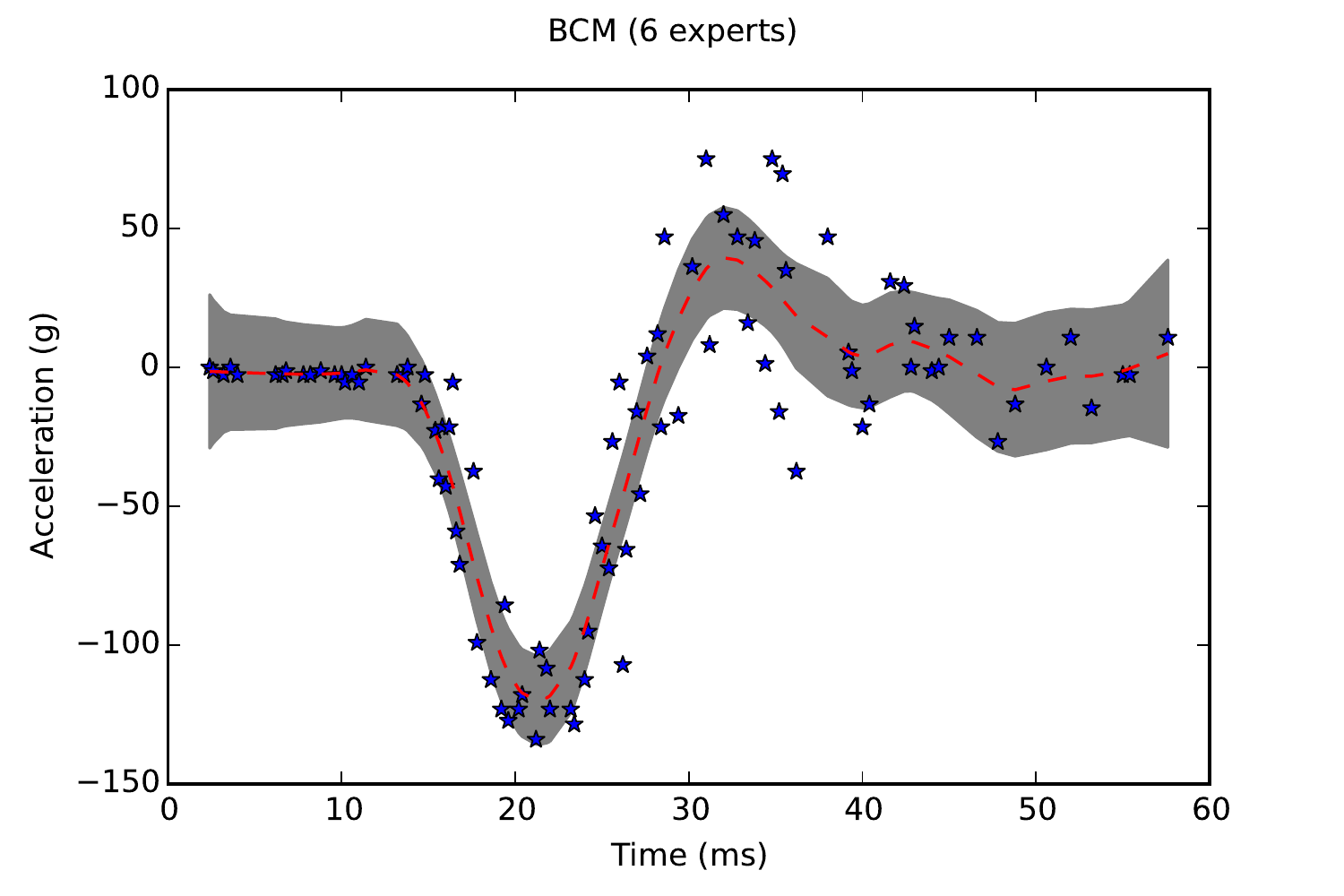}\includegraphics[width=0.43\textwidth]{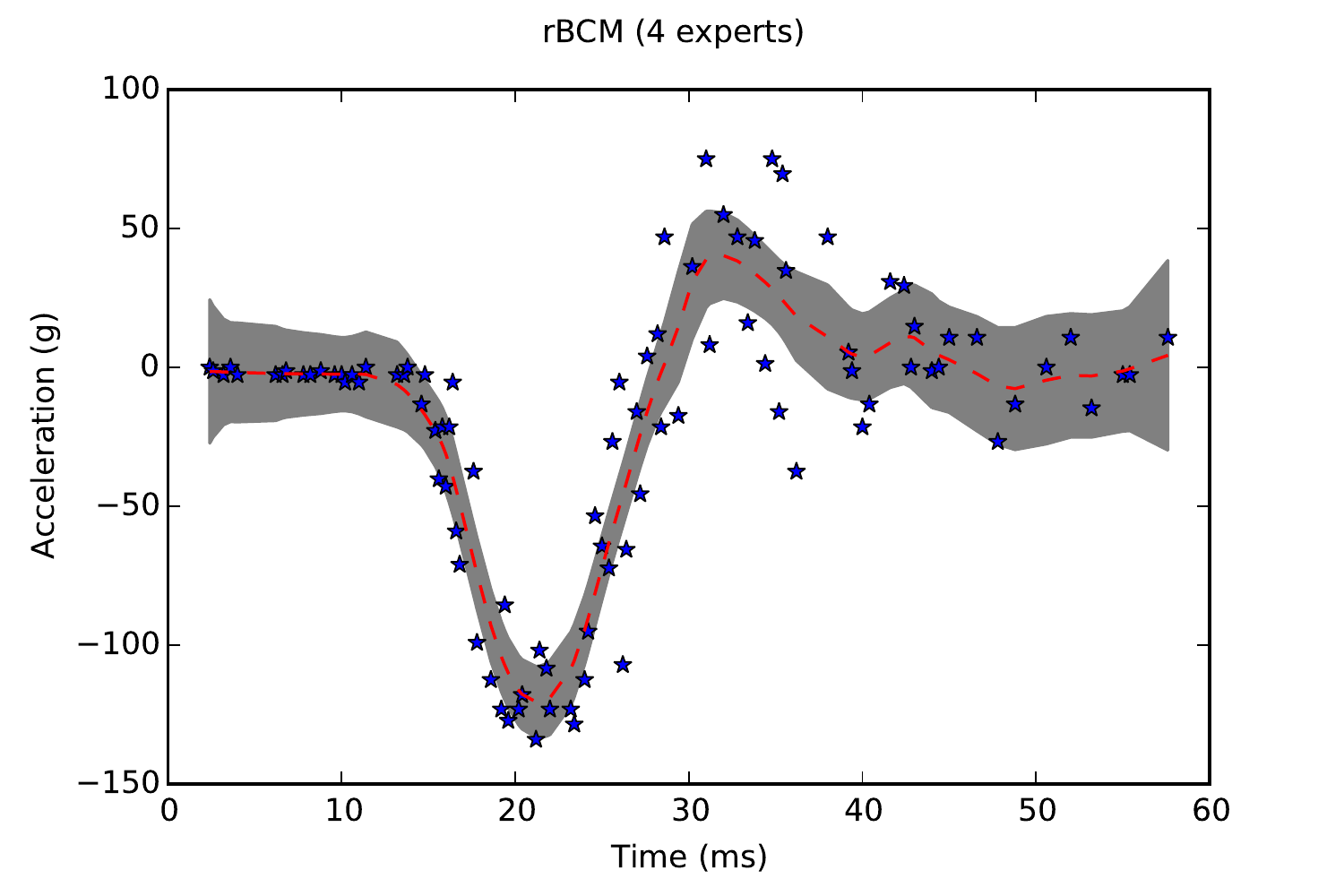}}
\centerline{\includegraphics[width=0.43\textwidth]{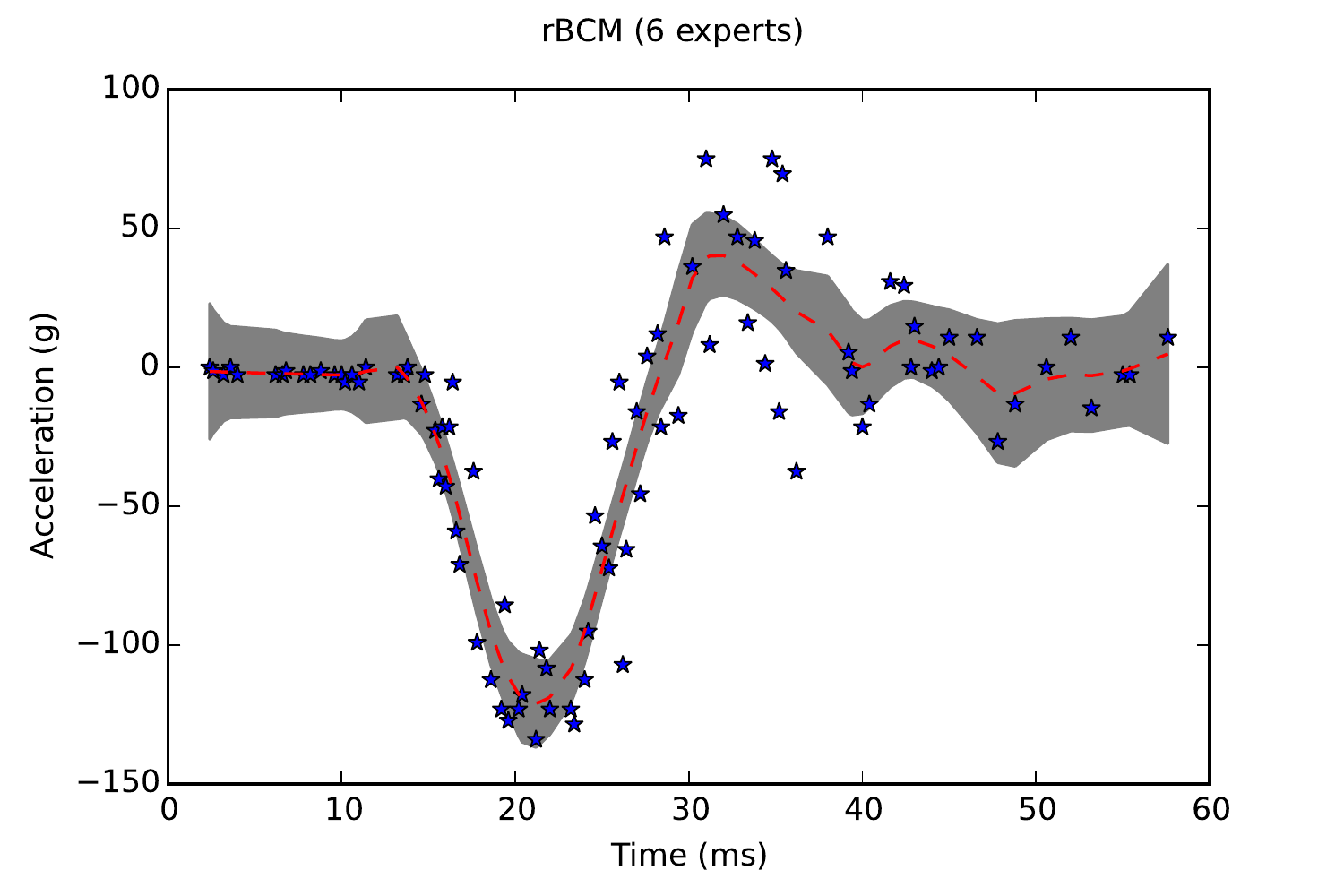}\includegraphics[width=0.43\textwidth]{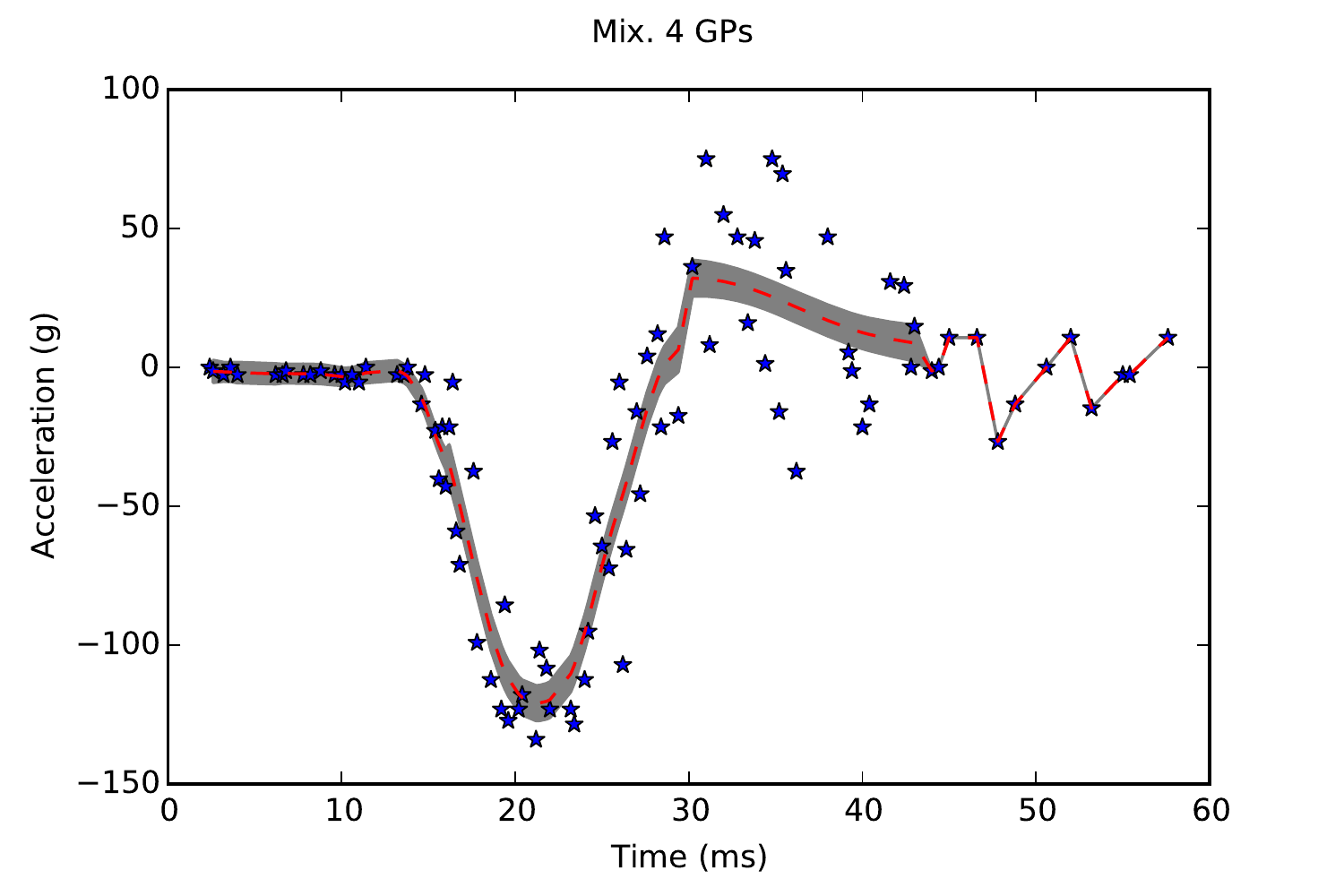}\includegraphics[width=0.43\textwidth]{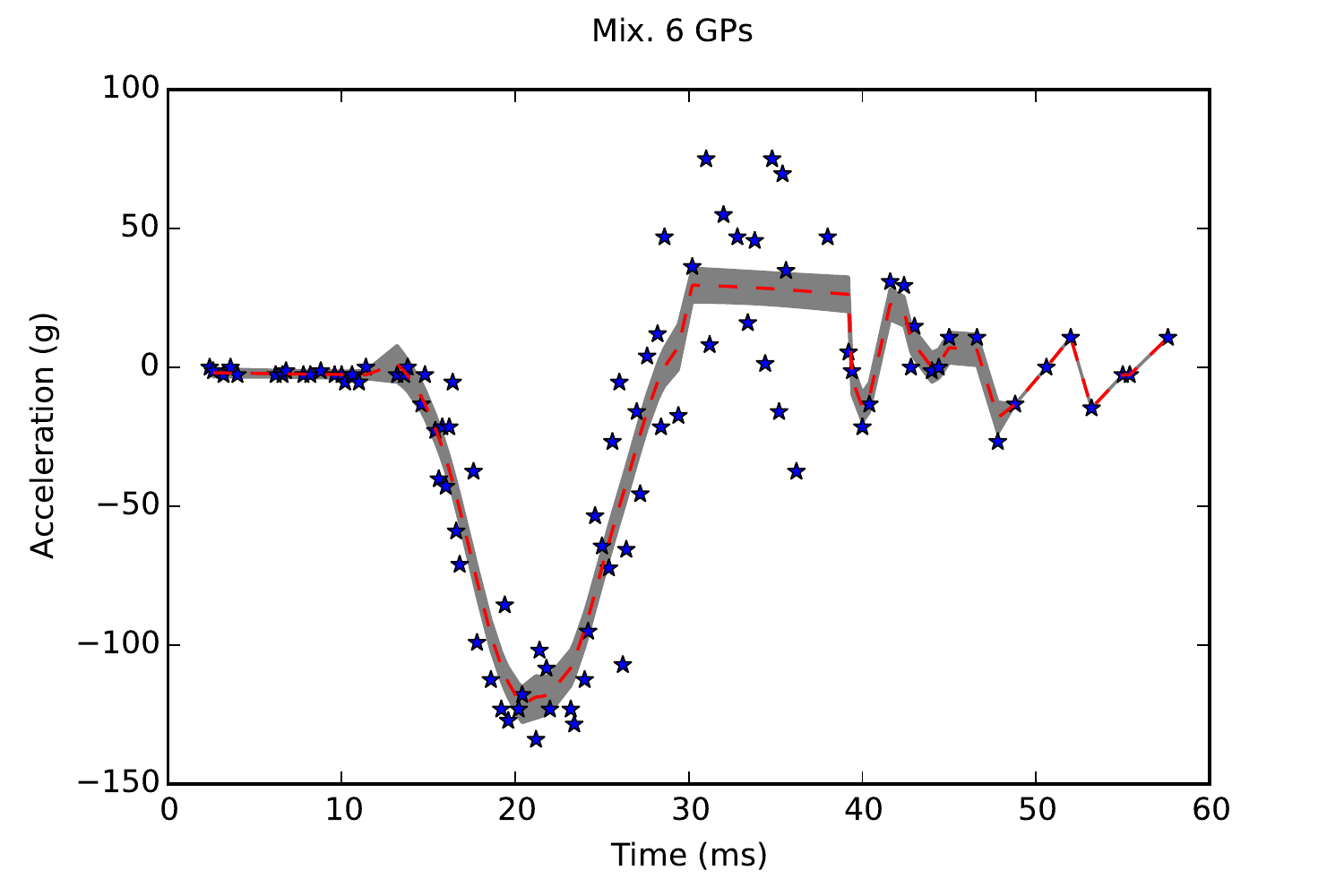}}
\caption{Bayesian nonparametric regressions on the motorcycle data set of \cite{silverman}. Models compared are  \emph{string GP} regression, vanilla GP regression, mixture of independent GP regression experts on a partition of the domain, the Bayesian committee machine (BCM) and the robust Bayesian committee machine (rBCM). Domain partitions were learned during \emph{string GP} maximum likelihood inference (red vertical bars), and reused in other experiments.  Blue stars are noisy samples, red lines are posterior means of the latent function and grey bands correspond to $\pm$ 2 predictive standard deviations of the (noise-free) latent function about its posterior mean.}
\label{fig:motorcycle}
\end{center}
\end{figure}
\end{landscape}

\subsection{Large Scale Regression}
\label{sct:airline}
To illustrate how our approach fares against competing alternatives on a standard large scale problem, we consider predicting arrival delays of commercial flights in the USA in 2008 as studied by \cite{gpbigdatareg}. We choose the same covariates as in \cite{gpbigdatareg}, namely the age of the aircraft (number of years since deployment), distance that needs to be covered, airtime, departure time, arrival time, day of the week, day of the month and month. Unlike \cite{gpbigdatareg} who only considered commercial flights between January 2008 and April 2008, we consider commercial throughout the whole year, for a total of $5.93$ million records. In addition to the whole data set, we also consider subsets so as to empirically illustrate the sensitivity of computational time to the number of samples. Selected subsets consist of $10,000$, $100,000$ and $1,000,000$ records selected uniformly at random. For each data set, we use $2/3$ of the records selected uniformly at random for training and we use the remaining $1/3$ for testing. In order to level the playing field between stationary and nonstationary approaches, we normalize training and testing data sets.\footnote{More precisely, we substract from every feature sample (both in-sample and out-of-sample) the in-sample mean of the feature and we divide the result by the in-sample standard deviation of the feature.} As competing alternatives to \emph{string GPs} we consider the SVIGP of \cite{gpbigdatareg}, the Bayesian committee machines (BCM) of \cite{tresp2000bayesian}, and the robust Bayesian committee machines (rBCM) of \cite{deisenroth2015distributed}. 

As previously discussed the prediction scheme operated by the BCM is Kolmogorov-inconsistent in that the resulting predictive distributions are not consistent by marginalization.\footnote{For instance the predictive distribution of the value of the latent function at a test input $x_1$, namely $f(x_1)$, obtained by using $\{x_1\}$ as set of test inputs in the BCM,  differs from the predictive distribution obtained by using $\{x_1, x_2\}$ as set of test inputs in the BCM and then marginalising with respect to the second input $x_2$.} Moreover, jointly predicting all function values by using the set of all test inputs as query set, as originally suggested in \cite{tresp2000bayesian}, would be impractical in this experiment given that the BCM requires inverting a covariance matrix of the size of the query set which, considering the numbers of test inputs in this experiment (which we recall can be as high as $1.97$ million), would be computationally intractable. To circumvent this problem we use the BCM algorithm to query one test input at a time. This approach is in-line with that adopted by \cite{deisenroth2015distributed}, where the authors did not address determining joint predictive distributions over multiple latent function values. For the BCM and rBCM, the number of experts is chosen so that each expert processes $200$ training points. For SVIGP we use the implementation made available by the \cite{gpy2014}, and we use the same configuration as in \cite{gpbigdatareg}. As for \emph{string GPs}, we use the symmetric sum as link function, and we run two types of experiments, one allowing for inference of change-points (String GP), and the other enforcing a single kernel configuration per input dimension (String GP*). The parameters $\bm{\alpha}$ and $\bm{\beta}$ are chosen so that the prior mean number of change-points in each input dimension is $5$\% of the number of distinct training and testing values in that input dimension, and so that the prior variance of the foregoing number of change-points is $50$ times the prior mean---the aim is to be uninformative about the number of change-points. We run $10,000$ iterations of our RJ-MCMC sampler and discarded the first $5,000$ as `burn-in'. After burn-in we record the states of the Markov chains for analysis using a 1-in-100 thinning rate. Predictive accuracies are reported in Table \ref{table:airline_bench} and CPU time requirements\footnote{We define CPU time as the cumulative CPU clock resource usage of the whole experiment (training and testing), across child processes and threads, and across CPU cores.} are illustrated in Figure \ref{fig:airline_benchmark}. We stress that all experiments were run on a multi-core machine, and that we prefer using the cumulative CPU clock resource as time complexity metric, instead of wall-clock time, so as to be agnostic to the number of CPU cores used in the experiments. This metric has the merit of illustrating how the number of CPU cores required grows as a function of the number of training samples for a fixed/desired wall-clock execution time, but also how the wall-clock execution time grows as a function of the number of training samples for a given number of available CPU cores.

The BCM and the rBCM perform the worst in this experiment both in terms of predictive accuracy (Table \ref{table:airline_bench}) and total CPU time (Figure \ref{fig:airline_benchmark}). The poor scalability of the BCM and the rBCM is primarily due to the testing phase. Indeed, if we denote $M$ the total number of experts, then $M = \lceil \frac{N}{300} \rceil$, as each expert processes $200$ training points, of which there are $\frac{2}{3}N$. In the prediction phase, each expert is required to make predictions about all $\frac{1}{3}N$ test inputs, which requires evaluating $M$ products of an $\frac{1}{3}N \times 200$ matrix with a $200 \times 200$ matrix, which results in a total CPU time requirement that grows in $\mathcal{O}(M\frac{1}{3}N 200^2)$, which is the same as $\mathcal{O}(N^2)$. Given that training CPU time grows linearly in $N$ the cumulative training and testing CPU time grows quadratically in $N$. This is well illustrated in Figure \ref{fig:airline_benchmark}, where it can be seen that the slopes of total CPU time profiles of the BCM and the rBCM in log-log scale are approximately $2$. The airline delays data set was also considered by \cite{deisenroth2015distributed}, but the authors restricted themselves to a fixed size of the test set of $100,000$ points. However, this limitation might be restrictive as in many `smoothing' applications, the test data set can be as large as the training data set---neither the BCM nor the rBCM would be sufficiently scalable in such applications.

As for SVIGP, although it was slightly more accurate than \emph{string GPs} on this data set, it can be noted from Figure \ref{fig:airline_benchmark} that \emph{string GPs} required $10$ times less CPU resources. In fact we were unable to run the experiment on the full data set with SVIGP---we gave up after 500 CPU hours, or more than a couple of weeks wall-clock time given that the GPy implementation of SVIGP makes little use of multiple cores. As a comparison, the full experiment took $91.0$ hours total CPU time ($\approx 15$ hours wall-clock time on our $8$ cores machine) when change-points were inferred and $83.11$ hours total CPU time ($\approx 14$ hours wall-clock time on our $8$ cores machine) when change-points were not inferred. Another advantage of additively separable \emph{string GPs} over GPs, and subsequently over SVIGP, is that they are more interpretable. Indeed, one can determine at a glance from the learned posterior mean \emph{string GPs} of Figure \ref{fig:airline_learned} the effect of each of the $8$ covariates considered on arrival delays. It turns out that the three most informative factors in predicting arrival delays are departure time, distance and arrival time, while the age of the aircraft, the day of the week and the day of the month seem to have little to no effect. Finally, posterior distributions of the number of change-points are illustrated in Figure \ref{fig:airline_n_cp}, and posterior distributions of the locations of change-points are illustrated in Figure \ref{fig:airline_pos_cp}.

% N -> CPU time plots 
% N -> MAE table
% Posterior change-point counts histogram
% Posterior change-points positions
%\begin{landscape}
\begin{table}
\centering
\begin{tabular}{l*{6}{c}r}
N             		& String GP & String GP* & BCM & rBCM & SVIGP  \\\midrule
$10,000$ 		& $1.03 \pm 0.10$ & $1.06 \pm 0.10$ & $1.06 \pm 0.10$  & $1.06 \pm 0.10$  & $\bm{0.90 \pm 0.09}$   \\
$100,000$          	& $0.93 \pm 0.03$ & $0.96 \pm 0.03$ & $1.66 \pm 0.03$ & $1.04 \pm 0.04$ &  $\bm{0.88 \pm 0.03}$  \\
$1,000,000$        	& $0.93 \pm 0.01$ & $0.92 \pm 0.01$ & N/A & N/A &  $\bm{0.82 \pm 0.01}$  \\
$5,929,413$      	& $0.90 \pm 0.01$ & $0.93 \pm 0.01$ & N/A & N/A &  N/A  \\
%\bottomrule
\end{tabular}
\caption{Predictive mean squared errors (MSEs) $\pm$ one standard error on the airline arrival delays experiment. Squared errors are expressed as fraction of the sample variance of airline arrival delays, and hence are unitless. With this normalisation, a MSE of $1.00$ is as good as using the training mean arrival delays as predictor. The * in String GP* indicates that inference was performed without allowing for change-points. N/A entries correspond to experiments that were not over after $500$ CPU hours.}
\label{table:airline_bench}
\end{table}
%\end{landscape}

\begin{figure}[p]
\begin{center}
\centerline{\includegraphics[width=0.75\textwidth]{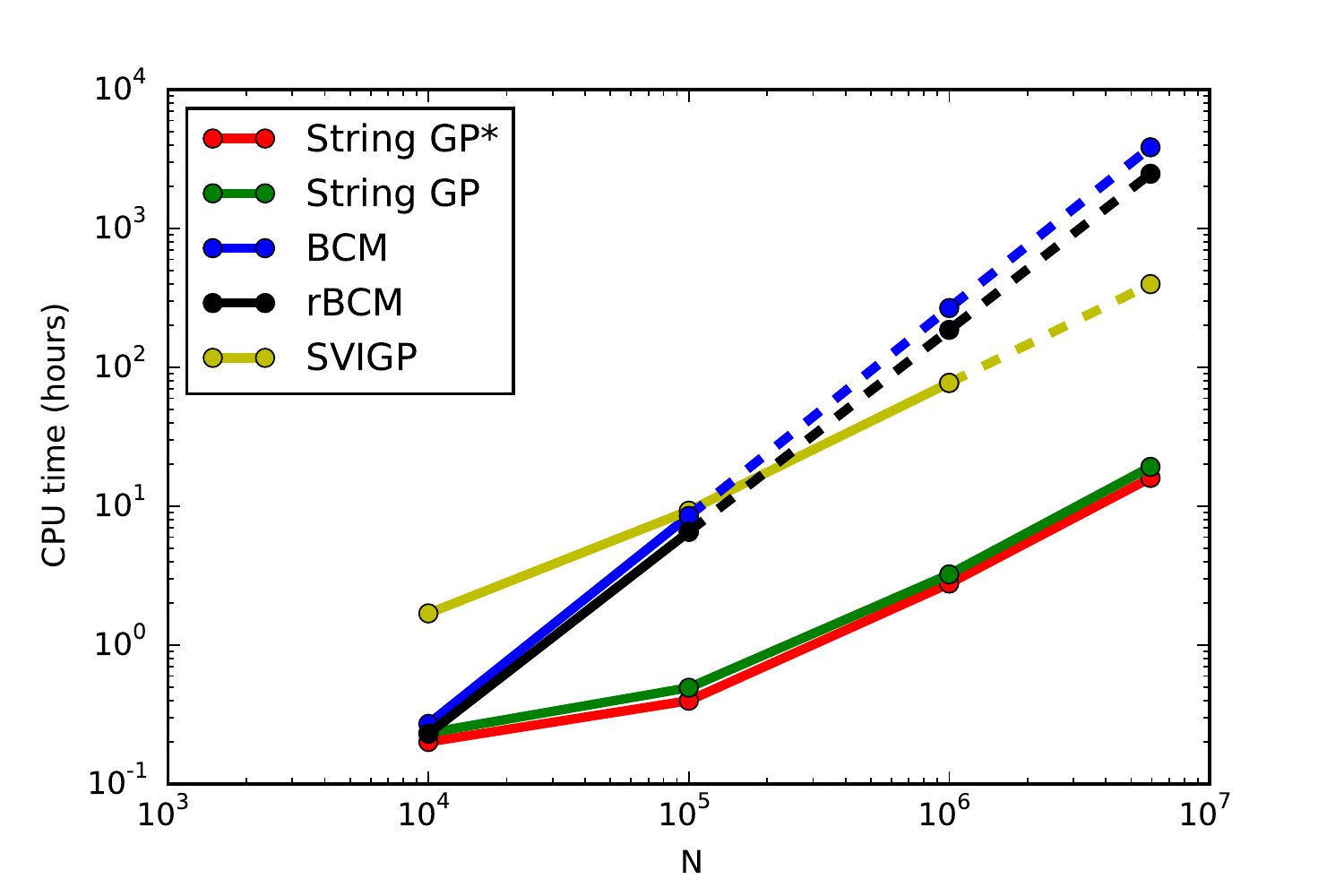}}
\caption{Total CPU time (training and testing) taken by various regression approaches on the airline delays data set as a function of the size of the subset considered, in log-log scale. The experimental setup is described in Section \ref{sct:airline}.  The CPU time reflects actual CPU clock resource usage in each experiment, and is therefore agnostic to the number of CPU cores used. It can be regarded as the wall-clock time the experiment would have taken to complete on a single-core computer (with the same CPU frequency). Dashed lines are extrapolated values, and correspond to experiments that did not complete after $500$ hours of CPU time.}
\label{fig:airline_benchmark}
\end{center}
\end{figure}

\begin{figure}[p]
\begin{center}
\centerline{\includegraphics[width=0.5\textwidth]{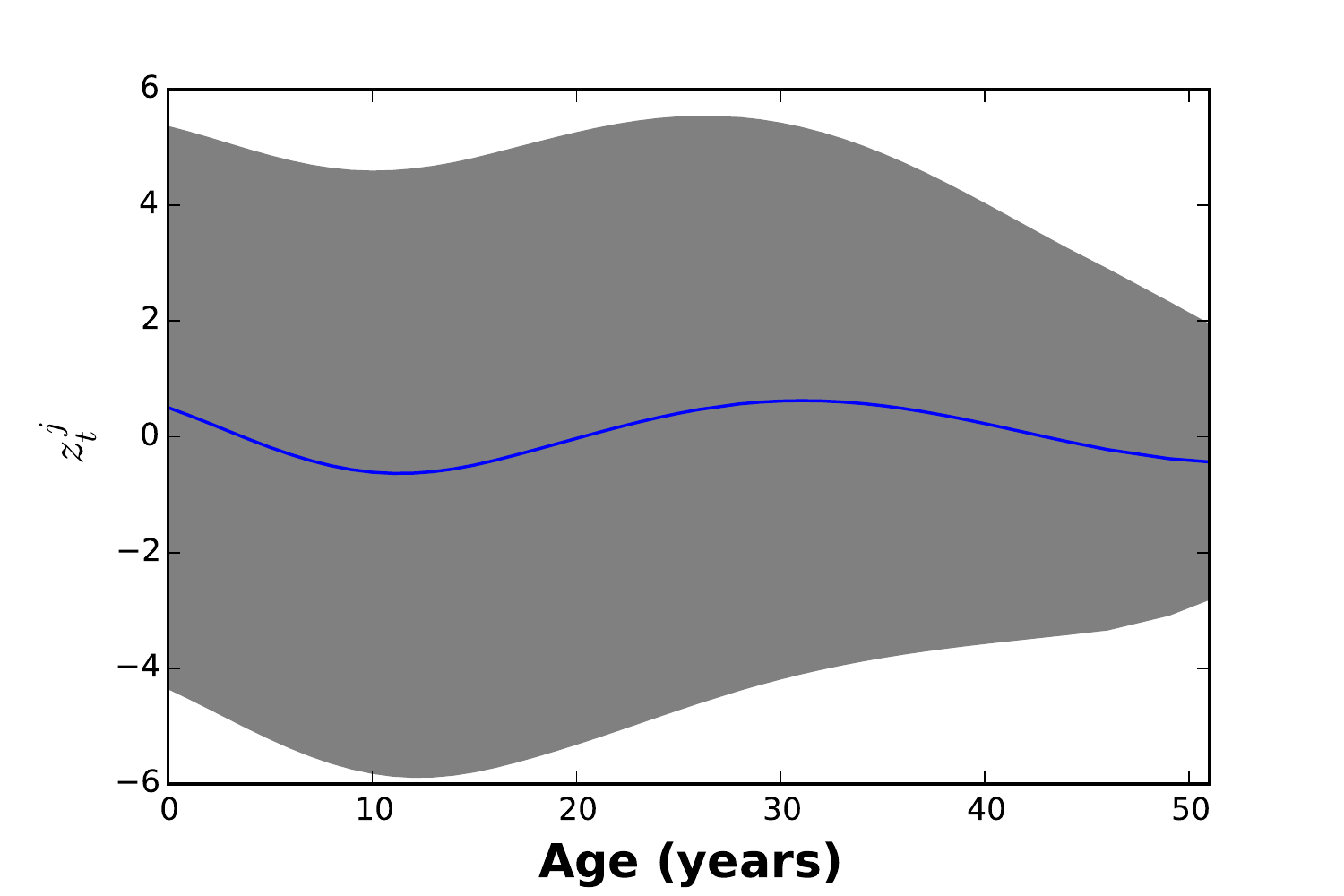}\includegraphics[width=0.5\textwidth]{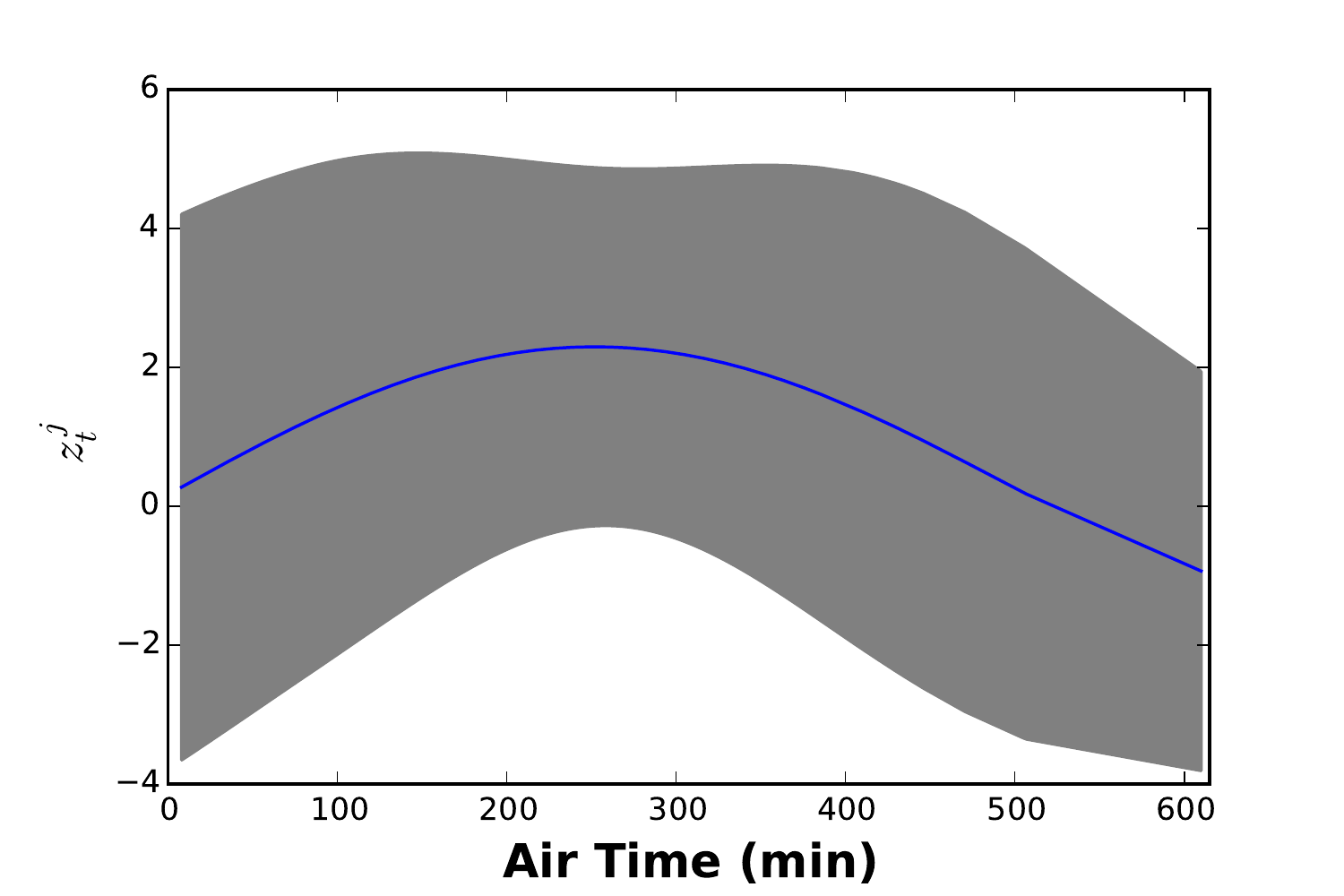}}
\centerline{\includegraphics[width=0.5\textwidth]{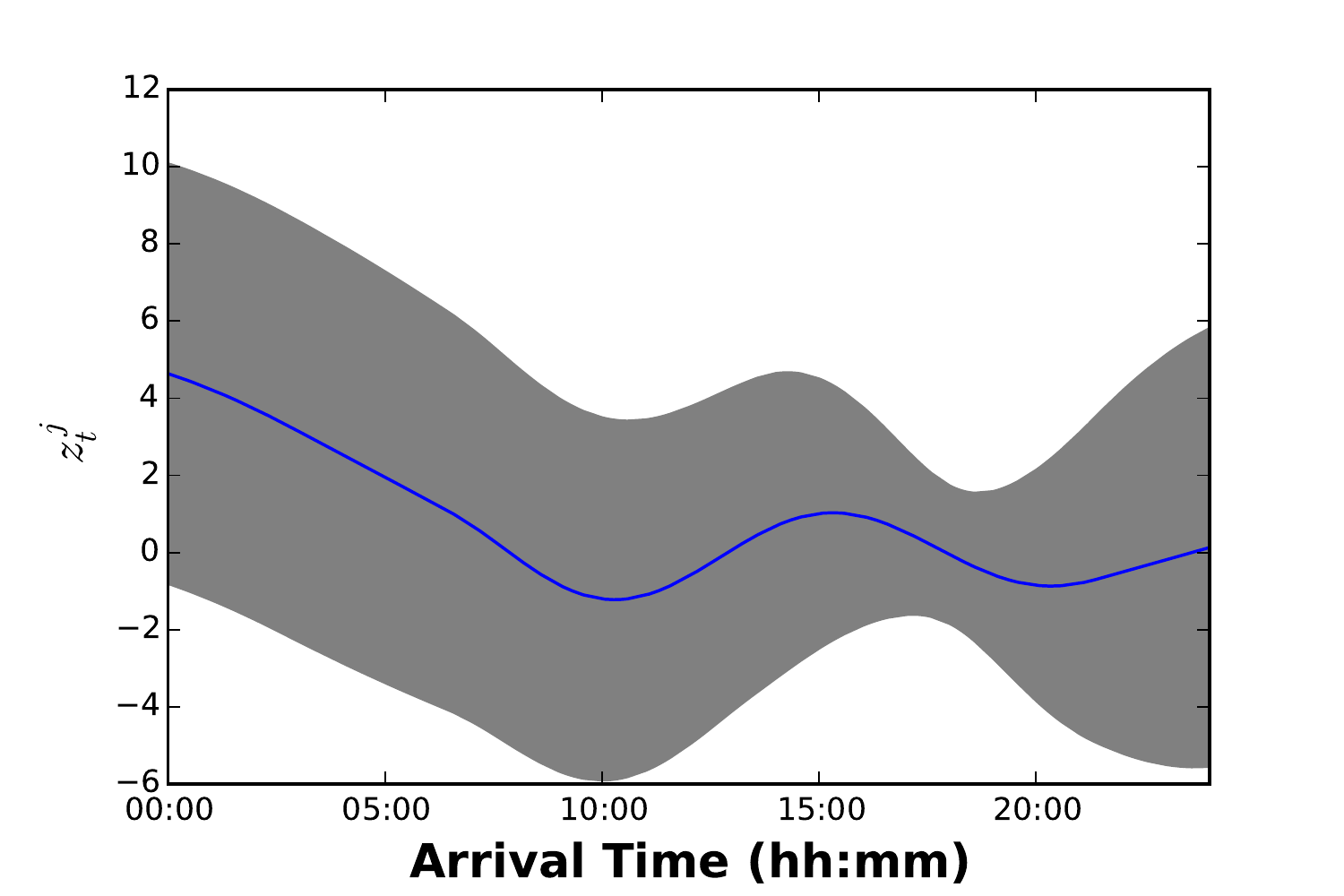}\includegraphics[width=0.5\textwidth]{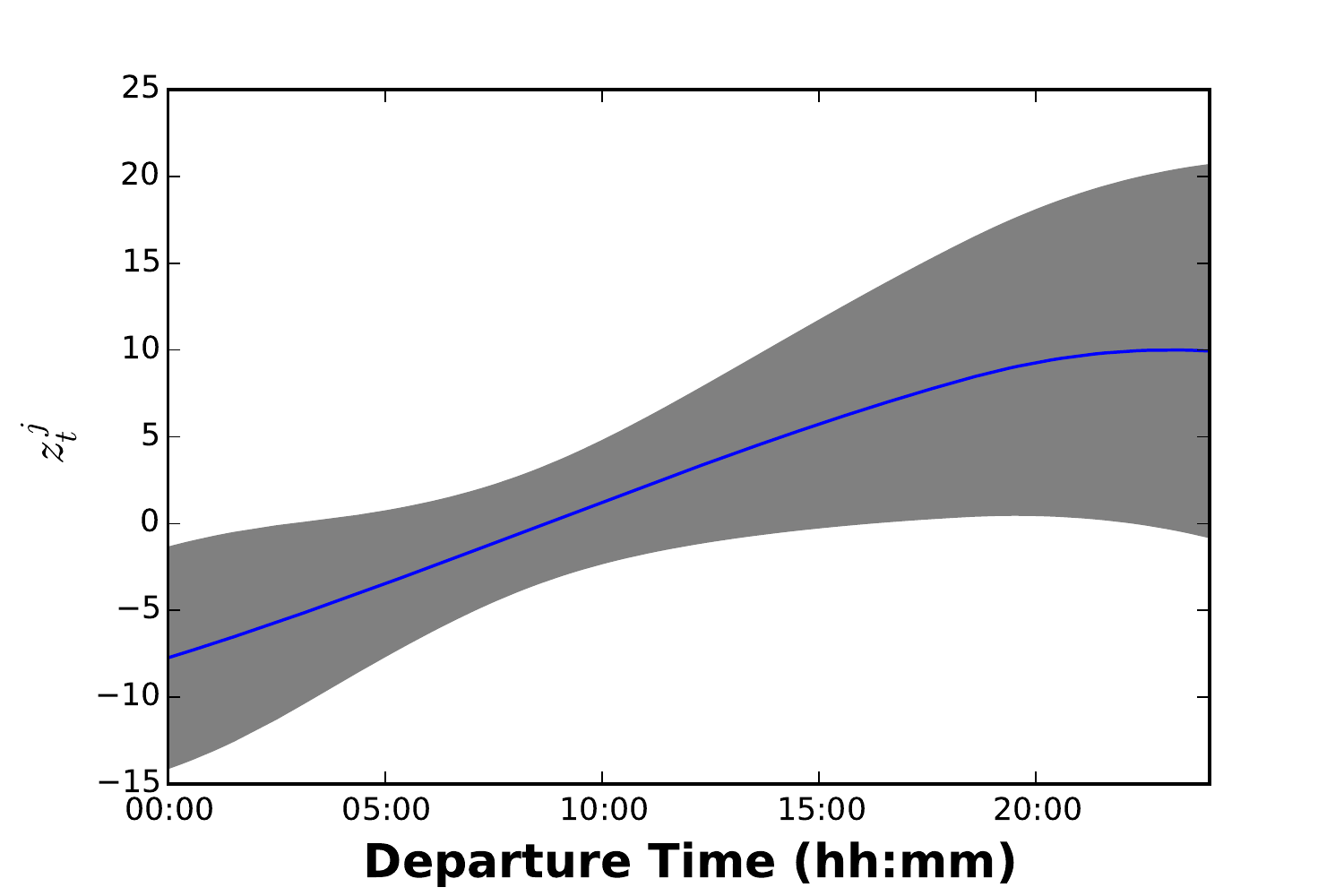}}
\centerline{\includegraphics[width=0.5\textwidth]{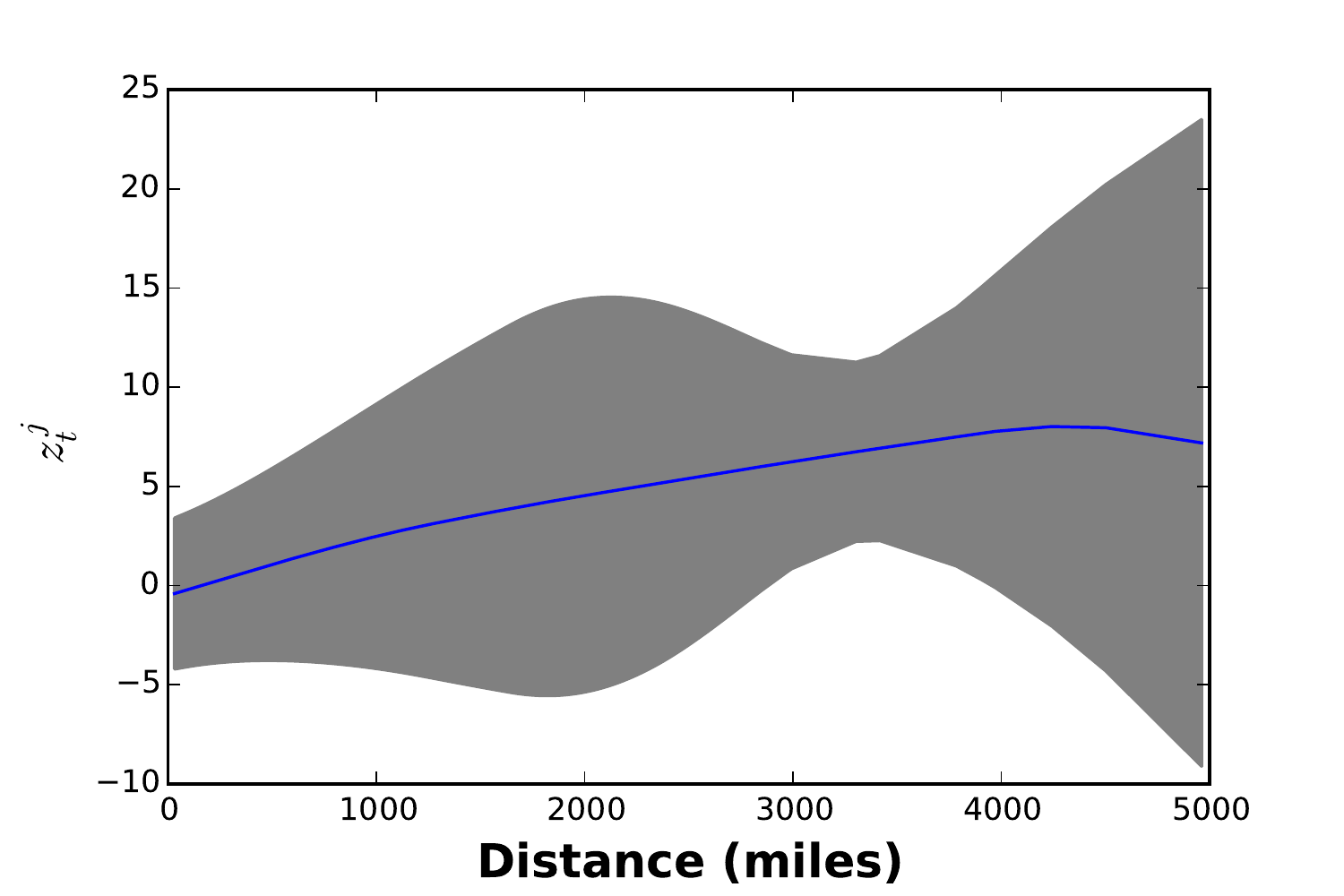}\includegraphics[width=0.5\textwidth]{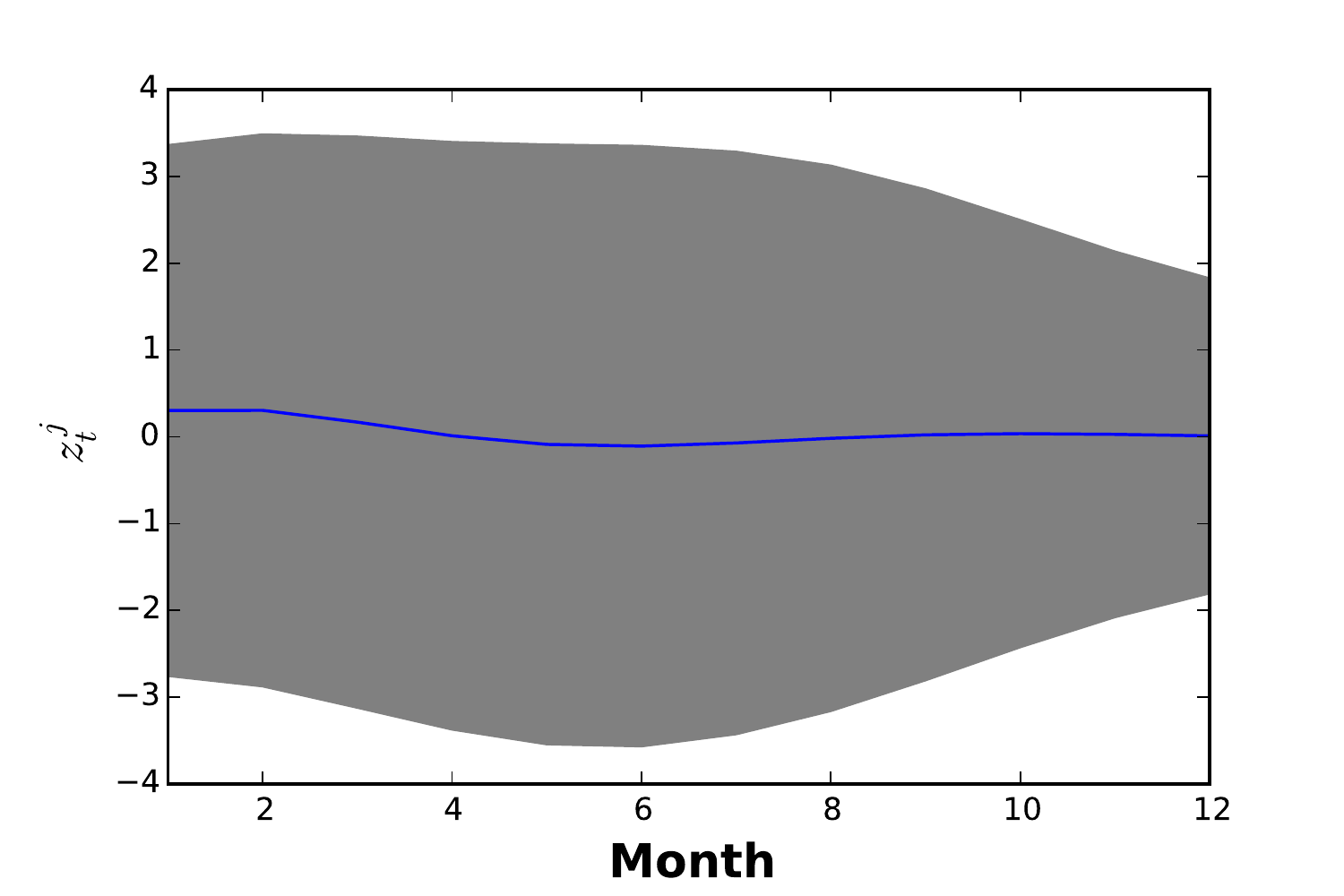}}
\centerline{\includegraphics[width=0.5\textwidth]{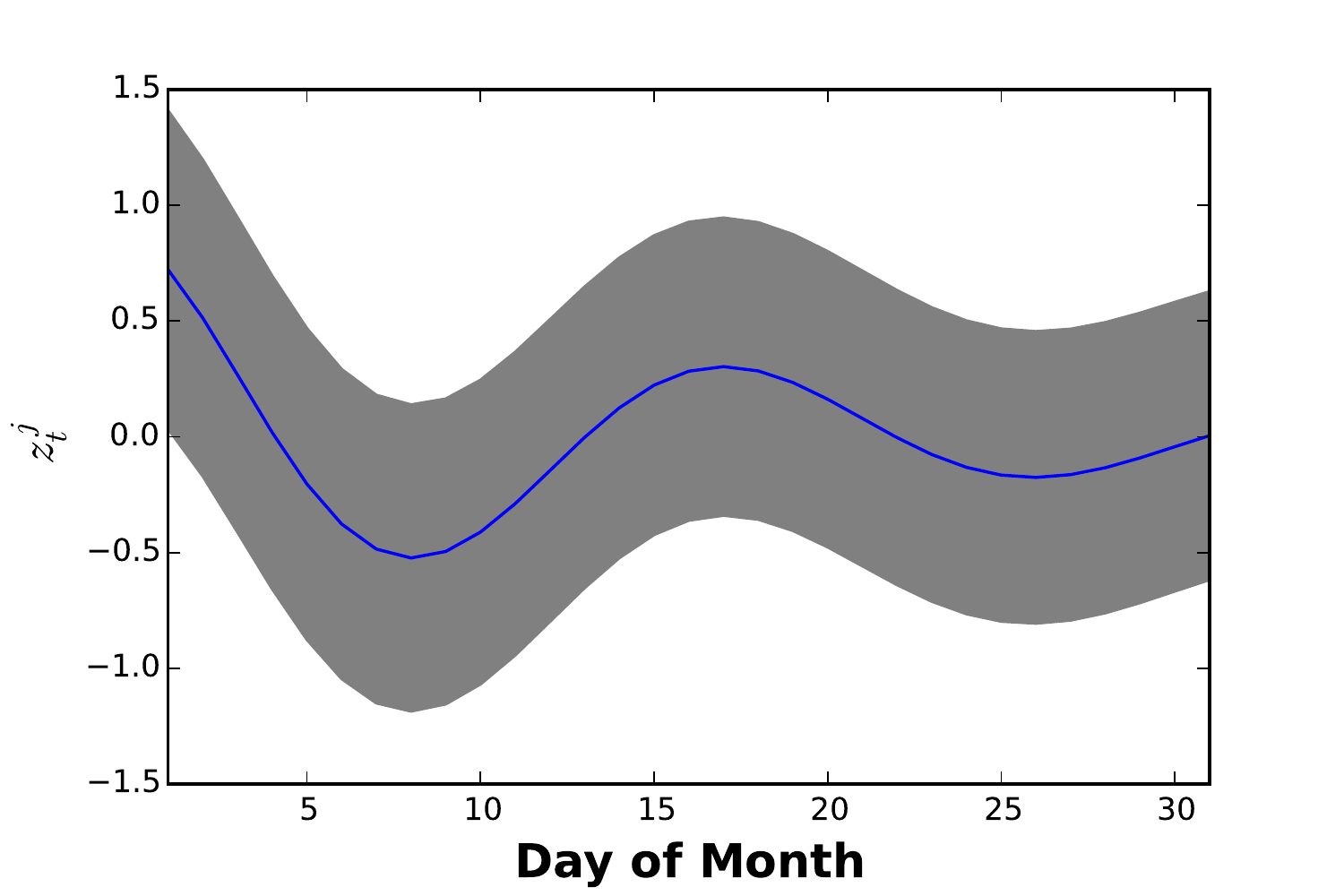}\includegraphics[width=0.5\textwidth]{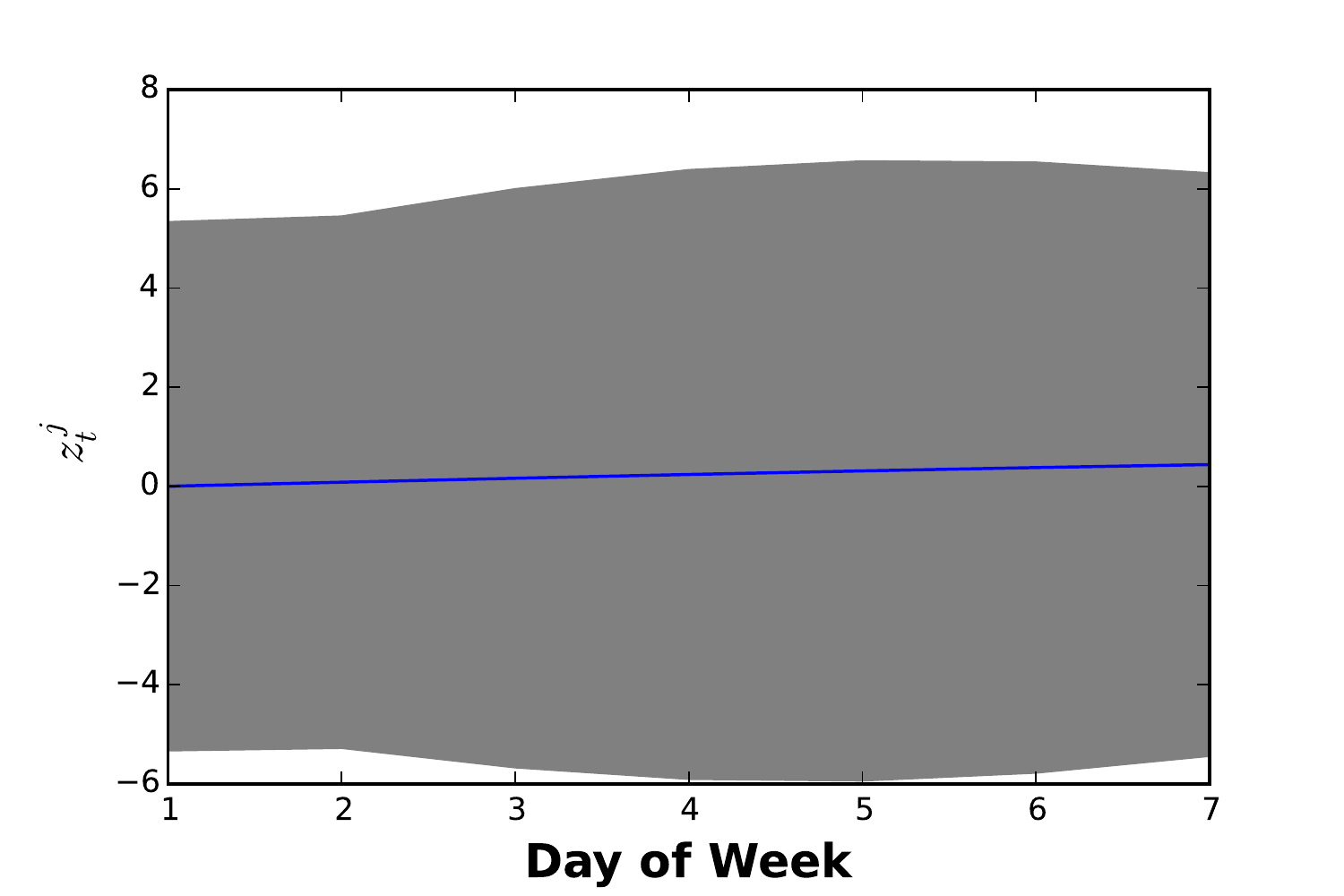}}
\caption{Posterior mean $\pm$ one posterior standard deviation of univariate \emph{string GPs} in the airline delays experiment of Section \ref{sct:airline}. Change-points were automatically inferred in this experiment.}
\label{fig:airline_learned}
\end{center}
\end{figure}

\begin{figure}[p]
\begin{center}
\centerline{\includegraphics[width=0.5\textwidth]{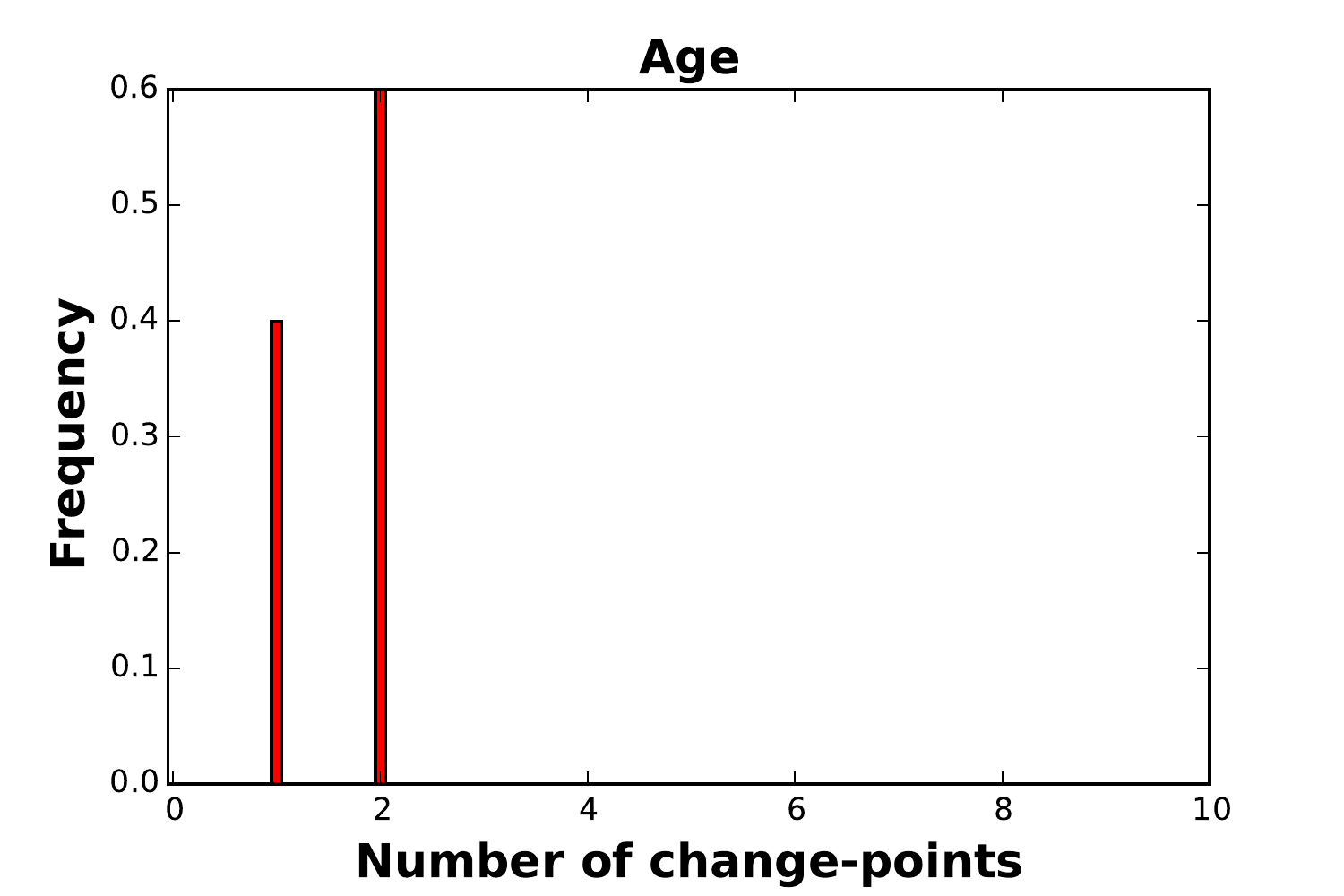}\includegraphics[width=0.5\textwidth]{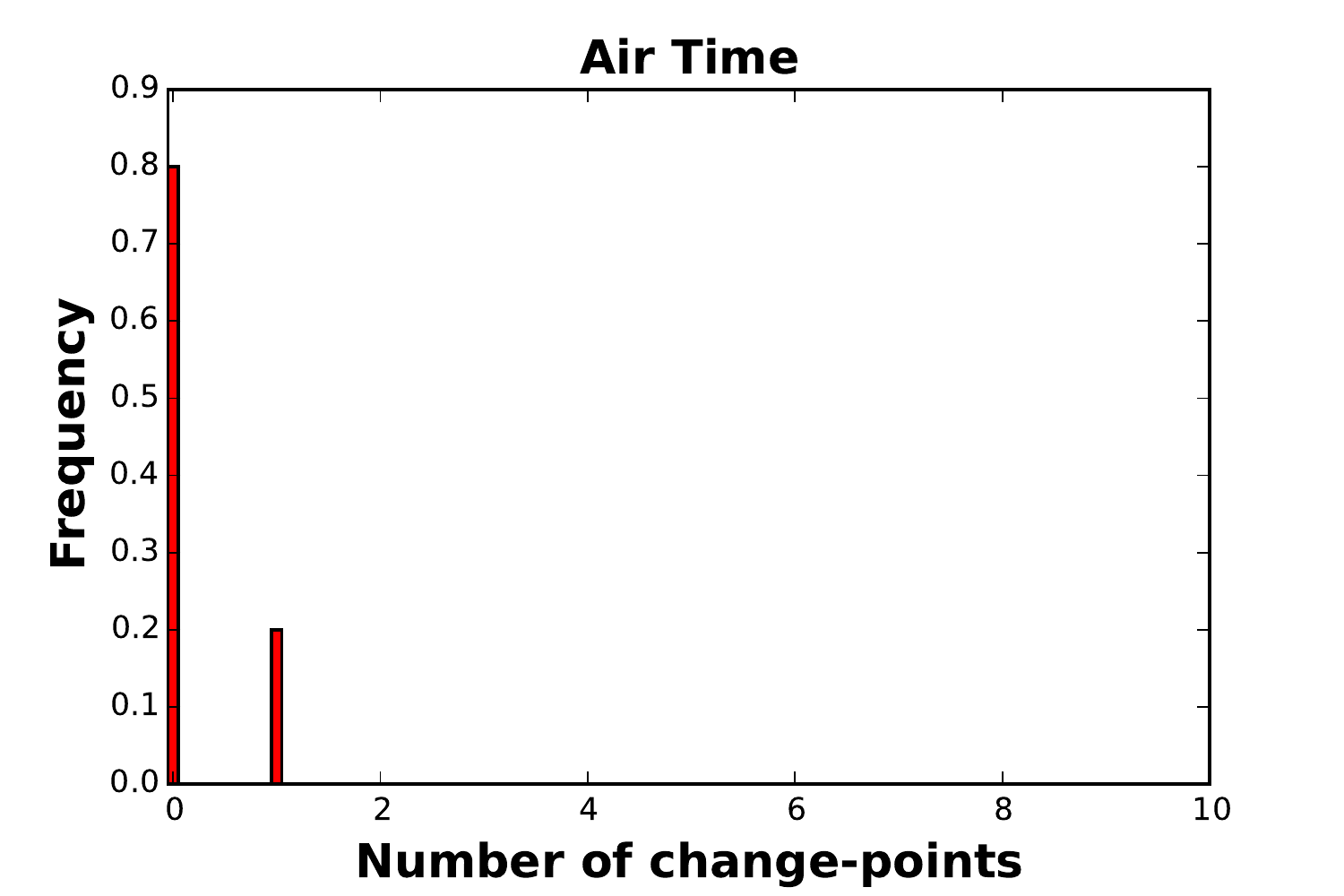}}
\centerline{\includegraphics[width=0.5\textwidth]{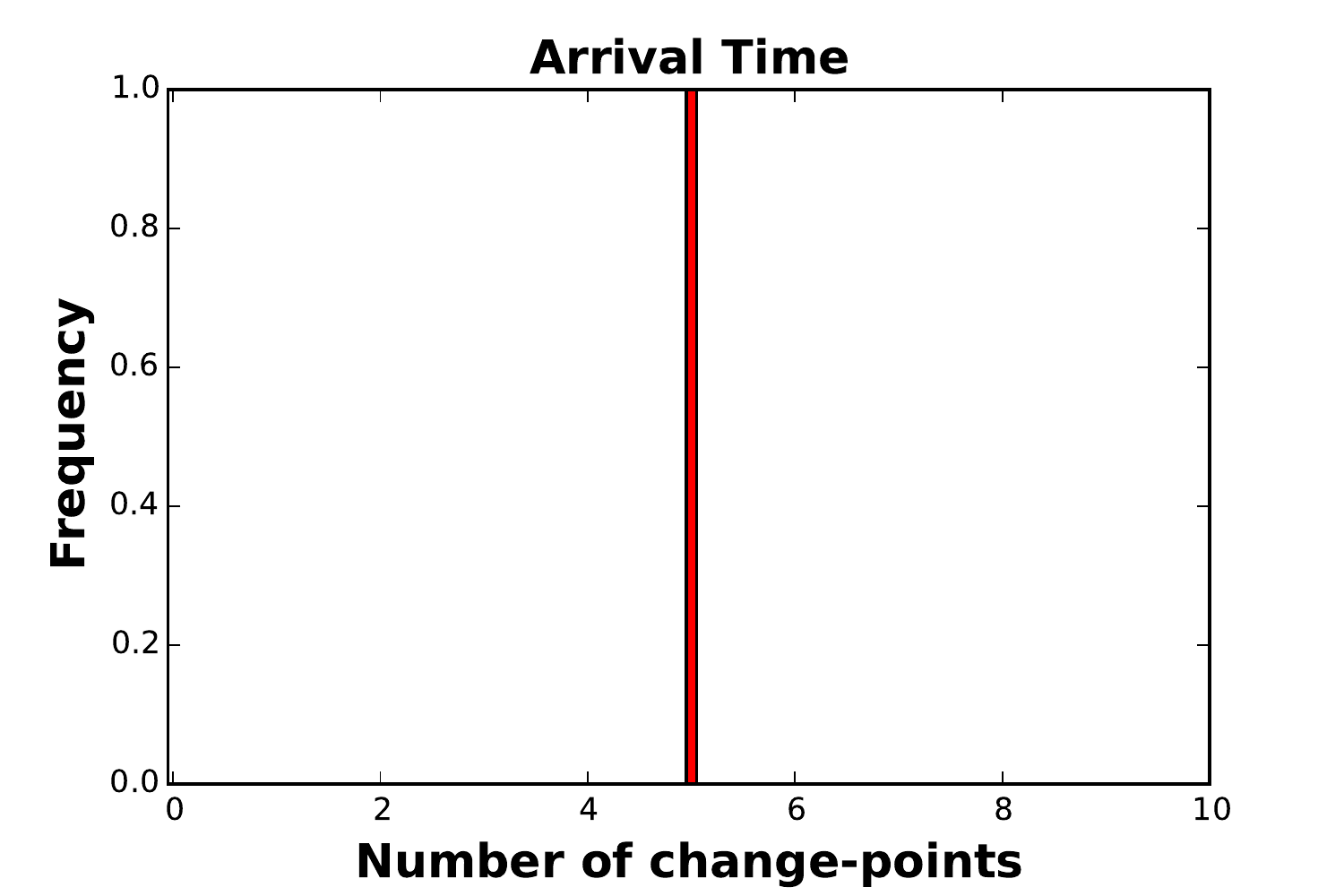}\includegraphics[width=0.5\textwidth]{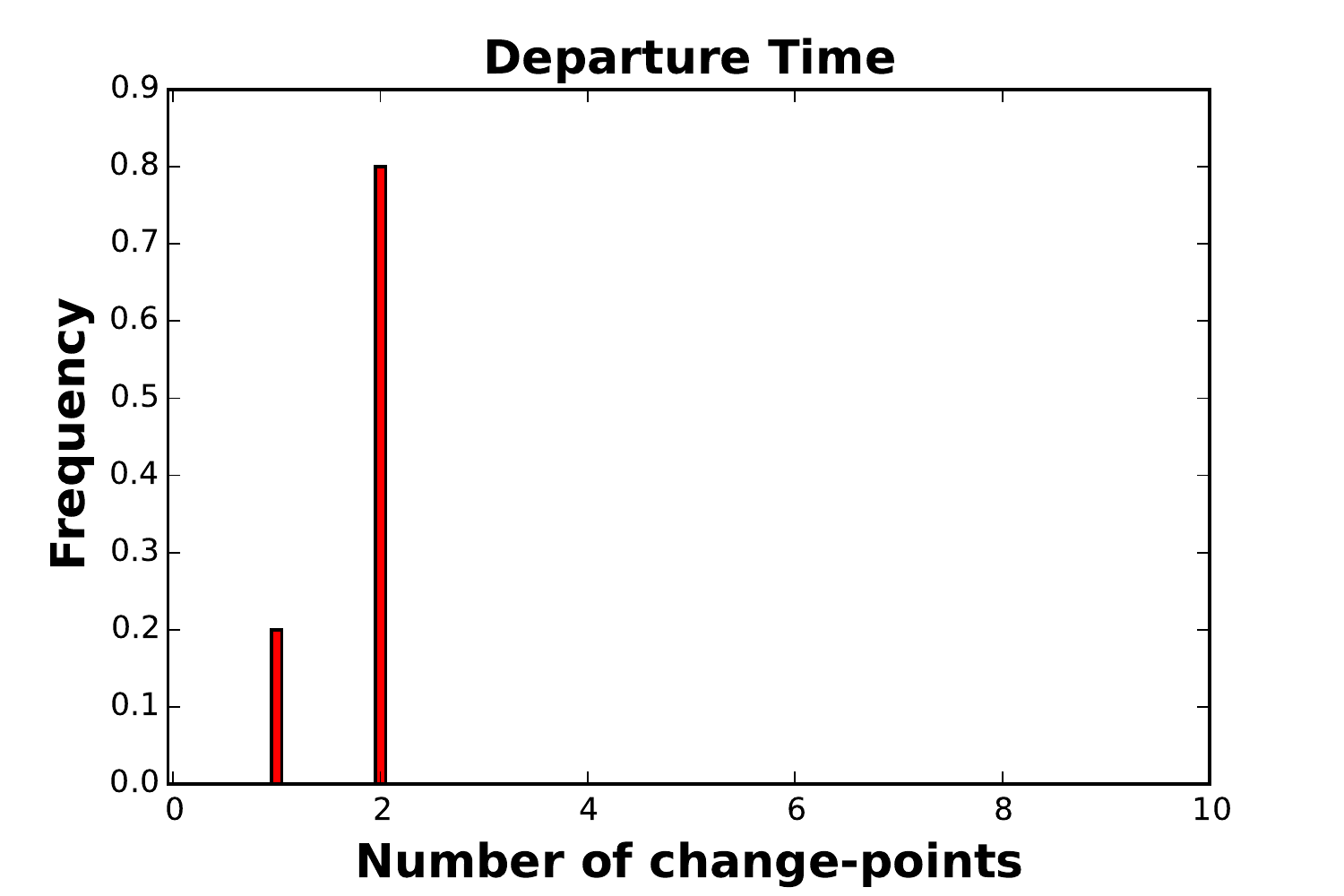}}
\centerline{\includegraphics[width=0.5\textwidth]{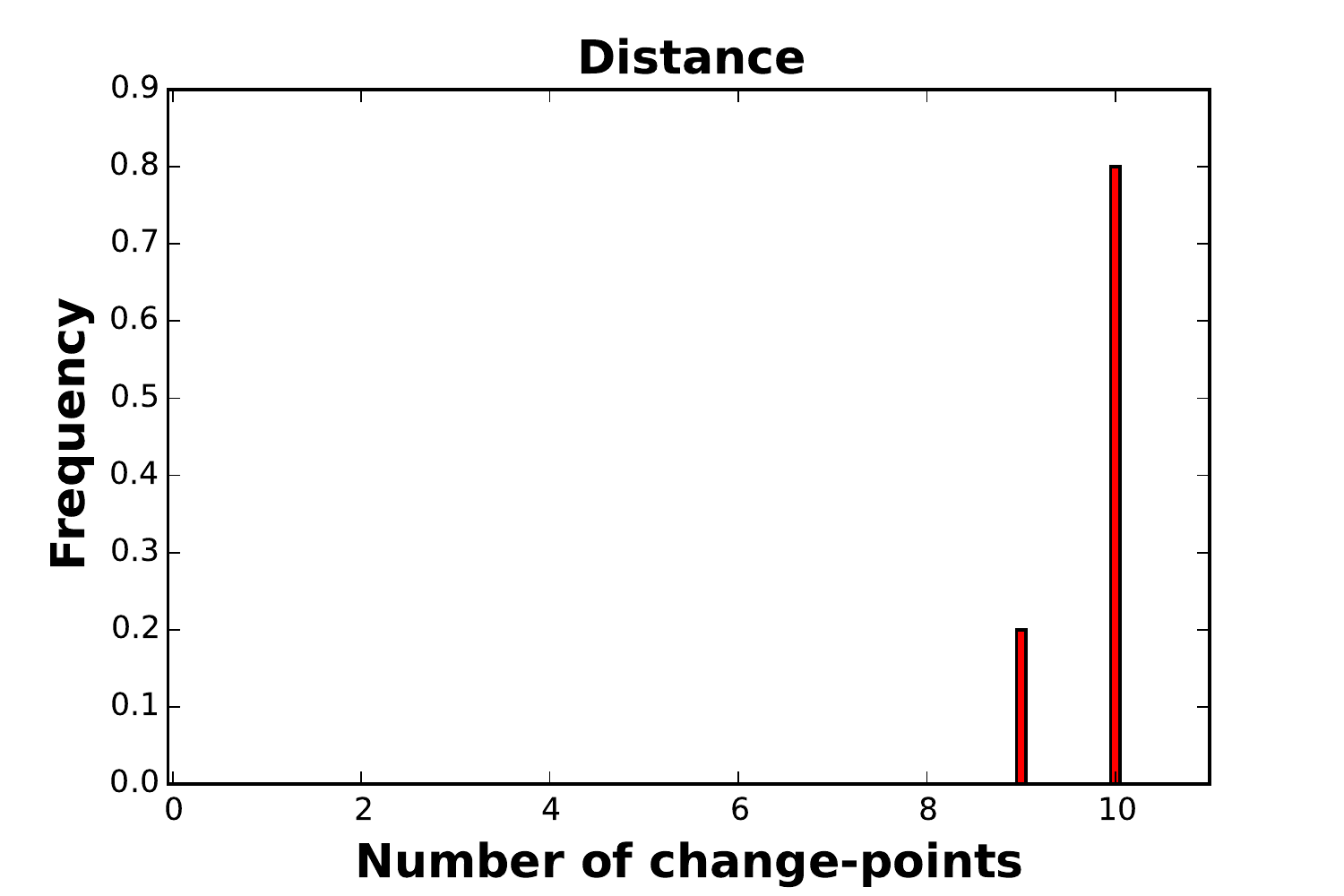}\includegraphics[width=0.5\textwidth]{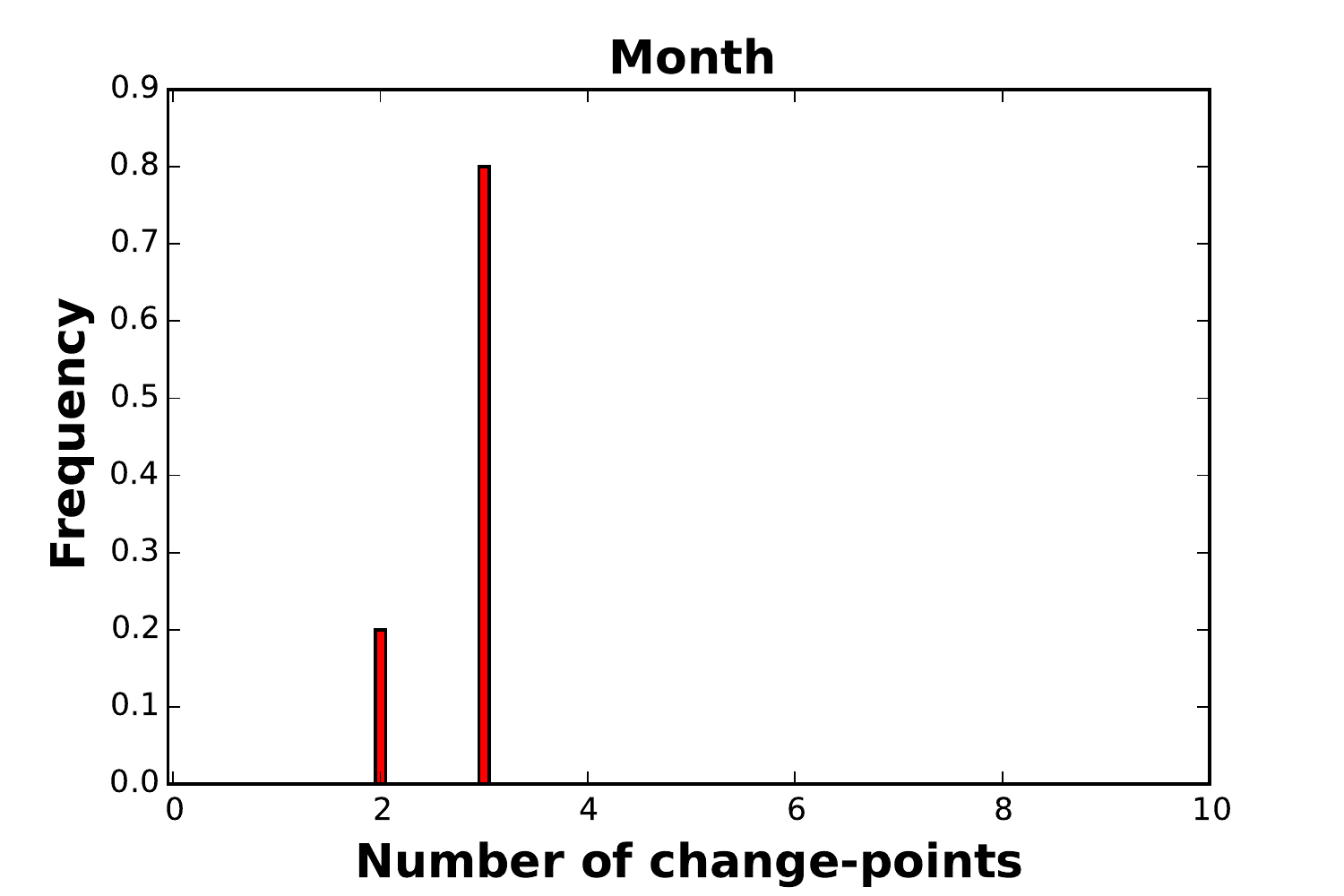}}
\centerline{\includegraphics[width=0.5\textwidth]{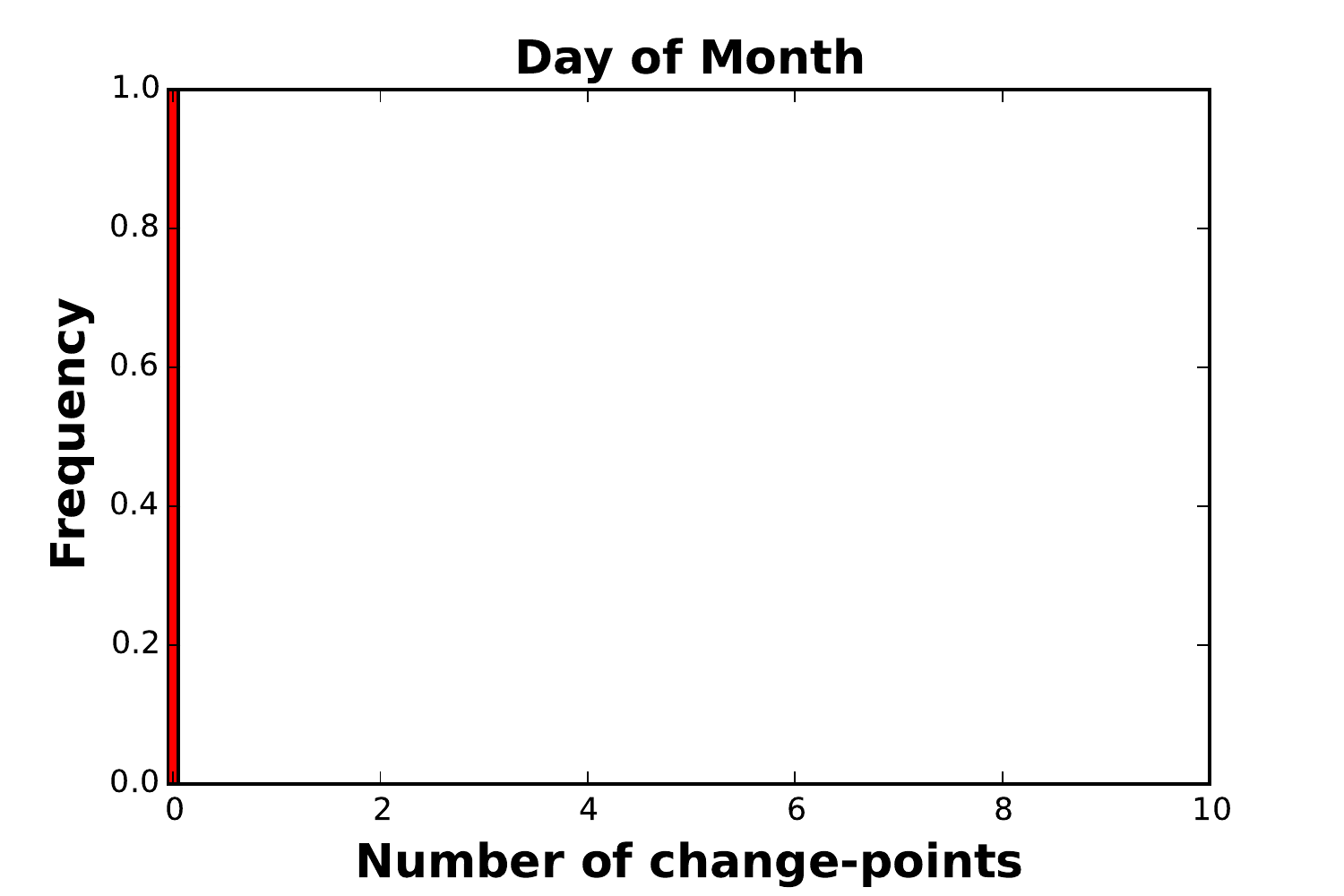}\includegraphics[width=0.5\textwidth]{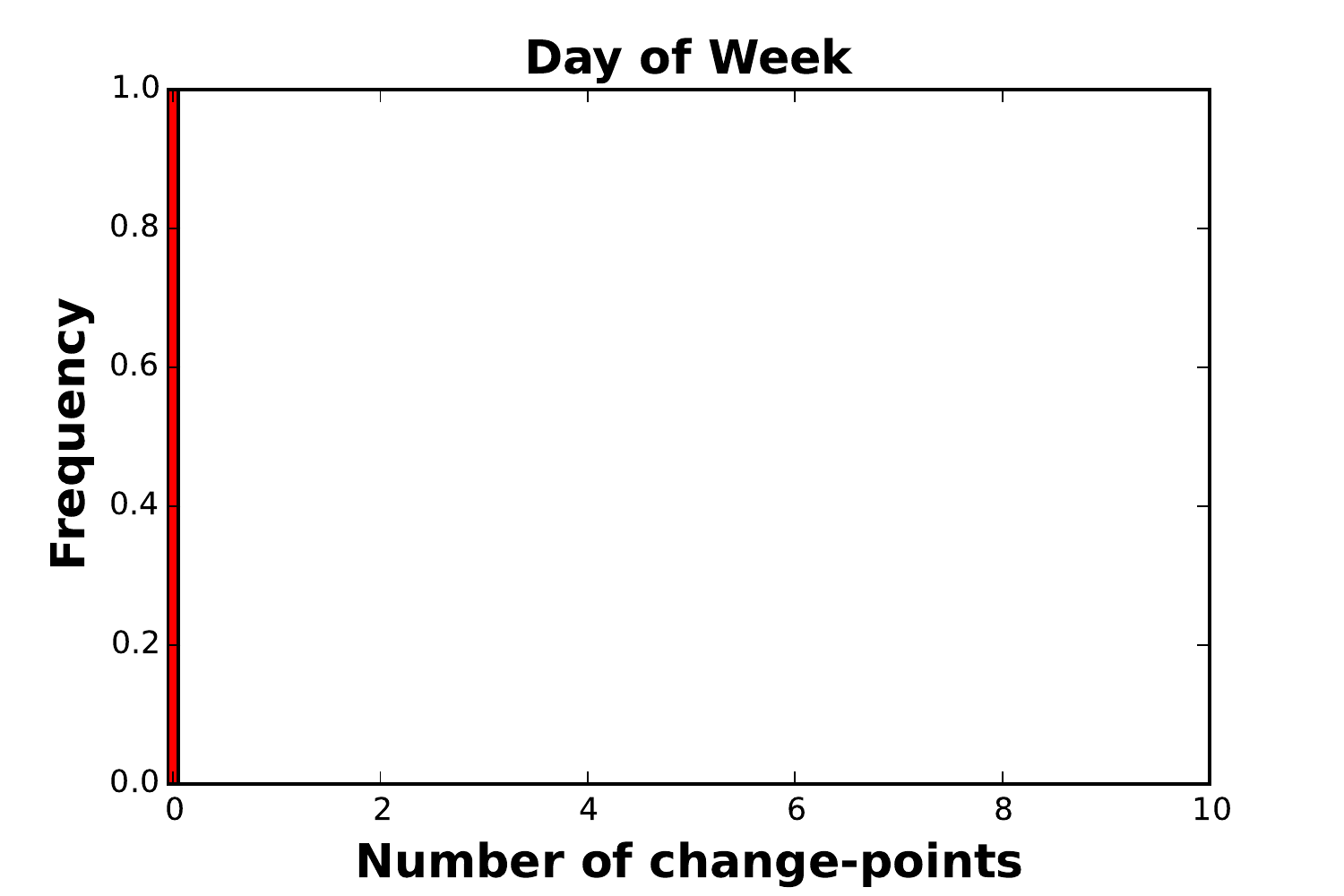}}
\caption{Posterior distributions of the numbers of change-points in each input dimension in the airline delays experiment of Section \ref{sct:airline}.}
\label{fig:airline_n_cp}
\end{center}
\end{figure}

\begin{figure}[p]
\begin{center}
\centerline{\includegraphics[width=0.5\textwidth]{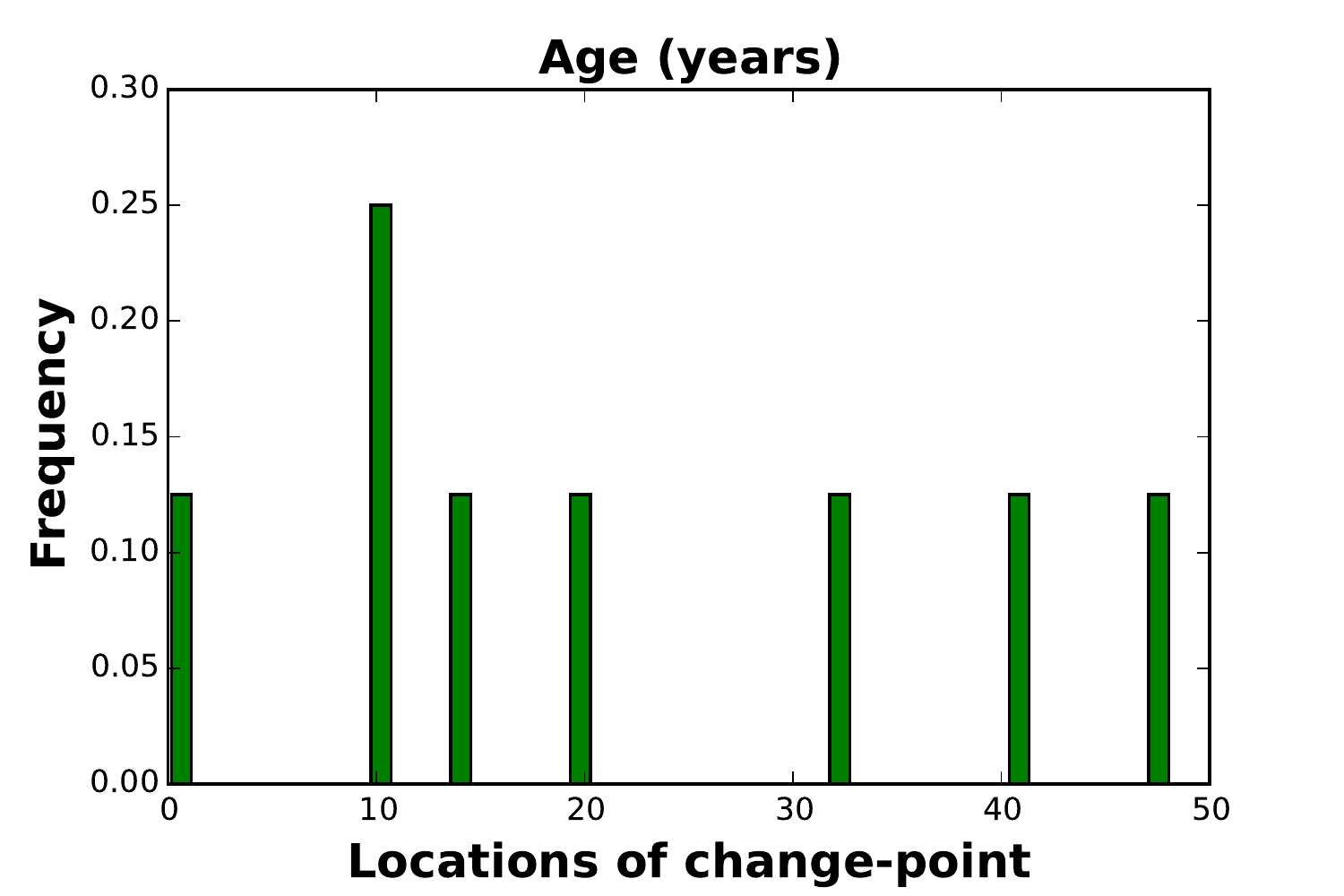}\includegraphics[width=0.5\textwidth]{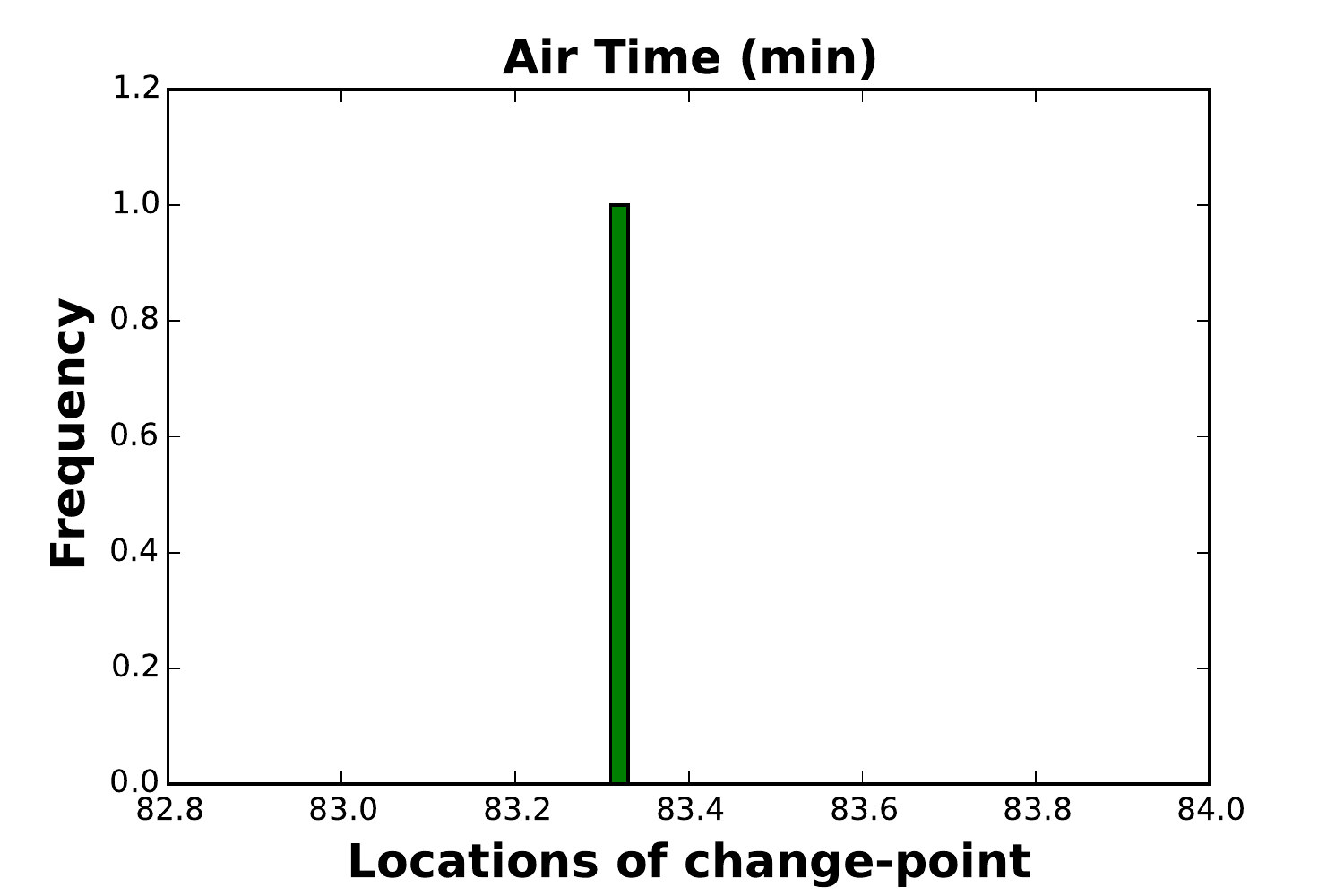}}
\centerline{\includegraphics[width=0.5\textwidth]{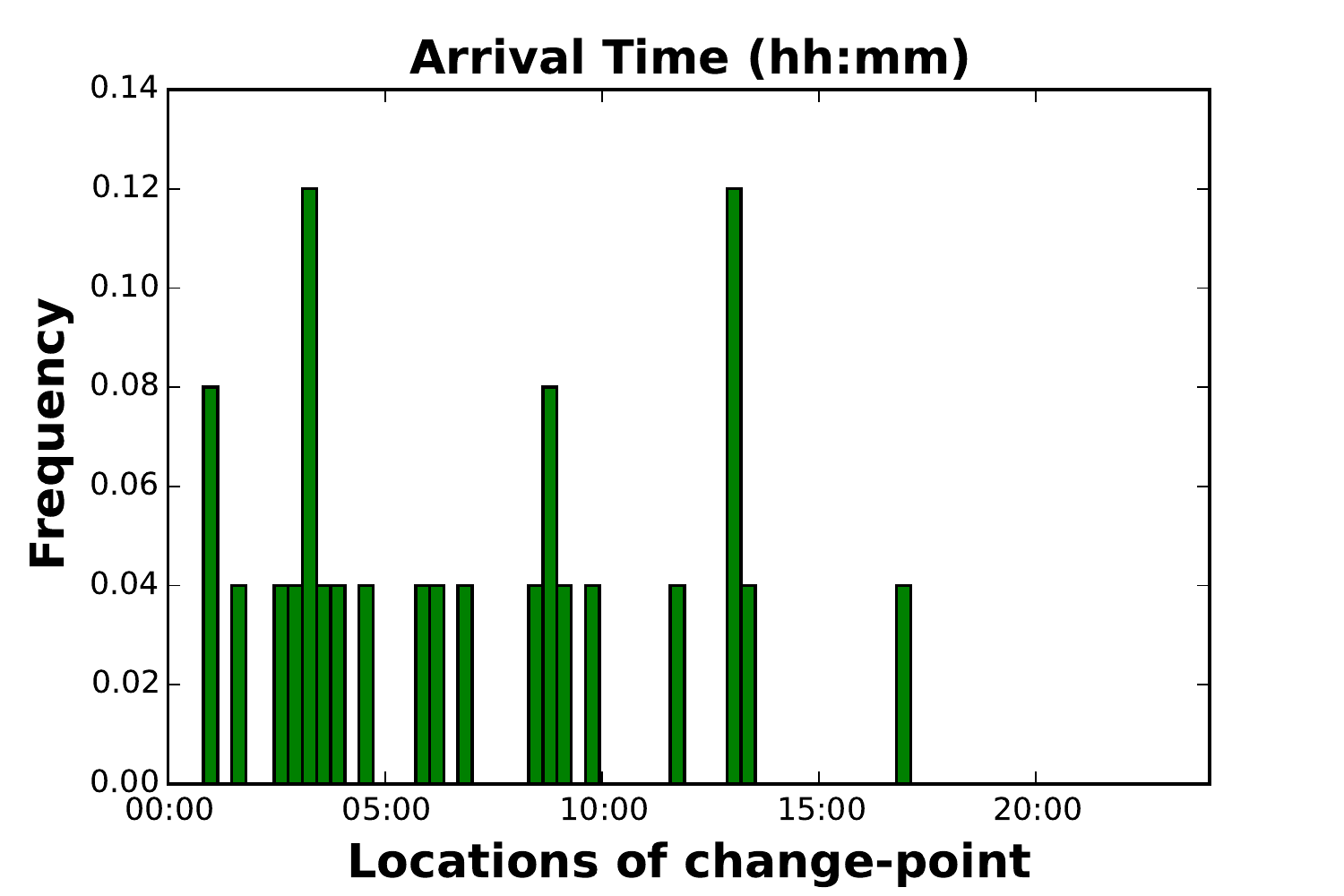}\includegraphics[width=0.5\textwidth]{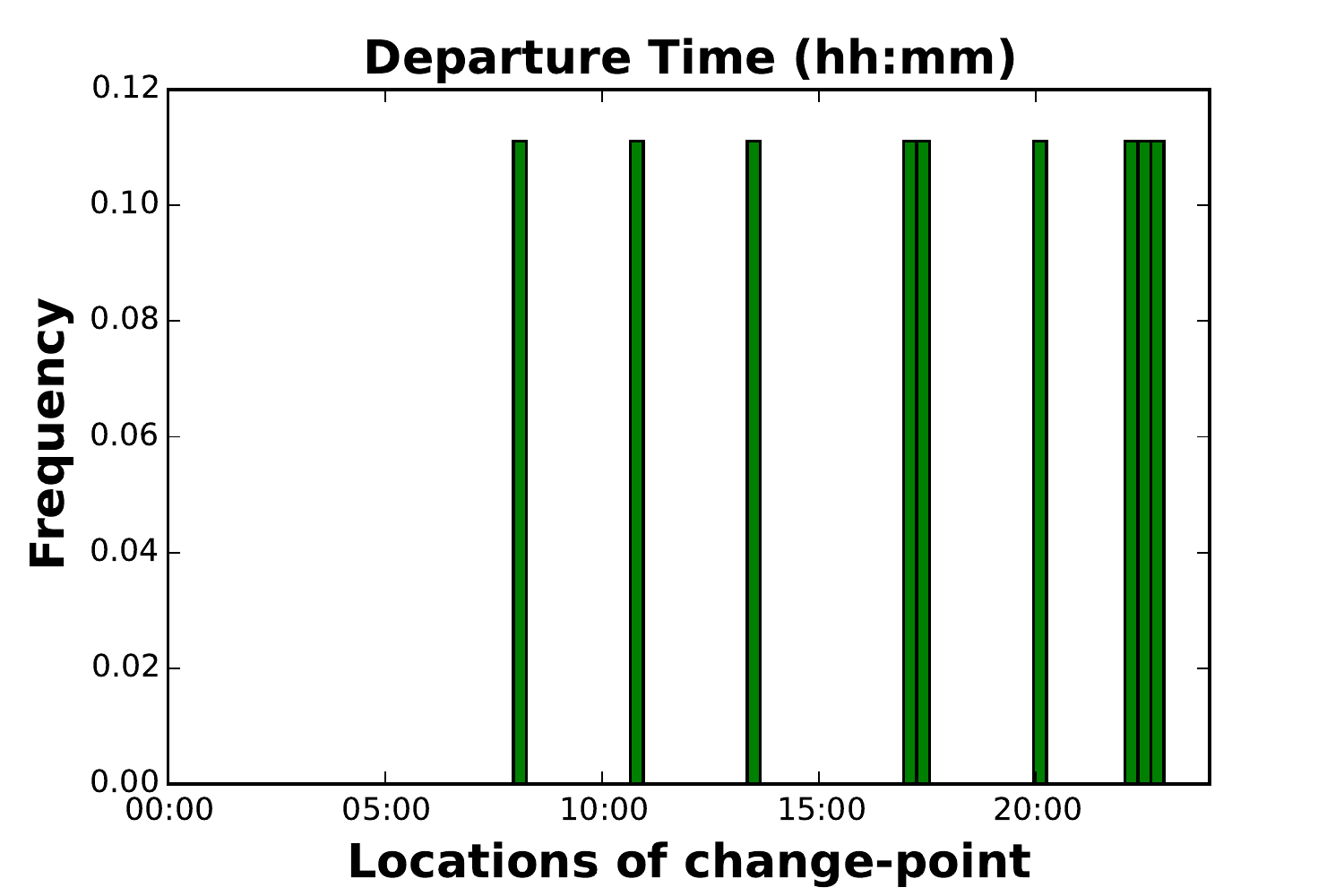}}
\centerline{\includegraphics[width=0.5\textwidth]{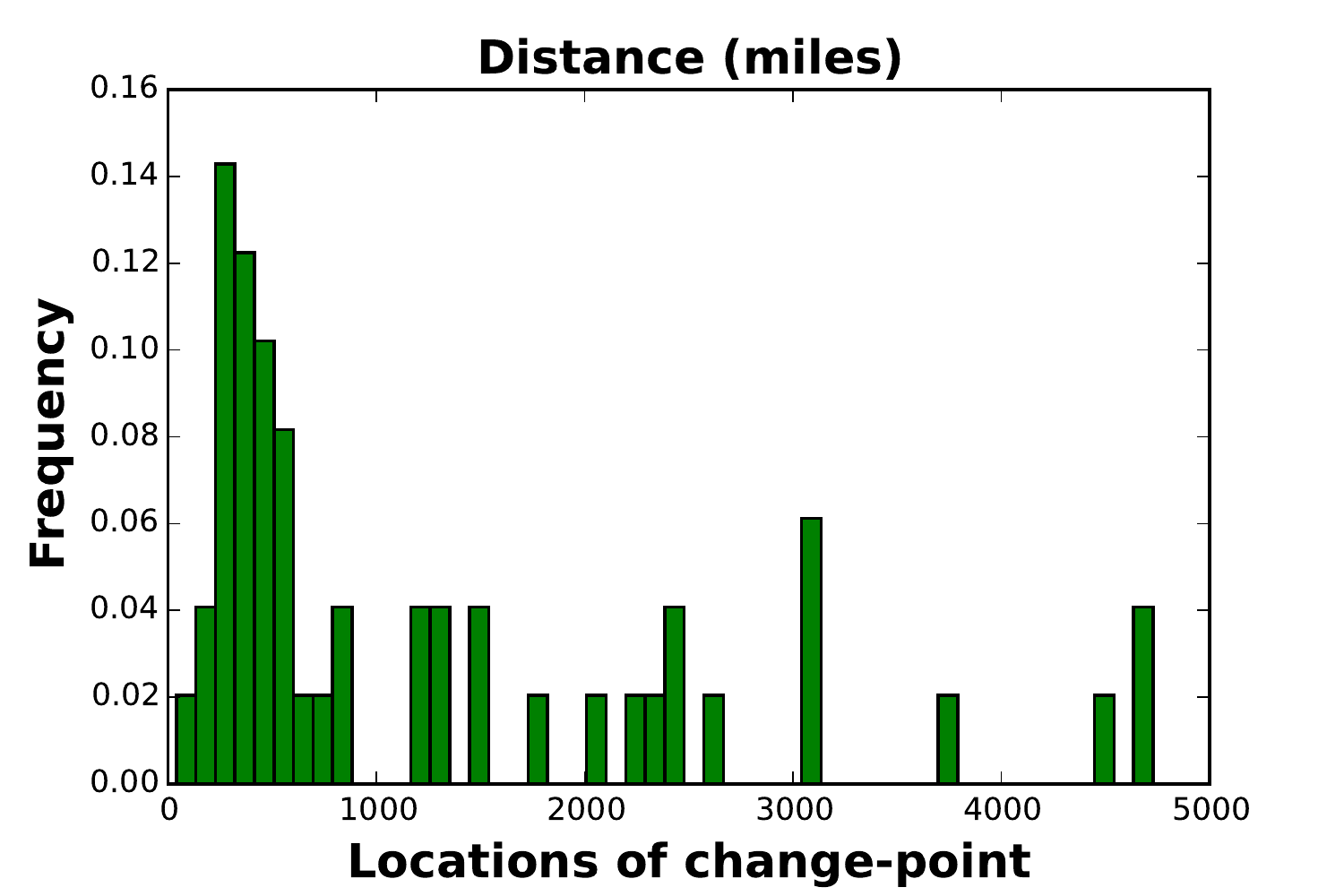}\includegraphics[width=0.5\textwidth]{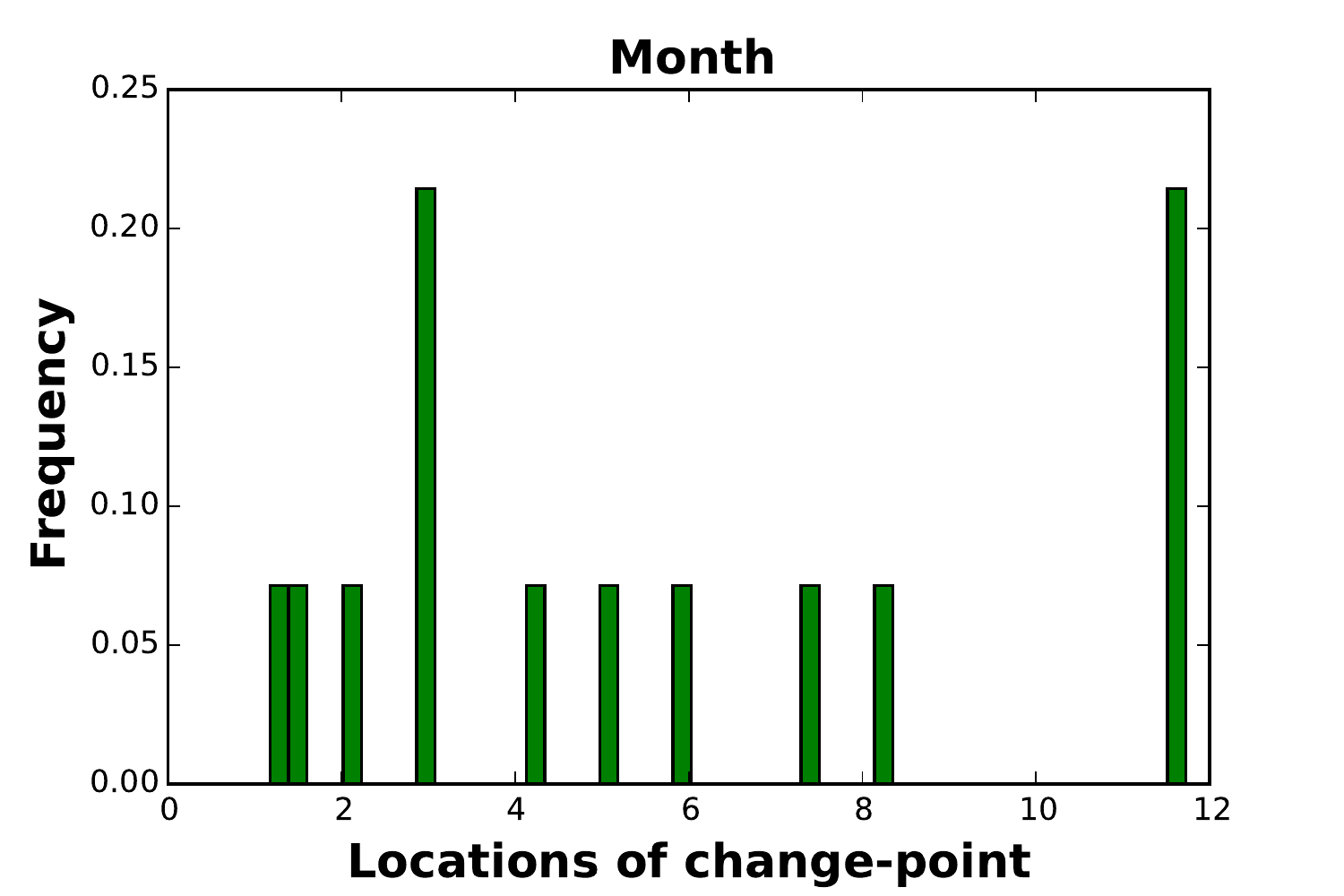}}
\caption{Posterior distributions of the locations of change-points in each input dimension in the airline delays experiment of Section \ref{sct:airline}. Dimensions that were learned to exhibit no change-point have been omitted here.}
\label{fig:airline_pos_cp}
\end{center}
\end{figure}

\clearpage

\subsection{Large Scale Dynamic Asset Allocation}
\label{sct:spt_ml}
An important feature of our proposed RJ-MCMC sampler (Algorithm \ref{alg:mcmc}) is that, unlike the BCM, the rBCM and SVIGP, which are restricted to Bayesian nonparametric regression and classification, Algorithm \ref{alg:mcmc} is agnostic with regard to the likelihood model, so long as it takes the form $p(\mathcal{D} \vert \mathbf{f}, \mathbf{u})$. Thus, it may be used as is on a wide variety of problems that go beyond classification and regression. In this experiment we aim to illustrate the efficacy of our approach on one such large scale problem in quantitative finance.

\subsubsection{Background} 
Let $\left(x_i(t)\right)_{t>0}$ for $i=1, \dots, n$ be $n$ stock price processes. Let $\left(X_i(t)\right)_{t>0}$ for $i=1, \dots, n$ denote the \emph{market capitalisation} processes, that is $X_i(t)=n_i(t) x_i(t)$ where $n_i(t)$ is the number of shares in company $i$ trading in the market at time $t$. We call \emph{long-only portfolio} any vector-valued stochastic process $\pi = (\pi_1, \dots, \pi_n)$ taking value on the unit simplex on $\mathbb{R}^n$, that is $$\forall i, t, ~\pi_i(t) \geq 0 ~~\text{ and }~~ \sum_{i=1}^n \pi_i(t) =1.$$ Each process $\pi_i$ represents the proportion of an investor's wealth invested in (holding) shares in asset $i$. An example long-only portfolio is the market portfolio $\mu = (\mu_1, \dots, \mu_n)$ where 
\begin{equation}
\mu_i(t) = \frac{X_i(t)}{X_1(t) + \dots + X_n(t)}
\end{equation}
is the market weight of company $i$ at time $t$, that is its size relative to the total market size (or that of the universe of stocks considered). The market portfolio is very important to practitioners as it is often perceived not be to subject to idiosyncracies, but only to systemic risk. It is often used as an indicator of how the stock market (or a specific universe of stocks) performs as a whole. We denote $Z_{\pi}$ the value process of a portfolio $\pi$ with initial capital $Z_{\pi}(0)$. That is, $Z_{\pi}(t)$ is the wealth at time $t$ of an investor who had an initial wealth of $Z_{\pi}(0)$, and dynamically re-allocated all his wealth between the $n$ stocks in our universe up to time $t$ following the continuous-time strategy $\pi$. 

A mathematical theory has recently emerged, namely \emph{stochastic portfolio theory} (SPT) \citep[see][]{karatzas2009stochastic} that studies the stochastic properties of the wealth processes of certain portfolios called \emph{functionally-generated portfolio} under realistic assumptions on the market capitalisation processes $\left(X_i(t)\right)_{t>0}$. Functionally-generated portfolios are rather specific in that the allocation at time $t$, namely $(\pi_1(t), \dots, \pi_n(t))$, solely depends on the market weights vector $(\mu_1(t), \dots, \mu_n(t))$. Nonetheless, some functionally-generated portfolios $\pi^*$ have been found that, under the mild (so-called diversity) condition
\begin{equation}
\exists \mu_{\text{max}},  0 < \mu_{\text{max}} < 1 ~~~\text{s.t.}~~~ \forall i \leq n, ~t \leq T, ~~~\mu_i(t) \leq  \mu_{\text{max}},
\end{equation}
outperform the market portfolio over the time horizon $[0, T]$ with probability one \citep[see][]{Vervuurt2015,karatzas2009stochastic}. More precisely,
\begin{equation}
\mathbb{P}\left(Z_{\pi^*}(T) \geq Z_{\mu}(T)\right) = 1 ~~~\text{and}~~~ \mathbb{P}\left(Z_{\pi^*}(T) > Z_{\mu}(T)\right) >0.
\end{equation} 
Galvanized by this result, we here consider the inverse problem consisting of learning from historical market data a portfolio whose wealth process has desirable user-specified properties. This inverse problem is perhaps more akin to the problems faced by investment professionals: i) their benchmarks depend on the investment vehicles pitched to investors and may vary from one vehicle to another, ii) they have to take into account liquidity costs, and iii) they often find it more valuable to go beyond market weights and leverage multiple company characteristics in their investment strategies. 

\subsubsection{Model Construction} 
We consider portfolios $\pi^f=\left(\pi_1^f, \dots, \pi_n^f\right)$ of the form 
\begin{equation}
\label{eq:pf_form}
\pi_i^f(t)  = \frac{f(c_i(t))}{f(c_1(t)) + \dots + f(c_n(t))},
\end{equation}
where $c_i(t) \in \mathbb{R}^d$ are some quantifiable characteristics of asset $i$ that may be observed in the market at time $t$, and $f$ is a positive-valued function. Portfolios of this form include all functionally-generated portfolios studied in SPT as a special case.\footnote{We refer the reader to \cite{karatzas2009stochastic} for the definition of  functionally-generated portfolios.} A crucial departure of our approach from the aforementioned type of portfolios is that the market characteristics processes $c_i$ need not be restricted to size-based information, and may contain additional information such as social media sentiments, stock price path-properties, but also characteristics relative to other stocks such as performance relative to the best/worst performing stock last week/month/year etc. We place a mean-zero \emph{string GP} prior on $\log f$. Given some historical data $\mathcal{D}$ corresponding to a training time horizon $[0, T]$, the likelihood model $p\left(\mathcal{D} \vert \pi^f\right)$ is defined by the investment professional and reflects the extent to which applying the investment strategy $ \pi^f$ over the training time horizon would have achieved a specific investment objective. An example investment objective is to achieve a high excess return relative to a benchmark portfolio $\alpha$
\begin{equation}
\mathcal{U}_{\text{ER}}\left( \pi^f \right) =  \log Z_{\pi^f}(T)- \log Z_{\alpha}(T).
\end{equation}
$\alpha$ can be the market portfolio (as in SPT) or any stock index. Other risk-adjusted investment objectives may also be used. One such objective is to achieve a high Sharpe-ratio, defined as 
\begin{equation}
\mathcal{U}_{\text{SR}}\left( \pi^f \right) = \frac{\bar{r}\sqrt{252}}{\sqrt{\frac{1}{T} \sum_{t=1}^T (r(t)-\bar{r})^2}},
\end{equation}
where the time $t$ is in days, $r(t) := \log Z_{\pi^f}(t)- \log Z_{\pi^f}(t-1)$ are the daily returns the portfolio $\pi^f$ and $\bar{r} = \frac{1}{T} \sum_{t=1}^T r(t)$ its average daily return. More generally, denoting $\mathcal{U}\left(\pi^f\right)$ the performance of the portfolio $\pi^f$ over the training horizon $[0, T]$ (as per the user-defined investment objective), we may choose as likelihood model a distribution over $\mathcal{U}\left(\pi^f\right)$ that reflects what the investment professional considers good and bad performance. For instance, in the case of the excess return relative to a benchmark portfolio or the Sharpe ratio, we may choose $\mathcal{U}\left(\pi^f\right)$ to be supported on $]0, +\infty[$ (for instance $\mathcal{U}\left(\pi^f\right)$ can be chosen to be Gamma distributed) so as to express that portfolios that do not outperform the benchmark or loose money overall in the training data are not of interest. We may then choose the mean and standard deviation of the Gamma distribution based on our expectation as to what performance a good candidate portfolio can achieve, and how confident we feel about this expectation. Overall we have, 
\begin{equation}
p\left(\mathcal{D} \vert \pi^f\right) = \gamma\left(\mathcal{U}\left(\pi^f\right); \alpha_e, \beta_e\right),
\end{equation}
where $\gamma(. ; \alpha, \beta)$ is the probability density function of the Gamma distribution. Noting, from Equation (\ref{eq:pf_form}) that $\pi^f(t)$ only depends on $f$ through its values at $(c_1(t), \dots, c_n(t))$, and assuming that $\mathcal{U}\left(\pi^f\right)$ only depends on $\pi^f$ evaluated at a finite number of times (as it is the case for excess returns and the Sharpe ratio), it follows that $\mathcal{U}(\pi^f)$ only depends on $\mathbf{f}$, a vector of values of $f$ at a finite number of points. Hence the likelihood model, which we may rewrite as 
\begin{equation}
\label{eq:sptt_lik}
p(\mathcal{D} \vert \mathbf{f}) = \gamma\left(\mathcal{U}\left(\pi_i^f\right); \alpha_e, \beta_e\right),
\end{equation}
is of the form required by the RJ-MCMC sampler previously developed. By sampling from the posterior distribution $p(\mathbf{f}, \mathbf{f}^*, \nabla \mathbf{f}, \nabla \mathbf{f}^* \vert \mathcal{D})$, the hope is to learn a portfolio that did well during the training horizon, to analyse the sensitivity of its investment strategy to the underlying market characteristics through the gradient of $f$, and to evaluate the learned investment policy on future market conditions.

\subsubsection{Experimental Setup} 
The universe of stocks we considered for this experiment are the constituents of the S\&P 500 index, accounting for changes in constituents over time and corporate events. We used the period 1\textsuperscript{st} January 1990 to 31\textsuperscript{st} December 2004 for training and we tested the learned portfolio during the period 1\textsuperscript{st} January 2005 to  31\textsuperscript{st} December 2014. We rebalanced the portfolio daily, for a total of $2.52$ million input points at which the latent function $f$ must be learned. We considered as market characteristics the market weight (CAP), the latest return on asset (ROA) defined as the ratio between the net yearly income and the total assets as per the latest balance sheet of the company known at the time of investment, the previous close-to-close return (PR), the close-to-close return before the previous  (PR2), and the S\&P long and short term credit rating (LCR and SCR). While the market weight is a company size characteristic, the ROA reflects how well a company performs relative to its size, and we hope that S\&P credit ratings will help further discriminate successful companies from others. The close-to-close returns are used to learn possible `momentum' patterns from the data. The data originate from the CRSP and Compustat databases. In the experiments we considered as performance metric the annualised excess return $\mathcal{U}_{\text{ER-EWP}}$ relative to the equally-weighted portfolio. We found the equally-weighted portfolio to be a harder benchmark to outperform than the market portfolio. We chose $\alpha_e$ and $\beta_e$ in Equation (\ref{eq:sptt_lik}) so that the mean of the Gamma distribution is $10.0$ and its variance $0.5$, which expresses a very greedy investment target.

It is worth pointing out that none of the scalable GP alternatives previously considered can cope with our likelihood model Equation (\ref{eq:sptt_lik}). We compared the performance of the learned \emph{string GP} portfolio out-of-sample to those of the best three SPT portfolios studied in \cite{Vervuurt2015}, namely the equally weighted portfolio
\begin{equation}
\pi_i^{\text{EWP}}(t) = \frac{1}{n},
\end{equation}
and the diversity weighted portfolios 
\begin{equation}
\pi_i^{\text{DWP}}(t; p) = \frac{\mu_i(t)^p}{\mu_1(t)^p + \dots + \mu_n(t)^p},
\end{equation}
with parameter $p$ equals to $-0.5$ and $0.5$, and the market portfolio. Results are provided Table \ref{table:spt_ml_bench}, and Figure \ref{fig:sptml_benchmark} displays the evolution of the wealth process of each strategy. It can be seen that the learned \emph{string GP} strategy considerably outperforms the next best SPT portfolio. This experiment not only demonstrates (once more) that \emph{string GPs} scale to large scale problems, it also illustrates that our inference scheme is able to unlock commercial value in new intricate large scale applications where no alternative is readily available. In effect, this application was first introduced by \cite{kom2016stochastic}, where the authors used a Gaussian process prior on a Cartesian grid and under a separable covariance function so as to speed up inference with Kronecker techniques. Although the resulting inference scheme has time complexity that is linear in the total number of points on the grid, for a given grid resolution, the time complexity grows \emph{exponentially} in the dimension of the input space (that is the number of trading characteristics), which is impractical for $d \geq 4$. On the other hand, \emph{string GPs} allow for a time complexity that grows linearly with the number of trading characteristics, thereby enabling the learning of subtler market inefficiencies from the data.

\begin{figure}
\begin{center}
\centerline{\includegraphics[width=0.75\textwidth]{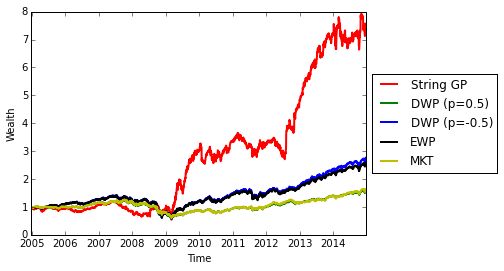}}
\caption{Evolution of the wealth processes of various long-only trading strategies on the S\&P $500$ universe of stocks between 1\textsuperscript{st} January 2005 (where we assume a starting wealth of $1$) and 31\textsuperscript{st} December 2014. The String GP strategy was learned using market data from 1\textsuperscript{st} January 1990 to 31\textsuperscript{st} December 2004 as descried in Section \ref{sct:spt_ml}. EWP refers to the equally-weighted portfolio, MKT refers to the market portfolio (which weights stocks proportionally to their market capitalisations) and DWP (p) refers to the diversity-weighted portfolio with exponent $p$ (which weights stocks proportionally to the $p$-th power of their market weights).}
\label{fig:sptml_benchmark}
\end{center}
\end{figure}

\begin{table}
\centering
\begin{tabular}{lccc}
Strategy			&Sharpe Ratio         & $Z_{\pi}(T)/Z_{\text{EWP}}(T)$	& Avg. Ann. Ret. 	\\\midrule
String GP			& $\bm{0.73}$ 	& $\bm{2.87}$ 				& $\bm{22.07}$\textbf{\%} 	\\
DWP ($p=-0.5$)       	& $0.55$ 		& $1.07$ 					& $10.56$\%  \\
EWP          			& $0.53$ 		& $1.00$ 					& $9.84$\% 	\\
MKT				& $0.34$		& $0.62$					& $4.77$\% \\
DWP ($p=0.5$)     		& $0.33$ 		& $0.61$ 					& $4.51$\%  
\end{tabular}
\caption{Performance of various long-only trading strategies on the S\&P $500$ universe of stocks between 1\textsuperscript{st} January 2005 (where we assume a starting wealth of $1$) and 31\textsuperscript{st} December 2014. The String GP strategy was learned using market data from 1\textsuperscript{st} January 1990 to 31\textsuperscript{st} December 2004 as descried in Section \ref{sct:spt_ml}. EWP refers to the equally-weighted portfolio, MKT refers to the market portfolio (which weights stocks proportionally to their market capitalisations) and DWP (p) refers to the diversity-weighted portfolio with exponent $p$  (which weights stocks proportionally to the $p$-th power of the market weight of the asset). $Z_{\pi}(T)$ denotes the terminal wealth of strategy $\pi$, and Avg. Ann. Ret. is the strategy's equivalent constant annual return over the test horizon.}
\label{table:spt_ml_bench}
\end{table}

\section{Discussion}
\label{sct:discussion}

In this paper, we introduce a novel class of smooth functional priors (or stochastic processes), which we refer to as \emph{string GPs}, with the aim of simultaneously addressing the lack of scalability and the lack of flexibility of Bayesian kernel methods. Unlike existing approaches, such as Gaussian process priors (\cite{rasswill}) or student-t process priors (\cite{shah2014student}), which are parametrised by \emph{global} mean and covariance functions, and which postulate \emph{fully dependent} finite-dimensional marginals, the alternative construction we propose adopts a \emph{local} perspective and the resulting finite-dimensional marginals exhibit \emph{conditional independence} structures. Our local approach to constructing \emph{string GPs} provides  a principled way of postulating that the latent function we wish to learn might exhibit locally homogeneous patterns, while the conditional independence structures constitute the core ingredient needed for developing scalable inference methods. Moreover, we provide theoretical results relating our approach to Gaussian processes, and we illustrate that our approach can often be regarded as a more scalable and/or more flexible extension. We argue and empirically illustrate that \emph{string GPs} present an unparalleled opportunity for learning local patterns in small scale regression problems using nothing but standard Gaussian process regression techniques. More importantly, we propose a novel \emph{scalable} RJ-MCMC inference scheme to learn latent functions in a wide variety of machine learning tasks, while simultaneously determining \emph{whether} the data set exhibits local patterns, \emph{how many} types of local patterns the data might exhibit, and \emph{where} do changes in these patterns are likely to occur. The proposed scheme has time complexity and memory requirement that are both \emph{linear} in the sample size $N$. When the number of available computing cores is at least equal to the dimension $d$ of the input space, the time complexity is independent from the dimension of the input space. Else, the time complexity grows in $\mathcal{O}(dN)$. The memory requirement grows in $\mathcal{O}(dN)$. We empirically illustrate that our approach scales considerably better than competing alternatives on a standard benchmark data set, and is able to process data sizes that competing approaches cannot handle in a reasonable time.

\subsection{Limitations}
The main limitation of our approach is that, unlike the \emph{standard GP paradigm} in which the time complexity of marginal likelihood evaluation does not depend on the dimension of the input space (other than through the evaluation of the Gram matrix), the \emph{string GP paradigm} requires a number of computing cores that increases linearly with the dimension of the input space, or alternatively has a time complexity linear in the input space dimension on single-core machines. This is a by-product of the fact that in the \emph{string GP paradigm}, we jointly infer the latent function and its gradient. If the gradient of the latent function is inferred in the \emph{standard GP paradigm}, the resulting complexity will also be linear in the input dimension. That being said, overall our RJ-MCMC inference scheme will typically scale  better per iteration to large input dimensions than gradient-based marginal likelihood inference in the \emph{standard GP paradigm}, as the latter typically requires numerically evaluating an Hessian matrix, which requires computing the marginal likelihood a number of times per iterative update that grows quadratically with the input dimension. In contrast, a Gibbs cycle in our MCMC sampler has worst case time complexity that is linear in the input dimension.

\subsection{Extensions}
Some of the assumptions we have made in the construction of \emph{string GPs} and \emph{membrane GPs} can be relaxed, which we consider in detail below.
\subsubsection{Stronger Global Regularity}
We could have imposed more (multiple continuous differentiability) or less (continuity) regularity as boundary conditions in the construction of \emph{string GPs}. We chose continuous differentiability as it is a relatively mild condition guaranteed by most popular kernels, and yet the corresponding treatment can be easily generalised to other regularity requirements. It is also possible to allow for discontinuity at a boundary time $a_k$ by replacing $\mu_k^b$ and $\Sigma_k^b$ in Equation (\ref{eq:bound_pdf}) with ${}_kM_{a_k}$ and ${}_k\textbf{K}_{a_k;a_k}$ respectively, or equivalently by preventing any communication between the $k$-th and the $(k+1)$-th strings. This would effectively be equivalent to having two independent \emph{string GPs} on $[a_0, a_k]$ and $]a_k, a_K]$.

\subsubsection{Differential Operators as Link Functions}
Our framework can be further extended to allow differential operators as link functions, thereby considering the latent multivariate function to infer as the response of a differential system to independent univariate \emph{string GP} excitations. The RJ-MCMC sampler we propose in Section \ref{sct:infer} would still work in this framework, with the only exception that, when the differential operator is of first order, the latent multivariate function will be continuous but not differentiable, except if global regularity is upgraded as discussed above. Moreover, Proposition \ref{prop:is_gp} can be generalised to first order linear differential operators.

\subsubsection{Distributed String GPs}
The RJ-MCMC inference scheme we propose may be easily adapted to handle applications where the data set is so big that it has to be stored across multiple clusters, and inference techniques have to be developed as data flow graphs\footnote{A data flow graph is a computational (directed) graph whose nodes represent calculations (possibly taking place on different computing units) and directed edges correspond to data flowing between calculations or computing units.} (for instance using libraries such as TensorFlow). 

To do so, the choice of string boundary times can be adapted so that each string has the same number of inner input coordinates, and such that in total there are as many strings across dimensions as a target number of available computing cores. We may then place a prior on kernel memberships similar to that of previous sections. Here, the change-points may be restricted to coincide with boundary times, and we may choose priors such that the sets of change-points are independent between input dimensions. In each input dimension the prior on the number of change-points can be chosen to be a truncated Poisson distribution (truncated to never exceed the total number of boundary times), and conditional on their number we may choose change-points to be uniformly distributed in the set of boundary times. In so doing, any two strings whose shared boundary time is not a change-point will be driven by the same kernel configuration.

This new setup presents no additional theoretical or practical challenges, and the RJ-MCMC techniques previously developed are easily adaptable to jointly learn change-points and function values. Unlike the case we developed in previous sections where an update of the univariate \emph{string GP} corresponding to an input dimension, say the $j$-th, requires looping through all distinct $j$-th input coordinates, here no step in the inference scheme requires a full view of the data set in any input dimension. Full RJ-MCMC inference can be constructed as a data flow graph. An example such graph is constructed as follows. The leaves correspond to computing cores responsible for generating change-points and kernel configurations, and mapping strings to kernel configurations. The following layer is made of compute cores that use kernel configurations coming out of the previous layer to sequentially compute boundary conditions corresponding to a specific input dimension---there are $d$ such compute cores, where $d$ is the input dimension. These compute cores then pass computed boundary conditions to subsequent compute cores we refer to as string compute cores. Each string compute core is tasked with computing \emph{derivative string GP} values for a specific input dimension and for a specific string in that input dimension, conditional on previously computed boundary conditions. These values are then passed to a fourth layer of compute cores, each of which being tasked with computing function and gradient values corresponding to a small subset of training inputs from previously computed \emph{derivative string GP} values. The final layers then computes the log-likelihood using a distributed algorithm such as Map-Reduce when possible. This proposal data flow graph is illustrated Figure \ref{fig:sgp_inference_flow}.

We note that the approaches of \cite{kim}, \cite{gramacy}, \cite{tresp2000bayesian}, and \cite{deisenroth2015distributed} also allow for fully-distributed inference on regression problems. Distributed \emph{string GP} RJ-MCMC inference improves on these in that it places little restriction on the type of likelihood. Moreover, unlike \cite{kim} and \cite{gramacy} that yield discontinuous latent functions, \emph{string GPs} are continuously differentiable, and unlike \cite{tresp2000bayesian} and  \cite{deisenroth2015distributed}, local experts in the \emph{string GP paradigm} (i.e. strings) are driven by possibly different sets of hyper-parameters, which facilitates the learning of local patterns.
\subsubsection{Approximate MCMC for I.I.D. Observations Likelihoods}
As discussed in Section \ref{sct:gen_rjmcmc}, the bottleneck of our proposed inference scheme is the evaluation of likelihood. When the likelihood factorises across training samples, the linear time complexity of our proposed approach can be further reduced using a Monte Carlo approximation of the log-likelihood (see for instance \cite{bardenet2014towards} and references therein). Although the resulting Markov chain will typically not converge to the true posterior distribution, in practice its stationary distribution can be sufficiently close to the true posterior when reasonable Monte Carlo sample sizes are used. Convergence results of such approximations have recently been studied by \cite{bardenet2014towards} and \cite{alquier2016noisy}. We expect this extension to speed-up inference when the number of compute cores is in the order of magnitude of the input dimension, but we would recommend the previously mentioned fully-distributed string GP inference extension when compute cores are not scarce.

\subsubsection{Variational Inference}
It would be useful to develop suitable variational methods for inference under \emph{string GP} priors, that we hope will scale similarly to our proposed RJ-MCMC sampler but will converge faster. We anticipate that the main challenge here will perhaps be the learning of model complexity, that is the number of distinct kernel configurations in each input dimension.

%
%data_flow_graph
\begin{landscape}
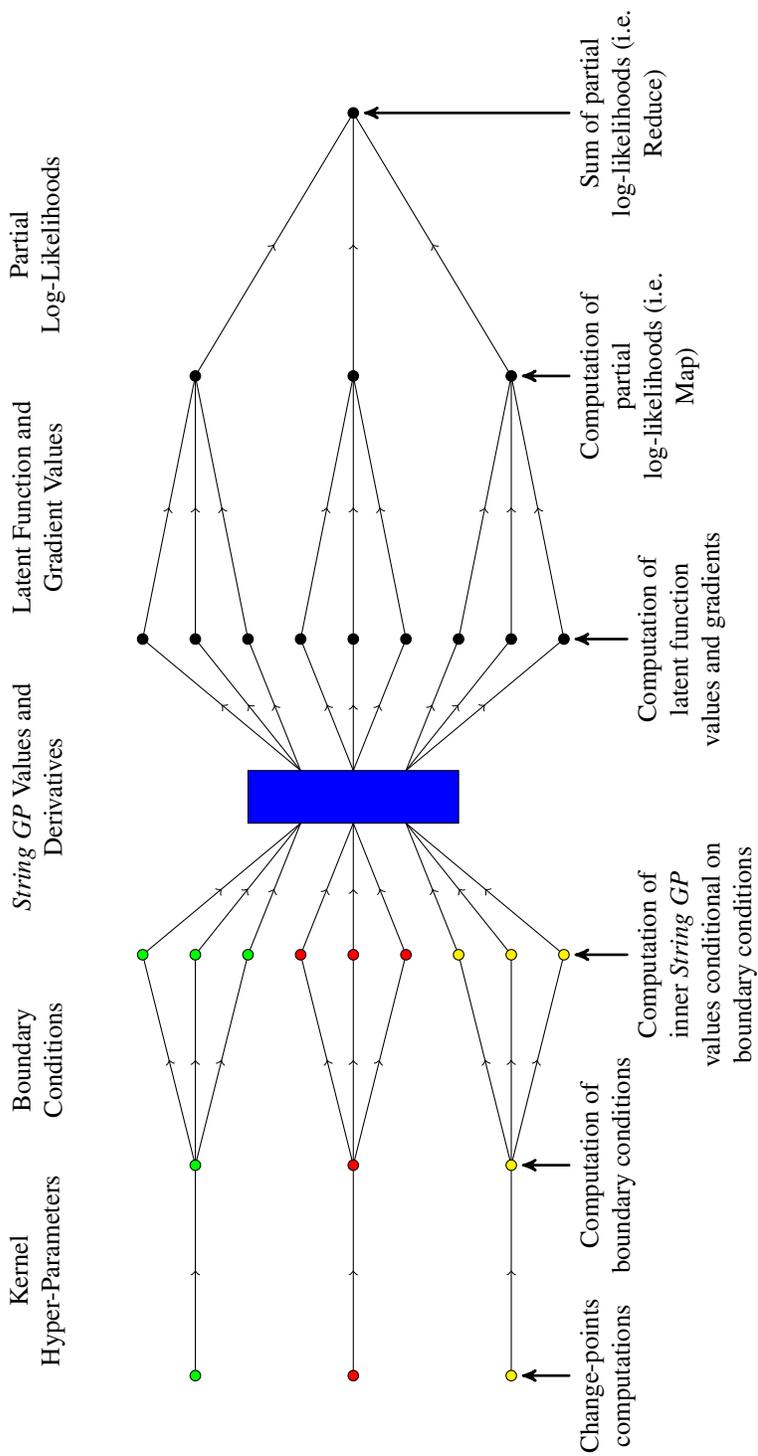
\begin{figure}
\centering
% Declare layers
\pgfdeclarelayer{background}
\pgfsetlayers{background,main}
\tikzstyle{information text}=[text badly centered,font=\small,text width=3cm]
\begin{tikzpicture}[scale=.7,cap=round]
    % The graphic
    \begin{scope}[>=stealth', line width=1pt]
        \draw[->] (-3,.9) node[below, information text]
            {Change-points computations} -- (-3,1.8);
        \draw[->] (1,.9) node[below, information text]
            {Computation of boundary conditions} -- (1,1.8);
        \draw[->] (5,-.2) node[below,information text]
            {Computation of inner \emph{String GP} values conditional on boundary conditions} -- (5,.8);
        \draw[->] (11,-.2) node[below,information text]
            {Computation of latent function values and gradients} -- (11,.8);
        \draw[->] (16,0.9) node[below,information text]
            {Computation of partial log-likelihoods (i.e. Map) } -- (16,1.8);
        \draw[->] (21,0.9) node[below,information text]
            {Sum of partial log-likelihoods (i.e. Reduce) } -- (21,4.8);
    \end{scope}
    \draw (3,11) node[information text] {Boundary Conditions };
    \draw (8,11) node[information text, text width=4cm, ]
        { \emph{String GP} Values and Derivatives };
    \draw (13.5,11) node[information text] { Latent Function and Gradient Values};
    \draw (-1,11) node[information text] {Kernel Hyper-Parameters};
    \draw (18.5,11) node[information text] {Partial Log-Likelihoods};
    % draw the nodes
    \filldraw[fill=yellow] (1,2) circle (0.1);
    \filldraw[fill=red] (1,5) circle (0.1);
    \filldraw[fill=green] (1,8) circle (0.1);
    \filldraw[fill=black] (16,2) circle (0.1);
    \filldraw[fill=black] (16,5) circle (0.1);
    \filldraw[fill=black] (16,8) circle (0.1);    
    \filldraw[fill=yellow] (-3,2) circle (0.1);
    \filldraw[fill=red] (-3,5) circle (0.1);
    \filldraw[fill=green] (-3,8) circle (0.1);
    \foreach \x in {5}
        \foreach \y in {1,2,3} {
            \filldraw[fill=yellow] (\x,\y) circle (0.1);
        }
    \foreach \x in {5}
        \foreach \y in {4,5,6} {
            \filldraw[fill=red] (\x,\y) circle (0.1);
        }
    \foreach \x in {5}
        \foreach \y in {7,8,9} {
            \filldraw[fill=green] (\x,\y) circle (0.1);
        }
    \foreach \x in {11}
        \foreach \y in {1,2,3} {
            \filldraw[fill=black] (\x,\y) circle (0.1);
        }
    \foreach \x in {11}
        \foreach \y in {4,5,6} {
            \filldraw[fill=black] (\x,\y) circle (0.1);
        }
    \foreach \x in {11}
        \foreach \y in {7,8,9} {
            \filldraw[fill=black] (\x,\y) circle (0.1);
        }
    \filldraw[fill=black] (21, 5) circle (0.1);
    \filldraw[fill=blue] (7.5,3.0) rectangle (8.5,7.0);
    %

    % The lines connecting the nodes are drawn in the background layer.
    % This way we can hide the lines behind the nodes and don't worry
    % about the width of each node.    
    \begin{pgfonlayer}{background}
	\begin{scope}[decoration={markings,mark=at position 0.5 with {\arrow{>}}}] 
        		\draw[postaction={decorate}] (-3,2) -- (1,2);
        		\draw[postaction={decorate}] (-3,5) -- (1,5);
        		\draw[postaction={decorate}] (-3,8) -- (1,8);
        		
        		\draw[postaction={decorate}] (1,2) -- (5,1);
        		\draw[postaction={decorate}] (1,2) -- (5,2);
        		\draw[postaction={decorate}] (1,2) -- (5,3);
        		\draw[postaction={decorate}] (1,5) -- (5,4);
        		\draw[postaction={decorate}] (1,5) -- (5,5);
        		\draw[postaction={decorate}] (1,5) -- (5,6);
        		\draw[postaction={decorate}] (1,8) -- (5,7);
        		\draw[postaction={decorate}] (1,8) -- (5,8);
        		\draw[postaction={decorate}] (1,8) -- (5,9);
        		
        		\foreach \ya in {1 , 2, 3}{
            			\draw[postaction={decorate}] (5,\ya) -- (7.5,4);
            			\draw[postaction={decorate}]  (8.5,4) -- (11,\ya);
		}

        		\foreach \ya in {4 , 5, 6}{
            			\draw[postaction={decorate}] (5,\ya) -- (7.5,5);
            			\draw[postaction={decorate}]  (8.5,5) -- (11,\ya);
		}

        		\foreach \ya in {7 , 8, 9}{
            			\draw[postaction={decorate}] (5,\ya) -- (7.5,6);
            			\draw[postaction={decorate}]  (8.5,6) -- (11,\ya);
		}

        		\foreach \ya in {7,8,9}
            			\draw[postaction={decorate}] (11,\ya) -- (16,8);

        		\foreach \ya in {4,5,6}
            			\draw[postaction={decorate}] (11,\ya) -- (16,5);
        			
        		\foreach \ya in {1,2,3}
            			\draw[postaction={decorate}] (11,\ya) -- (16,2);
            			
        		\draw[postaction={decorate}] (16,8) -- (21,5);
        		\draw[postaction={decorate}] (16,5) -- (21,5);
        		\draw[postaction={decorate}] (16,2) -- (21,5);
        \end{scope}
   \end{pgfonlayer}
\end{tikzpicture}
\caption{Example data flow graph for fully-distributed \emph{string GP} inference under an i.i.d observations likelihood model. Here the input space is three-dimensional to ease illustration. Filled circles represent compute cores, and edges correspond to flows of data. Compute cores with the same colour (green, red or yellow) perform operations pertaining to the same input dimension, while black-filled circles represent compute cores performing cross-dimensional operations. The blue rectangle plays the role of a hub that relays \emph{string GP} values to the compute cores that need them to compute the subset of latent function values and gradients they are responsible for. These values are then used to compute the log-likelihood in a distributed fashion using the Map-Reduce algorithm. Each calculation in the corresponding RJ-MCMC sampler would be initiated at one of the compute cores, and would trigger updates of all edges accessible from that compute core.}
\label{fig:sgp_inference_flow}
\end{figure}
\end{landscape}

\acks{Yves-Laurent is a Google Fellow in Machine Learning and would like to acknowledge support from the Oxford-Man Institute. Wharton Research Data Services (WRDS) was used in preparing the data for Section \ref{sct:spt_ml} of this paper. This service and the data available thereon constitute valuable intellectual property and trade secrets of WRDS and/or its third-party suppliers.}
\begin{appendices}
\renewcommand\thesection{Appendix \Alph{section}}
\section{}
We begin by recalling Kolmogorov's extension theorem, which we will use to prove the existence of \textit{derivative Gaussian processes} and \textit{string Gaussian processes}.

\begin{theorem}(Kolmogorov's extension theorem, \citep[Theorem 2.1.5]{oks})\\
Let $I$ be an interval, let all $t_1, \dots, t_i \in I, ~ i, n \in \mathbb{N}^{*}$, let $\nu_{t_1, \dots, t_i}$ be probability measures on $\mathbb{R}^{ni}$ such that:
\begin{equation}
\label{cond:perm}
\nu_{t_{\pi(1)}, \dots, t_{\pi(i)}}(F_{\pi(1)}, \dots, F_{\pi(i)}) = \nu_{t_1, \dots, t_i}(F_1, \dots, F_i) 
\end{equation}
for all permutations $\pi$ on $\{1, \dots, i\}$ and 
\begin{align}
\label{cond:margin}
&\nu_{t_1, \dots, t_i}(F_1, \dots, F_i) = \nu_{t_1, \dots, t_i, t_{i+1}, \dots, t_{i+m}}(F_1, \dots, F_i, \mathbb{R}^{n}, \dots, \mathbb{R}^{n}) 
\end{align}
for all $m \in \mathbb{N}^{*}$ where the set on the right hand side has a total of $i+m$ factors. Then there exists a probability space $(\Omega, \mathcal{F}, \mathbb{P})$ and an $\mathbb{R}^n$ valued stochastic process $(X_t)_{t \in I}$ on $\Omega$, \[X_t : \Omega \to \mathbb{R}^n\] such that 
\begin{equation}
\nu_{t_1, \dots, t_i}(F_1, \dots, F_i) = \mathbb{P}(X_{t_1} \in F_1, \dots, X_{t_i} \in F_i)
\end{equation}
for all $t_1, \dots, t_i \in I, ~ i \in \mathbb{N}^{*}$ and for all Borel sets $F_1, \dots, F_i$.
\end{theorem}
It is easy to see that every stochastic process satisfies the permutation and marginalisation conditions ($\ref{cond:perm}$) and ($\ref{cond:margin}$). The power of Kolmogorov's extension theorem is that it states that those two conditions are sufficient to guarantee the existence of a stochastic process. 
%
%
%
%
%
%	APPENDIX B
%
%
%
%
\section{Proof of Proposition \ref{prop:derivative_processes}}
\label{app:derivative_processes}
In this section we prove Proposition \ref{prop:derivative_processes}, which we recall below.\\\\
\textbf{Proposition \ref{prop:derivative_processes} (Derivative Gaussian processes)}\\
Let $I$ be an interval, $k: I \times I \rightarrow \mathbb{R}$ a $\mathcal{C}^2$ symmetric positive semi-definite function,\footnote{$\mathcal{C}^1$ (resp. $\mathcal{C}^2$) functions denote functions that are once (resp. twice) continuously differentiable on their domains.} $m: I \rightarrow \mathbb{R}$ a $\mathcal{C}^1$ function.

\noindent (A) There exists a $\mathbb{R}^2$-valued stochastic process $\left(D_t\right)_{t \in I}, ~D_t=(z_t, z_t^\prime)$, such that for all $t_1, \dots, t_n \in I$, 
$$(z_{t_1}, \dots, z_{t_n}, z_{t_1}^\prime, \dots, z_{t_n}^\prime)$$ is a Gaussian vector with mean 
$$\left(m(t_1), \dots, m(t_n), \frac{\text{d}m}{\text{dt}}(t_1), \dots, \frac{\text{d}m}{\text{dt}}(t_n)\right)$$ and covariance matrix such that 
$$\text{cov}(z_{t_i}, z_{t_j})=k(t_i, t_j), ~~~ \text{cov}(z_{t_i}, z_{t_j}^\prime)=\frac{\partial k}{\partial y }(t_i, t_j), ~~~\text{and}~~~ \text{cov}(z_{t_i}^\prime, z_{t_j}^\prime)=\frac{\partial^2 k}{\partial x \partial y }(t_i, t_j).$$ 
We herein refer to $(D_t)_{t \in I}$ as a \textbf{derivative Gaussian process}.

\noindent (B) $(z_t)_{t \in I}$ is a Gaussian process with mean function $m$, covariance function $k$ and that is $\mathcal{C}^1$ in the $L^2$ (mean square) sense. 

\noindent (C) $(z^\prime_t)_{t \in I}$ is a Gaussian process with mean function $\frac{\text{d}m}{\text{dt}}$ and covariance function $\frac{\partial^2 k}{\partial x \partial y }$. Moreover, $(z^\prime_t)_{t \in I}$ is the $L^2$ derivative of the process $(z_t)_{t \in I}$.\\

\begin{proof}
%
%
%
%
%
%	APPENDIX B 1
%
%
%
%
\subsection{Proof of Proposition \ref{prop:derivative_processes} (A)} Firstly, we need to show that the matrix suggested in the proposition as the covariance matrix of $(z_{t_1}, \dots, z_{t_n}, z_{t_1}^\prime, \dots, z_{t_n}^\prime)$ is indeed positive semi-definite. To do so, we will show that it is the limit of positive definite matrices (which is sufficient to conclude it is positive semi-definite, as $x^TM_nx \ge 0$ for a convergent sequence of positive definite matrices implies $ x^T M_{\infty} x \ge 0 $). 

Let $k$ be as in the proposition, $h$ such that $\forall i \leq n, ~ t_i +h \in I$ and $(\tilde{z}_t)_{t \in I}$ be a Gaussian process with covariance function $k$. The vector $$\left(\tilde{z}_{t_1}, \dots, \tilde{z}_{t_n}, \frac{\tilde{z}_{t_1+h}-\tilde{z}_{t_1}}{h}, \dots, \frac{\tilde{z}_{t_n+h}-\tilde{z}_{t_n}}{h}\right)$$ is a Gaussian vector whose covariance matrix is positive definite and such that 
\begin{align}
\label{eq:p1}
\text{cov}\left(\tilde{z}_{t_i}, \tilde{z}_{t_j}\right) = k(t_i, t_j),
\end{align}
\begin{align}
\label{eq:p2}
\text{cov}\left(\tilde{z}_{t_i}, \frac{\tilde{z}_{t_j+h}-\tilde{z}_{t_j}}{h}\right)=\frac{k(t_i, t_j + h)-k(t_i, t_j)}{h},
\end{align}
and
\begin{align}
\label{eq:p3}
&\text{cov}\left(\frac{\tilde{z}_{t_i+h}-\tilde{z}_{t_i}}{h}, \frac{\tilde{z}_{t_j+h}-\tilde{z}_{t_j}}{h}\right) \nonumber \\
&=\frac{1}{h^2 }\left(k(t_i + h, t_j + h)-k(t_i + h, t_j) -k(t_i, t_j+ h) + k(t_i,t_j)\right).
\end{align}
As $k$ is $\mathcal{C}^2$, $h \to k(x, y+h)$ admits a second order Taylor expansion about $h=0$ for every $x$, and we have:
\begin{align}
\label{eq:tay1}
k(x, y + h) &= k(x, y) + \frac{\partial k}{\partial y}(x, y) h + \frac{1}{2} \frac{\partial^2 k}{\partial y^2}(x, y) h^2 + o(h^2) = k(y + h, x).
\end{align}
Similarly, $h \to  k(x+h, y+h)$ admits a second order Taylor expansion about $h=0$ for every $x, ~y$ and we have:
\begin{align}
\label{eq:tay2}
k(x + h, y + h) =& ~k(x, y) + \left[\frac{\partial k}{\partial x}(x, y) + \frac{\partial k}{\partial y}(x, y)\right]h + \Bigg[\frac{\partial^2 k}{\partial x \partial y}(x, y) + \frac{1}{2} \frac{\partial^2 k}{\partial x^2}(x, y) \nonumber \\ 
&+\frac{1}{2} \frac{\partial^2 k}{\partial y^2}(x, y)\Bigg]h^2  + o(h^2).
\end{align}
Hence, \begin{align}
\label{eq:p4}
k(t_i, t_j + h)-k(t_i, t_j) = \frac{\partial k}{\partial y} (t_i, t_j)h + o(h),
\end{align} and 
\begin{align}
\label{eq:p5}
&k(t_i + h, t_j + h)-k(t_i + h, t_j) -k(t_i, t_j+ h)+k(t_i,t_j)  = \frac{\partial^2 k}{\partial x \partial y}(t_i, t_j) h^2 +o(h^2).
\end{align}
Dividing Equation (\ref{eq:p4}) by $h$, dividing Equation (\ref{eq:p5}) by $h^2$, and taking the limits, we obtain:
\[\underset{h \to 0}{\lim}~\text{cov}\left(\tilde{z}_{t_i}, \frac{\tilde{z}_{t_j+h}-\tilde{z}_{t_j}}{h}\right)=\frac{\partial k}{\partial y} (t_i, t_j),\] and 
\[\lim_{h \to 0}~\text{cov}\left(\frac{\tilde{z}_{t_i+h}-\tilde{z}_{t_i}}{h}, \frac{\tilde{z}_{t_j+h}-\tilde{z}_{t_j}}{h}\right) = \frac{\partial^2 k}{\partial x \partial y}(t_i, t_j), \]
which corresponds to the covariance structure of Proposition \ref{prop:derivative_processes}. In other words the proposed covariance structure is indeed positive semi-definite.

Let $\nu^{\mathcal{N}}_{t_1, \dots, t_n}$ be the Gaussian probability measure corresponding to the joint distribution of $(z_{t_1}, \dots, z_{t_n}, z_{t_1}^\prime, \dots, z_{t_n}^\prime)$  as per the Proposition \ref{prop:derivative_processes}, and let $\nu^{D}_{t_1, \dots, t_n}$ be the measure on the Borel $\sigma$-algebra $\mathcal{B}(\underbrace{\mathbb{R}^{2} \times \dots \times \mathbb{R}^{2})}_{n \text{ times }}$ such that for any $2n$ intervals $I_{11}, I_{12}, \dots, I_{n1}, I_{n2}$, 
\begin{align}
&\nu^{D}_{t_1, \dots, t_n}(I_{11}\times I_{12}, \dots, I_{n1} \times I_{n2}) :=\nu^{\mathcal{N}}_{t_1, \dots, t_n}(I_{11}, \dots, I_{n1}, I_{12}, \dots, I_{n2}).
\end{align} The measures $\nu^D_{t_1, \dots, t_n}$ are the finite dimensional measures corresponding to the stochastic object $(D_t)_{t \in I}$ sampled at times $t_1, \dots, t_n$. They satisfy the time permutation and marginalisation conditions of Kolmogorov's extension theorem as the Gaussian measures $\nu^{\mathcal{N}}_{t_1, \dots, t_n}$ do. Hence, the $\mathbb{R}^2$-valued stochastic process $(D_t)_{t \in I}$ defined in Proposition \ref{prop:derivative_processes} does exist.
%
%
%
%
%
%	APPENDIX B 2
%
%
%
%
\subsection{Proof of Proposition \ref{prop:derivative_processes} (B)} That $(z_t)_{t \in I}$ is a Gaussian process results from the fact that the marginals $(z_{t_1}, \dots, z_{t_n})$ are Gaussian vectors with mean $(m(t_1), \dots, m(t_n))$ and covariance matrix $[k(t_i, t_j)]_{i,j \in [1..n]}$. The fact that $(z_t)_{t \in I}$ is $\mathcal{C}^1$ in the $L^2$ sense is a direct consequence of the twice continuous differentiability of $k$.
%
%
%
%
%
%	APPENDIX B 3
%
%
%
%
\subsection{Proof of Proposition \ref{prop:derivative_processes} (C)}  In effect, it follows from  Proposition \ref{prop:derivative_processes}(A) that $\frac{z_{t+h}-z_t}{h}-z^\prime_t$ is a Gaussian random variable with mean 
$$\frac{m(t+h)-m(t)}{h}- \frac{\text{d} m}{\text{dt}}(t)$$ 
and variance
$$\frac{k(t+h, t+h) -2k(t+h, t) + k(t, t)-2\frac{\partial k}{\partial y}(t+h, t)h+2\frac{\partial k}{\partial y}(t, t)h + \frac{\partial^2 k}{\partial x \partial y}(t, t) h^2}{h^2}.$$
Taking the second order Taylor expansion of the numerator in the fraction above about $h=0$ we get $o(h^2)$, hence 
$$\underset{h \to 0}{\lim}~ \text{Var}\left(\frac{z_{t+h}-z_t}{h}-z^\prime_t\right)=0.$$  
We also have  
$$\underset{h \to 0}{\lim}~ \text{E}\left(\frac{z_{t+h}-z_t}{h}-z^\prime_t\right)= \frac{\text{d}m}{dt}(t) - \text{E}(z^\prime_t)  = 0.$$
Therefore, 
$$\underset{h \to 0}{\lim}~ \text{E}\left[\left(\frac{z_{t+h}-z_t}{h}-z^\prime_t\right)^2\right]=0,$$
which proves that $(z^\prime_t)$ is the $L^2$ derivative of $(z_t)$. The fact that $(z^\prime_t)$ is a Gaussian process with mean function $\frac{\text{d}m}{dt}$ and covariance function $\frac{\partial^2 k}{\partial x \partial y}$ is a direct consequence of the distribution of the marginals $(z_{t_1}^\prime, \dots, z_{t_n}^\prime)$. Moreover, the continuity of $(z^\prime_t)$ in the $L^2$ sense is a direct consequence of the continuity of $\frac{\partial^2 k}{\partial x \partial y}$ \citep[see][p. 81 4.1.1]{rasswill}.
\end{proof}
%
%
%
%
%
%	APPENDIX C
%
%
%
%
\section{Proof of Theorem \ref{theo:sgp}}
\label{app:sgp}
%
%
%
%
%
%	APPENDIX C1
%
%
%
%
In this section we prove Theorem \ref{theo:sgp} which we recall below.\\\\
\textbf{Theorem  \ref{theo:sgp} (String Gaussian process)}\\Let $a_0<\dots<a_k< \dots<a_K$, $I=[a_0, a_K]$ and let $p_\mathcal{N}(x; \mu, \Sigma)$ be the multivariate Gaussian density with mean vector $\mu$ and covariance matrix $\Sigma$. Furthermore, let $(m_k:  [a_{k-1}, a_k] \to \mathbb{R})_{k \in [1..K]}$ be $\mathcal{C}^1$ functions, and   $(k_k: [a_{k-1}, a_k] \times  [a_{k-1}, a_k]\to \mathbb{R})_{k \in [1..K]}$ be $\mathcal{C}^3$ symmetric positive semi-definite functions, neither degenerate at $a_{k-1}$, nor degenerate at $a_k$ given $a_{k-1}$.

\noindent (A) There exists an $\mathbb{R}^2$-valued stochastic process $(SD_{t})_{t \in I}, ~ SD_t=(z_t, z_t^\prime)$  satisfying the following conditions:

\noindent 1) The probability density of $(SD_{a_0}, \dots, SD_{a_K})$ reads: 
\begin{equation*}
p_{b}(x_0, \dots, x_K) := \prod_{k=0}^K p_\mathcal{N}\left(x_k;  \mu^b_k, \Sigma_k^b\right)
\end{equation*}
\begin{equation*}
\text{where:}~~~~
\Sigma_0^b = {}_1\textbf{K}_{a_0; a_0}, ~~
\forall ~ k>0 ~~~\Sigma_k^b = {}_k\textbf{K}_{a_k; a_k} - {}_k\textbf{K}_{a_k; a_{k-1}}~{}_k\textbf{K}_{a_{k-1}; a_{k-1}}^{-1}~{}_k\textbf{K}_{a_k; a_{k-1}}^T,
\end{equation*}
\begin{equation*}
\mu_0^b={}_1\textbf{M}_{a_0}, ~~
\forall ~ k>0 ~~~\mu^b_k={}_k\textbf{M}_{a_k} + {}_k\textbf{K}_{a_k; a_{k-1}}~{}_k\textbf{K}_{a_{k-1}; a_{k-1}}^{-1}(x_{k-1}-{}_k\textbf{M}_{a_{k-1}}), 
\end{equation*}
\[
\text{with}~~~~{}_k\textbf{K}_{u;v} = 
\begin{bmatrix} 
k_k(u, v) & \frac{\partial k_k}{\partial y}(u, v) \\
\frac{\partial k_k}{\partial x}(u, v)  & \frac{\partial^2 k_k}{\partial x \partial y}(u, v) 
\end{bmatrix}, ~~~~\text{and}~~~~
{}_k\textbf{M}_u =\begin{bmatrix} m_k(u) \\ \frac{d m_k}{dt}(u) \end{bmatrix}.
\]

\noindent 2) Conditional on $(SD_{a_k} = x_k)_{k \in [0..K]}$, the restrictions $(SD_{t})_{t \in  ]a_{k-1}, a_k[},~ k \in [1..K]$ are \textbf{independent conditional derivative Gaussian processes}, respectively with unconditional mean function $m_k$ and unconditional covariance function $k_k$ and that are conditioned to take values $x_{k-1}$ and $x_k$ at $a_{k-1}$ and $a_k$ respectively. We refer to $(SD_{t})_{t \in I}$  as a \textbf{string derivative Gaussian process}, and to its first coordinate $(z_{t})_{t \in I}$ as a \textbf{string Gaussian process} namely, \[(z_{t})_{t \in I} \sim \mathcal{SGP}(\{a_k\}, \{m_k\}, \{k_k\}).\]

\noindent (B) The \textbf{string Gaussian process} $(z_t)_{t \in I}$ defined in (A) is $\mathcal{C}^1$ in the $L^2$ sense and its $L^2$ derivative is the process $(z_t^\prime)_{t \in I}$ defined in (A).\\

\begin{proof}
\subsection{Proof of Theorem \ref{theo:sgp} (A)} We will once again turn to Kolmogorov's extension theorem to prove the existence of the stochastic process $(SD_t)_{t \in I}$. The core of the proof is in the finite dimensional measures implied by Theorem \ref{theo:sgp} (A-1) and (A-2). Let $\left\{t^k_i \in ]a_{k-1}, a_k[ \right\}_{i \in [1..N_k], k \in [1..K]}$ be $n$ times. We first formally construct the finite dimensional measures implied by Theorem \ref{theo:sgp} (A-1) and (A-2), and then verify that they satisfy the conditions of Kolmogorov's extension theorem.

\noindent Let us define the measure $\nu^{SD}_{t_1^1,\dots, t_{N_1}^1, \dots, t_1^K,\dots, t_{N_K}^K, a_0, \dots, a_K}$ as the probability measure having density with respect to the Lebesgue measure on $\mathcal{B}(\underbrace{\mathbb{R}^{2} \times \dots \times \mathbb{R}^{2})}_{1+n+K \text{ times }}$ that reads:
\begin{align}
\label{eq:proof_density}
 p_{SD}(x_{t_1^1}, \dots, x_{t_{N_1}^1}, \dots, x_{t_1^K}, \dots, x_{t_{N_K}^K}, x_{a_0}, \dots, x_{a_K}) =&p_{b}(x_{a_0}, \dots, x_{a_K}) \times \nonumber \\  
 & \prod_{k=1}^{K} p^{x_{a_{k-1}}, x_{a_k}}_{\mathcal{N}}(x_{t_1^k}, \dots, x_{t_{N_k}^k})
 \end{align}
where $p_{b}$ is as per Theorem \ref{theo:sgp} (A-1) and $p^{x_{a_{k-1}}, x_{a_k}}_{\mathcal{N}}(x_{t_1^k}, \dots, x_{t_{N_k}^k})$ is the (Gaussian) pdf of the joint distribution of the values at times $\{t_i^k \in ]a_{k-1}, a_k[\}$ of the \textit{conditional derivative Gaussian process} with unconditional mean functions $m_k$ and unconditional covariance functions $k_k$ that is conditioned to take values $x_{a_{k-1}}=(z_{a_{k-1}}, z^\prime_{a_{k-1}})$ and $x_{a_k}=(z_{a_{k}}, z^\prime_{a_{k}})$ at times $a_{k-1}$ and $a_k$ respectively (the corresponding---conditional---mean and covariance functions are derived from Equations (\ref{eq:double_cond_mean} and \ref{eq:double_cond_cov}). Let us extend the family of measures $\nu^{SD}$ to cases where some or all boundary times $a_k$ are missing, by integrating out the corresponding variables in Equation (\ref{eq:proof_density}). For instance when $a_0$ and $a_1$ are missing,
\begin{align}
\label{eq:margin}
&\nu^{SD}_{t_1^1,\dots, t_{N_1}^1, \dots, t_1^K,\dots, t_{N_K}^K, a_2, \dots, a_K}(T_1^1,\dots, T_{N_1}^1, \dots, T_1^K,\dots, T_{N_K}^K, A_2, \dots, A_K) \nonumber \\
&  := \nu^{SD}_{t_1^1,\dots, t_{N_1}^1, \dots, t_1^K,\dots, t_{N_K}^K, a_0, \dots, a_K}(T_1^1,\dots, T_{N_1}^1, \dots, T_1^K,\dots, T_{N_K}^K, \mathbb{R}^2, \mathbb{R}^2, A_2, \dots, A_K) \nonumber \\
 \end{align}
where $A_i$ and $T^i_j$ are rectangle in $\mathbb{R}^2$. Finally, we extend the family of measures $\nu^{SD}$ to any arbitrary set of indices $\{t_1, \dots, t_n\}$ as follows:
\begin{align}
\label{eq:ext_perm}
\nu^{SD}_{t_1, \dots, t_n}(T_1, \dots, T_n):=\nu^{SD}_{t_{\pi^{*}(1)}, \dots, t_{\pi^{*}(n)}}(T_{\pi^{*}(1)}, \dots, T_{\pi^{*}(n)}),
\end{align}
where $\pi^{*}$ is a permutation of $\{1, \dots, n\}$ such that $\{t_{\pi^{*}(1)}, \dots, t_{\pi^{*}(n)}\}$ verify the following conditions:
\begin{enumerate}
\item $\forall ~i, ~ j, ~\text{if } t_i \in ]a_{k_1 -1}, a_{k_1}[, ~ t_j \in ]a_{k_2 -1}, a_{k_2}[, \text{ and } k_1 < k_2,$ then $\text{Idx}(t_i) < \text{Idx}(t_j).$ Where $\text{Idx}(t_i)$ stands for the index of $t_i$ in $\{t_{\pi^{*}(1)}, \dots, t_{\pi^{*}(n)}\}$;
\item \text{if } $t_i \notin \{a_0, \dots, a_K\} \text{ and } t_j \in \{a_0, \dots, a_K\}$ then $\text{Idx}(t_i) < \text{Idx}(t_j);$
\item \text{if } $t_i \in \{a_0, \dots, a_K\} \text{ and } t_j \in \{a_0, \dots, a_K\}$ then $\text{Idx}(t_i) < \text{Idx}(t_j)$ if and only if $t_i < t_j.$
\end{enumerate}
Any such measure $\nu^{SD}_{t_{\pi^{*}(1)}, \dots, t_{\pi^{*}(n)}}$ will fall in the category of either Equation (\ref{eq:proof_density}) or Equation (\ref{eq:margin}). Although $\pi^{*}$ is not unique, any two permutations satisfying the above conditions will only differ by a permutation of times belonging to the same string interval $]a_{k-1}, a_k[$.  Moreover, it follows from Equations (\ref{eq:proof_density}) and (\ref{eq:margin}) that the measures $\nu^{SD}_{t_{\pi^{*}(1)}, \dots, t_{\pi^{*}(n)}}$ are invariant by permutation of times belonging to the same string interval $]a_{k-1}, a_k[$, and as a result any two $\pi^{*}$ satisfying the above conditions will yield the same probability measure.

The finite dimensional probability measures $\nu^{SD}_{t_1, \dots, t_n}$ are the measures implied by Theorem \ref{theo:sgp}. The permutation condition (\ref{cond:perm}) of Kolmogorov's extension theorem is met by virtue of Equation (\ref{eq:ext_perm}). In effect for every permutation $\pi$ of $\{1, \dots, n\}$, if we let $\pi^\prime: ~ \{\pi(1), \dots, \pi(n)\} \to \{\pi^{*}(1), \dots, \pi^{*}(n)\}$, then
\begin{align}
\nu^{SD}_{t_{\pi(1)}, \dots, t_{\pi(n)}}(T_{\pi(1)}, \dots, T_{\pi(n)}) &:= \nu^{SD}_{t_{\pi^{\prime}\circ\pi(1)}, \dots, t_{\pi^{\prime}\circ\pi(n)}}(T_{\pi^{\prime}\circ\pi(1)}, \dots, T_{\pi^{\prime}\circ\pi(n)}) \nonumber \\
& = \nu^{SD}_{t_{\pi^{*}(1)}, \dots, t_{\pi^{*}(n)}}(T_{\pi^{*}(1)}, \dots, T_{\pi^{*}(n)}) \nonumber \\
& = \nu^{SD}_{t_1, \dots, t_n}(T_1, \dots, T_n) \nonumber.
\end{align}
As for the marginalisation condition (\ref{cond:margin}), it is met for every boundary time by virtue of how we extended $\nu^{SD}$ to missing boundary times. All we need to prove now is that the marginalisation condition is also met at any non-boundary time. To do so, it is sufficient to prove that the marginalisation condition holds for $t_1^1$, that is:
\begin{align}
\label{eq:to_proove}
&\nu^{SD}_{t_1^1,\dots, t_{N_1}^1, \dots, t_1^K,\dots, t_{N_K}^K, a_0, \dots, a_K}(\mathbb{R}^2, T_2^1,\dots, T_{N_1}^1, \dots, T_1^K,\dots, T_{N_K}^K, A_0, \dots, A_K)  \nonumber \\
& = \nu^{SD}_{t_2^1,\dots, t_{N_1}^1, \dots, t_1^K,\dots, t_{N_K}^K, a_0, \dots, a_K}(T_2^1,\dots, T_{N_1}^1, \dots, T_1^K,\dots, T_{N_K}^K, A_0, \dots, A_K) \nonumber \\
 \end{align}
 for every rectangles $A_i$ and $T^i_j$ in $\mathbb{R}^2$. In effect, cases where some boundary times are missing are special cases with the corresponding rectangles $A_j$ set to $\mathbb{R}^2$. Moreover, if we prove Equation (\ref{eq:to_proove}), the permutation property (\ref{cond:perm}) will allow us to conclude that the marginalisation also holds true for any other (single) non-boundary time. Furthermore, if Equation (\ref{eq:to_proove}) holds true, it can be shown that the marginalisation condition will also hold over multiple non-boundary times by using the permutation property (\ref{cond:perm}) and marginalising one non-boundary time after another.

By Fubini's theorem, and considering Equation (\ref{eq:proof_density}), showing that Equation (\ref{eq:to_proove}) holds true is equivalent to showing that: 
\begin{align}
& \int_{\mathbb{R}^2} p^{x_{a_0}, x_{a_1}}_{\mathcal{N}}(x_{t_1^1}, \dots, x_{t_{N_1}^1}) dx_{t_1^1} =  p^{x_{a_{0}}, x_{a_1}}_{\mathcal{N}}(x_{t_2^1}, \dots, x_{t_{N_1}^1}) \end{align} which holds true as $p^{x_{a_{0}}, x_{a_1}}_{\mathcal{N}}(x_{t_1^1}, \dots, x_{t_{N_1}^1})$ is a multivariate Gaussian density, and the corresponding marginal is indeed the density of the same \textit{conditional derivative Gaussian process} at times $t_2^1, \dots, t_{N_1}^1$.

This concludes the proof of the existence of the stochastic process $(SD_t)_{t \in I}$.
%
%
%
%
%
%	APPENDIX C2
%
%
%
%
\subsection{Proof of Theorem \ref{theo:sgp} (B)} 

As conditional on boundary conditions the restriction of a \textit{string derivative Gaussian process} on a string interval $[a_{k-1}, a_k]$ is a \textit{derivative Gaussian process}, it follows from Proposition \ref{prop:derivative_processes} (C) that
\begin{align}
&\forall ~ \tilde{x}_{a_0}, \dots, \tilde{x}_{a_K}, ~ \forall ~ t, t+h \in [a_{k-1}, a_k],\nonumber \\
&\lim_{h \to 0}~ \text{E}\left( \left[\frac{z_{t+h}-z_t}{h} - z_t^\prime\right]^2 \bigg\vert x_{a_0} = \tilde{x}_{a_0}, \dots,  x_{a_K} = \tilde{x}_{a_K}\right) = 0,
\end{align}
or equivalently that:
\begin{align}
\label{eq:miss1}
 \Delta z_h := \text{E}\left( \left[\frac{z_{t+h}-z_t}{h} - z_t^\prime\right]^2 \bigg\vert x_{a_0}, \dots,  x_{a_K}\right) \underset{h \to 0}{\overset{a.s.}{\longrightarrow}} 0.
\end{align}
Moreover, 
\begin{equation}
\label{eq:dzh}
\Delta z_h  = \text{Var}\left(\frac{z_{t+h}-z_t}{h} - z_t^\prime \bigg\vert x_{a_0}, \dots,  x_{a_K}\right) + \text{E}\left(\frac{z_{t+h}-z_t}{h} - z_t^\prime \bigg\vert x_{a_0}, \dots,  x_{a_K}\right)^2.
\end{equation}
As both terms in the sum of the above equation are non-negative, it follows that 
\[\text{Var}\left(\frac{z_{t+h}-z_t}{h} - z_t^\prime \bigg\vert x_{a_0}, \dots,  x_{a_K}\right) \underset{h \to 0}{\overset{a.s.}{\longrightarrow}} 0 ~~~~\text{and}~~~~ \text{E}\left(\frac{z_{t+h}-z_t}{h} - z_t^\prime \bigg\vert x_{a_0}, \dots,  x_{a_K}\right)^2 \underset{h \to 0}{\overset{a.s.}{\longrightarrow}} 0.
\]
From which we deduce 
$$\text{E}\left(\frac{z_{t+h}-z_t}{h} - z_t^\prime \Big\vert x_{a_0}, \dots,  x_{a_K}\right) \underset{h \to 0}{\overset{a.s.}{\longrightarrow}} 0.$$ As $\text{E}\left(\frac{z_{t+h}-z_t}{h} - z_t^\prime \Big\vert x_{a_0}, \dots,  x_{a_K}\right)$
depends linearly on the boundary conditions, and as the boundary conditions are jointly-Gaussian (see \ref{app:is_gp} step 1), it follows that $$\text{E}\left(\frac{z_{t+h}-z_t}{h} - z_t^\prime \Big\vert x_{a_0}, \dots,  x_{a_K}\right)$$ is Gaussian. Finally we note that $$\text{Var}\left(\frac{z_{t+h}-z_t}{h} - z_t^\prime \Big\vert x_{a_0}, \dots,  x_{a_K}\right)$$ does not depend on the values of the boundary conditions $x_{a_k}$ (but rather on the boundary times), and we recall that convergence almost sure of Gaussian random variables implies convergence in $L^2$. Hence, taking the expectation on both side of Equation (\ref{eq:dzh}) and then the limit as $h$ goes to $0$ we get 
\[\text{E}\left( \left[\frac{z_{t+h}-z_t}{h} - z_t^\prime \right]^2 \right) = \text{E}(\Delta z_h) \underset{h \to 0}{\longrightarrow} 0,\]
which proves that the \textit{string GP} $(z_t)_{t \in I}$ is differentiable in the $L^2$ sense on $I$ and has derivative $(z_t^\prime)_{t \in I}$.
 
We prove the continuity in the $L^2$ sense of $(z_t^\prime)_{t \in I}$ in a similar fashion, noting that conditional on the boundary conditions, $(z^\prime_t)_{t \in I}$ is a Gaussian process whose mean function $\frac{d m_{ck}^{a_{k-1}, a_k}}{dt}$ and covariance function $\frac{\partial^2 k_{ck}^{a_{k-1}, a_k}}{\partial x \partial y}$ are continuous, thus is continuous in the $L^2$ sense on $[a_{k-1}, a_k]$ (conditional on the boundary conditions). We therefore have that:
 \begin{align}
&\forall ~ \tilde{x}_{a_0}, \dots, \tilde{x}_{a_K}, ~\forall ~ t, t+h \in [a_{k-1}, a_k], ~ \lim_{h \to 0} \text{E}\left((z^\prime_{t+h} - z_t^\prime)^2 \big\vert x_{a_0} = \tilde{x}_{a_0}, \dots,  x_{a_K} = \tilde{x}_{a_K}\right) = 0,
\end{align} from which we get that:
\begin{align}
\label{eq:miss2}
\Delta z_h^\prime := \text{E}\left( \left[z^\prime_{t+h} - z_t^\prime \right]^2 \big\vert x_{a_0}, \dots,  x_{a_K}\right) \underset{h \to 0}{\overset{a.s.}{\longrightarrow}} 0.
\end{align}
Moreover, 
\begin{equation}
\label{eq:dzhp}
\Delta z_h^\prime = \text{Var}\left( z^\prime_{t+h} - z_t^\prime \big\vert x_{a_0}, \dots,  x_{a_K}\right) + \text{E}\left( z^\prime_{t+h} - z_t^\prime \big\vert x_{a_0}, \dots,  x_{a_K}\right)^2,
\end{equation}
which implies that $$\text{Var}\left( z^\prime_{t+h} - z_t^\prime \big\vert x_{a_0}, \dots,  x_{a_K}\right) \underset{h \to 0}{\overset{a.s.}{\longrightarrow}} 0$$ and $$\text{E}\left( z^\prime_{t+h} - z_t^\prime \big\vert x_{a_0}, \dots,  x_{a_K}\right)^2 \underset{h \to 0}{\overset{a.s.}{\longrightarrow}} 0,$$ as both terms in the sum in Equation (\ref{eq:dzhp}) are non-negative. Finally, $$\text{Var}\left( z^\prime_{t+h} - z_t^\prime \big\vert x_{a_0}, \dots,  x_{a_K}\right)$$ does not depend on the values of the boundary conditions, and $$\text{E}\left( z^\prime_{t+h} - z_t^\prime \big\vert x_{a_0}, \dots,  x_{a_K}\right)$$ is Gaussian for the same reason as before. Hence, taking the expectation on both sides of Equation (\ref{eq:dzhp}), we get that 
\[\text{E}\left( \left[z_{t+h}^\prime- z_t^\prime \right]^2 \right) = \text{E}(\Delta z_h^\prime) \underset{h \to 0}{\longrightarrow} 0,
\]
which proves that $(z_t^\prime)$ is continuous in the $L^2$ sense.
\end{proof}
%
%
%
%
%
%	APPENDIX D
%
%
%
%
\section{Proof of the Condition for Pathwise Regularity Upgrade of \textit{String GPs} from $L^2$}
\label{app:path_reg}
In this section we prove that a sufficient condition for the process $(z_t^\prime)_{t \in I}$ in Theorem \ref{theo:sgp} to be almost surely continuous and to be the almost sure derivative of the string Gaussian process $(z_t)_{t \in I}$, is that the Gaussian processes on $I_k=[a_{k-1}, a_k]$ with mean and covariance functions $m_{ck}^{a_{k-1}, a_k}$ and $k^{a_{k-1}, a_k}_{ck}$ (as per Equations \ref{eq:double_cond_mean} and \ref{eq:double_cond_cov} with $m:= m_k$ and $k:= k_k$) are themselves almost surely $\mathcal{C}^1$ for every boundary condition. 

Firstly we note that the above condition guarantees that the result holds at non-boundary times. As for boundary times, the condition implies that the \textit{string GP} is almost surely right differentiable (resp. left differentiable) at every left (resp. right) boundary time, including $a_0$ and $a_K$. Moreover, the \textit{string GP} being differentiable in $L^2$, the right hand side and left hand side almost sure derivatives are the same, and are equal to the $L^2$ derivative, which proves that the $L^2$ derivatives at inner boundary times are also in the almost sure sense. A similar argument holds to conclude that the right (resp. left) hand side derivative at $a_0$ (resp. $a_K$) is also in the almost sure sense. Moreover, the derivative process $(z_t^\prime)_{t \in I}$ admits an almost sure right hand side limit and an almost sure left hand side limit at every inner boundary time and both are equal as the derivative is continuous in $L^2$, which proves its almost sure continuity at inner boundary times. Almost sure continuity of $(z_t^\prime)_{t \in I}$ on the right (resp. left) of $a_0$ (resp. $a_K$) is a direct consequence of the above condition.
%
%
%
%
%
%	APPENDIX E
%
%
%
%
\section{Proof of Proposition \ref{prop:diversity}}
\label{app:diversity}
In this section, we prove Proposition \ref{prop:diversity}, which we recall below.\\\\
\textbf{Proposition \ref{prop:diversity}} \textbf{(Additively separable \emph{string GPs} are flexible)}\\
Let $k(x, y) := \rho\left(\vert\vert x-y \vert\vert^2_{{L^2}}\right)$ be a stationary covariance function generating a.s. $\mathcal{C}^1$ GP paths indexed on $\mathbb{R}^d, ~ d>0$, and $\rho$ a function that is $\mathcal{C}^2$ on $]0, +\infty[$ and continuous at $0$. Let $\phi_s(x_1, \dots, x_d)=\sum_{j=1}^d x_j$, let $(z_t^j)_{t \in I^j, ~ j \in [1..d]}$ be independent stationary Gaussian processes with mean $0$ and covariance function $k$ (where the $L^2$ norm is on $\mathbb{R}$), and let $f(t_1, \dots, t_d)=\phi_s(z_{t_1}^1, \dots, z_{t_d}^d)$ be the corresponding stationary string GP. Finally, let $g$ be an isotropic Gaussian process indexed on $I^1 \times \dots \times I^d$ with mean 0 and covariance function $k$ (where the $L^2$ norm is on $\mathbb{R}^d$). Then: \\
1)$~\forall ~ x \in I^1 \times \dots \times I^d, ~ H(\nabla f(x)) = H(\nabla g(x))$,\\
2)$~\forall ~ x \neq y \in I^1 \times \dots \times I^d, ~ I(\nabla f(x); \nabla f(y)) \leq I(\nabla g(x); \nabla g(y))$.\\ \\
To prove Proposition \ref{prop:diversity} we need a lemma, which we state and prove below.
\begin{lemma}
\label{lem2}
Let $X_n$ be a sequence of Gaussian random vectors with auto-covariance matrix $\Sigma_n$ and mean $\mu_n$, converging almost surely to $X_{\infty}$. If $\Sigma_n \to \Sigma_{\infty}$ and $\mu_n \to \mu_{\infty}$ then $X_{\infty}$ is Gaussian with mean $\mu_{\infty}$ and auto-covariance matrix $\Sigma_{\infty}$.
\end{lemma}
\begin{proof}
We need to show that the characteristic function of $X_{\infty}$ is $$\phi_{X_{\infty}}(t):=\text{E}(e^{it^TX_{\infty}})= e^{it^T\mu_{\infty} - \frac{1}{2} t^T \Sigma_{\infty} t}.$$
As $\Sigma_{n}$ is positive semi-definite, $\forall n, ~\vert e^{it^T\mu_{n} - \frac{1}{2} t^T \Sigma_{n} t} \vert = e^{-\frac{1}{2} t^T \Sigma_{n} t} \leq 1.$ Hence, by Lebesgue's dominated convergence theorem, \[\phi_{X_{\infty}}(t) = \text{E}(\lim_{n \to +\infty}~ e^{it^TX_n} ) = \lim_{n \to +\infty}~ \text{E}(e^{it^TX_n}) =  \lim_{n \to +\infty}~ e^{it^T\mu_{n} - \frac{1}{2} t^T \Sigma_{n} t} = e^{it^T\mu_{\infty} - \frac{1}{2} t^T \Sigma_{\infty} t}.\]
\end{proof}

%
%
%
%
%
%	APPENDIX E1
%
%
%
%
\subsection{Proof of Proposition \ref{prop:diversity} 1)}
Let $x=(t_1^x, \dots, t_d^x) \in I^1 \times \dots \times I^d$. We want to show that $H(\nabla f(x)) =  H(\nabla g(x))$ where $f$ and $g$ are as per Proposition \ref{prop:diversity}, and $H$ is the entropy operator. Firstly, we note from Equation (\ref{eq:sgp_gradient}) that 
\begin{equation}
\label{eq:f_grad}
\nabla f(x) = \left(z_{t_1^x}^{1 \prime}, \dots, z_{t_d^x}^{d \prime}\right),
\end{equation}
where the joint law of the GP $(z_{t}^{j})_{t \in I^j}$ and its derivative $(z_{t}^{j \prime})_{t \in I^j}$ is provided in Proposition \ref{prop:derivative_processes}.
As the processes $\left(z_t^j, z_{t}^{j \prime}\right)_{t \in I^j}, ~ j \in [1..d]$ are assumed to be independent of each other, $\nabla f(x)$ is a Gaussian vector and its covariance matrix reads:
\begin{equation}
\Sigma_{\nabla f(x)} = -2 \frac{\text{d} \rho}{\text{d}x}(0) \text{I}_d,
\end{equation}
where $\text{I}_d$ is the $d \times d$ identity matrix. Hence,
\begin{equation}
H(\nabla f(x)) = \frac{d}{2}\left(1+ \ln(2\pi)\right) + \frac{1}{2} \ln \vert \Sigma_{\nabla f(x)} \vert.
\end{equation}
Secondly, let $e_j$ denote the $d$-dimensional vector whose $j$-th coordinate is $1$ and every other coordinate is $0$, and let $h \in \mathbb{R}$. As the proposition assumes the covariance function $k$ generates almost surely $\mathcal{C}^1$ surfaces, the vectors $\left(\frac{g(x+he_1) - g(x)}{h}, \dots, \frac{g(x+he_d) - g(x)}{h}\right)$ are Gaussian vectors converging almost surely as $h \to 0$. Moreover, their mean is $0$ and their covariance matrices have as element on the $i$-th row and $j$-th column ($i \neq j$):
\begin{align}
\label{eq:cross_seq}
& \text{cov}\left(\frac{g(x+he_i) - g(x)}{h}, \frac{g(x+he_j) - g(x)}{h}\right) =  \frac{\rho(2h^2)-2\rho(h^2)+\rho(0)}{h^2}
\end{align}
and as diagonal terms:
\begin{align}
\label{eq:diag_seq}
& \text{Var}\left(\frac{g(x+he_j) - g(x)}{h}\right)= 2\frac{\rho(0)-\rho(h^2)}{h^2}.
\end{align}
Taking the limit of Equations (\ref{eq:cross_seq}) and (\ref{eq:diag_seq})  using the first order Taylor expansion of $\rho$ (which the Proposition assumes is $\mathcal{C}^2$), we get that:
\begin{equation}
\label{eq:gp_surf_grad}
\Sigma_{\nabla g(x)} = -2 \frac{\text{d} \rho}{\text{d}x}(0) \text{I}_d = \Sigma_{\nabla f(x)},
\end{equation}
It then follows from Lemma \ref{lem2} that the limit $\nabla g(x)$ of $$\left(\frac{g(x+he_1) - g(x)}{h}, \dots, \frac{g(x+he_d) - g(x)}{h}\right)$$ is also a Gaussian vector, which proves that $H(\nabla f(x)) = H(\nabla g(x))$.
%
%
%
%
%
%	APPENDIX E2
%
%
%
%
\subsection{Proof of Proposition \ref{prop:diversity} 2)}
We start by stating and proving another lemma we will later use.
\begin{lemma}
\label{lem}
Let $A$ and $B$ be two $d$-dimensional jointly Gaussian vectors with diagonal covariance matrices $\Sigma_A$ and  $\Sigma_B$ respectively. Let $\Sigma_{A, B}$ be the cross-covariance matrix between $A$ and $B$, and let $\text{diag}(\Sigma_{A, B})$ be the diagonal matrix whose diagonal is that of $\Sigma_{A, B}$. Then:
\[\text{det}\left( \begin{bmatrix} \Sigma_A & \text{diag}(\Sigma_{A, B}) \\ \text{diag}(\Sigma_{A, B}) & \Sigma_B \end{bmatrix} \right) \geq \text{det}\left( \begin{bmatrix} \Sigma_A & \Sigma_{A, B} \\ \Sigma_{A, B}^T & \Sigma_B \end{bmatrix} \right).\]
\end{lemma}

\begin{proof}
Firstly we note that 
\[\text{det}\left( \begin{bmatrix} \Sigma_A & \text{diag}(\Sigma_{A, B}) \\ \text{diag}(\Sigma_{A, B}) & \Sigma_B \end{bmatrix} \right) = \text{det}(\Sigma_A)\text{det}\left(\Sigma_B - \text{diag}(\Sigma_{A, B}) \Sigma_A^{-1} \text{diag}(\Sigma_{A, B})\right)\]
and
\[\text{det}\left( \begin{bmatrix} \Sigma_A & \Sigma_{A, B} \\ \Sigma_{A, B} & \Sigma_B \end{bmatrix} \right) = \text{det}(\Sigma_A)\text{det}\left(\Sigma_B - \Sigma_{A, B}^T \Sigma_A^{-1} \Sigma_{A, B}\right).\]
As the matrix $\Sigma_A$ is positive semi-definite, $\text{det}(\Sigma_A) \geq 0$. The case $\text{det}(\Sigma_A) = 0$ is straight-forward. Thus we assume that $\text{det}(\Sigma_A) > 0$, so that all we need to prove is that \[\text{det}\left(\Sigma_B - \text{diag}(\Sigma_{A, B}) \Sigma_A^{-1} \text{diag}(\Sigma_{A, B})\right) \geq  \text{det}\left(\Sigma_B - \Sigma_{A, B}^T \Sigma_A^{-1} \Sigma_{A, B}\right).\]
Secondly, the matrix $\Sigma_{B \vert A}^{\text{diag}} :=  \Sigma_B - \text{diag}(\Sigma_{A, B}) \Sigma_A^{-1} \text{diag}(\Sigma_{A, B})$ being diagonal, its determinant is the product of its diagonal terms:
\[\text{det}(\Sigma_{B \vert A}^{\text{diag}}) = \prod_{i=1}^d \Sigma_{B \vert A}^{\text{diag}}[i,i] = \prod_{i=1}^d  \left(\Sigma_B[i,i] - \frac{\Sigma_{A,B}[i,i]^2}{\Sigma_A[i,i]}\right).\]
As for the matrix $\Sigma_{B \vert A}:= \Sigma_B - \Sigma_{A, B}^T \Sigma_A^{-1} \Sigma_{A, B}$, we note that it happens to be the covariance matrix of the (Gaussian) distribution of  B given A, and thus is positive semi-definite and admits a Cholesky decomposition $\Sigma_{B \vert A} = LL^T$. It follows that 
\begin{align}
\text{det}(\Sigma_{B \vert A}) = \prod_{i=1}^d L[i,i]^2 &\leq \prod_{i=1}^d \Sigma_{B \vert A}[i,i] = \prod_{i=1}^d  \left(\Sigma_B[i,i] - \sum_{j=1}^d \frac{\Sigma_{A,B}[j,i]^2}{\Sigma_A[j,j]}\right) \nonumber \\
&\leq  \prod_{i=1}^d  \left(\Sigma_B[i,i] - \frac{\Sigma_{A,B}[i,i]^2}{\Sigma_A[i,i]}\right) = \text{det}(\Sigma_{B \vert A}^{\text{diag}}),
\end{align}
where the first inequality results from the fact that $\Sigma_{B \vert A}[i,i] = \sum_{j=1}^{i} L[j, i]^2$ by definition of the Cholesky decomposition.
This proves that \[\text{det}\left(\Sigma_B - \text{diag}(\Sigma_{A, B}) \Sigma_A^{-1} \text{diag}(\Sigma_{A, B})\right) \geq  \text{det}\left(\Sigma_B - \Sigma_{A, B}^T \Sigma_A^{-1} \Sigma_{A, B}\right),\]
which as previously discussed concludes the proof of the lemma.
\end{proof}
\\
\textbf{Proof of Proposition \ref{prop:diversity} 2)}: Let $x=(t_1^x, \dots, t_d^x), ~ y=(t_1^y, \dots, t_d^y) \in I^1 \times \dots \times I^d$, $x \neq y$. We want to show that $I(\nabla f(x); \nabla f(y)) \leq  I(\nabla g(x); \nabla g(y))$ where $f$ and $g$ are as per Proposition \ref{prop:diversity}, and \[I(X;Y) = H(X)+ H(Y)- H(X, Y)\] is the mutual information between $X$ and $Y$. As we have proved that $ \forall~x, ~H(\nabla f(x)) =  H(\nabla g(x)),$ all we need to prove now is that \[H(\nabla f(x), \nabla f(y)) \geq H(\nabla g(x), \nabla g(y)).\]

\noindent Firstly, it follows from Equation (\ref{eq:f_grad}) and the fact that the \textit{derivative Gaussian processes} $(z_t^j, z_{t}^{j \prime})_{t \in I^j}$ are independent that $(\nabla f(x), \nabla f(y))$ is a jointly Gaussian vector. Moreover, the cross-covariance matrix $\Sigma_{\nabla f(x), \nabla f(y)}$ is diagonal with diagonal terms:
\begin{align}
\label{eq:cross_cov_f}
\Sigma_{\nabla f(x), \nabla f(y)}[i, i]= -2\left[\frac{\text{d} \rho}{\text{d}x}\left(\vert \vert  x-y \vert \vert^2_{L^2} \right)  + 2(t^x_i-t^y_i)^2\frac{\text{d}^2 \rho}{\text{d}x^2}\left(\vert \vert  x-y \vert \vert^2_{L^2}  \right) \right].
\end{align}

\noindent Secondly, it follows from a similar argument to the previous proof that $(\nabla g(x), \nabla g(y))$ is also a jointly Gaussian vector, and the terms $\Sigma_{\nabla g(x), \nabla g(y)}[i, j]$ are evaluated as limit of the cross-covariance terms $\text{cov}\left(\frac{g(x+he_i) - g(x)}{h}, \frac{g(y+he_j) - g(y)}{h}\right)$ as $h  \to 0.$ For $i = j$,
\begin{align}
\label{eq:cross_seq_3}
& \text{cov}\left(\frac{g(x+he_i) - g(x)}{h}, \frac{g(y+he_i) - g(y)}{h}\right) = \frac{1}{h^2}\bigg\{2\rho\left(\sum_{k} (t^x_k-t^y_k)^2\right)  \nonumber \\
& - \rho\left(\sum_{k \neq i} (t^x_k-t^y_k)^2+ (t^x_i+h-t^y_i)^2\right) -\rho\left(\sum_{k \neq i} (t^x_k-t^y_k)^2+ (t^x_i-h-t^y_i)^2\right)  \bigg\},
\end{align}
As $\rho$ is assumed to be $\mathcal{C}^2$, the below Taylor expansions around $h=0$ hold true:
\begin{align}
\label{eq:dl1}
&\rho\left(\sum_{k} (t^x_k-t^y_k)^2\right)  - \rho\left(\sum_{k \neq i} (t^x_k-t^y_k)^2+ (t^x_i-h-t^y_i)^2\right) =   2(t^x_i-t^y_i)h \frac{\text{d} \rho}{\text{d}x}\left(\sum_{k} (t^x_k-t^y_k)^2 \right) \\ \nonumber 
& - \left[ \frac{\text{d} \rho}{\text{d}x}\left(\sum_{k} (t^x_k-t^y_k)^2 \right)  + 2(t^x_i-t^y_i)^2\frac{\text{d}^2 \rho}{\text{d}x^2}\left(\sum_{k} (t^x_k-t^y_k)^2 \right) \right]h^2  + o(h^2)
\end{align}

\begin{align}
\label{eq:dl2}
&\rho\left(\sum_{k} (t^x_k-t^y_k)^2\right)  - \rho\left(\sum_{k \neq i} (t^x_k-t^y_k)^2+ (t^x_i+h-t^y_i)^2\right) =   -2(t^x_i-t^y_i)h \frac{\text{d} \rho}{\text{d}x}\left(\sum_{k} (t^x_k-t^y_k)^2 \right) \\ \nonumber 
& - \left[ \frac{\text{d} \rho}{\text{d}x}\left(\sum_{k} (t^x_k-y_k)^2 \right)  + 2(t^x_i-t^y_i)^2\frac{\text{d}^2 \rho}{\text{d}x^2}\left(\sum_{k} (t^x_k-t^y_k)^2 \right) \right]h^2  + o(h^2)
\end{align}
Plugging Equations (\ref{eq:dl1}) and (\ref{eq:dl2}) into Equation (\ref{eq:cross_seq_3}) and taking the limit we obtain: 
\begin{align}
\label{eq:cross_cov_g_1}
\Sigma_{\nabla g(x), \nabla g(y)}[i, i]&= -2\left[\frac{\text{d} \rho}{\text{d}x}\left(\vert \vert  x-y \vert \vert^2_{L^2} \right)  + 2(t^x_i-t^y_i)^2\frac{\text{d}^2 \rho}{\text{d}x^2}\left(\vert \vert  x-y \vert \vert^2_{L^2}  \right) \right] \nonumber \\
&= \Sigma_{\nabla f(x), \nabla f(y)}[i, i].
\end{align}
Similarly for $i \neq j$,
\begin{align}
\label{eq:cross_seq_2}
& \text{cov}\left(\frac{g(x+he_i) - g(x)}{h}, \frac{g(y+he_j) - g(y)}{h}\right) \nonumber \\
& = \frac{1}{h^2}\Bigg\{\rho\Bigg(\sum_{k \neq i, j} (t^x_k-t^y_k)^2+ (t^x_i+h-t^y_i)^2 + (t^x_j-h-t^y_j)^2\Bigg)   \nonumber \\
&-  \rho\left(\sum_{k \neq i} (t^x_k-t^y_k)^2+ (t^x_i+h-t^y_i)^2\right)  -  \rho\left(\sum_{k \neq j} (t^x_k-t^y_k)^2+ (t^x_i-h-t^y_j)^2\right) \nonumber \\
&+ \rho\left(\sum_{k} (t^x_k-t^y_k)^2\right)    \Bigg\},
\end{align}
and
\begin{align}
\label{eq:dl3}
& \rho\left(\sum_{k \neq i, j} (t^x_k-t^y_k)^2+ (t^x_i+h-t^y_i)^2 + (t^x_j-h-t^y_j)^2\right) - \rho\left(\sum_{k} (t^x_k-t^y_k)^2\right) \nonumber \\
& = \left[ 2\frac{\text{d} \rho}{\text{d}x}\left(\sum_{k} (t^x_k-t^y_k)^2  \right)  + 2\left((t^x_i- t^y_i) -  (t^x_j  - t^y_j)\right)^2  \times \frac{\text{d}^2 \rho}{\text{d}x^2}\left(\sum_{k} (t^x_k-t^y_k)^2 \right) \right]h^2 \nonumber \\ 
&+ 2\left(t^x_i-t^y_i - t^x_j + t^y_j\right) \frac{\text{d} \rho}{\text{d}x}\left(\sum_{k} (t^x_k-t^y_k)^2 \right)h  + o(h^2).
\end{align}
Plugging Equations (\ref{eq:dl1}), (\ref{eq:dl2}) and (\ref{eq:dl3}) in Equation (\ref{eq:cross_seq_2}) and taking the limit we obtain:
\begin{align}
\label{eq:cross_cov_g_2}
\Sigma_{\nabla g(x), \nabla g(y)}[i, j] = -4(t^x_i-t^y_i)(t^x_j-t^y_j)\frac{\text{d}^2 \rho}{\text{d}x^2}\left(\vert \vert  x-y \vert \vert^2_{L^2} \right).
\end{align}

To summarize, $(\nabla f(x), \nabla f(y))$ and $(\nabla g(x), \nabla g(y))$ are both jointly Gaussian vectors; $\nabla f(x)$, $\nabla g(x)$, $\nabla f(y)$, and $\nabla g(y)$ are (Gaussian) identically distributed with a diagonal covariance matrix; $\Sigma_{\nabla f(x), \nabla f(y)}$ is diagonal; $\Sigma_{\nabla g(x), \nabla g(y)}$ has the same diagonal as $\Sigma_{\nabla f(x), \nabla f(y)}$ but has possibly non-zero off-diagonal terms. Hence, it follows from Lemma \ref{lem} that the determinant of the auto-covariance matrix of $(\nabla f(x), \nabla f(y))$ is higher than that of the auto-covariance matrix of $(\nabla g(x), \nabla g(y))$; or equivalently the entropy of $(\nabla f(x), \nabla f(y))$ is higher than that of $(\nabla g(x), \nabla g(y))$ (as both are Gaussian vectors), which as previously discussed is sufficient to conclude that the mutual information between $\nabla f(x)$ and $\nabla f(y)$ is smaller than that between $\nabla g(x)$ and $\nabla g(y)$.
%
%
%
%
%
%	APPENDIX F
%
%
%
%
\section{Proof of Proposition \ref{prop:extension}}
\label{app:extension}
In this section, we prove Proposition \ref{prop:extension}, which we recall below.\\\\
\textbf{Proposition \ref{prop:extension} (Extension of the \textit{standard GP paradigm})}\\
Let $K \in \mathbb{N}^{*}$, let $I=[a_0, a_K]$ and $I_k= [a_{k-1}, a_k]$ be intervals with $a_0 < \dots < a_K$. Furthermore, let $m: I \to \mathbb{R}$ be a $\mathcal{C}^1$ function, $m_k$ the restriction of $m$ to $I_k$, $h: I \times I \to \mathbb{R}$ a $\mathcal{C}^3$ symmetric positive semi-definite function, and $h_k$ the restriction of $h$ to $I_k \times I_k$. If \[(z_t)_{t \in I} \sim \mathcal{SGP}(\{a_k\}, \{m_k\}, \{h_k\}),\] then \[ \forall ~ k \in [1..K], ~ (z_t)_{t \in I_k} \sim \mathcal{GP}(m, h).\]

\begin{proof}\\
To prove Proposition \ref{prop:extension}, we consider the \textit{string derivative Gaussian process} (Theorem \ref{theo:sgp}) $(SD_t)_{t \in I}$, $SD_t=(z_t, z_t^\prime)$ with unconditional string mean and covariance functions as per Proposition \ref{prop:extension} and prove that its restrictions on the intervals $I_k=[a_{k-1}, a_k]$ are \textit{derivative Gaussian processes} with the same mean function $m$ and covariance function $h$. Proposition \ref{prop:derivative_processes}(B) will then allow us to conclude that $(z_t)_{t \in I_k}$ are GPs with mean $m$ and covariance function $h$.

Let $t_1, \dots, t_n \in ]a_{k-1}, a_k[$ and let $p_{D}(x_{a_{k-1}})$ (respectively $p_{D}(x_{a_{k}} \vert x_{a_{k-1}})$ and\\ $p_{D}( x_{t_1}, \dots, x_{t_n}\vert x_{a_{k-1}}, x_{a_{k}})$) denote the pdf of the value of the \textit{derivative Gaussian process} with mean function $m$ and covariance function $h$ at $a_{k-1}$ (respectively its value at $a_k$ conditional on its value at $a_{k-1}$, and its values at $t_1, \dots, t_n$ conditional on its values at $a_{k-1}$ and $a_{k}$). Saying that the restriction of the \textit{string derivative Gaussian process} $(SD_t)$ on $[a_{k-1}, a_k]$ is the \textit{derivative Gaussian process} with mean $m$ and covariance $h$ is equivalent to saying that all finite dimensional marginals of the \textit{string derivative Gaussian process} $p_{SD}(x_{a_{k-1}}, x_{t_1}, \dots, x_{t_n}, x_{a_k}),~ t_i \in [a_{k-1}, a_k], ~$ factorise as\footnote{We emphasize that the terms on the right hand-side of this equation involve $p_{D}$ not $p_{SD}$.}: \[p_{SD}(x_{a_{k-1}}, x_{t_1}, \dots, x_{t_n}, x_{a_k}) = p_{D}(x_{a_{k-1}})p_{D}(x_{a_{k}} \vert x_{a_{k-1}})p_{D}( x_{t_1}, \dots, x_{t_n}\vert x_{a_{k-1}}, x_{a_{k}}).\]
Moreover, we know from Theorem \ref{theo:sgp} that by design, $p_{SD}(x_{a_{k-1}}, x_{t_1}, \dots, x_{t_n}, x_{a_k})$ factorises as \[p_{SD}(x_{a_{k-1}}, x_{t_1}, \dots, x_{t_n}, x_{a_k}) = p_{SD}(x_{a_{k-1}})p_{D}(x_{a_{k}} \vert x_{a_{k-1}})p_{D}( x_{t_1}, \dots, x_{t_n}\vert x_{a_{k-1}}, x_{a_{k}}).\]
 In other words, all we need to prove is that \[p_{SD}(x_{a_k})=p_{D}(x_{a_k})\] for every boundary time, which we will do by induction. We note by integrating out every boundary condition but the first in $p_{b}$ (as per Theorem \ref{theo:sgp} (a-1)) that \[p_{SD}(x_{a_0})=p_{D}(x_{a_0}).\]If we assume that $p_{SD}(x_{a_{k-1}})=p_{D}(x_{a_{k-1}})$ for some $k>0$, then as previously discussed the restriction of the \textit{string derivative Gaussian process} on $[a_{k-1}, a_k]$ will be the \textit{derivative Gaussian process} with the same mean and covariance functions, which will imply that $p_{SD}(x_{a_{k}})=p_{D}(x_{a_{k}})$. This concludes the proof.
\end{proof}

%
%
%
%
%
%	APPENDIX G
%
%
%
%
\section{Proof of Lemma \ref{lem:gaussian_message}}
\label{app:gaussian_message}
In this section, we will prove  Lemma \ref{lem:gaussian_message} that we recall below.\\\\
\textbf{Lemma \ref{lem:gaussian_message}} Let $X$ be a multivariate Gaussian with mean $\mu_X$ and covariance matrix $\Sigma_X$. If conditional on $X$, $Y$ is a multivariate Gaussian with mean $MX + A$  and covariance matrix $\Sigma_Y^c$ where $M$, $A$ and $\Sigma_Y^c$ do not depend on $X$, then $(X, Y)$ is a jointly Gaussian vector with mean
$$\mu_{X;Y}=\begin{bmatrix} \mu_X \\ M\mu_X + A \end{bmatrix},$$
and covariance matrix
$$\Sigma_{X;Y}=\begin{bmatrix} \Sigma_X & \Sigma_X M^T \\ M\Sigma_X & \Sigma_Y^c + M \Sigma_X M^T\end{bmatrix}.$$

\begin{proof}
To prove this lemma we introduce two vectors $\tilde{X}$ and $\tilde{Y}$ whose lengths are the same as those of $X$ and $Y$ respectively, and such that $(\tilde{X}, \tilde{Y})$ is jointly Gaussian with mean $\mu_{X;Y}$ and covariance matrix $\Sigma_{X;Y}$. We then prove that the (marginal) distribution of $\tilde{X}$ is the same as the distribution of $X$ and that the distribution of $\tilde{Y}|\tilde{X}=x$ is the same as $Y|X=x$ for any $x$, which is sufficient to conclude that $(X, Y)$ and $(\tilde{X}, \tilde{Y})$ have the same distribution.

It is obvious from the joint $(\tilde{X}, \tilde{Y})$ that $\tilde{X}$ is Gaussian distribution with mean $\mu_X$ and covariance matrix $\Sigma_X$. As for the distribution of $\tilde{Y}$ conditional on $\tilde{X}=x$, it follows from the usual Gaussian identities that it is Gaussian with mean 
\[ M\mu_X + c + M\Sigma_X \Sigma_X^{-1} (x-\mu_X) = Mx+c,
\]
and covariance matrix
\[
 \Sigma_Y^c + M \Sigma_X M^T - M\Sigma_X \Sigma_X^{-1}\Sigma_X^TM^T = \Sigma_Y^c,
\]
which is the same distribution as that of $Y|X=x$ since the covariance matrix $\Sigma_X$ is symmetric. This concludes our proof.
\end{proof}
%
%
%
%
%
%	APPENDIX H
%
%
%
%
\section{Proof of Proposition \ref{prop:is_gp}}
\label{app:is_gp}
In this section we will prove that \textit{string GPs} with link function $\phi_s$ are GPs, or in other words that if $f$ is a \textit{string GP} indexed on $\mathbb{R}^d, ~d>0$ with link function $\phi_s(x_1, \dots, x_d) = \sum_{j=1}^d x_j$, then $(f(x_1), \dots, f(x_n))$ has a multivariate Gaussian distribution for every set of distinct points $x_1, \dots, x_n \in \mathbb{R}^d$. \\

\begin{proof}
As the sum of independent Gaussian processes is a Gaussian process, a sufficient condition for additively separable \textit{string GPs} to be GPs in dimensions $d>1$ is that \textit{string GPs} be GPs in dimension $1$. Hence, all we need to do is to prove that \textit{string GPs} are GPs in dimension $1$.\\

Let $(z_t^j, z_t^{j \prime})_{t \in I^j}$ be a string derivative GP in dimension $1$, with boundary times $a_0^j, \dots, a_{K^j}^j$, and unconditional string mean and covariance functions $m_k^j$ and $k_k^j$ respectively.  We want to prove that $(z_{t_1}^j, \dots, z_{t_n}^j)$ is jointly Gaussian for any $t_1, \dots, t_n \in I^j$.
\subsubsection{Step 1 $\left(z_{a_0}^j, z_{a_0}^{j \prime}, \dots, z_{a_{K^j}}^j, z_{a_{K^j}}^{j \prime}\right)$ is jointly Gaussian}
We first prove recursively that the vector $\left(z_{a_0}^j, z_{a_0}^{j \prime}, \dots, z_{a_{K^j}}^j, z_{a_{K^j}}^{j \prime}\right)$ is jointly Gaussian. We note from Theorem \ref{theo:sgp} that $(z_t^j, z_t^{j \prime})_{t \in [a_0, a_1]}$ is the \textit{derivative Gaussian process} with mean $m_1^j$ and covariance function $k_1^j$. Hence, $(z_{a_0}^j, z_{a_0}^{j \prime}, z_{a_1}^j, z_{a_1}^{j \prime})$ is jointly Gaussian. Moreover, let us assume that $\mathcal{B}_{k-1}:=(z_{a_0}^j, z_{a_0}^{j \prime}, \dots, z_{a_{k-1}}^j, z_{a_{k-1}}^{j \prime})$ is jointly Gaussian for some $k>1$. Conditional on $\mathcal{B}_{k-1}$,  $(z_{a_k}^j, z_{a_k}^{j \prime})$ is Gaussian with covariance matrix independent of $\mathcal{B}_{k-1}$, and with mean \[\begin{bmatrix}m_k^j(a_{k}^j) \\\frac{dm_k^j}{dt}(a_{k}^j) \end{bmatrix} + {}^j_k\textbf{K}_{a_{k}^j; a_{k-1}^j}~{}^j_k\textbf{K}_{a_{k-1}; a_{k-1}^j}^{-1}\begin{bmatrix}z^j_{a_{k-1}^j}-m_k^j(a_{k-1}^j) \\z^{j \prime}_{a_{k-1}^j} - \frac{dm_k^j}{dt}(a_{k-1}^j) \end{bmatrix},\] which depends linearly on $(z_{a_0}^j, z_{a_0}^{j \prime}, \dots, z_{a_{k-1}}^j, z_{a_{k-1}}^{j \prime})$. Hence by Lemma \ref{lem:gaussian_message}, $$(z_{a_0}^j, z_{a_0}^{j \prime}, \dots, z_{a_k}^j, z_{a_k}^{j \prime})$$ is jointly Gaussian.
\subsubsection{Step 2 $(z_{a_0}^j, z_{a_0}^{j \prime}, \dots, z_{a_{K^j}}^j, z_{a_{K^j}}^{j \prime}, \dots, z_{t_i^k}^j, z_{t_i^k}^{j \prime}, \dots)$ is jointly Gaussian}
Let $t^k_1, \dots, t^k_{n^k} \in ]a_{k-1}^j, a_k^j[, ~ k \leq K^j$ be distinct string times. We want to prove that the vector $(z_{a_0}^j, z_{a_0}^{j \prime}, \dots, z_{a_{K^j}}^j, z_{a_{K^j}}^{j \prime}, \dots, z_{t_i^k}^j, z_{t_i^k}^{j \prime}, \dots)$ where all boundary times are represented, and for any finite number of string times is jointly Gaussian. Firstly, we have already proved that \\$(z_{a_0}^j, z_{a_0}^{j \prime}, \dots, z_{a_{K^j}}^j, z_{a_{K^j}}^{j \prime})$ is jointly Gaussian. Secondly, we note from Theorem \ref{theo:sgp} that conditional on $(z_{a_0}^j, z_{a_0}^{j \prime}, \dots, z_{a_{K^j}}^j, z_{a_{K^j}}^{j \prime})$, $( \dots, z_{t_i^k}^j, z_{t_i^k}^{j \prime}, \dots)$ is a Gaussian vector whose covariance matrix does not depend on $(z_{a_0}^j, z_{a_0}^{j \prime}, \dots, z_{a_{K^j}}^j, z_{a_{K^j}}^{j \prime})$, and whose mean depends linearly on $$\left(z_{a_0}^j, z_{a_0}^{j \prime}, \dots, z_{a_{K^j}}^j, z_{a_{K^j}}^{j \prime}\right).$$ Hence,  $$\left(z_{a_0}^j, z_{a_0}^{j \prime}, \dots, z_{a_{K^j}}^j, z_{a_{K^j}}^{j \prime}, \dots, z_{t_i^k}^j, z_{t_i^k}^{j \prime}, \dots \right)$$ is jointly Gaussian (by Lemma \ref{lem:gaussian_message}).
\subsubsection{Step 3 $(z_{t_1}^j, \dots, z_{t_n}^j)$ is jointly Gaussian}
$(z_{t_1}^j, z_{t_1}^{j \prime}, \dots, z_{t_n}^j, z_{t_n}^{j \prime})$ is jointly Gaussian as it can be regarded as the marginal of some joint distribution of the form $(z_{a_0}^j, z_{a_0}^{j \prime}, \dots, z_{a_{K^j}}^j, z_{a_{K^j}}^{j \prime}, \dots, z_{t_i^k}^j, z_{t_i^k}^{j \prime}, \dots)$. Hence, its marginal $(z_{t_1}^j, \dots, z_{t_n}^j)$ is also jointly Gaussian, which concludes our proof.
\end{proof}

%
%
%
%
%
%	APPENDIX H
%
%
%
%
\section{Derivation of Global String GP Mean and Covariance Functions}
\label{app:global_mean_cov}
We begin with \emph{derivative string GPs} indexed on $\mathbb{R}$. Extensions to \emph{membrane GPs} are easily achieved for a broad range of link functions. In our exposition, we focus on the class of \textit{elementary symmetric polynomials} (\cite{macdonald}). In addition to containing the link function $\phi_s$ previously introduced, this family of polynomials yields global covariance structures that have many similarities with existing kernel approaches, which we discuss in Section \ref{sct:ex_kern}.

For $n\leq d$, the $n$-th order elementary symmetric polynomial is given by
\begin{align}
e_0(x_1, \dots, x_d) := 1, ~~~ \forall 1 \leq n \leq d ~~ e_n(x_1, \dots, x_d)= \sum_{1 \leq j_1 < j_2 < \dots < j_n \leq d} ~~\prod_{k=1}^n x_{j_k}.
\end{align}
As an illustration, \[e_1(x_1, \dots, x_d) = \sum_{j=1}^{d} x_j = \phi_s(x_1, \dots, x_d),\] \[e_2(x_1, \dots, x_d) = x_1x_2 + x_1x_3 + \dots +  x_1x_d + \dots +  x_{d-1}x_d,\]\[\dots\]
\[e_d(x_1, \dots, x_d) = \prod_{j=1}^{d} x_j = \phi_p(x_1, \dots, x_d).\]
Let $f$ denote a \emph{membrane GP} indexed on $\mathbb{R}^d$ with link function $e_n$ and by $(z^1_t), \dots, (z^d_t)$ its independent building block \emph{string GPs}. Furthermore, let $m_k^j$ and $k_k^j$ denote  the unconditional mean and covariance functions corresponding to the $k$-th string of $(z^j_t)$ defined on $[a_{k-1}^j, a_k^j]$. Finally, let us define 
\begin{equation*}
\boxed{\bar{m}^j(t):= \text{E}(z^j_t),~~~\bar{m}^{j\prime}(t):= \text{E}(z^{j\prime}_t),}
\end{equation*}
the global mean functions of the $j$-th building block \emph{string GP} and of its derivative, where $\forall t \in ~I^j$. It follows from the independence of the building block \emph{string GPs} $(z^j_t)$ that:
\begin{equation*}
\boxed{\bar{m}^f(t_1, \dots, t_d) := \text{E}(f(t_1, \dots, t_d))=e_n(\bar{m}^1(t_1), \dots, \bar{m}^d(t_d)).}
\end{equation*}
Moreover, noting that $$\frac{\partial e_n}{\partial x_j} = e_{n-1}(x_1, \dots, x_{j-1}, x_{j+1}, \dots, x_d),$$ it follows that:
\begin{align*}
\boxed{\bar{m}^{\nabla f}(t_1, \dots, t_d) :=\text{E}\left(\nabla f(t_1, \dots, t_d)\right) = \begin{bmatrix}\bar{m}^{1 \prime}(t_1) e_{n-1}\left(\bar{m}^{2}(t_2), \dots, \bar{m}^{d}(t_d)\right) \\ \dots \\ \bar{m}^{d \prime}(t_d) e_{n-1}\left(\bar{m}^{1}(t_1), \dots, \bar{m}^{d-1}(t_{d-1})\right) \end{bmatrix}.}
\end{align*}
Furthermore, for any $u_j, v_j \in I^j$ we also have that 
\begin{align*}
\boxed{\text{cov}\left(f(u_1, \dots, u_d), f(v_1, \dots, v_d)\right) = e_n\left(\text{cov}(z_{u_1}^1, z_{v_1}^1), \dots, \text{cov}(z_{u_d}^d, z_{v_d}^d)\right),}
\end{align*}
\begin{align*}
\boxed{\text{cov}\bigg(\frac{\partial f}{\partial x_i} (u_1, \dots, u_d), f(v_1, \dots, v_d)\bigg) = e_n\left(\text{cov}(z_{u_1}^1, z_{v_1}^1), \dots, \text{cov}(z_{u_i}^{i \prime}, z_{v_i}^i), \dots, \text{cov}(z_{u_d}^d, z_{v_d}^d)\right),}
\end{align*}
and for $i \leq j$
\begin{empheq}[box=\widefbox]{align}
&\text{cov}\left(\frac{\partial f}{\partial x_i} (u_1, \dots, u_d), \frac{\partial f}{\partial x_j} (v_1, \dots, v_d)\right) \nonumber\\
&=\begin{cases}
    e_n\left(\text{cov}(z_{u_1}^1, z_{v_1}^1), \dots, \text{cov}(z_{u_i}^{i \prime}, z_{v_i}^i), \dots, \text{cov}(z_{u_j}^{j \prime}, z_{v_j}^j), \dots, \text{cov}(z_{u_d}^d, z_{v_d}^d)\right), & \text{if $i<j$}\\
     e_n\left(\text{cov}(z_{u_1}^1, z_{v_1}^1), \dots, \text{cov}(z_{u_i}^{i \prime}, z_{v_i}^{i \prime}), \dots, \text{cov}(z_{u_d}^d, z_{v_d}^d)\right), & \text{if $i=j$}
  \end{cases}.\nonumber
\end{empheq}
Overall, for any elementary symmetric polynomial link function, multivariate mean and covariance functions are easily deduced from the previously boxed equations and the univariate quantities 
\[\bar{m}^j(u),~\bar{m}^{j\prime}(u),~\text{and}~ {}^j\bar{\textbf{K}}_{u; v} :=
\begin{bmatrix} 
\text{cov}(z_u^j, z_v^j) & \text{cov}(z_u^{j}, z_v^{j \prime}) \\ 
\text{cov}(z_u^{j \prime}, z_v^{j})  & \text{cov}(z_u^{j \prime}, z_v^{j \prime})
\end{bmatrix}={}^j\bar{\textbf{K}}_{v; u}^T,\]
which we now derive. In this regards, we will need the following lemma.
\begin{lemma}
\label{lem:gaussian_message}
Let $X$ be a multivariate Gaussian with mean $\mu_X$ and covariance matrix $\Sigma_X$. If conditional on $X$, $Y$ is a multivariate Gaussian with mean $MX + A$  and covariance matrix $\Sigma_Y^c$ where $M$, $A$ and $\Sigma_Y^c$ do not depend on $X$, then $(X, Y)$ is a jointly Gaussian vector with mean
$$\mu_{X;Y}=\begin{bmatrix} \mu_X \\ M\mu_X + A \end{bmatrix},$$
and covariance matrix
$$\Sigma_{X;Y}=\begin{bmatrix} \Sigma_X & \Sigma_X M^T \\ M\Sigma_X & \Sigma_Y^c + M \Sigma_X M^T\end{bmatrix}.$$
\end{lemma}
\begin{proof}
See \ref{app:gaussian_message}.
\end{proof}
%
%
%
%
%
%
%
%	SUBSUBSECTION: STRING GP MEAN FUNCTIONS
%
%
%
%
%
%
%
\subsubsection{Global String GP Mean Functions}
\label{sct:str_gp_m}
We now turn to evaluating the univariate global mean functions $\bar{m}^{j}$ and $\bar{m}^{j \prime}$. We start with boundary times and then generalise to other times.\\\\
\textbf{\underline{Boundary times}:} We note from Theorem \ref{theo:sgp} that the restriction $\left(z_t^j, z_t^{j \prime}\right)_{t \in  [a_0^j, a_{1}^j]}$ is the \textit{derivative Gaussian process} with mean and covariance functions $m^j_1$ and $k^j_1$. Thus,
\begin{empheq}[box=\widefbox]{align}
\label{eq:init_b_mean}
\begin{bmatrix}\bar{m}^{j}(a_0^j) \\ \bar{m}^{j \prime}(a_0^j)\end{bmatrix}= \begin{bmatrix}m^j_1(a_0^j) \\ \frac{d m^j_1}{dt}(a_0^j)\end{bmatrix}, ~~~~\text{and}~~~~\begin{bmatrix}\bar{m}^{j}(a_1^j) \\ \bar{m}^{j \prime}(a_1^j)\end{bmatrix}= \begin{bmatrix}m^j_1(a_1^j) \\ \frac{d m^j_1}{dt}(a_1^j)\end{bmatrix}.\nonumber
\end{empheq}
 For $k>1$, we recall that conditional on $\left(z^j_{a_{k-1}^j}, z^{j \prime}_{a_{k-1}^j}\right)$, $\left(z^j_{a_{k}^j}, z^{j \prime}_{a_{k}^j}\right)$ is Gaussian with mean 
 \[\begin{bmatrix}m_k^j(a_{k}^j) \\\frac{dm_k^j}{dt}(a_{k}^j) \end{bmatrix} + {}^j_k\textbf{K}_{a_{k}^j; a_{k-1}^j}~{}^j_k\textbf{K}_{a_{k-1}^j; a_{k-1}^j}^{-1}\begin{bmatrix}z^j_{a_{k-1}^j}-m_k^j(a_{k-1}^j) \\z^{j \prime}_{a_{k-1}^j} - \frac{dm_k^j}{dt}(a_{k-1}^j) \end{bmatrix},\]\[ \text{with}~{}^j_k\textbf{K}_{u; v}=\begin{bmatrix} k_k^j(u, v) & \frac{\partial k_k^j}{\partial y}(u, v) \\ \frac{\partial k_k^j}{\partial x}(u, v)  & \frac{\partial^2 k_k^j}{\partial x \partial y}(u, v) \end{bmatrix}.\]
It then follows from the law of total expectations that for all $k>1$
\begin{empheq}[box=\widefbox]{align}
\begin{bmatrix}\bar{m}^{j}(a_{k}^j) \\ \bar{m}^{j \prime}(a_{k}^j)\end{bmatrix} = \begin{bmatrix}m_k^j(a_{k}^j) \\\frac{dm_k^j}{dt}(a_{k}^j) \end{bmatrix} + {}^j_k\textbf{K}_{a_{k}^j; a_{k-1}^j}~{}^j_k\textbf{K}_{a_{k-1}^j; a_{k-1}^j}^{-1} \begin{bmatrix}\bar{m}^{j}(a_{k-1}^j)-m_k^j(a_{k-1}^j) \\\bar{m}^{j \prime}(a_{k-1}^j) - \frac{dm_k^j}{dt}(a_{k-1}^j) \end{bmatrix}.\nonumber
\end{empheq}
\textbf{\underline{String times}:} As for non-boundary times $t \in ]a_{k-1}^j, a_k^j[$, conditional on $\left(z_{a_{k-1}}^j, z_{a_{k-1}}^{j \prime}\right)$ and $\left(z_{a_k}^j, z_{a_k}^{j \prime}\right)$, $\left(z_t^j, z_t^{j \prime}\right)$ is Gaussian with mean 
$$\begin{bmatrix}m_k^j(t) \\\frac{dm_k^j}{dt}(t) \end{bmatrix} + {}^j_k\textbf{K}_{t; \left(a_{k-1}^j, a_{k}^j\right)}~{}^j_k\textbf{K}_{\left(a_{k-1}^j, a_{k}^j\right); \left(a_{k-1}^j, a_{k}^j\right)}^{-1}\begin{bmatrix}z^j_{a_{k-1}^j}-m_k^j(a_{k-1}^j) \\z^{j \prime}_{a_{k-1}^j} - \frac{dm_k^j}{dt}(a_{k-1}^j) \\ z^j_{a_{k}^j}-m_k^j(a_{k}^j) \\z^{j \prime}_{a_{k}^j} - \frac{dm_k^j}{dt}(a_{k}^j)\end{bmatrix},$$
with $${}^j_k\textbf{K}_{\left(a_{k-1}^j, a_{k}^j\right); \left(a_{k-1}^j, a_{k}^j\right)} = 
\begin{bmatrix} 
{}^j_k\textbf{K}_{a_{k-1}^j; a_{k-1}^j} & {}^j_k\textbf{K}_{a_{k-1}^j; a_k^j} \\
{}^j_k\textbf{K}_{a_k^j; a_{k-1}^j} & {}^j_k\textbf{K}_{a_k^j; a_k^j} \\
\end{bmatrix}$$
and 
$${}^j_k\textbf{K}_{t; (a_{k-1}^j, a_{k}^j)} = 
\begin{bmatrix} 
{}^j_k\textbf{K}_{t; a_{k-1}^j} & {}^j_k\textbf{K}_{t; a_k^j}
\end{bmatrix}.$$
Hence, using once again the law of total expectation, it follows that for any $t \in ]a_{k-1}^j, a_k^j[$,
\begin{empheq}[box=\widefbox]{align}
\label{eq:nb_mean}
\begin{bmatrix}\bar{m}^{j}(t) \\ \bar{m}^{j \prime}(t)\end{bmatrix} =\begin{bmatrix}m_k^j(t) \\\frac{dm_k^j}{dt}(t) \end{bmatrix} + {}^j_k\textbf{K}_{t; (a_{k-1}^j, a_{k}^j)}~{}^j_k\textbf{K}_{(a_{k-1}^j, a_{k}^j); (a_{k-1}^j, a_{k}^j)}^{-1}\begin{bmatrix}\bar{m}^{j}(a_{k-1}^j)-m_k^j(a_{k-1}^j) \\\bar{m}^{j \prime}(a_{k-1}^j) - \frac{dm_k^j}{dt}(a_{k-1}^j) \\ \bar{m}^{j}(a_{k}^j)-m_k^j(a_{k}^j) \\\bar{m}^{j \prime}(a_{k}^j) - \frac{dm_k^j}{dt}(a_{k}^j)\end{bmatrix}. \nonumber
\end{empheq}
We note in particular that when $\forall ~j, k, ~m_k^j =0$, it follows that $\bar{m}^{j}=0, \bar{m}^{f}=0, \bar{m}^{\nabla f}=0$.
%
%
%
%
%
%	SUBSUBSECTION: STRING GP KERNELS
%
%
%
%
%
\subsubsection{Global String GP Covariance Functions}
\label{sct:str_gp_k}
As for the evaluation of  ${}^j\bar{\textbf{K}}_{u, v}$, we start by noting that the covariance function of a univariate \emph{string GP} is the same as that of another \emph{string GP} whose strings have the same unconditional kernels but unconditional mean functions $m_k^j=0$, so that to evaluate univariate \emph{string GP} kernels we may assume that $\forall~j,k, ~m_k^j=0$ without loss of generality. We start with the case where $u$ and $v$ are both boundary times, after which we will generalise to other times.\\\\
\textbf{\underline{Boundary times}:} As previously discussed, the restriction $\left(z_t^j, z_t^{j \prime}\right)_{t \in  [a_0^j, a_{1}^j]}$ is the \textit{derivative Gaussian process} with mean $0$ and covariance function $k^j_1$. Thus,
\begin{align}
{}^j\bar{\textbf{K}}_{a_0^j; a_0^j} = {}^j_1\textbf{K}_{a_0^j; a_0^j}, ~~~~ {}^j\bar{\textbf{K}}_{a_1^j; a_1^j} = {}^j_1\textbf{K}_{a_1^j; a_1^j}, ~~~~ {}^j\bar{\textbf{K}}_{a_0^j; a_1^j} = {}^j_1\textbf{K}_{a_0^j; a_1^j}.
\end{align}
We recall that conditional on the boundary conditions at or prior to $a_{k-1}^j$, $\left(z_{a_k^j}^j, z_{a_k^j}^{j \prime}\right)$ is Gaussian with mean 
\[
 {}^b_kM
\begin{bmatrix}
z^j_{a_{k-1}^j} \\
z^{j \prime}_{a_{k-1}^j}
\end{bmatrix}~~~~\text{with}~~~~
{}^b_kM= {}^j_k\textbf{K}_{a_{k}^j; a_{k-1}^j}~{}^j_k\textbf{K}_{a_{k-1}^j; a_{k-1}^j}^{-1},
\] 
and covariance matrix 
\[
{}^b_k\Sigma =  {}^j_k\textbf{K}_{a_k^j; a_k^j} - {}^b_kM ~{}^j_k\textbf{K}_{a_{k-1}^j; a_k^j}.
\] 
Hence using Lemma \ref{lem:gaussian_message} with $M=\begin{bmatrix}{}^b_kM & 0 & \dots & 0\end{bmatrix}$ where there are $(k-1)$ null block $2 \times 2$ matrices, and noting that $\left(z_{a_0^j}^j, z_{a_0^j}^{j \prime}, \dots, z_{a_{k-1}^j}^j, z_{a_{k-1}^j}^{j \prime}\right)$ is jointly Gaussian,  it follows that the vector $\left(z_{a_0^j}^j, z_{a_0^j}^{j \prime}, \dots, z_{a_k^j}^j, z_{a_k^j}^{j \prime}\right)$ is jointly Gaussian, that $\left(z_{a_k^j}^j, z_{a_k^j}^{j \prime}\right)$ has covariance matrix
\begin{empheq}[box=\widefbox]{align}
{}^j\bar{\textbf{K}}_{a_k^j; a_k^j} = {}^b_k\Sigma + {}^b_kM~{}^j\bar{\textbf{K}}_{a_{k-1}^j; a_{k-1}^j}~{}^b_kM^T, \nonumber
\end{empheq}
and that the covariance matrix between the boundary conditions at $a_k^j$ and at any earlier boundary time $a_l^j, ~ l<k$ reads:
\begin{empheq}[box=\widefbox]{align}
 {}^j\bar{\textbf{K}}_{a_k^j; a_l^j} = {}^b_kM~{}^j\bar{\textbf{K}}_{a_{k-1}^j; a_l^j}. \nonumber
\end{empheq}
\textbf{\underline{String times}:}  Let $u \in [a_{p-1}^j, a_p^j], ~ v\in [a_{q-1}^j, a_q^j]$. By the law of total expectation, we have that
\[
{}^j\bar{\textbf{K}}_{u; v} := \text{E}\left( 
\begin{bmatrix}
z^j_{u} \\
z^{j \prime}_{u}
\end{bmatrix}
\begin{bmatrix}
z^j_{v} & z^{j \prime}_{v}
\end{bmatrix}\right) =
\text{E}\left( \text{E}\left( 
\begin{bmatrix}
z^j_{u} \\
z^{j \prime}_{u}
\end{bmatrix}
\begin{bmatrix}
z^j_{v} & z^{j \prime}_{v}
\end{bmatrix}\bigg| \mathcal{B}(p, q)\right) \right),
\]
where $\mathcal{B}(p, q)$ refers to the boundary condtions at the boundaries of the $p$-th and $q$-th strings, in other words $\bigg\{z^j_{x}, z^{j \prime}_{x}, ~x \in \big\{a_{p-1}^j, a_p^j,a_{q-1}^j, a_q^j \big\}\bigg\}$. Furthermore,  using the definition of  the covariance matrix under the conditional law, it follows that
\begin{equation}
\label{eq:tot_kern_cond_cov}
\text{E}\bigg( 
\begin{bmatrix}
z^j_{u} \\
z^{j \prime}_{u}
\end{bmatrix}
\begin{bmatrix}
z^j_{v} & z^{j \prime}_{v}
\end{bmatrix}\bigg| \mathcal{B}(p, q)\bigg) = {}^j_c\bar{\textbf{K}}_{u; v} + \text{E}\bigg( 
\begin{bmatrix}
z^j_{u} \\
z^{j \prime}_{u}
\end{bmatrix}
\bigg| \mathcal{B}(p, q)\bigg)
 \text{E}\bigg( 
\begin{bmatrix}
z^j_{v} & z^{j \prime}_{v}
\end{bmatrix}\bigg| \mathcal{B}(p, q)\bigg),
\end{equation}
where ${}^j_c\bar{\textbf{K}}_{u; v}$ refers to the covariance matrix between $(z^j_u, z^{j \prime}_u)$ and $(z^j_v, z^{j \prime}_v)$ conditional on the boundary conditions $\mathcal{B}(p, q)$, and can be easily evaluated from Theorem \ref{theo:sgp}. In particular,
\begin{equation}
\label{eq:kern_cond_cov}
\text{if}~ p \neq q, ~{}^j_c\bar{\textbf{K}}_{u; v} =0, ~\text{and if}~ p=q, ~{}^j_c\bar{\textbf{K}}_{u; v} = {}^j_p\textbf{K}_{u;v} - {}^j_p\Lambda_u
\begin{bmatrix}
{}^j_p\textbf{K}_{v;a_{p-1}^j}^T 
\\  {}^j_p\textbf{K}_{v;a_p^j}^T 
\end{bmatrix},
\end{equation}
where 
\[\forall ~ x, l, ~{}^j_l\Lambda_x = 
\begin{bmatrix}
{}^j_l\textbf{K}_{x;a_{l-1}^j} &  {}^j_l\textbf{K}_{x;a_l^j} 
\end{bmatrix} 
\begin{bmatrix} 
{}^j_l\textbf{K}_{a_{l-1}^j;a_{l-1}^j} &  {}^j_l\textbf{K}_{a_{l-1}^j;a_l^j} \\
{}^j_l\textbf{K}_{a_l^j;a_{l-1}^j} &  {}^j_l\textbf{K}_{a_l^j;a_l^j}
\end{bmatrix}^{-1}.\]
We also note that \[
\text{E}\bigg( 
\begin{bmatrix}
z^j_{u} \\
z^{j \prime}_{u}
\end{bmatrix}
\bigg| \mathcal{B}(p, q)\bigg) = {}^j_p\Lambda_u 
\begin{bmatrix}
z^j_{a_{p-1}^j} \\
z^{j \prime}_{a_{p-1}^j} \\
z^j_{a_p^j} \\
z^{j \prime}_{a_p^j} \\
\end{bmatrix}~\text{and}~~
\text{E}\bigg( 
\begin{bmatrix}
z^j_{v} & z^{j \prime}_{v}
\end{bmatrix}
\bigg| \mathcal{B}(p, q)\bigg) = 
\begin{bmatrix}
z^j_{a_{q-1}^j} & z^{j \prime}_{a_{q-1}^j} & z^j_{a_q^j} & z^{j \prime}_{a_q^j} 
\end{bmatrix}
{}^j_q\Lambda_v ^T.
\]
Hence, taking the expectation with respect to the boundary conditions on both sides of Equation (\ref{eq:tot_kern_cond_cov}), we obtain:
\begin{empheq}[box=\widefbox]{align}
\forall ~u \in [a^j_{p-1}, a^j_{p}], ~  v \in [a^j_{q-1}, a^j_{q}],~~
{}^j\bar{\textbf{K}}_{u; v} = {}^j_c\bar{\textbf{K}}_{u; v} +  {}^j_p\Lambda_u 
\begin{bmatrix} 
{}^j\bar{\textbf{K}}_{a_{p-1}^j;a_{q-1}^j} &  {}^j\bar{\textbf{K}}_{a_{p-1}^j;a_q^j} \\
{}^j\bar{\textbf{K}}_{a_p^j;a_{q-1}^j} &  {}^j\bar{\textbf{K}}_{a_p^j;a_q^j}
\end{bmatrix}
{}^j_q\Lambda_v ^T,\nonumber
\end{empheq}
where $ {}^j_c\bar{\textbf{K}}_{u; v}$ is provided in Equation (\ref{eq:kern_cond_cov}).
\end{appendices}
\newpage
\bibliography{string_gp}

\begin{thebibliography}{58}
\providecommand{\natexlab}[1]{#1}
\providecommand{\url}[1]{\texttt{#1}}
\expandafter\ifx\csname urlstyle\endcsname\relax
  \providecommand{\doi}[1]{doi: #1}\else
  \providecommand{\doi}{doi: \begingroup \urlstyle{rm}\Url}\fi

\bibitem[Adams and Stegle(2008)]{ggpm}
R.~P. Adams and O.~Stegle.
\newblock {G}aussian process product models for nonparametric nonstationarity.
\newblock In \emph{Proceedings of the International Conference on Machine
  Learning}, pages 1--8, 2008.

\bibitem[Adler and Taylor(2011)]{adlertaylor}
R.~J. Adler and J.~E. Taylor.
\newblock \emph{Topological Complexity of Smooth Random Functions: {\'E}cole
  D'{\'E}t{\'e} de Probabilit{\'e}s de Saint-Flour XXXIX-2009}.
\newblock Lecture Notes in Mathematics / {\'E}cole d'{\'E}t{\'e} de
  Probabilit{\'e}s de Saint-Flour. Springer, 2011.

\bibitem[Alquier et~al.(2016)Alquier, Friel, Everitt, and
  Boland]{alquier2016noisy}
P.~Alquier, N.~Friel, R.~Everitt, and A.~Boland.
\newblock Noisy monte carlo: Convergence of {M}arkov chains with approximate
  transition kernels.
\newblock \emph{Statistics and Computing}, 26\penalty0 (1-2):\penalty0 29--47,
  2016.

\bibitem[Bardenet et~al.(2014)Bardenet, Doucet, and
  Holmes]{bardenet2014towards}
R.~Bardenet, A.~Doucet, and C.~Holmes.
\newblock Towards scaling up {M}arkov chain monte carlo: an adaptive
  subsampling approach.
\newblock In \emph{Proceedings of the International Conference on Machine
  Learning}, pages 405--413, 2014.

\bibitem[Brooks et~al.(2011)Brooks, Gelman, Jones, and
  Meng]{brooks2011handbook}
S.~Brooks, A.~Gelman, G.~Jones, and X.-L. Meng.
\newblock \emph{Handbook of {M}arkov Chain Monte Carlo}.
\newblock CRC press, 2011.

\bibitem[Calandra et~al.(2014)Calandra, Peters, Rasmussen, and
  Deisenroth]{calandra2014manifold}
R.~Calandra, J.~Peters, C.~E. Rasmussen, and M.~P. Deisenroth.
\newblock Manifold {G}aussian processes for regression.
\newblock \emph{arXiv preprint arXiv:1402.5876}, 2014.

\bibitem[Cao and Fleet(2014)]{cao2014generalized}
Y.~Cao and D.~J. Fleet.
\newblock Generalized product of experts for automatic and principled fusion of
  {G}aussian process predictions.
\newblock \emph{arXiv preprint arXiv:1410.7827}, 2014.

\bibitem[Daley and Vere-Jones(2008)]{Daley08}
D.~J. Daley and D.~Vere-Jones.
\newblock \emph{An Introduction to the Theory of Point Processes}.
\newblock Springer-Verlag, 2008.

\bibitem[Deisenroth and Ng(2015)]{deisenroth2015distributed}
M.~Deisenroth and J.~W. Ng.
\newblock Distributed {G}aussian processes.
\newblock In \emph{Proceedings of the International Conference on Machine
  Learning}, pages 1481--1490, 2015.

\bibitem[Doob(1944)]{doob44}
J.~L. Doob.
\newblock The elementary {G}aussian processes.
\newblock \emph{The Annals of Mathematical Statistics}, 15\penalty0
  (3):\penalty0 229--282, 1944.

\bibitem[Durrande et~al.(2012)Durrande, Ginsbourger, and Roustant]{durrande}
N.~Durrande, D.~Ginsbourger, and O.~Roustant.
\newblock Additive covariance kernels for high-dimensional {G}aussian process
  modeling.
\newblock \emph{Annales de la Facult\'e de Sciences de Toulouse}, 21\penalty0
  (3), 2012.

\bibitem[Duvenaud et~al.(2011)Duvenaud, Nickisch, and
  Rasmussen]{DuvenaudNR2012}
D.~Duvenaud, H.~Nickisch, and C.~E. Rasmussen.
\newblock Additive {G}aussian processes.
\newblock In \emph{Proceedings of the Conference on Advances in Neural
  Information Processing Systems}, pages 226--234, 2011.

\bibitem[Ferguson(1973)]{ferguson1973bayesian}
T.~S. Ferguson.
\newblock A {B}ayesian analysis of some nonparametric problems.
\newblock \emph{The Annals of Statistics}, pages 209--230, 1973.

\bibitem[Foti and Williamson(2015)]{foti2015survey}
N.~J. Foti and S.~Williamson.
\newblock A survey of non-exchangeable priors for {B}ayesian nonparametric
  models.
\newblock \emph{IEEE Transactions on Pattern Analysis and Machine
  Intelligence}, 37\penalty0 (2):\penalty0 359--371, 2015.

\bibitem[Gramacy and Lee(2008)]{gramacy}
R.~B. Gramacy and H.~K.~H. Lee.
\newblock {B}ayesian treed {G}aussian process models with an application to
  computer modeling.
\newblock \emph{Journal of the American Statistical Association}, 103\penalty0
  (483), 2008.

\bibitem[Green(1995)]{green1995reversible}
P.~J. Green.
\newblock Reversible jump {M}arkov chain monte carlo computation and {B}ayesian
  model determination.
\newblock \emph{Biometrika}, 82\penalty0 (4):\penalty0 711--732, 1995.

\bibitem[Green and Hastie(2009)]{green09}
P.~J. Green and D.~I. Hastie.
\newblock Reversible jump {MCMC}.
\newblock \emph{Genetics}, 155\penalty0 (3):\penalty0 1391--1403, 2009.

\bibitem[Hensman et~al.(2013)Hensman, Fusi, and Lawrence]{gpbigdatareg}
J.~Hensman, N.~Fusi, and N.~D. Lawrence.
\newblock {G}aussian processes for big data.
\newblock In \emph{Proceedings of the Conference on Uncertainty in Artificial
  Intellegence}, pages 282--290, 2013.

\bibitem[Hensman et~al.(2015)Hensman, Matthews, and Ghahramani]{gpbigdatacls}
J.~Hensman, A.~Matthews, and Z.~Ghahramani.
\newblock Scalable variational {G}aussian process classification.
\newblock In \emph{Proceedings of the International Conference on Artificial
  Intelligence and Statistics}, pages 351--360, 2015.

\bibitem[Hoffman et~al.(2013)Hoffman, Blei, Wang, and
  Paisley]{JMLR:v14:hoffman13a}
M.~D. Hoffman, D.~M. Blei, C.~Wang, and J.~Paisley.
\newblock Stochastic variational inference.
\newblock \emph{Journal of Machine Learning Research}, 14:\penalty0 1303--1347,
  2013.

\bibitem[Karatzas and Fernholz(2009)]{karatzas2009stochastic}
I.~Karatzas and R.~Fernholz.
\newblock Stochastic {P}ortfolio {T}theory: {A}n {O}verview.
\newblock \emph{Handbook of Numerical Analysis}, 15:\penalty0 89--167, 2009.

\bibitem[Kim et~al.(2005)Kim, K., and C.]{kim}
H.~Kim, Mallick~B. K., and Holmes~C. C.
\newblock Analyzing nonstationary spatial data using piecewise {G}aussian
  processes.
\newblock \emph{Journal of the American Statistical Association}, 100\penalty0
  (470):\penalty0 653--668, 2005.

\bibitem[Kingman(1967)]{kingman1967completely}
J.~Kingman.
\newblock Completely random measures.
\newblock \emph{Pacific Journal of Mathematics}, 21\penalty0 (1):\penalty0
  59--78, 1967.

\bibitem[Lawrence et~al.(2003)Lawrence, Seeger, and Herbrich]{lawr03}
N.~Lawrence, M.~Seeger, and R.~Herbrich.
\newblock Fast sparse {G}aussian process methods: the informative vector
  machine.
\newblock In \emph{Proceedings of the Conference on Advances in Neural
  Information Processing Systems}, pages 625--632, 2003.

\bibitem[Lazaro-Gredilla et~al.(2010)Lazaro-Gredilla, Quinonero-Candela,
  Rasmussen, and Figueiras-Vida]{sparsespectrum}
M.~Lazaro-Gredilla, J.~Quinonero-Candela, C.~E. Rasmussen, and A.~R.
  Figueiras-Vida.
\newblock Sparse spectrum {G}aussian process regression.
\newblock \emph{Journal of Machine Learning Research}, 11:\penalty0 1866--1881,
  2010.

\bibitem[Le et~al.(2013)Le, Sarl{\'o}s, and Smola]{le13}
Q.~Le, T.~Sarl{\'o}s, and A~Smola.
\newblock Fastfood-approximating kernel expansions in loglinear time.
\newblock In \emph{Proceedings of the International Conference on Machine
  Learning}, 2013.

\bibitem[Macdonald(1995)]{macdonald}
I.~G. Macdonald.
\newblock \emph{Symmetric Functions and Hall Polynomials}.
\newblock Oxford University Press, 1995.

\bibitem[MacKay(1998)]{gp_intro}
D.~J.~C. MacKay.
\newblock Introduction to {G}aussian processes.
\newblock In \emph{NATO ASI Series F: Computer and Systems Sciences}, pages
  133--166. Springer, Berlin, 1998.

\bibitem[Meeds and Osindero(2006)]{Meeds_analternative}
E.~Meeds and S.~Osindero.
\newblock An alternative infinite mixture of {G}aussian process experts.
\newblock In \emph{Proceedings of the Conference on Advances In Neural
  Information Processing Systems}, 2006.

\bibitem[Murray et~al.(2010)Murray, Adams, and MacKay]{Murray09b}
I.~Murray, R.~P. Adams, and D.~J.~C. MacKay.
\newblock Elliptical slice sampling.
\newblock In \emph{Proceedings of the International Conference on Artificial
  Intelligence and Statistics}, pages 9--16, 2010.

\bibitem[Neal(1996)]{neal}
R.~Neal.
\newblock \emph{{B}ayesian learning for neural networks}.
\newblock Lecture notes in Statistics. Springer, 1996.

\bibitem[Nguyen and Bonilla(2014)]{nguyen2014fast}
T.~Nguyen and E.~Bonilla.
\newblock Fast allocation of {G}aussian process experts.
\newblock In \emph{Proceedings of the International Conference on Machine
  Learning}, pages 145--153, 2014.

\bibitem[{\O}ksendal(2003)]{oks}
B.~{\O}ksendal.
\newblock \emph{Stochastic Differential Equations: An Introduction with
  Applications}.
\newblock Hochschultext / Universitext. Springer, 2003.

\bibitem[Paciorek and Schervish(2004)]{paciorek2004nonstationary}
C.~Paciorek and M.~Schervish.
\newblock Nonstationary covariance functions for {G}aussian process regression.
\newblock In \emph{Proceedings of the Conference on Advances in Neural
  Information Processing Systems}, pages 273--280, 2004.

\bibitem[Pitman and Yor(1997)]{pitman1997two}
J.~Pitman and M.~Yor.
\newblock The two-parameter {P}oisson-{D}irichlet distribution derived from a
  stable subordinator.
\newblock \emph{The Annals of Probability}, pages 855--900, 1997.

\bibitem[Plagemann et~al.(2008)Plagemann, Kersting, and
  Burgard]{plagemann08ecml}
C.~Plagemann, K.~Kersting, and W.~Burgard.
\newblock Nonstationary {G}aussian process regression using point estimate of
  local smoothness.
\newblock In \emph{Proceedings of the European Conference on Machine Learning},
  pages 204--219, 2008.

\bibitem[Quinonero-Candela and Rasmussen(2005)]{FTCI}
J.~Quinonero-Candela and C.~E. Rasmussen.
\newblock A unifying view of sparse approximate {G}aussian process regression.
\newblock \emph{Journal of Machine Learning Research}, 6:\penalty0 1939--1959,
  2005.

\bibitem[Rahimi and Recht(2007)]{rahimi07}
A.~Rahimi and B.~Recht.
\newblock Random features for large-scale kernel machines.
\newblock In \emph{Proceedings of the Conference on Advances in Neural
  Information Processing Systems}, pages 1177--1184, 2007.

\bibitem[Rasmussen and Ghahramani(2001)]{Rasmussen01infinitemixtures}
C.~E. Rasmussen and Z.~Ghahramani.
\newblock Infinite mixtures of {G}aussian process experts.
\newblock In \emph{Proceedings of the Conference on Advances in Neural
  Information Processing Systems}, pages 881--888, 2001.

\bibitem[Rasmussen and Williams(2006)]{rasswill}
C.~E. Rasmussen and C.~K.~I. Williams.
\newblock \emph{{G}aussian Processes for Machine Learning}.
\newblock The MIT Press, 2006.

\bibitem[Ross and Dy(2013)]{icml2013_ross13a}
J.~Ross and J.~Dy.
\newblock Nonparametric mixture of {G}aussian processes with constraints.
\newblock In \emph{Proceedings of the International Conference on Machine
  Learning}, pages 1346--1354, 2013.

\bibitem[Saatchi(2011)]{saatchi11}
Y.~Saatchi.
\newblock \emph{Scalable Inference for Structured {G}aussian Process Models}.
\newblock PhD thesis, University of Cambridge, 2011.

\bibitem[Schmidt and O'Hagan(2003)]{Schmidt03}
A.~M. Schmidt and A.~O'Hagan.
\newblock {{B}ayesian inference for nonstationary spatial covariance structure
  via spatial deformations}.
\newblock \emph{Journal of the Royal Statistical Society: Series B (Statistical
  Methodology)}, 65\penalty0 (3):\penalty0 743--758, 2003.

\bibitem[Seeger(2003{\natexlab{a}})]{Seeger03bayesiangaussian}
M.~Seeger.
\newblock {B}ayesian {G}aussian process models: Pac-{B}ayesian generalisation
  error bounds and sparse approximations.
\newblock Technical report, 2003{\natexlab{a}}.

\bibitem[Seeger(2003{\natexlab{b}})]{seeger2003pac}
M.~Seeger.
\newblock {PAC}-{B}ayesian generalisation error bounds for {G}aussian process
  classification.
\newblock \emph{Journal of Machine Learning Research}, 3:\penalty0 233--269,
  2003{\natexlab{b}}.

\bibitem[Shah et~al.(2014)Shah, Wilson, and Ghahramani]{shah2014student}
A.~Shah, A.~G. Wilson, and Z.~Ghahramani.
\newblock Student-t processes as alternatives to {G}aussian processes.
\newblock In \emph{Proceedings of the International Conference on Artificial
  Intelligence and Statistics}, pages 877--885, 2014.

\bibitem[Silverman(1985)]{silverman}
B.~W. Silverman.
\newblock {Some Aspects of the Spline Smoothing Approach to Non-Parametric
  Regression Curve Fitting}.
\newblock \emph{Journal of the Royal Statistical Society. Series B
  (Methodological)}, 47\penalty0 (1):\penalty0 1--52, 1985.

\bibitem[Smola and Bartlett(2001)]{Smola01sparsegreedy}
A.~J. Smola and P.~Bartlett.
\newblock Sparse greedy {G}aussian process regression.
\newblock In \emph{Proceedings of the Conference on Advances in Neural
  Information Processing Systems}, pages 619--625. MIT Press, 2001.

\bibitem[Snelson and Ghahramani(2006)]{Snelson06sparsegaussian}
E.~Snelson and Z.~Ghahramani.
\newblock Sparse {G}aussian processes using pseudo-inputs.
\newblock In \emph{Proceedings of the Conference on Advances in Neural
  Information Processing Systems}, pages 1257--1264, 2006.

\bibitem[{The GPy authors}(2012--2016)]{gpy2014}
{The GPy authors}.
\newblock {GPy}: A {G}aussian process framework in python.
\newblock \url{http://github.com/SheffieldML/GPy}, 2012--2016.

\bibitem[Tresp(2000)]{tresp2000bayesian}
V.~Tresp.
\newblock A {B}ayesian {C}ommittee {M}achine.
\newblock \emph{Neural Computation}, 12\penalty0 (11):\penalty0 2719--2741,
  2000.

\bibitem[Tresp(2001)]{tresp}
V.~Tresp.
\newblock Mixtures of {G}aussian processes.
\newblock In \emph{Proceedings of the Conference on Advances in Neural
  Information Processing Systems}, pages 654--660, 2001.

\bibitem[Vervuurt and Karatzas(2015)]{Vervuurt2015}
A.~Vervuurt and I.~Karatzas.
\newblock Diversity-weighted portfolios with negative parameter.
\newblock \emph{Annals of Finance}, 11\penalty0 (3):\penalty0 411--432, 2015.

\bibitem[Williams and Seeger(2001)]{Williams01usingthe}
C.~Williams and M.~Seeger.
\newblock Using the {N}ystr\"{o}m method to speed up kernel machines.
\newblock In \emph{Proceedings of the Conference on Advances in Neural
  Information Processing Systems}, pages 682--688, 2001.

\bibitem[Wilson and Adams(2013)]{wilson2013gaussian}
A.~G. Wilson and R.~P. Adams.
\newblock {G}aussian process kernels for pattern discovery and extrapolation.
\newblock In \emph{Proceedings of the International Conference on Machine
  Learning}, pages 1067--1075, 2013.

\bibitem[Wilson and Nickisch(2015)]{wilson2015kernel}
A.~G. Wilson and H.~Nickisch.
\newblock Kernel interpolation for scalable structured {G}aussian processes.
\newblock In \emph{Proceedings of the International Conference on Machine
  Learning}, pages 1775--1784, 2015.

\bibitem[Wilson et~al.(2014)Wilson, Gilboa, and Nehorai]{GPatt}
A.~G. Wilson, E.~Gilboa, and J.~P. Nehorai, A.and~Cunningham.
\newblock Fast kernel learning for multidimensional pattern extrapolation.
\newblock In \emph{Proceedings of the Conference on Advances in Neural
  Information Processing Systems}, pages 3626--3634. 2014.

\bibitem[Yang et~al.(2015)Yang, Smola, Song, and Wilson]{le15}
Z.~Yang, A.~Smola, L.~Song, and A.~G. Wilson.
\newblock A la carte -- learning fast kernels.
\newblock In \emph{Proceedings of the International Conference on Artificial
  Intelligence and Statistics}, pages 1098--1106, 2015.

\end{thebibliography}
\end{document}